\documentclass{article} 
\usepackage{iclr2026_conference,times}

\usepackage{amsmath,amsfonts,bm}

\def\eqref#1{equation~\ref{#1}}

\def\1{\mathbbm{1}}

\DeclareMathAlphabet{\mathsfit}{\encodingdefault}{\sfdefault}{m}{sl}
\SetMathAlphabet{\mathsfit}{bold}{\encodingdefault}{\sfdefault}{bx}{n}

\newcommand{\softmax}{\mathrm{softmax}}

\usepackage[utf8]{inputenc} 
\usepackage[T1]{fontenc}    
\usepackage{hyperref}       
\usepackage{url}            
\usepackage{booktabs}       
\usepackage{amsfonts}       
\usepackage{nicefrac}       
\usepackage{microtype}      
\usepackage{xcolor}         

\usepackage{microtype}
\usepackage{graphicx}
\usepackage{subfigure}
\usepackage{enumitem}
\usepackage{booktabs} 
\usepackage{multirow}
\usepackage{paracol}
\usepackage{colortbl}

\usepackage[ruled,linesnumbered]{algorithm2e}
\usepackage{algorithmic}

\usepackage{amsmath}
\usepackage{amssymb}
\usepackage{mathtools}
\usepackage{amsthm}

\usepackage[capitalize]{cleveref}

\theoremstyle{plain}
\newtheorem{theorem}{Theorem}[section]

\newtheorem{lemma}[theorem]{Lemma}

\theoremstyle{definition}

\newcommand{\draft}[1]{\textcolor{red}{#1}{}}
\newcolumntype{X}{>{\centering\arraybackslash}m{1cm}}
\newcolumntype{Y}{>{\centering\arraybackslash}m{1.4cm}}
\newcolumntype{Z}{>{\centering\arraybackslash}m{2cm}}
\newcommand{\sheng}[1]{\textcolor{blue}{(Sheng: #1)}{}}
\newcommand{\xy}[1]{\textcolor{green}{(xy: #1)}{}}

\def\alg{{\texttt{M2CL}}}

\title{Context Learning for Multi-Agent Discussion}

\author{Xingyuan Hua\textsuperscript{\rm 1}, Sheng Yue\textsuperscript{\rm 2}\thanks{Correspondence to: \texttt{yuesh5@mail.sysu.edu.cn}}, Xinyi Li\textsuperscript{\rm 3}, Yizhe Zhao\textsuperscript{\rm 1}, Jinrui Zhang\textsuperscript{\rm 1}, Ju Ren\textsuperscript{\rm 1,4}\\
\textsuperscript{\rm 1}Department of Computer Science and Technology, Tsinghua University, \\
\textsuperscript{2}School of Cyber Science and Technology, Sun Yat-sen University Shenzhen Campus, \\
\textsuperscript{3}College of Computer Science, Northwest University, \\
\textsuperscript{4}State Key Laboratory of Internet Architecture, Tsinghua University
}
\newcommand{\fix}{\marginpar{FIX}}
\newcommand{\new}{\marginpar{NEW}}

\iclrfinalcopy 
\begin{document}
	
\maketitle

\begin{abstract}
	Multi-Agent Discussion (MAD) has garnered increasing attention very recently, where multiple LLM instances collaboratively solve problems via structured discussion. However, we find that current MAD methods easily suffer from discussion inconsistency—LLMs fail to reach a coherent solution—due to the misalignment between their individual contexts. In this paper, we introduce a multi-LLM context learning method (\alg{}) that learns a context generator for each agent, capable of dynamically generating context instructions per discussion round via automatic information organization and refinement. Specifically, inspired by our theoretical insights on the context instruction, \alg{} trains the generators to control context coherence and output discrepancies via a carefully crafted self-adaptive mechanism. It enables LLMs to avoid premature convergence on “majority noise” and progressively reach the correct consensus. We evaluate \alg{} on challenging tasks, including academic reasoning, embodied tasks, and mobile control. The results show that the performance of \alg{} significantly surpasses existing methods by $20\%$--$50\%$, while enjoying favorable transferability and computational efficiency.\footnote{Code is available at \url{https://github.com/HansenHua/M2CL-ICLR26}.}
\end{abstract}
\section{Introduction}
\label{sec:introduction}
Large Language Models (LLMs) have demonstrated transformative impact across a large number of real-world domains, including education, healthcare, and scientific research, where an LLM instance is employed to automate the process of content generation, reasoning, or decision-making~\citep{zheng2023codegeex,wang2023air,imani2023mathprompter,yue2024momentum,zhang2024simulating,goyal2024healai,wu2025evaluating,yu2025cookiechecker}. Yet, it has been recognized that single-LLM-built systems often struggle in the problems requiring complex multi-step reasoning or multi-tool using, such as complicated proof~\citep{cobbe2021training}, large-scale code generation~\citep{wang2025codeflowbench}, and embodied agentic tasks~\citep{ahn2022can,shen2024small,yue2024context}, because its single viewpoint of a problem easily limits the ability to explore multiple reasoning paths, leverage external tools effectively, or adapt to dynamic task requirements~\citep{li2023camel,smit2024should,yue2024ollie}.

Very recently, research has shifted towards \textit{Multi-Agent Discussion} (MAD)~\citep{smit2024should}, where multiple LLM instances collaboratively solve problems via structured discussions~\citep{du2023improving,liu2024dynamic}. In a typical MAD framework, each LLM instance is pre-assigned a set of crafted \textit{contexts} that represent diverse solution perspectives of the problem to be solved. Equipped with these distinct context instructions, LLMs continue to discuss with each other for a solution consensus~\citep{park2023generative,shanahan2023role,wei2023multi,lu2024llm,liu2024dynamic}. Such `society-of-mind' paradigms are expected to improve reasoning accuracy by enhancing creativity and expanding the search space for possible solutions, and have shown great potential across various complex tasks, including software engineering~\citep{gu2023llm} and scientific discovery~\citep{sprueill2024chemreasoner}.

Albeit achieving improved performance over single-LLM settings, we find that current multi-agent collaboration approaches typically suffer from discussion inconsistency, that is, the majority of LLM instances fail to reach an agreement on a coherent solution (as showcased in \cref{fig:intro}), easily making the collaborative decision dominated by noise rather than principled reasoning. The underlying reason primarily lies in \textit{context misalignment} between LLMs. On one hand, the pre-assigned contexts or role-based instructions lack a nuanced understanding of the task; they are often rigid, incomplete or biased, which would misguide the reasoning of individual LLMs~\citep{jang2025reasoning}. On the other hand, these contexts often fall short in effective fusion of information exchanged among LLMs, and thus hardly steer the discussion towards coherent solutions. As illustrated in \cref{fig:introduction_exp}, for a multi-step mathematical proof task, one LLM agent may correctly derive an intermediate result, yet another LLM—despite receiving this result in the extended context—does not effectively incorporate it into its own reasoning chain. Since the context instruction does not explicitly enforce leveraging conclusions from other LLMs, other LLMs may redundantly re-derive the step or even give an inconsistent argument (see \cref{sec:case_study}) for full results).

This paper aims to answer: \textit{``how can we obtain the contexts for MAD that can continually guide multi-LLM discussion towards a correct consensus?''} A straightforward solution is to manually adapt instructions in contexts as the discussion progresses. It, however, is labor-intensive and requires extra expert knowledge, rendering it impractical for complex tasks or large-scale collaboration. Instead, a more reasonable approach is to develop context learning mechanisms that enable evolving the context instructions based on the intermediate discussion results. Although promising, it is highly challenging to evaluate the contribution of LLMs' contexts to the final solution and control the coherence among inter- and intra-LLM outputs.

To tackle these challenges, we propose a \textit{multi-LLM context learning} (\alg{}) method for efficient MAD, which learns a context generator for each agent, capable of dynamically generating context instructions per discussion round via automatic information organization and refinement. First, we characterize the impact of initial and evolving contexts on the discussion performance. Building upon the analytical insights, we train the generators to control context coherence and output discrepancies. To strike the right tradeoff therein, we devise a self-adaptive balancing mechanism, enabling LLMs to progressively align on correct consensus while avoiding premature convergence on ``majority noise''. Further, we develop a lightweight context initialization approach, which tends to assign LLMs with diverse initial instructions that are approximately orthogonal in the latent space, enabling sufficient coverage of complementary solution perspectives.

\begin{figure}[t]
	\centering
	\includegraphics[width=0.99\linewidth]{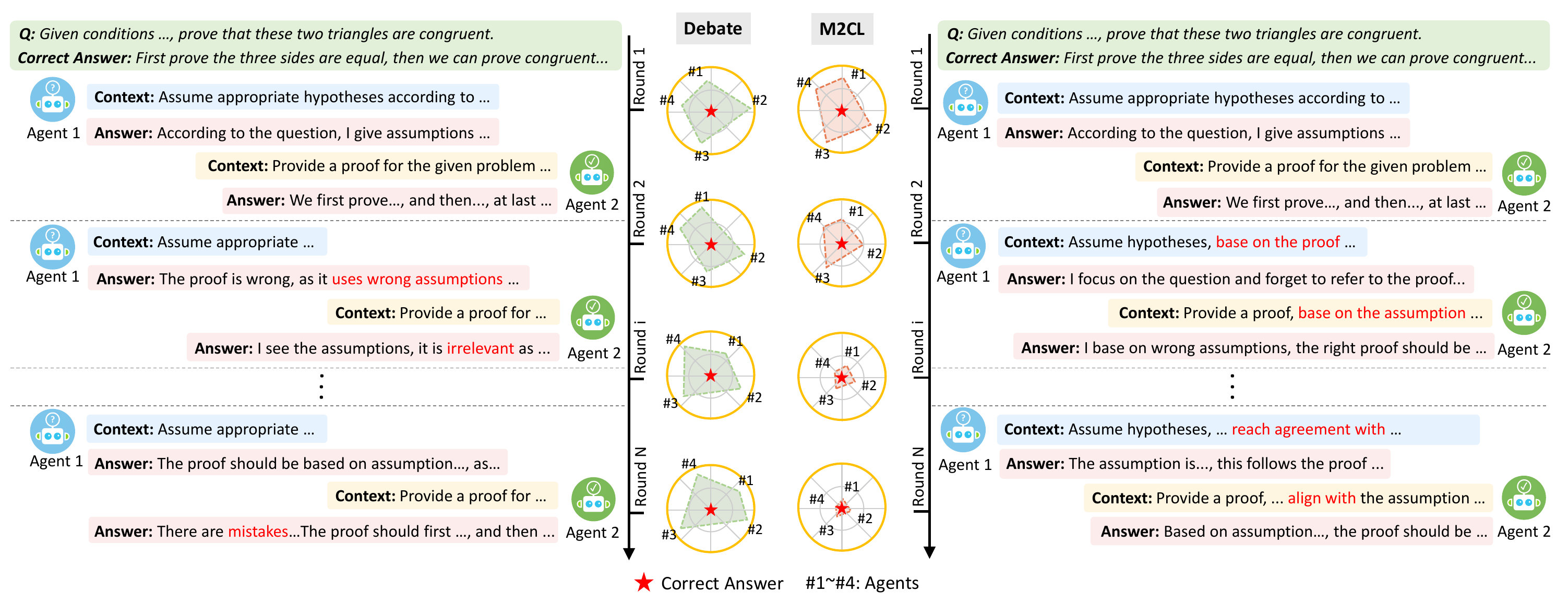}
    \vspace{-1.5em}
	\caption{An illustration of context misalignment of an existing method (\texttt{Debate}~\cite{du2023improving}) on a multi-step proof task. Pre-assigned context instructions (in the blue and yellow boxes of the left part) provide insufficient guidance on information fusion, leading to conflict in reasoning.}
    \vspace{-1em}
	\label{fig:introduction_exp}
\end{figure}

\begin{figure}[t]
	\centering
	\includegraphics[width=0.99\linewidth]{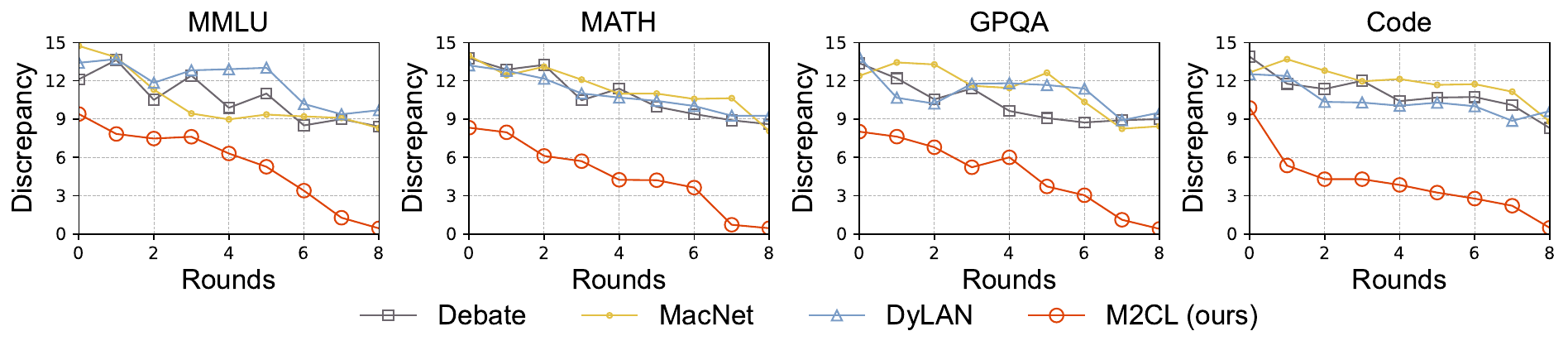}
    \vspace{-1.5em}
	\caption{The discrepancy between the answers of participating LLM instances. The discrepancy is characterized by the maximum distance between participating LLMs' output embeddings.}
    \vspace{-1em}
	\label{fig:intro}
\end{figure}

We systematically evaluate the proposed method across 9 challenging benchmarks, including LLM reasoning, embodied agentic tasks, and mobile GUI control. The results demonstrate that \alg{} significantly enhances the MAD performance under various numbers of participating LLMs, consistently outperforming existing methods by $\bm{20\%}$--$\bm{50\%}$, particularly in complex GUI control tasks. \alg{} enjoys a more favorable MAD ``scaling law'' and exhibits great efficiency, where a runtime overhead of at most $10\%$ suffices to achieve more than $20\%$ performance gains. Further, we find that the learned context generators can be migrated to different LLM architectures with consistent performance improvement. 

Our main contributions are summarized as follows:
\begin{itemize}[leftmargin=*,topsep=0pt,itemsep=0pt]
    \item We propose \alg{}, a principled multi-LLM context learning method that learns a context generator for each agent, capable of dynamically generating context instructions per discussion round.
    \item We devise a lightweight context initialization approach that can assign LLMs with diverse initial instructions, thereby enabling sufficient coverage of complementary solution perspectives.   
    \item We systematically evaluate the proposed method across a range of challenging benchmarks and corroborate the efficacy and efficiency of the proposed method.
\end{itemize}

\section{Related Work}
\label{sec:related_work}
\textbf{Multi-agent framework.} Multi-agent frameworks have been proposed to enhance the reasoning capabilities of single-LLM prompting methods. Several works utilize two LLMs to iteratively generate and evaluate to refine the final answer~\citep{raman2022planning,yao2023react,paul2023refiner,madaan2024self,yue2024federated}. However, scaling them to a larger number of LLMs to make full use of the wisdom of crowd is challenging due to the pairwise dependency of generation and evaluation. To incorporate more LLMs, several studies have explored another multi-agent framework in which each LLM has access to the history of all preceding LLMs' responses~\citep{chan2023chateval,du2023improving,liang2023encouraging,smit2024should,liu2024dynamic,zhuge2024gptswarm}. Most existing methods remain constrained by manually defined inter-LLM topologies or workflows. More recently, researches have attempted to overcome this limitation by automatically optimizing workflows~\citep{zhuge2024gptswarm,liu2024dynamic,zhang2025aflow}. However, such evolution is often one-shot. For instance, ~\citet{zhuge2024gptswarm} describe language agent systems as optimizable computational graphs and enables automatic improvements of LLM prompts and inter-LLM orchestration. ~\citet{zhang2025aflow} employs Monte Carlo Tree Search to construct complex Multi-LLM systems tailored to a specific task domain. ~\citet{liu2024dynamic} propose a multi-LLM collaboration method which constructs dynamic communication structure by scoring other LLMs. While these approaches allow for scaling the number of LLMs, designing appropriate contexts for them to collaboratively solve problems remains difficult, as single-viewpoint context instructions fall short in guidance on how to organize and refine information across agents.

\textbf{Context learning.} Context learning has recently attracted significant attention, as it allows LLMs to adjust their behavior at inference time by modifying the input context, without requiring any gradient-based training~\cite{von2023transformers,todd2024function,li2024context}. One line of work focuses on context selection, which identifies the most relevant examples or information to include in the context~\cite{zhang2022active,lu2023dynamic,xiong2024dqlore,purohit2025sample}. To eliminate manual selection, ~\citet{lu2023dynamic} leverage reinforcement learning to learn a context selection policy. ~\citet{xiong2024dqlore} query LLM to obtain knowledge and then query a retriever to obtain the final context. Beyond selection, some recent methods have proposed to generate or evolve context~\citep{zhuge2024gptswarm,li2024implicit,zhang2025g}. \citet{madaan2024self} leverage feedback generated by LLMs as extra prompt and iteratively incorporate it into revised drafts, aiming to enhance the coherence of the generated text. \citet{pandita2025prorefine} dynamically refines prompts during inference using textual feedback from prior outputs, aiming to improve contextual alignment and generate more consistent responses. Albeit with promising results, these methods are grounded in single-LLM formulations and struggle with inter-LLM inconsistency during multi-agent dscussion, where it is crucial to guide LLMs to make full use of others' intermediate results.

\section{Preliminaries}
\label{sec:preliminaries}
In this section, we provide necessary backgrounds and definitions of our investigated problem.

\textbf{Context learning.} To formally define context learning, we begin with the standard probabilistic model of an autoregressive LLM. An autoregressive model, parameterized by $\phi$, generates an output sequence $X=(x_1,x_2,\dots,x_L)$ given an input context $C$ by maximizing the conditional probability:
\begin{align}
	P_\phi(X|C)=\prod_{l=1}^LP_\phi(x_l|x_{<l},C).
\end{align}
Historically, in the paradigm of prompt engineering, the context $C$ was treated as a static string of text. This view is insufficient for complex agentic tasks which leverage dynamic, structured, and multifaceted information stream. To address this, context learning is redefined as a dynamically organized collection of information, denoted as $\{c_1,c_2,\dots,c_n\}$. Each component $c_k$ can be instructions, external knowledge~\citep{lewis2020retrieval}, available external tools~\citep{qin2023toolllm}, and memory~\citep{zhang2025g}. These components are sourced, filtered, and formatted by a set of functions $f_k$, and orchestrated into a coherent representation by a high-level assembly function $\mathcal{A}$:
\begin{align}
	\label{eq:context_formulation}
	C=\mathcal{A}(f_1(c_1),f_2(c_2),\dots,f_n(c_n)).
\end{align}
\textbf{Multi-LLM context learning.} MAD involves a set of $N$ LLM instances collaboratively solving a task via multi-round inter-LLM discussion (debate). Each LLM $i$ is endowed with an evolving instruction context in terms of the task description, available external tools and current knowledge aggregation. 
At the $t$-th round, we use the concatenation function as the assembly function $\mathcal{A}$ and construct the context with three components: (\textit{i}) the task goal $P\in \mathbb{R}^{d_{model}\times n}$, (\textit{ii}) the concatenation of responses from all other LLMs in the previous round $\bar{X}_i^{t-1}=[X_j^{t-1}]_{j \neq i}\in \mathbb{R}^{d_{model}\times n}$, where $X_j^{t-1}$ is the response of LLM $j$ at round $t-1$, and (\textit{iii}) the current instruction context $I_i^t\in \mathbb{R}^{d_{model}\times n}$ ($d_{model}$ refers to the dimension of embedding). The task goal represents the ultimate objective of collaboration and remains invariant across rounds. The second component serves as a dynamic memory that incorporates cross-LLM interaction history. For the third component, in constrast to existing efforts relying on static preassigned roles, we employ an instruction generator $\mathcal{G}$, parameterized by $\theta_i$, to adaptively refine it into a per-step instruction $I_i^t$, conditioned on the task goal $P$, the initial instruction $I_i^b$, and the concatenated response $\bar{X}_i^{t-1}$:
\begin{align}
	\label{eq:context_generation}
	I_i^t = \mathcal{G}_{\theta_i}([P;I_i^b;\bar{X}_i^{t-1}]).
\end{align}
Given context $C_i^t=[I_i^t,\bar{X}_i^{t-1},P]$, each LLM $i$, parameterized by $\phi_i$, generate its response by:
\begin{align}
	\label{eq:response_generation}
	X_i^{t}=\arg\max_X P_{\phi_i}(X|C_i^t).
\end{align}
After $T$ rounds of interaction, the final result is obtained by a majority vote on the LLMs' outputs generated in the final round.

\section{Motivation}
\label{sec:motivation}
In this section, we investigate the quantified impact of the contexts on the MAD performance. First, we introduce the formulation of attention activation, denoted as $a(\cdot)$, as follows \citep{vaswani2017attention}:
\begin{align}
	\label{eq:activation}
	a(C_i^t)\doteq W_V[I_i^t;\bar{X}_i^{t-1};P]\softmax\bigg(\frac{(W_K[I_i^t;\bar{X}_i^{t-1};P])^TW_Q{P}}{\sqrt{d}}\bigg)
\end{align}
where $\sqrt{d}$ is a scaling factor. $W_Q$, $W_K$, and $W_V$ denote the parameter weight matrices in the attention mechanism. $a(C_i^t)\in\mathbb{R}^{d_{model}\times n}$ is a matrix with the same shape as $[P;X;I]$

From this, we have the following theorem characterizing the total distance between the activations of the correct answer $a_c$ and that induced by context $C_j^{t}$. Here, we utilize activation distance instead of token embedding distance as attention activation captures deep representational similarity learned through the model’s internal reasoning process, making it more robust to superficial linguistic variations.
\begin{theorem}
	\label{thm:activation_difference}
	Assume that the attention activation function is $L_a$-smooth and define the weight vector as $\omega\doteq[\omega_1, \omega_2, \dots, \omega_N]$, with the initial context $C_i^b\doteq[I_i^b; P]$. Then, the following fact holds:
	\begin{align}
		\sum_{i=1}^N\|a_c-a(C_i^{t})\|&\leq \sum_{i=1}^N\Big(\sum_{j=1}^N\|a(C_i^{t})-a(C_j^{t})\|+(N+1)L_a\|C_i^t-C_i^b\|\Big)\nonumber\\
		&+N\min_{\omega}\|a_c-\sum_{i=1}^N\omega_ia(C_i^b)\|.
	\end{align}
\end{theorem}
\begin{proof}
	For a detailed proof, please refer to \cref{sec:activation_difference}.
\end{proof}
\cref{thm:activation_difference} corroborates the necessity of multi-LLM context learning, consisting of both initialization and evolution. The first term, $\|a(C_i^{t})-a(C_j^{t})\|+(N+1)L_a\|C_i^t-C_i^b\|$, captures the divergence among LLMs’ activations as well as the deviation from their initial contexts. This indicates contexts must be continuously evolved to reduce inter-LLM discrepancies while keeping coherent, promoting consistency of reasoning chains. The second term, $\min_{\omega}\|a_c-\sum_{i=1}^N\omega_ia(C_i^b)\|$, depends solely on the initial contexts $I_i^b$. It implies that orthogonality among initial activations provides a comprehensive basis, allowing the contexts to approximate the correct activation more effectively. This highlights the importance of diverse solution perspectives at the initialization stage. Together, the result motivates a two-stage design: initializing contexts to ensure diverse solution perspectives and evolving intermediate contexts to drive a consensus across LLMs.

\section{Multi-LLM Context Learning}
\label{sec:methodology}
In this section, we introduce a multi-LLM context learning method that can select appropriate contexts for each LLM and dynamically adapt the context according to the evolving task completion status.

\subsection{Context Initialization}
\label{sec:context_init}
The context initialization serves as the foundation of the entire multi-agent interaction process, as it frames the capability scope of LLMs and fundamentally shapes the chain of subsequent information exchange. Therefore, we propose to select initial contexts from a predefined pool $\{I_i^b\}_{i=1}^M$ containing prompts with diverse perspectives.

Motivated by \cref{thm:activation_difference}, we provide the following context initialization mechanism:
\begin{align}
	\label{eq:context_init}
	\bm I^b = \{I^b_1,\dots,I^b_N\} = \arg\min\limits_{\bm I^b}\Big\{\min\limits_\omega\|\sum_{i=1}^N \omega_i a([I^b_i; P]-a_c)\|\Big\}.
\end{align}
\cref{eq:context_init} identifies a subset of contexts $\bm I^b$ whose activation best reconstruct the target activation $a_c$. As the dimension of the activation matrix $a([I^b_i;P])\in\mathbb{R}^{d_{model}\times n}$ is far larger than the number of selected contexts $N$, a set of matrices that aims to best reconstruct the correct activation $a_c$ naturally tends toward forming a set of basis-like directions. The resulting near-orthogonal activations form a compact basis, ensuring each context contributes unique, non-overlapping information for subsequent discussion.

However, \cref{eq:context_init} is impractical since the correct activation $a_c$ is not accessible during initialization. Inspired by prior work~\citep{yang2024chain,yang2025lf}, we project the activation into a latent space with function $f$:
\begin{align}
	\label{eq:context_init_1}
    \bm I^b = \{I^b_1,\dots,I^b_N\} = \arg\min_{\bm I^b}\Big\{\min_\omega\|\sum_{i=1}^N\omega_if(a([I_i^b;P]))-f(a_c)\|\Big\},
\end{align}
Then, we use the problem space as the selected contexts are solely dependent on the problem. This projection preserves the original orthogonality properties of activations, ensuring that diverse perspectives remain distinguishable in other spaces. Therefore, we reformulate the initialization mechanism:
\begin{align}
	\label{eq:context_init_2}
    \bm I^b = \{I^b_1,\dots,I^b_N\} = \arg\min_{\bm I^b}\Big\{\min_\omega\|\sum_{i=1}^N\omega_if(a([I_i^b;P]))-v_P\|\Big\},
\end{align}
where $v_P$ denotes the sentence vector of the question $P$. To obtain the projection $f(\cdot)$ (parameterized by $\phi_f$), we utilize the activation of inputting the answer $A$ with the question $P$ as the correct activation $a_c=a([A;P])$ and provide the loss function as:
\begin{align}
	L(\phi_f)=\|v_P-f(a([A;P]))\|.
\end{align}
However, directly leveraging \cref{eq:context_init_2} to initialize context is computational costly as it requires all the context activations in the context pool. Therefore, we distill a lightweight $\mathcal{F}(\cdot)$ (parameterized by $\phi_F$) which directly projects the initial context into the problem space through the following loss function:
\begin{align}
	L(\phi_F)=\|\mathcal{F}([I_i^b;P])-f(a([I_i^b;P]))\|.
\end{align}
Then, we provide the final context initialization as:
\begin{align}
	\label{eq:context_init_3}
    \bm I^b = \{I^b_1,\dots,I^b_N\} = \arg\min_{\bm I^b}\Big\{\min_\omega\|\sum_{i=1}^N\omega_i\mathcal{F}([I_i^b;P])-v_P\|\Big\}.
\end{align}
\cref{eq:context_init_3} provides a computational efficient context initialization (as illustrated in \cref{fig:runtime_init}) that aligns with \cref{eq:context_init} to encourage orthogonality of selected context activations, thereby providing diverse reasoning perspectives and expanding the search space for solutions.

\subsection{Context Evolution}
\label{sec:context_evolve}
As discussed in \cref{sec:context_init}, context initialization aims to select contexts with diverse perspectives to solve the problem. Nevertheless, these individual contexts address the problem from a single viewpoint while lacking explicit instructions on how to incorporate perspectives provided by other LLMs. Directly using them often results in inconsistent reasoning across LLMs due to misaligned inter-LLM guidance. To address this, we iteratively refine the instruction to integrate collaborative instructions that guide the LLMs more effectively toward generating the desired solution.

\subsubsection{Evaluate the Contribution of Contexts}
\label{sec:context_evaluation}

To achieve efficient context evolution, it is crucial to quantify the contribution of different contexts (refer to a utility function). A natural solution is to assign utilities only at the final round. However, this leads to inefficient and unstable training process due to its sparsity~\cite{liu2025ghpo}. On the other hand, current methods typically use correctness as the only criterion~\citep{zelikman2022star,rafailov2023direct,yue2024leverage,guo2025deepseek}. Albeit with promising performance in single-LLM question answering, this simple criterion ignores the complex dependency among LLMs with different contexts in MAD. As a result, an LLM that provides crucial insights which enable others to reach the correct answer may still receive a low utility, simply because it fails to produce the correct final answer itself. To tackle these challenges, we introduce a novel round-wise criterion to evaluate the contribution of $C_i^t$ in multi-agent interaction:
\begin{align}
	\label{eq:reward_activation}
	\max_{j\in [N]}\Big\{-\alpha\|C_i^t-C_i^b\|-\|a(C_i^t)-a(C_j^t)\|\Big\},
\end{align}
where $\alpha\ge0$ is a weighting parameter and $\|a(C_i^t)-a(C_j^t)\|$ represents the activation difference between contexts.

\textbf{\textit{Remarks.}} Of note, if removing the second term of \cref{eq:reward_activation}, the context generation reduces to that of prompt engineering: $\min_{I_{1:N}^t}\sum_{i\in[N]}\alpha\|C_i^t-C_i^b\|$, which utilize fixed instructions in contexts to LLMs throughout the discussion~\citep{liu2024dynamic,lu2024llm}. As mentioned in \cref{sec:introduction}, this easily results in inconsistency across LLMs. Thus, we introduce $\|a(C_i^t)-a(C_j^t)\|$ into the objective to encourage LLMs to remain aligned with each other.

This design addresses the two aforementioned challenges. First, this round-wise criterion provides denser feedback, mitigating the inefficiency and instability caused by sparse final-round utilities. Second, the activation-based alignment term captures inter-LLM dependencies, thereby encouraging them to offer guidance on how to integrate others’ responses.

Applying \cref{eq:reward_activation} as the criterion for LLM $i$ requires the contexts of all other LLMs. However, since these contexts are simultaneously being optimized, a bias inevitably arises as each LLM is updated based on stale snapshots of others’ contexts. Such outdated information introduces a biased utility, which in turn hinders consistency and degrades the quality of the learned contexts. To overcome this issue, we decouple the inter-LLM dependencies by designing an alternative per-LLM criterion:
\begin{align}
	\label{eq:reward_decouple}
	-\alpha\|C_i^t-C_i^b\|-\|a([I_i^t,P])-a([X_i^{t-1},P])\|.
\end{align}
In \cref{lem:activation_decouple}, we can prove that the summation of \cref{eq:reward_decouple} over all LLMs serves as an upper bound of the summation of \cref{eq:reward_activation} over all LLMs. Intuitively, this criterion conveys a clear meaning: the first term, $-\alpha\|C_i^t-C_i^b\|$, preserves the fundamental problem-solving capability endowed by the initialization, while the second term, $-\|a([I_i^t,P])-a([X_i^{t-1},P])\|$, enforces consistency between an LLM’s current instruction and its own previous response. This local consistency constraint implicitly aligns all LLMs, because when every LLM evolves its context in a temporally coherent way, the divergence among them is gradually reduced. Hence, the collective effect of all LLMs following this rule is that their contexts evolve coherently across rounds, progressively shrinking activation differences and ultimately driving the MAD toward consistency in their final answers.

To intuitively show the tightness of the bound, consider the case where $I_i^t=X_i^{t-1}$ holds, i.e., the instruction of each LLM in the next round equals its response in the previous round. That means $a(I_i^t)=a(X_i^{t-1})$. Given that all LLMs receive the same context, the second term of both \cref{eq:reward_activation} and \cref{eq:reward_decouple} becomes zero. In this case, the two formulations coincide.

\subsubsection{Multi-Round Context Evolving}
\label{sec:context_learning}

We slightly abuse notation by denoting $a(I_i^t)=a([I_i^t;P])$ and $a(X_i^{t-1})=a([X_i^{t-1};P])$, and provide the multi-round context evolution objective by accumulating contributions as follows:
\begin{align}
	\label{eq:objective_multi_round}
	\sum_{t=1}^T\max_{I_i^t}\Big\{-\alpha\|C_i^{t}-C_i^b\|-\|a(I_i^t)-a({X}_i^{t-1})\|\Big\}.
\end{align}
However, selecting the optimal value for the weight $\alpha$ is non-trivial, as it must be carefully tuned for the evolving discussion process. To alleviate the need for manual tuning, we recast \cref{eq:objective_multi_round} into a constrained optimization problem, where the context adjustment $\|I_i^t-I_i^b\|$ is treated as a constraint:
\begin{align}
	\label{eq:objective_multi_round_final}
	\min_{I_{i}^1} &\bigg(\|a(I_{i}^1)-a(X_{i}^0)\|+\min_{I_{i}^2} \Big(\|a(I_{i}^2)-a(X_{i}^1)\|+\cdots+ \min_{I_{i}^T}\|a(I_{i}^T)-a(X_{i}^{T-1})\|\Big)\bigg)\nonumber\\
	\text{s.t.}&~~ \|C_{i}^t-C_i^b\|\leq \beta, \quad \forall t,i.
\end{align}
According to the detailed derivation in \cref{sec:loss}, we prove that taking the dual of this problem recovers \cref{eq:objective_multi_round} and yields an auxiliary update for the dual variable $\alpha$. Consequently, the optimization of Problem~(\ref{eq:objective_multi_round_final}) can be implemented using an approximate dual gradient descent procedure, alternating gradient updates of $L(\theta_i)$ and $L(\alpha_i)$.
\begin{align}
	\label{eq:loss}
	&L(\theta_i)=\big\|a(\mathcal{G}_{\theta_i}(P, I_i^b, \bar{X}_i^{t-1}))-a(X_i^{t-1})\big\|+\alpha\big\|C_i^t-C_i^b\big\|\nonumber\\
	&L(\alpha_i)=\alpha_i\big(\beta-\big\|\mathcal{G}_{\theta_i}(P, I_i^b, \bar{X}_i^{t-1})-C_i^b\big\|\big).
\end{align}
At the beginning of the discussion, the initial contexts are intentionally diverse to encourage multi-perspective reasoning of the problem. At this stage, when the answers differ greatly, $\alpha$ decreases rapidly, thereby weakening the constraint on the distance between the generated contexts and their initial contexts. This guides the generation of contexts towards promoting faster convergence on a unified solution. As the discussion progresses and the LLMs gradually reach agreement, $\alpha$ will be kept at a certain level. This adjustment prevents premature consensus and supports a richer, more comprehensive final solution by shifting the contexts’ focus toward preserving multiple perspectives and exploring nuanced differences.

Overall, we name our proposed algorithm \alg{}, with its pseudocode detailed in \cref{sec:algorithm}.

\section{Experiment}
\label{sec:experiment}

\begin{table*}[t]
    \vspace{-1em}
	\centering
	\caption{Accuracy (\%) on different datasets. The number of LLMs is $4$ for all dataset. We exhibit the performance advantage with \texttt{BoN} and highlight the \colorbox{blue!10}{best} result.}
	\centering
	\vskip 0.1in
	\resizebox{1.0\textwidth}{!}{
		\begin{tabular}{YZYYYYYYYY}
			\toprule
			\textbf{Model} & \textbf{Method} & \textbf{MMLU} & \textbf{MATH} & \textbf{GPQA} & \textbf{Code} & \textbf{ALFWorld} & \textbf{SciWorld} & \textbf{GAIA} & \textbf{PDDL} \\
			\midrule
			\multirow{7}{*}{Qwen-7B}
			& \texttt{Single} & 61.2{\tiny\textcolor{green}{$\downarrow$13.0}} & 12.9{\tiny\textcolor{green}{$\downarrow$12.0}} & 20.2{\tiny\textcolor{green}{$\downarrow$16.2}} & 51.2{\tiny\textcolor{green}{$\downarrow$11.3}} & 23.8{\tiny\textcolor{green}{$\downarrow$7.7}} & 25.2{\tiny\textcolor{green}{$\downarrow$10.1}} & 15.6{\tiny\textcolor{green}{$\downarrow$5.5}} & 21.0{\tiny\textcolor{green}{$\downarrow$5.3}} \\
			& \texttt{BoN} & 74.2{\tiny\textcolor{red}{$\uparrow$0.0}} & 24.9{\tiny\textcolor{red}{$\uparrow$0.0}} & 36.4{\tiny\textcolor{red}{$\uparrow$0.0}} & 62.5{\tiny\textcolor{red}{$\uparrow$0.0}} & 31.5{\tiny\textcolor{red}{$\uparrow$0.0}} & 35.3{\tiny\textcolor{red}{$\uparrow$0.0}} & 21.1{\tiny\textcolor{red}{$\uparrow$0.0}} & 26.3{\tiny\textcolor{red}{$\uparrow$0.0}} \\
			& \texttt{Debate} & 71.1{\tiny\textcolor{green}{$\downarrow$3.1}} & 19.9{\tiny\textcolor{green}{$\downarrow$5.0}} & 28.7{\tiny\textcolor{green}{$\downarrow$7.7}} & 60.0{\tiny\textcolor{green}{$\downarrow$2.5}} & 30.6{\tiny\textcolor{green}{$\downarrow$0.9}} & 32.2{\tiny\textcolor{green}{$\downarrow$3.1}} & 21.0{\tiny\textcolor{green}{$\downarrow$0.1}} & 24.8{\tiny\textcolor{green}{$\downarrow$1.5}} \\
			& \texttt{DyLAN} & 74.3{\tiny\textcolor{red}{$\uparrow$0.1}} & 26.7{\tiny\textcolor{red}{$\uparrow$1.8}} & 35.4{\tiny\textcolor{green}{$\downarrow$1.0}} & 63.4{\tiny\textcolor{red}{$\uparrow$0.9}} & 29.8{\tiny\textcolor{green}{$\downarrow$1.7}} & 29.4{\tiny\textcolor{green}{$\downarrow$5.9}} & 18.4{\tiny\textcolor{green}{$\downarrow$2.7}} & 23.4{\tiny\textcolor{green}{$\downarrow$2.9}} \\
			& \texttt{GPTSwarm} & 76.3{\tiny\textcolor{red}{$\uparrow$2.1}} & 26.2{\tiny\textcolor{red}{$\uparrow$1.3}} & 35.6{\tiny\textcolor{green}{$\downarrow$0.8}} & 62.7{\tiny\textcolor{red}{$\uparrow$0.2}} & 29.9{\tiny\textcolor{green}{$\downarrow$1.6}} & 30.8{\tiny\textcolor{green}{$\downarrow$4.5}} & 20.6{\tiny\textcolor{green}{$\downarrow$0.5}} & 24.5{\tiny\textcolor{green}{$\downarrow$1.8}} \\
			& \texttt{MacNet} & 71.5{\tiny\textcolor{green}{$\downarrow$2.7}} & 21.4{\tiny\textcolor{green}{$\downarrow$3.5}} & 30.8{\tiny\textcolor{green}{$\downarrow$5.6}} & 59.3{\tiny\textcolor{green}{$\downarrow$3.2}} & 33.6{\tiny\textcolor{red}{$\uparrow$2.1}} & 37.0{\tiny\textcolor{red}{$\uparrow$1.7}} & 21.5{\tiny\textcolor{red}{$\uparrow$0.4}} & 29.5{\tiny\textcolor{red}{$\uparrow$3.2}} \\
			& {\alg{} (ours)} & \cellcolor{blue!10}92.5{\tiny\textcolor{red}{$\uparrow$18.3}} & \cellcolor{blue!10}47.8{\tiny\textcolor{red}{$\uparrow$22.9}} & \cellcolor{blue!10}66.1{\tiny\textcolor{red}{$\uparrow$29.7}} & \cellcolor{blue!10}80.3{\tiny\textcolor{red}{$\uparrow$17.8}} & \cellcolor{blue!10}39.9{\tiny\textcolor{red}{$\uparrow$8.4}} & \cellcolor{blue!10}45.3{\tiny\textcolor{red}{$\uparrow$10.0}} & \cellcolor{blue!10}33.6{\tiny\textcolor{red}{$\uparrow$12.5}} & \cellcolor{blue!10}34.7{\tiny\textcolor{red}{$\uparrow$8.4}} \\
			\midrule
			\multirow{7}{*}{Qwen-14B}
			& \texttt{Single} & 67.2{\tiny\textcolor{green}{$\downarrow$12.5}} & 21.6{\tiny\textcolor{green}{$\downarrow$6.2}} & 21.2{\tiny\textcolor{green}{$\downarrow$11.6}} & 56.7{\tiny\textcolor{green}{$\downarrow$12.7}} & 30.1{\tiny\textcolor{green}{$\downarrow$7.4}} & 31.3{\tiny\textcolor{green}{$\downarrow$10.3}} & 18.7{\tiny\textcolor{green}{$\downarrow$6.9}} & 25.7{\tiny\textcolor{green}{$\downarrow$2.8}} \\
			& \texttt{BoN} & 79.7{\tiny\textcolor{red}{$\uparrow$0.0}} & 27.8{\tiny\textcolor{red}{$\uparrow$0.0}} & 32.8{\tiny\textcolor{red}{$\uparrow$0.0}} & 69.4{\tiny\textcolor{red}{$\uparrow$0.0}} & 37.5{\tiny\textcolor{red}{$\uparrow$0.0}} & 41.6{\tiny\textcolor{red}{$\uparrow$0.0}} & 25.6{\tiny\textcolor{red}{$\uparrow$0.0}} & 28.5{\tiny\textcolor{red}{$\uparrow$0.0}} \\
			& \texttt{Debate} & 77.2{\tiny\textcolor{green}{$\downarrow$2.5}} & 27.4{\tiny\textcolor{green}{$\downarrow$0.4}} & 30.3{\tiny\textcolor{green}{$\downarrow$2.5}} & 66.9{\tiny\textcolor{green}{$\downarrow$2.5}} & 36.8{\tiny\textcolor{green}{$\downarrow$0.7}} & 39.4{\tiny\textcolor{green}{$\downarrow$2.2}} & 26.4{\tiny\textcolor{red}{$\uparrow$0.8}} & 30.0{\tiny\textcolor{red}{$\uparrow$1.5}} \\
			& \texttt{DyLAN} & 86.8{\tiny\textcolor{red}{$\uparrow$7.1}} & 31.2{\tiny\textcolor{red}{$\uparrow$3.4}} & 39.4{\tiny\textcolor{red}{$\uparrow$6.6}} & 76.6{\tiny\textcolor{red}{$\uparrow$7.2}} & 34.8{\tiny\textcolor{green}{$\downarrow$2.7}} & 36.0{\tiny\textcolor{green}{$\downarrow$5.6}} & 22.8{\tiny\textcolor{green}{$\downarrow$2.8}} & 28.0{\tiny\textcolor{green}{$\downarrow$0.5}} \\
			& \texttt{GPTSwarm} & 87.0{\tiny\textcolor{red}{$\uparrow$7.3}} & 31.3{\tiny\textcolor{red}{$\uparrow$3.5}} & 38.7{\tiny\textcolor{red}{$\uparrow$5.9}} & 75.9{\tiny\textcolor{red}{$\uparrow$6.5}} & 34.9{\tiny\textcolor{green}{$\downarrow$2.6}} & 35.9{\tiny\textcolor{green}{$\downarrow$5.7}} & 25.0{\tiny\textcolor{green}{$\downarrow$0.6}} & 28.4{\tiny\textcolor{green}{$\downarrow$0.1}} \\
			& \texttt{MacNet} & 78.9{\tiny\textcolor{green}{$\downarrow$0.8}} & 28.7{\tiny\textcolor{red}{$\uparrow$0.9}} & 31.9{\tiny\textcolor{green}{$\downarrow$0.9}} & 70.4{\tiny\textcolor{red}{$\uparrow$1.0}} & 39.2{\tiny\textcolor{red}{$\uparrow$1.7}} & 46.0{\tiny\textcolor{red}{$\uparrow$4.4}} & 32.2{\tiny\textcolor{red}{$\uparrow$6.6}} & 32.7{\tiny\textcolor{red}{$\uparrow$4.2}} \\
			& {\alg{} (ours)} & \cellcolor{blue!10}93.7{\tiny\textcolor{red}{$\uparrow$14.0}} & \cellcolor{blue!10}51.7{\tiny\textcolor{red}{$\uparrow$23.9}} & \cellcolor{blue!10}66.2{\tiny\textcolor{red}{$\uparrow$33.4}} & \cellcolor{blue!10}91.1{\tiny\textcolor{red}{$\uparrow$21.7}} & \cellcolor{blue!10}48.2{\tiny\textcolor{red}{$\uparrow$10.7}} & \cellcolor{blue!10}56.1{\tiny\textcolor{red}{$\uparrow$14.5}} & \cellcolor{blue!10}42.0{\tiny\textcolor{red}{$\uparrow$16.4}} & \cellcolor{blue!10}43.0{\tiny\textcolor{red}{$\uparrow$14.5}} \\
			\midrule
			\multirow{7}{*}{Qwen-72B}
			& \texttt{Single} & 72.5{\tiny\textcolor{green}{$\downarrow$11.7}} & 31.6{\tiny\textcolor{green}{$\downarrow$19.4}} & 34.9{\tiny\textcolor{green}{$\downarrow$11.0}} & 59.1{\tiny\textcolor{green}{$\downarrow$13.1}} & 48.2{\tiny\textcolor{green}{$\downarrow$9.3}} & 50.4{\tiny\textcolor{green}{$\downarrow$11.7}} & 31.1{\tiny\textcolor{green}{$\downarrow$10.1}} & 41.0{\tiny\textcolor{green}{$\downarrow$10.0}} \\
			& \texttt{BoN} & 84.2{\tiny\textcolor{red}{$\uparrow$0.0}} & 51.0{\tiny\textcolor{red}{$\uparrow$0.0}} & 45.9{\tiny\textcolor{red}{$\uparrow$0.0}} & 72.2{\tiny\textcolor{red}{$\uparrow$0.0}} & 57.5{\tiny\textcolor{red}{$\uparrow$0.0}} & 62.1{\tiny\textcolor{red}{$\uparrow$0.0}} & 41.2{\tiny\textcolor{red}{$\uparrow$0.0}} & 51.0{\tiny\textcolor{red}{$\uparrow$0.0}} \\
			& \texttt{Debate} & 82.7{\tiny\textcolor{green}{$\downarrow$1.5}} & 48.4{\tiny\textcolor{green}{$\downarrow$2.6}} & 43.4{\tiny\textcolor{green}{$\downarrow$2.5}} & 69.1{\tiny\textcolor{green}{$\downarrow$3.1}} & 60.4{\tiny\textcolor{red}{$\uparrow$2.9}} & 65.2{\tiny\textcolor{red}{$\uparrow$3.1}} & 42.9{\tiny\textcolor{red}{$\uparrow$1.7}} & 49.1{\tiny\textcolor{green}{$\downarrow$1.9}} \\
			& \texttt{DyLAN} & 91.5{\tiny\textcolor{red}{$\uparrow$7.3}} & 63.1{\tiny\textcolor{red}{$\uparrow$12.1}} & 51.6{\tiny\textcolor{red}{$\uparrow$5.7}} & 80.4{\tiny\textcolor{red}{$\uparrow$8.2}} & 55.1{\tiny\textcolor{green}{$\downarrow$2.4}} & 58.0{\tiny\textcolor{green}{$\downarrow$4.1}} & 40.4{\tiny\textcolor{green}{$\downarrow$0.8}} & 45.5{\tiny\textcolor{green}{$\downarrow$5.5}} \\
			& \texttt{GPTSwarm} & 91.5{\tiny\textcolor{red}{$\uparrow$7.3}} & 64.7{\tiny\textcolor{red}{$\uparrow$13.7}} & 52.4{\tiny\textcolor{red}{$\uparrow$6.5}} & 79.6{\tiny\textcolor{red}{$\uparrow$7.4}} & 56.8{\tiny\textcolor{green}{$\downarrow$0.7}} & 60.2{\tiny\textcolor{green}{$\downarrow$1.9}} & 40.3{\tiny\textcolor{green}{$\downarrow$0.9}} & 49.7{\tiny\textcolor{green}{$\downarrow$1.3}} \\
			& \texttt{MacNet} & 83.8{\tiny\textcolor{green}{$\downarrow$0.4}} & 52.9{\tiny\textcolor{red}{$\uparrow$1.9}} & 46.2{\tiny\textcolor{red}{$\uparrow$0.3}} & 70.5{\tiny\textcolor{green}{$\downarrow$1.7}} & 61.8{\tiny\textcolor{red}{$\uparrow$4.3}} & 68.4{\tiny\textcolor{red}{$\uparrow$6.3}} & 46.4{\tiny\textcolor{red}{$\uparrow$5.2}} & 53.7{\tiny\textcolor{red}{$\uparrow$2.7}} \\
			& {\alg{} (ours)} & \cellcolor{blue!10}95.1{\tiny\textcolor{red}{$\uparrow$10.9}} & \cellcolor{blue!10}72.5{\tiny\textcolor{red}{$\uparrow$21.5}} & \cellcolor{blue!10}78.9{\tiny\textcolor{red}{$\uparrow$33.0}} & \cellcolor{blue!10}90.7{\tiny\textcolor{red}{$\uparrow$18.5}} & \cellcolor{blue!10}79.0{\tiny\textcolor{red}{$\uparrow$21.5}} & \cellcolor{blue!10}88.9{\tiny\textcolor{red}{$\uparrow$26.8}} & \cellcolor{blue!10}67.2{\tiny\textcolor{red}{$\uparrow$26.0}} & \cellcolor{blue!10}70.5{\tiny\textcolor{red}{$\uparrow$19.5}} \\
			\bottomrule
		\end{tabular}
	}
	\label{tab:comparative_results}
\end{table*}

\begin{figure*}[t]
	\vspace{-0.7em}
	\centering
	\includegraphics[width=0.995\linewidth]{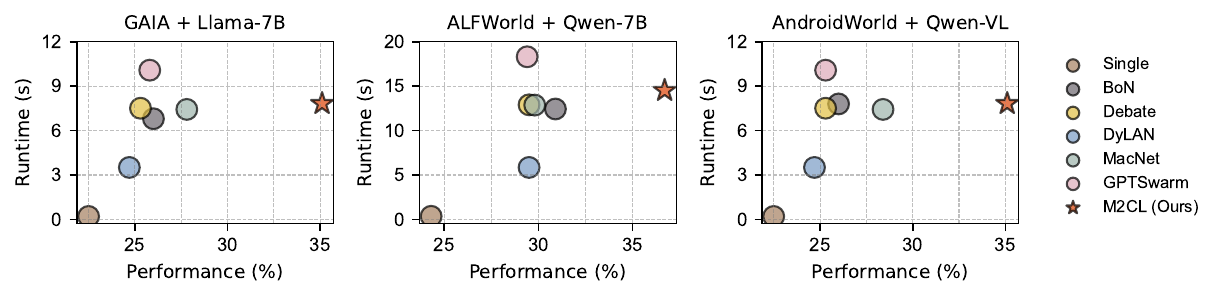}
	\vspace{-2.2em}
	\caption{Performance versus runtime under different settings. Circles closer to the lower-left corner indicate higher efficiency.}
	\label{fig:efficiency_performance}
	\vspace{-1em}
\end{figure*}

In this section, we conduct experiments to evaluate the performance of \alg{} by answering the following research questions:
\begin{enumerate}[itemsep=0pt,topsep=0pt]
	\renewcommand{\labelenumi}{Q\theenumi.}
	\item How does \alg{} perform compared to existing methods across various benchmarks, especially in complex agentic tasks?
	\item How does the performance of MAD scale with the number of LLMs?
	\item How is the performance affected by factors such as context constraint and components such as context initialization and context evolution?
	\item How do contexts promote consensus, and are they transferable to other models?
\end{enumerate}

\subsection{Experimental Setup}

\textbf{Dataset.} We run experiments with 3 domains including 9 datasets:
\textbf{1)} \textbf{LLM reasoning}, including MMLU~\citep{hendrycks2021measuring}, MATH~\citep{hendrycks2021measuring}, and GPQA~\citep{rein2023gpqa}, HumanEval~\citep{chen2021evaluating}. \textbf{2)} \textbf{Embodied Agentic}, including ALFWorld~\cite{shridhar2021alfworld}, SciWorld~\cite{wang2022scienceworld}, GAIA~\citep{mialon2023gaia}, and PDDL~\cite{chang2024agentboard}. \textbf{3)} \textbf{Mobile GUI} AndroidWorld~\cite{rawles2025androidworld}. Details on datasets can be found in \cref{sec:datasets}.

\textbf{Baselines.} We evaluate our method against six strong baseline methods:
\textbf{1)} \texttt{Single execution}, querying a single LLM to solve the task.
\textbf{2)} \texttt{Best-of-N}, querying a single LLM N times and sampling the most correct answer.
\textbf{3)} \texttt{Debate}~\citep{du2023improving}, a multi-agent framework where LLMs discuss their responses and reasoning processes over multiple rounds.
\textbf{4)} \texttt{DyLAN}~\citep{liu2024dynamic}, a multi-agent framework where LLMs score each other and collaborate dynamically.
\textbf{5)} \texttt{GPTSwarm}~\citep{zhuge2024gptswarm}, a multi-agent framework that refines LLM prompts and improves LLM orchestration by changing their connectivity.
\textbf{6)} \texttt{MacNet}~\citep{jiang2023llm}, a recent multi-LLM framework where LLMs are invoked between LLM interactions to provide actionable instructions to the next LLM based on the previous LLM’s outputs.

\textbf{Reproducibility.} All details of our experiments are provided in the appendices in terms of the tasks, network architectures, hyperparameters, etc. We conduct experiment on two series of LLM (Llama-2~\citep{touvron2023llama} and Qwen-2.5~\citep{yang2024qwen2}) and three model sizes each for LLM reasoning, and Qwen2.5-VL~\citep{bai2025qwen2} for GUI reasoning. All the experiments are run on Ubuntu 22.04.4 LTS with 8 NVIDIA H800 GPUs.

\begin{figure*}[t]
	\centering
	\includegraphics[width=0.995\linewidth]{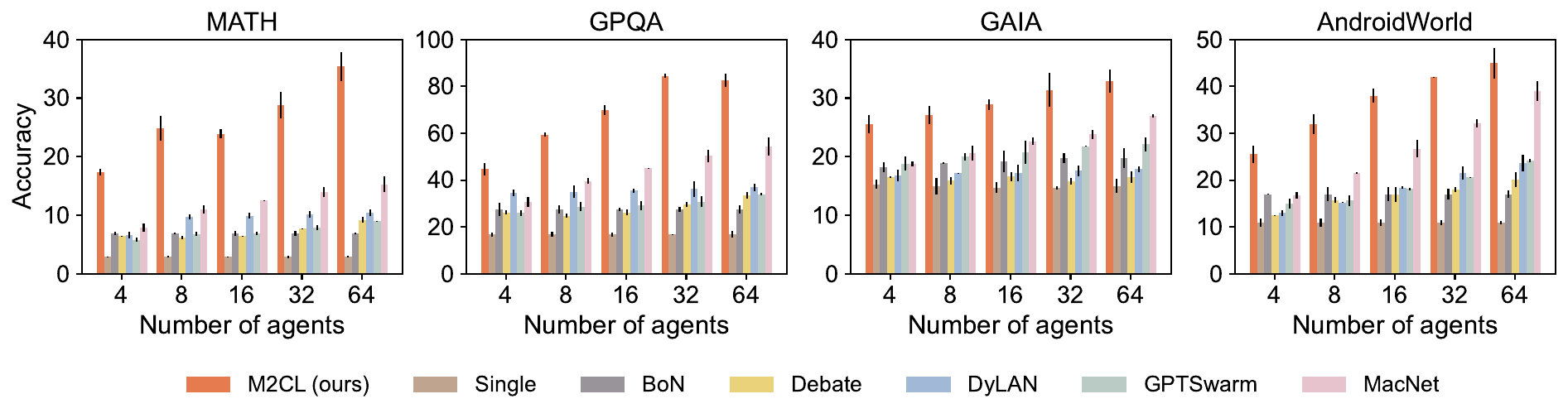}
	\vspace{-2em}
	\caption{Performance of varying the numbers of LLMs. Uncertainty intervals depict standard deviation over three seeds.}
    \vspace{-1em}
	\label{fig:number_agent}
\end{figure*}

\subsection{Experimental Results}
\textbf{Comparative results.} To answer the first question, we evaluate \alg{}'s performance across all datasets, with varying base models and number of LLMs (ranging from $4$ to $64$). We select Qwen series models and $4$ LLMs participating in \cref{tab:comparative_results} and provide full results in \cref{sec:number_agents,sec:agent}. We find \alg{} consistently outperforms baselines in all $9$ datasets, often by a significant margin in terms of performance. Of note, \texttt{BoN} outperforms most of the baselines especially in complex multi-round agentic tasks, revealing the drawback of fixed contexts, which, despite expanding the exploration space, do not converge and thus hinder LLMs from achieving true cooperative reasoning. In contrast, \alg{} can adapt contexts and enhance the relevance of responses and questions while ensuring creativity, indicating that \alg{} well avoids that LLMs with different contexts easily influence each other and successfully brings LLMs into cooperation by reaching a consensus through discussion.

In addition, we evaluate the efficiency of \alg{} by visualizing the performance versus runtime. As shown in \cref{fig:efficiency_performance}, \alg{} consistently delivers the highest performance improvement (more than $20\%$) while maintaining a modest increase in runtime (less than $10\%$). This clearly demonstrates the efficiency of \alg{} due to its lightweight context generator.

\textbf{Multi-agent scaling law.} To answer the second question, we run experiments with varying numbers of LLMs (ranging from 4 to 64). The data and parameter setup adhere to that of \cref{tab:parameter_setting}. We present selected results in \cref{fig:number_agent} and full results in \cref{tab:number_4,tab:number_8,tab:number_16,tab:number_32,tab:number_64,fig:number_agents_1,fig:number_agents_2,fig:number_agents_3,fig:number_agents_4,fig:number_agents_5,fig:number_agents_6} of \cref{sec:number_agents}. Scaling our method reveals a more efficient scaling law as the performance grows logarithmically before saturation and improves faster than baselines. We speculate this arises because collaborative instructions in our generated contexts enable genuine inter-LLM cooperation, thereby unleashing the multidimensional reasoning capabilities of the MAD.

\textbf{Context constraint.} To answer the third question, we vary the context constraint $\beta$ from $0$ to $10$ and run experiments across all datasets. Full results are shown in \cref{fig:context_constraint_1,fig:context_constraint_2,fig:context_constraint_3,fig:context_constraint_4,fig:context_constraint_5,} of \cref{sec:context_constraint}. The results clearly indicate that as $\beta$ increases, there is an initial improvement in performance; once it reaches a sufficiently large value, performance tends to drop. A strict (refers to a smaller $\beta$) context constraint leads to generated contexts getting closer to the initial contexts, thereby causing discussion inconsistency. On the other hand, a loose context constraint leads to naive consistency, where LLMs tend to generate the same answers, resulting in insufficient creativity.

\textbf{Ablation studies.} We assess the effect of key components by ablating them on all datasets under the same setting. \emph{1) Importance of context initialization.} \emph{2) Importance of tuning $\alpha$.} \emph{3) Importance of context evolution.} As illustrated in \cref{tab:ablation_llama_7b,tab:ablation_llama_13b,tab:ablation_llama_70b}, LLMs struggle to specialize and coordinate efficiently, whereas context initialization enables them to acquire high-impact contexts, significantly enhancing the foundational capabilities of MAD. Without tuning $\alpha$ during discussion rounds, LLMs tend to reach an agreement in the first round, leading to responses that lack creativity and diversity, which ultimately reduces problem-solving ability. Without context evolution, LLMs lack collaborative guidance considering previous responses, which prevents them from effectively leveraging the outputs of other LLMs.


\textbf{Discrepancy intensity.} To answer the fourth question, we explored the change of Discrepancy intensity ($\max_{i,j\in[N]}\|a_i-a_j\|^2$) with rounds to verify whether the generated context can help the LLMs gradually reach agreement with the discussion. Complete results are shown in \cref{fig:agreement_intensity_0,fig:agreement_intensity_1,fig:agreement_intensity_2,fig:agreement_intensity_3,fig:agreement_intensity_4}, of \cref{sec:agreement_intensity}. We find that discrepancy intensity from \alg{} decreases faster than other methods, corroborating its efficacy and superiority in converging the search space of multiple LLMs.

\textbf{Transferability of contexts.} To answer the fourth question, we also conduct experiments on transferring contexts directly into stronger LLMs to verify whether the generated contexts have better interpretation ability and efficacy. Full results are depicted in \cref{tab:context_transfer_all} of \cref{sec:transferability}. The results show that the transferred contexts deliver consistent improvement, indicating that the trained context generator can directly adapt to a wide range of models without additional retraining.

\section{Limitation and Discussion}
In this paper, we propose a novel context learning method, designed to expand LLMs' horizons and reach consensus in MAD. By initializing and adjusting LLMs' contexts based on the problem and discussion states, our method significantly improves problem-solving capabilities across diverse benchmarks while maintaining computational efficiency.
A limitation of \alg{} is the MAD framework in which diversity is brought about by the number of LLMs with heterogeneous characteristics which is computationally inefficient. An avenue for future work is to enable LLMs to truly capture the sub-tasks they are interested in or excel at.

\section*{Acknowledgments}
This research was supported in part by the National Natural Science Foundation of China under Grants 62572496 and 62402269, the Shenzhen Science and Technology Program under Grant JCYJ20250604175500001, and the Young Elite Scientist Sponsorship Program by CAST under Contract ZB2025-218.

\section*{Ethics Statement}
This work advances the field of large language model (LLM) research by introducing a context learning method for multi-agent discussion, demonstrating significant improvements in accuracy, diversity, and consensus across multiple benchmarks.

However, the broader implications of deploying such systems warrant careful consideration. Multi-agent discussion, while powerful, may inadvertently propagate biases introduced during context initialization or amplify errors during consensus-building. These risks are particularly critical in high-stakes domains like legal, financial, and healthcare applications, where erroneous outputs could lead to significant ethical, social, or economic consequences.
\section*{Reproducibility Statement}
The full algorithmic details of \alg{} are presented in the main paper, with additional implementation details and hyperparameter settings included in \cref{sec:experiment}. We provide source code at \url{https://github.com/HansenHua/M2CL-ICLR26}, which includes scripts for training, evaluation, and reproducing all reported experiments. We provide checkpoints at \url{https://huggingface.co/hansenhua/M2CL}. For theoretical results, we include complete proofs in \cref{sec:activation_difference,sec:reward_decouple,sec:softmax_attention}.

\bibliography{ref}
\bibliographystyle{iclr2026_conference}

\appendix
\newpage
\section{Usage of LLMs}
We use large language models (LLMs) solely as an assistive tool for polishing the writing and improving clarity of exposition. All content generated by LLMs was carefully reviewed, verified, and, where necessary, revised by the authors. The authors take full responsibility for the correctness and integrity of the final manuscript.

\section{Proof of \cref{thm:activation_difference}}
\label{sec:activation_difference}
\begin{proof}
	Define the optimal $\omega^*=\arg\min_{\omega}\|a_c-\sum_{i=1}^N\omega_ia(C_i^b)\|$ and its corresponding activation $a=\sum_{i=1}^N\omega^*_ia(C_i^b)$. Then, we can derive the upper bound of the activation difference using the triangle inequality:
	\begin{align}
		\label{eq:accuracy_bound}
		\sum_{i=1}^N\|a_c-a(C_i^{t})\|\leq\sum_{i=1}^N\|a(C_i^{t})-a\|+N\|a-a_c\|.
	\end{align}
	Using the triangle inequality, we can bound the first term in \cref{eq:accuracy_bound}:
	\begin{align}
		&\sum_{i=1}^N\|a(C_i^{t})-a\|=\sum_{i=1}^N\|\sum_{j=1}^N\omega^*_j(a(C_i^{t})-a(C_j^b))\|\nonumber\\
		\leq&\sum_{i=1}^N\sum_{j=1}^N|\omega^*_j|\cdot\|a(C_i^{t})-a(C_j^b)\|\leq \sum_{i=1}^N\sum_{j=1}^N\|a(C_i^{t})-a(C_j^b)\|\nonumber\\
        \leq&\sum_{i=1}^N\sum_{j=1}^N\|a(C_i^{t})-a(C_j^{t})\|+\sum_{i=1}^N\sum_{j=1}^N\|a(C_j^{t})-a(C_j^b)\|\nonumber\\
        \leq&\sum_{i=1}^N\sum_{j=1}^N\|a(C_i^{t})-a(C_j^{t})\|+\sum_{i=1}^N\sum_{j=1}^NL_a\|C_j^t-C_j^b\|.
	\end{align}
	Then, we bound the second term in \cref{eq:accuracy_bound}:
	\begin{align}
		\|a-a_c\|&\leq\|\sum_{i=1}^N\omega^*_ia(C_i^{t})-a_c\|=\|\sum_{i=1}^N\omega^*_ia(C_i^b)+\sum_{i=1}^N\omega^*_i(a(C_i^{t})-a(C_i^b))-a_c\|\nonumber\\
		&\leq\min_{\omega}\|a_c-\sum_{i=1}^N\omega_ia(C_i^b)\|+\sum_{i=1}^N\|\omega^*_i(a(C_i^{t})-a(C_i^b))\|\nonumber\\
		&\leq\min_{\omega}\|a_c-\sum_{i=1}^N\omega_ia(C_i^b)\|+\sum_{i=1}^N\|a(C_i^{t})-a(C_i^b)\|\nonumber\\
		&\leq\min_{\omega}\|a_c-\sum_{i=1}^N\omega_ia(C_i^b)\|+L_a\sum_{i=1}^N\|C_i^{t}-C_i^b\|.
	\end{align}
	In the first line, we add and subtract $\sum_{i=1}^Na(C_i^b)$. In the second line, we use the triangle inequality and the definition of $\omega^*$. In the third line, we scale $\omega^*$ to 1. Finally, we utilize the smoothness property of the activation function to bound the second term
	Substituting into \cref{eq:accuracy_bound}, we can derive:
	\begin{align}
		\sum_{i=1}^N\|a_c-a(C_i^{t})\|\leq \sum_{i=1}^N\Big(\sum_{j=1}^N\|a(C_i^{t})-a(C_j^{t})\|&+(N+1)L_a\|C_i^t-C_i^b\|\Big)\nonumber\\
		&+N\min_{\omega}\|a_c-\sum_{i=1}^N\omega_ia(C_i^b)\|,
	\end{align}
\end{proof}

\section{Decoupled Criterion Fucntion}
\label{sec:reward_decouple}
\begin{lemma}
	\label{lem:activation_decouple}
	Under the assumption of one-block transformer, LLMs' activation diversity can be bound by the activation difference between their instructions and responses.
	\begin{align}
		\|a(C_i^t)-a(C_j^t)\|\leq\|a(I_i^t)-a(X_i^{t})\|+\|a(I_j^t)-a(X_j^{t})\|+18L_VnN\exp(2\rho^2).
	\end{align}
\end{lemma}
\begin{proof}
	We first decompose the activation of context into a combination of its components:
	\begin{align}
		&a(C_i^t)=a(C_i^t)-[a(I_i^t)+\sum_{k\neq i}^Na(X_k^{t-1})-(N-1)a(P)]+[a(I_i^t)+\sum_{k\neq i}^Na(X_k^{t-1})-(N-1)a(P)]\nonumber\\
		&a(C_j^t)=a(C_j^t)-[a(I_j^t)+\sum_{k\neq j}^Na(X_k^{t-1})-(N-1)a(P)]+[a(I_j^t)+\sum_{k\neq j}^Na(X_k^{t-1})-(N-1)a(P)].
	\end{align}
	By using \cref{lem:error_decomposition}, we can derive the upper bound of the activation difference as:
	\begin{align}
		&\|a(C_i^t)-a(C_j^t)\|\nonumber\\
		\leq&\|a(C_i^t)-[a(I_i^t)+\sum_{k\neq i}^Na(X_k^{t-1})-(N-1)a(P)]\|+\|a(I_i^t)-a(X_i^{t-1})\|\nonumber\\
		&+\|a(C_j^t)-[a(I_j^t)+\sum_{k\neq j}^Na(X_k^{t-1})-(N-1)a(P)]\|+\|a(I_j^t)-a(X_j^{t-1})\|\nonumber\\
		\leq&\|a(I_i^t)-a(X_i^{t-1})\|+\|a(I_j^t)-a(X_j^{t-1})\|+18L_VnN\exp(2\rho^2).\nonumber\\
	\end{align}
\end{proof}

\section{Useful Lemma}
\label{sec:softmax_attention}

\begin{lemma}
	\label{lem:error_attention}
	Suppose $\|W_VX\|\leq L_V$ and $\|(W_KX)^TW_QX\|\leq\rho^2$, the difference between the activation of softmax attention $a(X)$ and linear attention $a'(X)$ can be bounded by:
	\begin{align}
		\|a(X)-a'(X)\|\leq3L_Vn\exp(2\rho^2)
	\end{align}
	where $a(X)=W_VX\softmax\big(\frac{(W_KX)^TW_QX}{\sqrt{d}}\big)$ and $a'(X)=W_VX(W_KX)^TW_QX$.
\end{lemma}
\begin{proof}
	Define $\Delta=D^{-1}\exp(S)-\exp(S)$, where $S=(W_KX)^TW_QX$ and $D=diag(\langle\exp(S),\mathbf{1}_n\rangle)$. Then, we can derive the upper bound of each element of $\Delta$.
	\begin{align}
		\Delta_{ij}=\frac{\exp(S_{ij})}{D_{ii}}-\exp(S_{ij})\leq\exp(S_{ij})(|\frac{1}{D_{ii}}|+1)\leq\exp(\rho^2)(1+\exp(\rho^2))\leq2\exp(2\rho^2)
	\end{align}
	Summing them up, we can derive:
	\begin{align}
		\|D^{-1}\exp(S)-\exp(S)\|\leq\|D^{-1}\exp(S)-\exp(S)\|_F=\sqrt{\sum_{i=1}^n\sum_{j=1}^n\Delta_{ij}^2}\leq 2n\exp(2\rho^2)
	\end{align}
	Define $\delta=\exp(S)-S$, where $S=(W_KX)^TW_QX$. Then, we derive the upper bound of $\delta$ by finding the upper bound of each element of $\delta$.
	\begin{align}
		\delta_{ij}=\exp(S_{ij})-S_{ij}\leq\max\{\exp(\rho^2)-\rho^2, \rho^2+\exp(-\rho^2)\}\leq\exp(\rho^2)
	\end{align}
	Summing them up, we can derive:
	\begin{align}
		\|\exp(S)-S\|\leq\|\exp(S)-S\|_F=\sqrt{\sum_{i=1}^n\sum_{j=1}^n\exp(S_{ij})-S_{ij}}\leq n\exp(\rho^2)
	\end{align}
	After bounding $\|D^{-1}\exp(S)-\exp(S)\|$ and $\|\exp(S)-S\|$, we can derive the upper bound of $\|D^{-1}\exp(S)-S\|$ by basic algebra as:
	\begin{align}
		&\|D^{-1}\exp(S)-S\|\nonumber\\
		=&\|D^{-1}\exp(S)-\exp(S)+\exp(S)-S\|\nonumber\\
		\leq&\|D^{-1}\exp(S)-\exp(S)\|+\|\exp(S)-S\|\nonumber\\
		\leq&2n\exp(2\rho^2)+n\exp(\rho^2)\nonumber\\
		\leq&3n\exp(2\rho^2)
	\end{align}
	Therefore, we derive the difference between the activation of softmax attention and linear attention as:
	\begin{align}
		&\|W_VX\softmax\bigg(\frac{(W_KX)^TW_QX}{\sqrt{d}}\bigg)-W_VX(W_KX)^TW_QX\|\nonumber\\
		\leq&\|W_VX\|\cdot\|\softmax\bigg(\frac{(W_KX)^TW_QX}{\sqrt{d}}\bigg)-(W_KX)^TW_QX\|\nonumber\\
		\leq&L_V\cdot\|D^{-1}\exp(S)-S\|\nonumber\\
		\leq&3L_Vn\exp(2\rho^2)
	\end{align}
\end{proof}

\begin{lemma}
	\label{lem:linear_attention}
	Define the activation of linear attention:
	\begin{align}
		a'(Y)&\doteq W_V[Y,P](W_K[Y,P])^TW_QY\nonumber\\
		a'(P)&\doteq W_VP(W_KP)^TW_QP,
	\end{align}
	then the activation of a long prompt $Y=[Y_1, Y_2, \dots,Y_N]$ can be derived by the activation of its component:
	\begin{align}
		a'(Y)=\sum_{i=1}^Na'({Y}_i)-(N-1)a'(P).
	\end{align}
\end{lemma}
\begin{proof}
	\begin{align}
		a'({Y})&=W_V[Y_1, Y_2, \dots,Y_N,P]\Big(W_K[Y_1, Y_2, \dots,Y_N,P]\Big)^TW_Q{P}\nonumber \\
		&=W_V[Y_1, Y_2, \dots,Y_N,P][Y_1, Y_2, \dots,Y_N,P]^TW_K^TW_Q{P}\nonumber\\
		&=W_V\big(\sum_{i=1}^NY_iY_i^T+PP^T\big)W_K^TW_Q{P}\nonumber\\
		&=W_V\Big[\sum_{i=1}^N(Y_iY_i^T+PP^T)-(N-1)PP^T\Big]W_K^TW_Q{P}\nonumber\\
		&=\sum_{i=1}^NW_V(Y_iY_i^T+P P^T)W_K^TW_Q{P}-W_V(N-1)PP^TW_K^TW_Q{P}\nonumber\\
		&=\sum_{i=1}^NW_V[Y_i,P](W_K[Y_i,P])^TW_Q{P}-(N-1)W_KP(W_KP)^TW_Q{P}\nonumber\\
		&=\sum_{i=1}^Na'(Y_i)-(N-1)a'(P).
	\end{align}
\end{proof}

\begin{lemma}
	\label{lem:error_decomposition}
	We derive the difference between the activation of a long prompt $Y$ and a combination of the activations of its components $Y=[Y_1, Y_2,\dots, Y_N, P]$ as:
	\begin{align}
		\|a(Y)-\sum_{i=1}^Na(Y_i)+(N-1)a(P)\|\leq9L_VnN\exp(\rho^2).
	\end{align}
\end{lemma}
\begin{proof}
	Denote the activation of linear attention as $a'$, we can complete the proof by using \cref{lem:linear_attention} to simplify the relationship of linear attention and \cref{lem:error_attention} to bound the difference:
	\begin{align}
		&\|a(Y)-\sum_{i=1}^Na(Y_i)+(N-1)a(P)\|\nonumber\\
		=&\|a(Y)-a'(Y)+a'(Y)-[\sum_{i=1}^Na'(Y_i)-(N-1)a'(P)]\nonumber\\
		&+[\sum_{i=1}^Na'(Y_i)-(N-1)a'(P)]-[\sum_{i=1}^Na(Y_i)-(N-1)a(P)]\|\nonumber\\
		\leq&\|a(Y)-a'(Y)\|+\sum_{i=1}^N\|a'(Y_i)-a(Y_i)\|+(N-1)\|a'(P)-a(P)\|\nonumber\\
		\leq&\exp(2\rho^2)[3L_VnN+3L_VNn+3L_V(N-1)n]\nonumber\\
		\leq&9L_VnN\exp(2\rho^2)
	\end{align}
\end{proof}

\begin{lemma}
	If $\|S\|\leq\rho$, we have $\langle\exp(S),\mathbf{1}_n\rangle\geq\exp(-\rho)$
\end{lemma}
\begin{proof}
	\begin{align}
		&\langle \exp(S), \mathbf{1}_n \rangle = \sum_{i=1}^{n} \exp(S_i) \geq \min_{i \in [n]} \exp(S_i)\geq \min_{i \in [n]} \exp(-|S_i|) \notag \\
		=& \exp(- \max_{i \in [n]} |S_i|)= \exp(-\|S\|_{\infty})\geq \exp(-\|S\|_2)\geq \exp(-\rho)
	\end{align}
\end{proof}

\section{Multi-Round Context Learning}
\label{sec:loss}
First, we define:
\begin{align}
	h(I_i^t) &\doteq \beta-\|I_i^t-I_i^b\|\\
	f(I_i^t) &\doteq \begin{cases}
		\|a(I_i^t)-a({X}_i^{t-1})\|, & \text{if }h(I_i^t) \geq 0 \\
		+\infty, & \text{otherwise}.
	\end{cases}
\end{align}
To solve the minimization optimization with inequality constraint, we can construct a Lagrangian expression with a Lagrange multiplier $\alpha_T$:
\begin{align}
	\text{minimize } f(I_i^t) \text{ s.t. } h(I_i^t) \geq 0
	\Leftrightarrow L(I_i^t, \alpha_t) = f(I_i^t) - \alpha_t h(I_i^t).
\end{align}
The optimization changes to:
\begin{align}
	\min_{I_i^t} f(I_i^t) = \max_{\alpha_t \geq 0} \min_{I_i^t} L(I_i^t, \alpha_t).
\end{align}
Therefore, to minimize $f(I_i^t)$, the dual problem is listed as below. Note that to make sure $\min_{I_i^t} f(I_i^t)$ is properly minimized and would not become $+\infty$, the constraint has to be satisfied.
\begin{align}
	&\min_{I_i^t} \Big[\|a(I_i^t)-a({X}_{i}^{t-1})\|\Big]\nonumber\\
	=& \min_{I_i^t} f(I_i^t) \nonumber\\
	=& \max_{\alpha_t \geq 0}  \min_{I_i^t} L(I_i^t, \alpha_t) \nonumber\\
	=& \max_{\alpha_t \geq 0}  \min_{I_i^t} f(I_i^t) - \alpha_t h(I_i^t)\nonumber \\ 
	=& \max_{\alpha_T \geq 0}  \min_{I_i^t} \Big[\|a(I_i^t)-a({X}_{i}^{t-1})\|\Big] - \alpha_t \Big(\beta-\big\|I_i^t-I_i^b\big\|\Big) \nonumber\\ 
	=& \max_{\alpha_T \geq 0}  \min_{I_i^t} \Big[\|a(I_i^t)-a({X}_i^{t-1})\|- \alpha_T \beta+\alpha_t \big\|I_i^t-I_i^b\big\|\Big].
\end{align}
Define the Q-function for this multi-round multi-agent collaboration as:
\begin{align}
	Q(a(I_i^t),a(X_i^{t-1}))=\|a(I_i^t)-a(X_i^{t-1})\|+\min_{I_i^t}\|a(I_i^{t+1})-a({X}_i^{t})\|+\alpha_{t+1}\|I_i^{t+1,*}-I_i^b\|.  
\end{align}
Here, $I_i^{t+1,*}$ denotes the optimal $I_{i}^{t+1}$. Therefore, the expected return is as follows, when we take one round further back to the round $T-1$:
\begin{align}
	&\min_{I_i^{T-1}} \Big(\|a(I_i^{T-1})-a(X_i^{T-2})\|+\min_{I_i^t}\|a(I_i^t)-a(X_i^{T-1})\|\Big)\nonumber\\
	=&\min_{I_i^{T-1}} \Big( Q(a(I_i^{T-1}),a({X}_i^{T-2})) - \alpha^*_T \|I_{i}^{T,*}-I_i^b\|\Big)\nonumber\\
	=&\max_{\alpha_{T-1}>0}\min_{I_i^{T-1}} \Big( Q(a(I_i^{T-1}),a(X_i^{T-2}))-\alpha_{T-1}\big(\beta-\big\|I_i^{T-1}-I_i^b\big\|\big) - \alpha^*_T \|I_i^{T,*}-I_i^b\|\Big)\nonumber\\
	=&\max_{\alpha_{T-1}>0}\min_{I_i^{T-1}} \Big( Q(a(I_i^{T-1}),a(X_i^{T-2}))-\alpha_{T-1}\beta+\alpha_{T-1}\big\|I_i^{T-1}-I_i^b\big\|\Big) - \alpha^*_T \|I_i^{T,*}-I_i^b\|.
\end{align}

Similar to the previous round,
\begin{align}
	I_i^t&= \arg\min_{I_i^t} \Big[\|a(I_i^t)-a(X_i^{t-1})\|- \alpha_t \beta+\alpha_t \big\|I_i^t-I_i^b\big\|\Big] \nonumber\\
	\alpha^{*}_t&= \arg\max_{\alpha_t \geq 0} \Big[\alpha_T \big\|I_i^t-I_i^b\big\|-\alpha_t \beta\Big].
\end{align}

By repeating this process, we can learn the optimal temperature parameter in every round. Hence, the loss functions for $\theta$ and $\alpha$ are as follows:
\begin{align}
	&L(\theta_i^t)=\big\|a(I_i^t)-a(X_i^{t-1})\big\|+\alpha_t \big\|I_i^t-I_i^b\big\|\nonumber\\
	&L(\alpha)=\alpha\big(\beta-\big\|I_i^t-I_i^b\big\|\big).
\end{align}

\clearpage
\section{Experiment Setup}
\subsection{Datasets}
\label{sec:datasets}
We evaluate our method on three areas with 7 datasets which are widely used in prior studies~\citep{dubey2024llama,qian2025scaling}. We elaborate on what follows.

\begin{itemize}[leftmargin=*,itemsep=0pt,topsep=0pt]
	\item MMLU~\citep{hendrycks2021measuring}, a comprehensive benchmark covering diverse subjects and difficulty levels, designed to test world knowledge and logical reasoning through multiple-choice questions.
	\item MATH~\citep{hendrycks2021measuring}, a dataset of challenging competition-level math problems requiring multi-step symbolic reasoning and advanced problem-solving ability.
	\item GPQA~\citep{rein2023gpqa}, a benchmark of graduate-level multiple-choice science questions that assesses deep domain knowledge and reasoning under uncertainty.
	\item HumanEval~\citep{chen2021evaluating}, a widely recognized benchmark for function-level code generation, designed to evaluate fundamental programming skills.
	\item ALFWorld~\citep{shridhar2021alfworld}, a text-based embodied environment featuring household tasks, where agents navigate and interact with objects via natural language commands.
	\item SciWorld~\citep{wang2022scienceworld}, a text-based embodied environment for interactive science tasks, requiring agents to navigate rooms, conduct experiments, and perform procedural reasoning.
	\item GAIA~\citep{mialon2023gaia}, a benchmark of real-world question answering tasks that integrate knowledge retrieval, reasoning, and multi-step tool use.
	\item PDDL~\citep{chang2024agentboard}, an environment comprising diverse strategic games, where agents must employ PDDL expressions to plan and execute complex tasks.
	\item AndroidWorld~\citep{rawles2025androidworld}, an environment with 116 dynamic tasks across 20 real-world Android apps, designed to evaluate mobile agents’ capabilities in app navigation and system-level control.
\end{itemize}

We use $20\%$ of the questions to construct the training dataset and the rest as testing dataset.
\subsection{Baselines}
We test our method against six baselines. We implement them based on their publicly available implementations. 

\begin{itemize}[leftmargin=*,itemsep=0pt,topsep=0pt]
	\item \textit{Single Execution} (\texttt{Single}), querying a single LLM to solve the task.
	\item \textit{Best-of-N sampling} (\texttt{BoN}), querying a single LLM N times and sampling the most correct answer. We set N to 32, as increasing N further does not yield additional performance benefits and aligns closely with the number of LLM calls used for multi-agent methods.
	\item \textit{Discussion}~\citep{du2023improving}, a multi-agent framework in which LLMs are assigned predefined distinct contexts and iteratively exchange their reasoning processes as additional prompts over multiple rounds, before producing a final answer via majority voting.
	\item \textit{Dynamic LLM-Powered Agent Network}~\citep{liu2024dynamic} (\texttt{DyLAN}),  a discussion-style framework which incorporates an LLM selection algorithm based on an unsupervised metric, namely the Agent Importance Score, which identifies the most contributive LLMs through a preliminary trial tailored to the specific task.
	\item \textit{GPTSwarm}~\citep{zhuge2024gptswarm}, formalizing a swarm of autonomous agents as computational graphs, with nodes as manually-customized functions and edges facilitating information flow, adaptively optimizing node prompts and modifying graph connectivity during collective reasoning.
	\item \textit{Multi-LLM Collaboration Network}~\citep{qian2025scaling} (\texttt{MacNet}), a representative framework for decentralized and scalable multi-LLM systems. It introduces edge agents that mediate interactions by generating actionable instructions for the next agent based on the outputs of the previous one.
\end{itemize}
\subsection{Implementation Details}
The context pool is constructed using GPT-4o, where we prompt it to generate a large collection of high-quality initial contexts across diverse domains, including mathematics, science, coding, and embodied reasoning, and various domain-specific sub-contexts are included in each domain. This ensures that the pool provides a broad coverage of reasoning perspectives and is shared across all tasks.

The context generators are implemented by T5-small model~\citep{raffel2020exploring} for all the tasks. The dimension of the generated context vectors is 512. The learning rates for the context generators and $\alpha$ are both $1e-4$. We use $20\%$ of the questions to train the context initialization and generator.
\begin{table}[h]
	\centering
	\caption{Hyperparameters (identical across datasets).}
	\vskip 0.1in
	\begin{tabular}{l|l} 
		Hyperparameter & Value \\ 
		\hline
		Sentence vector dimension & 512\\
		Optimizer & \texttt{Adam} \\
		Batchsize & 32 \\
		Learning rate of context & 1e-4 \\
		Learning rate of $\alpha$ & 1e-4 \\
		Maximum rounds for discussion & 8\\
		Training rounds & 100\\
        Size of context pool & 100\\
		\hline
	\end{tabular}
	\label{tab:parameter_setting}
\end{table}

We implement our code using Pytorch 2.3.0, built upon the open-source parameters of the llama2, Qwen2.5, and Qwen2.5-VL. Models provided at \url{https://huggingface.co/meta-llama} and \url{https://huggingface.co/Qwen}. All the experiments are run on Ubuntu 22.04.4 LTS with 8 NVIDIA H800 GPUs.
\subsection{Pseudocode of \alg{}}
\label{sec:algorithm}
We present the pseudocode of training \alg{} in \cref{alg:context_generation}. It begins with training the context initialization, then the context generators are trained along with the weight $\alpha$ during the discussion.
\begin{algorithm}[h]
	\caption{Pseudocode of \alg{}}
	\label{alg:context_generation}
    Initialize parameters $\phi_f$, $\phi_F$, $\{\theta_i\}_{i=1}^N$, and $\{\alpha_i\}_{i=1}^N$\;
    \For{each question}{
    $\phi_f\leftarrow\phi_f-\eta_f\nabla L(\phi_f)$,\quad    $\phi_F\leftarrow\phi_F-\eta_F\nabla L(\phi_F)$\;
    }
    \For{each epoch}{
    Obtain initial contexts $\{I_i^b\}_{i=1}^N$ via \cref{eq:context_init_3}\;
    \For{each round $t$}{
    \For{$i=1$ {\bfseries to} $N$}{
    Generate instructions $I_i^t$ via \cref{eq:context_generation} and obtain responses ${X}_i^t$ via \cref{eq:response_generation}\;
    $\theta_i\leftarrow\theta_i-\eta_\theta\nabla L(\theta_i)$,\quad $\alpha_i\leftarrow\alpha_i-\eta_\alpha\nabla L(\alpha_i)$\;
    }
    }
    }
\end{algorithm}
\clearpage
\section{Additional Results}
\subsection{Comparison with More Models}
\label{sec:comparative_result}
To further investigate \alg{}'s performance, we compare it with \texttt{best-of-N} using stronger close-sourced LLMs, including Qwen2.5-Max and GPT-4. As shown in \cref{tab:more_model}, llama series models perform worse when fewer LLMs participate but achieve higher accuracy as the number of LLMs increases, demonstrating that \alg{} can collaborate weaker models to achieve comparable performance.
\begin{table*}[h]
	\centering
	\caption{Accuracy using different base models on different datasets. We exhibit the performance advantage with \texttt{Qwen-max} and highlight the \colorbox{blue!10}{best} result.}
	\resizebox{\textwidth}{!}{
		\begin{tabular}{XZYYYYYYYY}
			\toprule
			& \textbf{Model} & \textbf{MMLU} & \textbf{MATH} & \textbf{GPQA} & \textbf{Code} & \textbf{ALFWorld} & \textbf{SciWorld} & \textbf{GAIA} & \textbf{PDDL} \\
			\midrule
			\multirow{5}{*}{n=1}
			& \texttt{Qwen-max} 
			& \cellcolor{blue!10}86.5{\tiny\textcolor{red}{$\uparrow$0.0}} & \cellcolor{blue!10}43.6{\tiny\textcolor{red}{$\uparrow$0.0}} & \cellcolor{blue!10}35.7{\tiny\textcolor{red}{$\uparrow$0.0}} & \cellcolor{blue!10}65.9{\tiny\textcolor{red}{$\uparrow$0.0}} & \cellcolor{blue!10}73.1{\tiny\textcolor{red}{$\uparrow$0.0}} & \cellcolor{blue!10}68.7{\tiny\textcolor{red}{$\uparrow$0.0}} & \cellcolor{blue!10}61.5{\tiny\textcolor{red}{$\uparrow$0.0}} & \cellcolor{blue!10}76.4{\tiny\textcolor{red}{$\uparrow$0.0}} \\
			& \texttt{GPT-4} & 84.3{\tiny\textcolor{green}{$\downarrow$2.2}} & 41.1{\tiny\textcolor{green}{$\downarrow$2.5}} & 33.3{\tiny\textcolor{green}{$\downarrow$2.4}} & 62.7{\tiny\textcolor{green}{$\downarrow$3.2}} & 70.2{\tiny\textcolor{green}{$\downarrow$2.9}} & 65.4{\tiny\textcolor{green}{$\downarrow$3.3}} & 58.7{\tiny\textcolor{green}{$\downarrow$2.8}} & 73.5{\tiny\textcolor{green}{$\downarrow$2.9}} \\
			& \texttt{Llama-7B} & 45.3{\tiny\textcolor{green}{$\downarrow$41.2}} & 2.9{\tiny\textcolor{green}{$\downarrow$40.7}} & 16.8{\tiny\textcolor{green}{$\downarrow$18.9}} & 15.8{\tiny\textcolor{green}{$\downarrow$50.1}} & 22.5{\tiny\textcolor{green}{$\downarrow$50.6}} & 20.1{\tiny\textcolor{green}{$\downarrow$48.6}} & 15.3{\tiny\textcolor{green}{$\downarrow$46.2}} & 24.2{\tiny\textcolor{green}{$\downarrow$52.2}} \\
			& \texttt{Llama-13B} & 54.8{\tiny\textcolor{green}{$\downarrow$31.7}} & 3.9{\tiny\textcolor{green}{$\downarrow$39.7}} & 25.0{\tiny\textcolor{green}{$\downarrow$10.7}} & 23.7{\tiny\textcolor{green}{$\downarrow$42.2}} & 28.5{\tiny\textcolor{green}{$\downarrow$44.6}} & 24.5{\tiny\textcolor{green}{$\downarrow$44.2}} & 18.6{\tiny\textcolor{green}{$\downarrow$42.9}} & 30.4{\tiny\textcolor{green}{$\downarrow$46.0}} \\
			& \texttt{Llama-70B} & 68.9{\tiny\textcolor{green}{$\downarrow$17.6}} & 31.6{\tiny\textcolor{green}{$\downarrow$12.0}} & 27.6{\tiny\textcolor{green}{$\downarrow$8.1}} & 35.5{\tiny\textcolor{green}{$\downarrow$30.4}} & 46.1{\tiny\textcolor{green}{$\downarrow$27.0}} & 40.2{\tiny\textcolor{green}{$\downarrow$28.5}} & 29.9{\tiny\textcolor{green}{$\downarrow$31.6}} & 48.4{\tiny\textcolor{green}{$\downarrow$28.0}} \\
			\midrule
			\multirow{5}{*}{n=4}
			& \texttt{Qwen-max} 
			& 87.5{\tiny\textcolor{red}{$\uparrow$0.0}} & \cellcolor{blue!10}44.2{\tiny\textcolor{red}{$\uparrow$0.0}} & 39.2{\tiny\textcolor{red}{$\uparrow$0.0}} & 67.5{\tiny\textcolor{red}{$\uparrow$0.0}} & \cellcolor{blue!10}74.0{\tiny\textcolor{red}{$\uparrow$0.0}} & \cellcolor{blue!10}69.8{\tiny\textcolor{red}{$\uparrow$0.0}} & \cellcolor{blue!10}62.5{\tiny\textcolor{red}{$\uparrow$0.0}} & \cellcolor{blue!10}77.5{\tiny\textcolor{red}{$\uparrow$0.0}} \\
			& \texttt{GPT-4} & 86.4{\tiny\textcolor{green}{$\downarrow$1.1}} & 43.9{\tiny\textcolor{green}{$\downarrow$0.3}} & 38.8{\tiny\textcolor{green}{$\downarrow$0.4}} & 67.0{\tiny\textcolor{green}{$\downarrow$0.5}} & 72.8{\tiny\textcolor{green}{$\downarrow$1.2}} & 68.0{\tiny\textcolor{green}{$\downarrow$1.8}} & 61.0{\tiny\textcolor{green}{$\downarrow$1.5}} & 76.0{\tiny\textcolor{green}{$\downarrow$1.5}} \\
			& \texttt{Llama-7B} & 59.1{\tiny\textcolor{green}{$\downarrow$28.4}} & 17.4{\tiny\textcolor{green}{$\downarrow$26.8}} & 44.7{\tiny\textcolor{red}{$\uparrow$5.5}} & 32.9{\tiny\textcolor{green}{$\downarrow$34.6}} & 35.1{\tiny\textcolor{green}{$\downarrow$38.9}} & 33.0{\tiny\textcolor{green}{$\downarrow$36.8}} & 25.5{\tiny\textcolor{green}{$\downarrow$37.0}} & 33.1{\tiny\textcolor{green}{$\downarrow$44.4}} \\
			& \texttt{Llama-13B} & 86.9{\tiny\textcolor{green}{$\downarrow$0.6}} & 23.4{\tiny\textcolor{green}{$\downarrow$20.8}} & 58.9{\tiny\textcolor{red}{$\uparrow$19.7}} & 47.5{\tiny\textcolor{green}{$\downarrow$20.0}} & 42.9{\tiny\textcolor{green}{$\downarrow$31.1}} & 39.9{\tiny\textcolor{green}{$\downarrow$29.9}} & 30.5{\tiny\textcolor{green}{$\downarrow$32.0}} & 40.9{\tiny\textcolor{green}{$\downarrow$36.6}} \\
			& \texttt{Llama-70B} & \cellcolor{blue!10}95.6{\tiny\textcolor{red}{$\uparrow$8.1}} & 40.3{\tiny\textcolor{green}{$\downarrow$3.9}} & \cellcolor{blue!10}78.7{\tiny\textcolor{red}{$\uparrow$39.5}} & \cellcolor{blue!10}70.6{\tiny\textcolor{red}{$\uparrow$3.1}} & 69.5{\tiny\textcolor{green}{$\downarrow$4.5}} & 65.5{\tiny\textcolor{green}{$\downarrow$4.3}} & 49.6{\tiny\textcolor{green}{$\downarrow$13.0}} & 66.5{\tiny\textcolor{green}{$\downarrow$11.0}} \\
			\midrule
			\multirow{5}{*}{n=8}
			& \texttt{Qwen-max} 
			& 88.5{\tiny\textcolor{red}{$\uparrow$0.0}} & \cellcolor{blue!10}45.0{\tiny\textcolor{red}{$\uparrow$0.0}} & 40.5{\tiny\textcolor{red}{$\uparrow$0.0}} & 68.8{\tiny\textcolor{red}{$\uparrow$0.0}} & 75.0{\tiny\textcolor{red}{$\uparrow$0.0}} & 71.0{\tiny\textcolor{red}{$\uparrow$0.0}} & \cellcolor{blue!10}63.5{\tiny\textcolor{red}{$\uparrow$0.0}} & \cellcolor{blue!10}79.0{\tiny\textcolor{red}{$\uparrow$0.0}} \\
			& \texttt{GPT-4} & 86.5{\tiny\textcolor{green}{$\downarrow$2.0}} & 44.6{\tiny\textcolor{green}{$\downarrow$0.4}} & 38.9{\tiny\textcolor{green}{$\downarrow$1.6}} & 67.1{\tiny\textcolor{green}{$\downarrow$1.7}} & 73.2{\tiny\textcolor{green}{$\downarrow$1.8}} & 69.5{\tiny\textcolor{green}{$\downarrow$1.5}} & 61.0{\tiny\textcolor{green}{$\downarrow$2.5}} & 76.0{\tiny\textcolor{green}{$\downarrow$3.0}} \\
			& \texttt{Llama-7B} & 63.8{\tiny\textcolor{green}{$\downarrow$24.7}} & 24.9{\tiny\textcolor{green}{$\downarrow$20.1}} & 59.3{\tiny\textcolor{red}{$\uparrow$18.8}} & 38.8{\tiny\textcolor{green}{$\downarrow$30.0}} & 37.2{\tiny\textcolor{green}{$\downarrow$37.8}} & 36.1{\tiny\textcolor{green}{$\downarrow$34.9}} & 27.1{\tiny\textcolor{green}{$\downarrow$36.4}} & 34.7{\tiny\textcolor{green}{$\downarrow$44.3}} \\
			& \texttt{Llama-13B} & 92.7{\tiny\textcolor{red}{$\uparrow$4.2}} & 23.0{\tiny\textcolor{green}{$\downarrow$22.0}} & 70.7{\tiny\textcolor{red}{$\uparrow$30.2}} & 55.4{\tiny\textcolor{green}{$\downarrow$13.4}} & 46.1{\tiny\textcolor{green}{$\downarrow$28.9}} & 44.7{\tiny\textcolor{green}{$\downarrow$26.3}} & 33.6{\tiny\textcolor{green}{$\downarrow$29.9}} & 43.8{\tiny\textcolor{green}{$\downarrow$35.2}} \\
			& \texttt{Llama-70B} & \cellcolor{blue!10}95.4{\tiny\textcolor{red}{$\uparrow$6.9}} & 43.5{\tiny\textcolor{green}{$\downarrow$1.5}} & \cellcolor{blue!10}87.6{\tiny\textcolor{red}{$\uparrow$47.1}} & \cellcolor{blue!10}81.0{\tiny\textcolor{red}{$\uparrow$12.2}} & \cellcolor{blue!10}78.9{\tiny\textcolor{red}{$\uparrow$3.9}} & \cellcolor{blue!10}76.2{\tiny\textcolor{red}{$\uparrow$5.2}} & 58.7{\tiny\textcolor{green}{$\downarrow$4.8}} & 73.6{\tiny\textcolor{green}{$\downarrow$5.4}} \\
			\midrule
			\multirow{5}{*}{n=16}
			& \texttt{Qwen-max} 
			& 88.5{\tiny\textcolor{red}{$\uparrow$0.0}} & 45.0{\tiny\textcolor{red}{$\uparrow$0.0}} & 40.5{\tiny\textcolor{red}{$\uparrow$0.0}} & 68.8{\tiny\textcolor{red}{$\uparrow$0.0}} & 75.0{\tiny\textcolor{red}{$\uparrow$0.0}} & 71.0{\tiny\textcolor{red}{$\uparrow$0.0}} & \cellcolor{blue!10}63.5{\tiny\textcolor{red}{$\uparrow$0.0}} & \cellcolor{blue!10}79.0{\tiny\textcolor{red}{$\uparrow$0.0}} \\
			& \texttt{GPT-4} & 87.9{\tiny\textcolor{green}{$\downarrow$0.6}} & 44.8{\tiny\textcolor{green}{$\downarrow$0.2}} & 40.1{\tiny\textcolor{green}{$\downarrow$0.4}} & 68.5{\tiny\textcolor{green}{$\downarrow$0.3}} & 74.3{\tiny\textcolor{green}{$\downarrow$0.7}} & 70.5{\tiny\textcolor{green}{$\downarrow$0.5}} & 63.0{\tiny\textcolor{green}{$\downarrow$0.5}} & 78.2{\tiny\textcolor{green}{$\downarrow$0.8}} \\
			& \texttt{Llama-7B} & 71.5{\tiny\textcolor{green}{$\downarrow$17.0}} & 23.9{\tiny\textcolor{green}{$\downarrow$21.1}} & 69.9{\tiny\textcolor{red}{$\uparrow$29.4}} & 48.3{\tiny\textcolor{green}{$\downarrow$20.5}} & 39.8{\tiny\textcolor{green}{$\downarrow$35.2}} & 38.4{\tiny\textcolor{green}{$\downarrow$32.6}} & 28.9{\tiny\textcolor{green}{$\downarrow$34.6}} & 37.0{\tiny\textcolor{green}{$\downarrow$42.0}} \\
			& \texttt{Llama-13B} & 94.5{\tiny\textcolor{red}{$\uparrow$6.0}} & 32.7{\tiny\textcolor{green}{$\downarrow$12.3}} & 88.5{\tiny\textcolor{red}{$\uparrow$48.0}} & 66.1{\tiny\textcolor{green}{$\downarrow$2.7}} & 48.7{\tiny\textcolor{green}{$\downarrow$26.3}} & 47.1{\tiny\textcolor{green}{$\downarrow$23.9}} & 36.2{\tiny\textcolor{green}{$\downarrow$27.3}} & 45.5{\tiny\textcolor{green}{$\downarrow$33.5}} \\
			& \texttt{Llama-70B} & \cellcolor{blue!10}93.9{\tiny\textcolor{red}{$\uparrow$5.4}} & \cellcolor{blue!10}51.5{\tiny\textcolor{red}{$\uparrow$6.5}} & \cellcolor{blue!10}91.3{\tiny\textcolor{red}{$\uparrow$50.8}} & \cellcolor{blue!10}97.2{\tiny\textcolor{red}{$\uparrow$28.4}} & \cellcolor{blue!10}84.9{\tiny\textcolor{red}{$\uparrow$9.9}} & \cellcolor{blue!10}81.7{\tiny\textcolor{red}{$\uparrow$10.7}} & 61.5{\tiny\textcolor{green}{$\downarrow$2.0}} & 78.1{\tiny\textcolor{green}{$\downarrow$0.9}} \\
			\midrule
			\multirow{5}{*}{n=32}
			& \texttt{Qwen-max} & 89.0{\tiny\textcolor{red}{$\uparrow$0.0}} & 45.5{\tiny\textcolor{red}{$\uparrow$0.0}} & 41.0{\tiny\textcolor{red}{$\uparrow$0.0}} & 69.5{\tiny\textcolor{red}{$\uparrow$0.0}} & 75.5{\tiny\textcolor{red}{$\uparrow$0.0}} & 71.5{\tiny\textcolor{red}{$\uparrow$0.0}} & 64.0{\tiny\textcolor{red}{$\uparrow$0.0}} & 79.5{\tiny\textcolor{red}{$\uparrow$0.0}} \\
			& \texttt{GPT-4} & 89.1{\tiny\textcolor{red}{$\uparrow$0.1}} & 45.4{\tiny\textcolor{green}{$\downarrow$0.1}} & 41.5{\tiny\textcolor{red}{$\uparrow$0.5}} & 69.7{\tiny\textcolor{red}{$\uparrow$0.2}} & 75.0{\tiny\textcolor{green}{$\downarrow$0.5}} & 71.0{\tiny\textcolor{green}{$\downarrow$0.5}} & 63.5{\tiny\textcolor{green}{$\downarrow$0.5}} & 78.5{\tiny\textcolor{green}{$\downarrow$1.0}} \\
			& \texttt{Llama-7B} & 79.1{\tiny\textcolor{green}{$\downarrow$9.9}} & 28.8{\tiny\textcolor{green}{$\downarrow$16.7}} & 84.5{\tiny\textcolor{red}{$\uparrow$43.5}} & 56.5{\tiny\textcolor{green}{$\downarrow$13.0}} & 42.5{\tiny\textcolor{green}{$\downarrow$33.0}} & 40.7{\tiny\textcolor{green}{$\downarrow$30.8}} & 31.4{\tiny\textcolor{green}{$\downarrow$32.6}} & 38.9{\tiny\textcolor{green}{$\downarrow$40.6}} \\
			& \texttt{Llama-13B} & 95.8{\tiny\textcolor{red}{$\uparrow$6.8}} & 41.4{\tiny\textcolor{green}{$\downarrow$4.1}} & 94.7{\tiny\textcolor{red}{$\uparrow$53.7}} & 75.6{\tiny\textcolor{red}{$\uparrow$6.1}} & 52.2{\tiny\textcolor{green}{$\downarrow$23.3}} & 50.7{\tiny\textcolor{green}{$\downarrow$20.8}} & 38.5{\tiny\textcolor{green}{$\downarrow$25.5}} & 48.0{\tiny\textcolor{green}{$\downarrow$31.5}} \\
			& \texttt{Llama-70B} & \cellcolor{blue!10}97.0{\tiny\textcolor{red}{$\uparrow$8.0}} & \cellcolor{blue!10}54.5{\tiny\textcolor{red}{$\uparrow$9.0}} & \cellcolor{blue!10}95.1{\tiny\textcolor{red}{$\uparrow$54.1}} & \cellcolor{blue!10}93.7{\tiny\textcolor{red}{$\uparrow$24.2}} & \cellcolor{blue!10}88.5{\tiny\textcolor{red}{$\uparrow$13.0}} & \cellcolor{blue!10}86.0{\tiny\textcolor{red}{$\uparrow$14.5}} & \cellcolor{blue!10}65.7{\tiny\textcolor{red}{$\uparrow$1.7}} & \cellcolor{blue!10}79.8{\tiny\textcolor{red}{$\uparrow$0.3}} \\
			\midrule
			\multirow{5}{*}{n=64}
			& \texttt{Qwen-max} & 89.5{\tiny\textcolor{red}{$\uparrow$0.0}} & 46.0{\tiny\textcolor{red}{$\uparrow$0.0}} & 41.5{\tiny\textcolor{red}{$\uparrow$0.0}} & 70.0{\tiny\textcolor{red}{$\uparrow$0.0}} & 76.0{\tiny\textcolor{red}{$\uparrow$0.0}} & 72.0{\tiny\textcolor{red}{$\uparrow$0.0}} & 64.5{\tiny\textcolor{red}{$\uparrow$0.0}} & 80.0{\tiny\textcolor{red}{$\uparrow$0.0}} \\
			& \texttt{GPT-4} & 88.5{\tiny\textcolor{green}{$\downarrow$1.0}} & 45.4{\tiny\textcolor{green}{$\downarrow$0.6}} & 41.8{\tiny\textcolor{red}{$\uparrow$0.3}} & 70.5{\tiny\textcolor{red}{$\uparrow$0.5}} & 75.2{\tiny\textcolor{green}{$\downarrow$0.8}} & 71.5{\tiny\textcolor{green}{$\downarrow$0.5}} & 64.0{\tiny\textcolor{green}{$\downarrow$0.5}} & 79.2{\tiny\textcolor{green}{$\downarrow$0.8}} \\
			& \texttt{Llama-7B} & 81.5{\tiny\textcolor{green}{$\downarrow$8.0}} & 35.4{\tiny\textcolor{green}{$\downarrow$10.6}} & 82.5{\tiny\textcolor{red}{$\uparrow$41.0}} & 60.6{\tiny\textcolor{green}{$\downarrow$9.4}} & 44.2{\tiny\textcolor{green}{$\downarrow$31.8}} & 42.9{\tiny\textcolor{green}{$\downarrow$29.1}} & 32.9{\tiny\textcolor{green}{$\downarrow$31.6}} & 40.9{\tiny\textcolor{green}{$\downarrow$39.1}} \\
			& \texttt{Llama-13B} & \cellcolor{blue!10}96.8{\tiny\textcolor{red}{$\uparrow$7.3}} & 40.0{\tiny\textcolor{green}{$\downarrow$6.0}} & 93.1{\tiny\textcolor{red}{$\uparrow$51.6}} & 82.5{\tiny\textcolor{red}{$\uparrow$12.5}} & 54.8{\tiny\textcolor{green}{$\downarrow$21.2}} & 52.8{\tiny\textcolor{green}{$\downarrow$19.2}} & 40.3{\tiny\textcolor{green}{$\downarrow$24.2}} & 49.8{\tiny\textcolor{green}{$\downarrow$30.2}} \\
			& \texttt{Llama-70B} & 96.3{\tiny\textcolor{red}{$\uparrow$6.8}} & \cellcolor{blue!10}58.0{\tiny\textcolor{red}{$\uparrow$12.0}} & \cellcolor{blue!10}96.9{\tiny\textcolor{red}{$\uparrow$55.4}} & \cellcolor{blue!10}92.0{\tiny\textcolor{red}{$\uparrow$22.0}} & \cellcolor{blue!10}90.8{\tiny\textcolor{red}{$\uparrow$14.8}} & \cellcolor{blue!10}88.9{\tiny\textcolor{red}{$\uparrow$16.9}} & \cellcolor{blue!10}68.2{\tiny\textcolor{red}{$\uparrow$3.7}} & \cellcolor{blue!10}82.2{\tiny\textcolor{red}{$\uparrow$2.2}} \\
			\bottomrule
		\end{tabular}
	}
	\label{tab:more_model}
\end{table*}

\clearpage

\subsection{Number of LLMs}
We evaluate \alg{}'s performance with varying number of LLMs, ranging from 4 to 64. The comparative results are shown in \cref{tab:number_4,tab:number_8,tab:number_16,tab:number_32,tab:number_64,fig:number_agents_1,fig:number_agents_2,fig:number_agents_3,fig:number_agents_4,fig:number_agents_5}.

\textbf{Summary of key findings.} The results show that \alg{} significantly improves MAD performance, consistently surpassing existing methods by $20\%-50\%$, particularly in complex tasks like math and tool-using. This highlights its ability to tackle intricate reasoning tasks. \alg{} also exhibits a more effective multi-agent scaling law, where performance consistently improves as the number of LLMs increases to 64, especially in agentic domains where a larger amount of LLMs enhances problem-solving accuracy and efficiency.

\label{sec:number_agents}
\begin{table*}[h]
	\centering
	\caption{Accuracy with varying number of LLMs on different datasets. The number of LLMs is $4$ for all datasets. We exhibit the performance advantage with \texttt{BoN} and highlight the \colorbox{blue!10}{best} result.}
	\vskip 0.1in
	\resizebox{\textwidth}{!}{
		\begin{tabular}{YZYYYYYYYY}
			\toprule
			\textbf{Model} & \textbf{Method} & \textbf{MMLU} & \textbf{MATH} & \textbf{GPQA} & \textbf{Code} & \textbf{ALFWorld} & \textbf{SciWorld} & \textbf{GAIA} & \textbf{PDDL} \\
			\midrule
			\multirow{7}{*}{Llama-7B}
			& \texttt{Single} & 45.3{\tiny\textcolor{green}{$\downarrow$5.2}} & 2.9{\tiny\textcolor{green}{$\downarrow$4.0}} & 16.8{\tiny\textcolor{green}{$\downarrow$10.7}} & 15.8{\tiny\textcolor{green}{$\downarrow$6.0}} & 22.5{\tiny\textcolor{green}{$\downarrow$2.9}} & 20.1{\tiny\textcolor{green}{$\downarrow$4.8}} & 15.3{\tiny\textcolor{green}{$\downarrow$2.9}} & 24.2{\tiny\textcolor{green}{$\downarrow$2.0}} \\
			& \texttt{BoN} & 50.5{\tiny\textcolor{red}{$\uparrow$0.0}} & 6.9{\tiny\textcolor{red}{$\uparrow$0.0}} & 27.5{\tiny\textcolor{red}{$\uparrow$0.0}} & 21.8{\tiny\textcolor{red}{$\uparrow$0.0}} & 25.4{\tiny\textcolor{red}{$\uparrow$0.0}} & 24.9{\tiny\textcolor{red}{$\uparrow$0.0}} & 18.2{\tiny\textcolor{red}{$\uparrow$0.0}} & 26.2{\tiny\textcolor{red}{$\uparrow$0.0}} \\
			& \texttt{Debate} & 47.3{\tiny\textcolor{green}{$\downarrow$3.2}} & 6.4{\tiny\textcolor{green}{$\downarrow$0.5}} & 26.2{\tiny\textcolor{green}{$\downarrow$1.3}} & 19.8{\tiny\textcolor{green}{$\downarrow$2.0}} & 25.3{\tiny\textcolor{green}{$\downarrow$0.1}} & 23.8{\tiny\textcolor{green}{$\downarrow$1.1}} & 16.5{\tiny\textcolor{green}{$\downarrow$1.7}} & 26.2{\tiny\textcolor{red}{$\uparrow$0.0}} \\
			& \texttt{DyLAN} & 48.1{\tiny\textcolor{green}{$\downarrow$2.4}} & 6.6{\tiny\textcolor{green}{$\downarrow$0.3}} & 34.5{\tiny\textcolor{red}{$\uparrow$7.0}} & 22.2{\tiny\textcolor{red}{$\uparrow$0.4}} & 24.7{\tiny\textcolor{green}{$\downarrow$0.7}} & 21.8{\tiny\textcolor{green}{$\downarrow$3.1}} & 16.8{\tiny\textcolor{green}{$\downarrow$1.4}} & 25.5{\tiny\textcolor{green}{$\downarrow$0.7}} \\
			& \texttt{GPTSwarm} & 48.0{\tiny\textcolor{green}{$\downarrow$2.5}} & 5.8{\tiny\textcolor{green}{$\downarrow$1.1}} & 26.0{\tiny\textcolor{green}{$\downarrow$1.5}} & 18.9{\tiny\textcolor{green}{$\downarrow$2.9}} & 25.8{\tiny\textcolor{red}{$\uparrow$0.4}} & 25.1{\tiny\textcolor{red}{$\uparrow$0.2}} & 18.8{\tiny\textcolor{red}{$\uparrow$0.6}} & 27.9{\tiny\textcolor{red}{$\uparrow$1.7}} \\
			& \texttt{MacNet} & 50.8{\tiny\textcolor{red}{$\uparrow$0.3}} & 7.9{\tiny\textcolor{red}{$\uparrow$1.0}} & 30.7{\tiny\textcolor{red}{$\uparrow$3.2}} & 24.0{\tiny\textcolor{red}{$\uparrow$2.2}} & 27.8{\tiny\textcolor{red}{$\uparrow$2.4}} & 26.5{\tiny\textcolor{red}{$\uparrow$1.6}} & 18.8{\tiny\textcolor{red}{$\uparrow$0.6}} & 29.3{\tiny\textcolor{red}{$\uparrow$3.1}} \\
			& {\alg{} (ours)} & \cellcolor{blue!10}59.1{\tiny\textcolor{red}{$\uparrow$8.6}} & \cellcolor{blue!10}17.4{\tiny\textcolor{red}{$\uparrow$10.5}} & \cellcolor{blue!10}44.7{\tiny\textcolor{red}{$\uparrow$17.2}} & \cellcolor{blue!10}32.9{\tiny\textcolor{red}{$\uparrow$11.1}} & \cellcolor{blue!10}35.1{\tiny\textcolor{red}{$\uparrow$9.7}} & \cellcolor{blue!10}33.0{\tiny\textcolor{red}{$\uparrow$8.1}} & \cellcolor{blue!10}25.5{\tiny\textcolor{red}{$\uparrow$7.3}} & \cellcolor{blue!10}33.1{\tiny\textcolor{red}{$\uparrow$6.9}} \\
			\midrule
			\multirow{7}{*}{Llama-14B}
			& \texttt{Single} & 54.8{\tiny\textcolor{green}{$\downarrow$16.2}} & 3.9{\tiny\textcolor{green}{$\downarrow$4.9}} & 25.0{\tiny\textcolor{green}{$\downarrow$12.7}} & 23.7{\tiny\textcolor{green}{$\downarrow$8.3}} & 28.5{\tiny\textcolor{green}{$\downarrow$2.7}} & 24.5{\tiny\textcolor{green}{$\downarrow$7.2}} & 18.6{\tiny\textcolor{green}{$\downarrow$3.2}} & 30.4{\tiny\textcolor{green}{$\downarrow$2.8}} \\
			& \texttt{BoN} & 71.0{\tiny\textcolor{red}{$\uparrow$0.0}} & 8.8{\tiny\textcolor{red}{$\uparrow$0.0}} & 37.7{\tiny\textcolor{red}{$\uparrow$0.0}} & 32.0{\tiny\textcolor{red}{$\uparrow$0.0}} & 31.2{\tiny\textcolor{red}{$\uparrow$0.0}} & 31.7{\tiny\textcolor{red}{$\uparrow$0.0}} & 21.8{\tiny\textcolor{red}{$\uparrow$0.0}} & 33.2{\tiny\textcolor{red}{$\uparrow$0.0}} \\
			& \texttt{Debate} & 64.5{\tiny\textcolor{green}{$\downarrow$6.5}} & 8.4{\tiny\textcolor{green}{$\downarrow$0.4}} & 36.3{\tiny\textcolor{green}{$\downarrow$1.4}} & 27.9{\tiny\textcolor{green}{$\downarrow$4.1}} & 30.8{\tiny\textcolor{green}{$\downarrow$0.4}} & 29.3{\tiny\textcolor{green}{$\downarrow$2.4}} & 19.8{\tiny\textcolor{green}{$\downarrow$2.0}} & 31.9{\tiny\textcolor{green}{$\downarrow$1.3}} \\
			& \texttt{DyLAN} & 69.7{\tiny\textcolor{green}{$\downarrow$1.3}} & 10.3{\tiny\textcolor{red}{$\uparrow$1.5}} & 46.5{\tiny\textcolor{red}{$\uparrow$8.8}} & 32.5{\tiny\textcolor{red}{$\uparrow$0.5}} & 31.3{\tiny\textcolor{red}{$\uparrow$0.1}} & 28.7{\tiny\textcolor{green}{$\downarrow$3.0}} & 19.4{\tiny\textcolor{green}{$\downarrow$2.4}} & 30.4{\tiny\textcolor{green}{$\downarrow$2.8}} \\
			& \texttt{MacNet} & 70.3{\tiny\textcolor{green}{$\downarrow$0.7}} & 9.2{\tiny\textcolor{red}{$\uparrow$0.4}} & 40.9{\tiny\textcolor{red}{$\uparrow$3.2}} & 35.2{\tiny\textcolor{red}{$\uparrow$3.2}} & 31.5{\tiny\textcolor{red}{$\uparrow$0.3}} & 31.8{\tiny\textcolor{red}{$\uparrow$0.1}} & 23.4{\tiny\textcolor{red}{$\uparrow$1.6}} & 35.3{\tiny\textcolor{red}{$\uparrow$2.1}} \\
			& \texttt{MacNet} & 70.3{\tiny\textcolor{green}{$\downarrow$0.7}} & 9.2{\tiny\textcolor{red}{$\uparrow$0.4}} & 40.9{\tiny\textcolor{red}{$\uparrow$3.2}} & 35.2{\tiny\textcolor{red}{$\uparrow$3.2}} & 31.5{\tiny\textcolor{red}{$\uparrow$0.3}} & 31.8{\tiny\textcolor{red}{$\uparrow$0.1}} & 23.4{\tiny\textcolor{red}{$\uparrow$1.6}} & 35.3{\tiny\textcolor{red}{$\uparrow$2.1}} \\
			& {\alg{} (ours)} & \cellcolor{blue!10}86.9{\tiny\textcolor{red}{$\uparrow$15.9}} & \cellcolor{blue!10}23.4{\tiny\textcolor{red}{$\uparrow$14.6}} & \cellcolor{blue!10}58.9{\tiny\textcolor{red}{$\uparrow$21.2}} & \cellcolor{blue!10}47.5{\tiny\textcolor{red}{$\uparrow$15.5}} & \cellcolor{blue!10}42.9{\tiny\textcolor{red}{$\uparrow$11.7}} & \cellcolor{blue!10}39.9{\tiny\textcolor{red}{$\uparrow$8.2}} & \cellcolor{blue!10}30.5{\tiny\textcolor{red}{$\uparrow$8.7}} & \cellcolor{blue!10}40.9{\tiny\textcolor{red}{$\uparrow$7.7}} \\
			\midrule
			\multirow{7}{*}{Llama-70B}
			& \texttt{Single} & 68.9{\tiny\textcolor{green}{$\downarrow$14.1}} & 31.6{\tiny\textcolor{green}{$\downarrow$3.2}} & 27.6{\tiny\textcolor{green}{$\downarrow$32.0}} & 35.5{\tiny\textcolor{green}{$\downarrow$12.5}} & 46.1{\tiny\textcolor{green}{$\downarrow$4.8}} & 40.2{\tiny\textcolor{green}{$\downarrow$6.5}} & 29.9{\tiny\textcolor{green}{$\downarrow$2.8}} & 48.4{\tiny\textcolor{green}{$\downarrow$4.4}} \\
			& \texttt{BoN} & 83.0{\tiny\textcolor{red}{$\uparrow$0.0}} & 34.8{\tiny\textcolor{red}{$\uparrow$0.0}} & 59.6{\tiny\textcolor{red}{$\uparrow$0.0}} & 48.0{\tiny\textcolor{red}{$\uparrow$0.0}} & 50.9{\tiny\textcolor{red}{$\uparrow$0.0}} & 46.7{\tiny\textcolor{red}{$\uparrow$0.0}} & 32.7{\tiny\textcolor{red}{$\uparrow$0.0}} & 52.8{\tiny\textcolor{red}{$\uparrow$0.0}} \\
			& \texttt{Debate} & 74.8{\tiny\textcolor{green}{$\downarrow$8.2}} & 29.8{\tiny\textcolor{green}{$\downarrow$5.0}} & 50.2{\tiny\textcolor{green}{$\downarrow$9.4}} & 43.2{\tiny\textcolor{green}{$\downarrow$4.8}} & 50.4{\tiny\textcolor{green}{$\downarrow$0.5}} & 47.8{\tiny\textcolor{red}{$\uparrow$1.1}} & 32.4{\tiny\textcolor{green}{$\downarrow$0.3}} & 52.9{\tiny\textcolor{red}{$\uparrow$0.1}} \\
			& \texttt{DyLAN} & 82.1{\tiny\textcolor{green}{$\downarrow$0.9}} & 31.5{\tiny\textcolor{green}{$\downarrow$3.3}} & 53.8{\tiny\textcolor{green}{$\downarrow$5.8}} & 44.5{\tiny\textcolor{green}{$\downarrow$3.5}} & 49.3{\tiny\textcolor{green}{$\downarrow$1.6}} & 45.1{\tiny\textcolor{green}{$\downarrow$1.6}} & 31.5{\tiny\textcolor{green}{$\downarrow$1.2}} & 50.3{\tiny\textcolor{green}{$\downarrow$2.5}} \\
			& \texttt{GPTSwarm} & 81.4{\tiny\textcolor{green}{$\downarrow$1.6}} & 31.5{\tiny\textcolor{green}{$\downarrow$3.3}} & 54.8{\tiny\textcolor{green}{$\downarrow$4.8}} & 43.7{\tiny\textcolor{green}{$\downarrow$4.3}} & 51.0{\tiny\textcolor{red}{$\uparrow$0.1}} & 45.2{\tiny\textcolor{green}{$\downarrow$1.5}} & 32.5{\tiny\textcolor{green}{$\downarrow$0.2}} & 51.5{\tiny\textcolor{green}{$\downarrow$1.3}} \\
			& \texttt{MacNet} & 72.9{\tiny\textcolor{green}{$\downarrow$10.1}} & 32.7{\tiny\textcolor{green}{$\downarrow$2.1}} & 49.1{\tiny\textcolor{green}{$\downarrow$10.5}} & 40.8{\tiny\textcolor{green}{$\downarrow$7.2}} & 52.7{\tiny\textcolor{red}{$\uparrow$1.8}} & 55.9{\tiny\textcolor{red}{$\uparrow$9.2}} & 36.5{\tiny\textcolor{red}{$\uparrow$3.8}} & 57.0{\tiny\textcolor{red}{$\uparrow$4.2}} \\
			& {\alg{} (ours)} & \cellcolor{blue!10}95.6{\tiny\textcolor{red}{$\uparrow$12.6}} & \cellcolor{blue!10}40.3{\tiny\textcolor{red}{$\uparrow$5.5}} & \cellcolor{blue!10}78.7{\tiny\textcolor{red}{$\uparrow$19.1}} & \cellcolor{blue!10}70.6{\tiny\textcolor{red}{$\uparrow$22.6}} & \cellcolor{blue!10}69.5{\tiny\textcolor{red}{$\uparrow$18.6}} & \cellcolor{blue!10}65.5{\tiny\textcolor{red}{$\uparrow$18.8}} & \cellcolor{blue!10}49.6{\tiny\textcolor{red}{$\uparrow$16.9}} & \cellcolor{blue!10}66.5{\tiny\textcolor{red}{$\uparrow$13.7}} \\
			\midrule
			\multirow{7}{*}{Qwen-7B}
			& \texttt{Single} & 61.2{\tiny\textcolor{green}{$\downarrow$13.0}} & 12.9{\tiny\textcolor{green}{$\downarrow$12.0}} & 20.2{\tiny\textcolor{green}{$\downarrow$16.2}} & 51.2{\tiny\textcolor{green}{$\downarrow$11.3}} & 24.3{\tiny\textcolor{green}{$\downarrow$6.6}} & 25.9{\tiny\textcolor{green}{$\downarrow$6.6}} & 15.3{\tiny\textcolor{green}{$\downarrow$4.7}} & 21.1{\tiny\textcolor{green}{$\downarrow$4.1}} \\
			& \texttt{BoN} & 74.2{\tiny\textcolor{red}{$\uparrow$0.0}} & 24.9{\tiny\textcolor{red}{$\uparrow$0.0}} & 36.4{\tiny\textcolor{red}{$\uparrow$0.0}} & 62.5{\tiny\textcolor{red}{$\uparrow$0.0}} & 30.9{\tiny\textcolor{red}{$\uparrow$0.0}} & 32.5{\tiny\textcolor{red}{$\uparrow$0.0}} & 20.0{\tiny\textcolor{red}{$\uparrow$0.0}} & 25.2{\tiny\textcolor{red}{$\uparrow$0.0}} \\
			& \texttt{Debate} & 69.3{\tiny\textcolor{green}{$\downarrow$4.9}} & 17.9{\tiny\textcolor{green}{$\downarrow$7.0}} & 28.7{\tiny\textcolor{green}{$\downarrow$7.7}} & 54.5{\tiny\textcolor{green}{$\downarrow$8.0}} & 29.5{\tiny\textcolor{green}{$\downarrow$1.4}} & 31.4{\tiny\textcolor{green}{$\downarrow$1.1}} & 20.6{\tiny\textcolor{red}{$\uparrow$0.6}} & 24.3{\tiny\textcolor{green}{$\downarrow$0.9}} \\
			& \texttt{DyLAN} & 74.2{\tiny\textcolor{red}{$\uparrow$0.0}} & 23.3{\tiny\textcolor{green}{$\downarrow$1.6}} & 34.2{\tiny\textcolor{green}{$\downarrow$2.2}} & 62.5{\tiny\textcolor{red}{$\uparrow$0.0}} & 29.5{\tiny\textcolor{green}{$\downarrow$1.4}} & 28.3{\tiny\textcolor{green}{$\downarrow$4.2}} & 19.0{\tiny\textcolor{green}{$\downarrow$1.0}} & 23.0{\tiny\textcolor{green}{$\downarrow$2.2}} \\
			& \texttt{GPTSwarm} & 74.5{\tiny\textcolor{red}{$\uparrow$0.3}} & 24.1{\tiny\textcolor{green}{$\downarrow$0.8}} & 37.2{\tiny\textcolor{red}{$\uparrow$0.8}} & 60.4{\tiny\textcolor{green}{$\downarrow$2.1}} & 29.4{\tiny\textcolor{green}{$\downarrow$1.5}} & 29.5{\tiny\textcolor{green}{$\downarrow$3.0}} & 20.0{\tiny\textcolor{red}{$\uparrow$0.0}} & 23.9{\tiny\textcolor{green}{$\downarrow$1.3}} \\
			& \texttt{MacNet} & 68.0{\tiny\textcolor{green}{$\downarrow$6.2}} & 20.3{\tiny\textcolor{green}{$\downarrow$4.6}} & 25.0{\tiny\textcolor{green}{$\downarrow$11.4}} & 57.1{\tiny\textcolor{green}{$\downarrow$5.4}} & 29.8{\tiny\textcolor{green}{$\downarrow$1.1}} & 33.8{\tiny\textcolor{red}{$\uparrow$1.3}} & 20.5{\tiny\textcolor{red}{$\uparrow$0.5}} & 26.6{\tiny\textcolor{red}{$\uparrow$1.4}} \\
			& {\alg{} (ours)} & \cellcolor{blue!10}88.7{\tiny\textcolor{red}{$\uparrow$14.5}} & \cellcolor{blue!10}40.8{\tiny\textcolor{red}{$\uparrow$15.9}} & \cellcolor{blue!10}58.4{\tiny\textcolor{red}{$\uparrow$22.0}} & \cellcolor{blue!10}76.8{\tiny\textcolor{red}{$\uparrow$14.3}} & \cellcolor{blue!10}36.7{\tiny\textcolor{red}{$\uparrow$5.8}} & \cellcolor{blue!10}41.7{\tiny\textcolor{red}{$\uparrow$9.2}} & \cellcolor{blue!10}30.2{\tiny\textcolor{red}{$\uparrow$10.2}} & \cellcolor{blue!10}33.0{\tiny\textcolor{red}{$\uparrow$7.8}} \\
			\midrule
			\multirow{7}{*}{Qwen-14B}
			& \texttt{Single} & 67.2{\tiny\textcolor{green}{$\downarrow$12.5}} & 21.6{\tiny\textcolor{green}{$\downarrow$6.2}} & 21.2{\tiny\textcolor{green}{$\downarrow$11.6}} & 56.7{\tiny\textcolor{green}{$\downarrow$12.7}} & 30.2{\tiny\textcolor{green}{$\downarrow$6.0}} & 31.7{\tiny\textcolor{green}{$\downarrow$7.7}} & 19.2{\tiny\textcolor{green}{$\downarrow$4.6}} & 26.4{\tiny\textcolor{green}{$\downarrow$2.4}} \\
			& \texttt{BoN} & 79.7{\tiny\textcolor{red}{$\uparrow$0.0}} & 27.8{\tiny\textcolor{red}{$\uparrow$0.0}} & 32.8{\tiny\textcolor{red}{$\uparrow$0.0}} & 69.4{\tiny\textcolor{red}{$\uparrow$0.0}} & 36.2{\tiny\textcolor{red}{$\uparrow$0.0}} & 39.4{\tiny\textcolor{red}{$\uparrow$0.0}} & 23.8{\tiny\textcolor{red}{$\uparrow$0.0}} & 28.8{\tiny\textcolor{red}{$\uparrow$0.0}} \\
			& \texttt{Debate} & 71.7{\tiny\textcolor{green}{$\downarrow$8.0}} & 24.7{\tiny\textcolor{green}{$\downarrow$3.1}} & 24.8{\tiny\textcolor{green}{$\downarrow$8.0}} & 62.7{\tiny\textcolor{green}{$\downarrow$6.7}} & 35.6{\tiny\textcolor{green}{$\downarrow$0.6}} & 38.5{\tiny\textcolor{green}{$\downarrow$0.9}} & 25.5{\tiny\textcolor{red}{$\uparrow$1.7}} & 29.9{\tiny\textcolor{red}{$\uparrow$1.1}} \\
			& \texttt{DyLAN} & 77.5{\tiny\textcolor{green}{$\downarrow$2.2}} & 26.9{\tiny\textcolor{green}{$\downarrow$0.9}} & 31.9{\tiny\textcolor{green}{$\downarrow$0.9}} & 67.9{\tiny\textcolor{green}{$\downarrow$1.5}} & 33.6{\tiny\textcolor{green}{$\downarrow$2.6}} & 34.9{\tiny\textcolor{green}{$\downarrow$4.5}} & 22.7{\tiny\textcolor{green}{$\downarrow$1.1}} & 27.8{\tiny\textcolor{green}{$\downarrow$1.0}} \\
			& \texttt{GPTSwarm} & 80.1{\tiny\textcolor{red}{$\uparrow$0.4}} & 26.9{\tiny\textcolor{green}{$\downarrow$0.9}} & 33.0{\tiny\textcolor{red}{$\uparrow$0.2}} & 68.1{\tiny\textcolor{green}{$\downarrow$1.3}} & 34.6{\tiny\textcolor{green}{$\downarrow$1.6}} & 34.8{\tiny\textcolor{green}{$\downarrow$4.6}} & 24.5{\tiny\textcolor{red}{$\uparrow$0.7}} & 28.3{\tiny\textcolor{green}{$\downarrow$0.5}} \\
			& \texttt{MacNet} & 71.2{\tiny\textcolor{green}{$\downarrow$8.5}} & 24.6{\tiny\textcolor{green}{$\downarrow$3.2}} & 27.6{\tiny\textcolor{green}{$\downarrow$5.2}} & 62.9{\tiny\textcolor{green}{$\downarrow$6.5}} & 36.1{\tiny\textcolor{green}{$\downarrow$0.1}} & 41.4{\tiny\textcolor{red}{$\uparrow$2.0}} & 28.6{\tiny\textcolor{red}{$\uparrow$4.8}} & 29.4{\tiny\textcolor{red}{$\uparrow$0.6}} \\
			& {\alg{} (ours)} & \cellcolor{blue!10}91.4{\tiny\textcolor{red}{$\uparrow$11.7}} & \cellcolor{blue!10}43.6{\tiny\textcolor{red}{$\uparrow$15.8}} & \cellcolor{blue!10}64.6{\tiny\textcolor{red}{$\uparrow$31.8}} & \cellcolor{blue!10}86.5{\tiny\textcolor{red}{$\uparrow$17.1}} & \cellcolor{blue!10}45.0{\tiny\textcolor{red}{$\uparrow$8.8}} & \cellcolor{blue!10}51.3{\tiny\textcolor{red}{$\uparrow$11.9}} & \cellcolor{blue!10}37.2{\tiny\textcolor{red}{$\uparrow$13.4}} & \cellcolor{blue!10}39.7{\tiny\textcolor{red}{$\uparrow$10.9}} \\
			\midrule
			\multirow{7}{*}{Qwen-72B}
			& \texttt{Single} & 72.5{\tiny\textcolor{green}{$\downarrow$11.7}} & 31.6{\tiny\textcolor{green}{$\downarrow$19.4}} & 34.9{\tiny\textcolor{green}{$\downarrow$11.0}} & 59.1{\tiny\textcolor{green}{$\downarrow$13.1}} & 49.3{\tiny\textcolor{green}{$\downarrow$7.3}} & 51.1{\tiny\textcolor{green}{$\downarrow$8.7}} & 31.4{\tiny\textcolor{green}{$\downarrow$7.2}} & 41.8{\tiny\textcolor{green}{$\downarrow$8.0}} \\
			& \texttt{BoN} & 84.2{\tiny\textcolor{red}{$\uparrow$0.0}} & 51.0{\tiny\textcolor{red}{$\uparrow$0.0}} & 45.9{\tiny\textcolor{red}{$\uparrow$0.0}} & 72.2{\tiny\textcolor{red}{$\uparrow$0.0}} & 56.6{\tiny\textcolor{red}{$\uparrow$0.0}} & 59.8{\tiny\textcolor{red}{$\uparrow$0.0}} & 38.6{\tiny\textcolor{red}{$\uparrow$0.0}} & 49.8{\tiny\textcolor{red}{$\uparrow$0.0}} \\
			& \texttt{Debate} & 76.6{\tiny\textcolor{green}{$\downarrow$7.6}} & 48.0{\tiny\textcolor{green}{$\downarrow$3.0}} & 39.6{\tiny\textcolor{green}{$\downarrow$6.3}} & 66.7{\tiny\textcolor{green}{$\downarrow$5.5}} & 58.3{\tiny\textcolor{red}{$\uparrow$1.7}} & 62.5{\tiny\textcolor{red}{$\uparrow$2.7}} & 40.7{\tiny\textcolor{red}{$\uparrow$2.1}} & 47.8{\tiny\textcolor{green}{$\downarrow$2.0}} \\
			& \texttt{DyLAN} & 84.6{\tiny\textcolor{red}{$\uparrow$0.4}} & 61.5{\tiny\textcolor{red}{$\uparrow$10.5}} & 45.5{\tiny\textcolor{green}{$\downarrow$0.4}} & 71.3{\tiny\textcolor{green}{$\downarrow$0.9}} & 54.2{\tiny\textcolor{green}{$\downarrow$2.4}} & 57.8{\tiny\textcolor{green}{$\downarrow$2.0}} & 38.2{\tiny\textcolor{green}{$\downarrow$0.4}} & 45.9{\tiny\textcolor{green}{$\downarrow$3.9}} \\
			& \texttt{GPTSwarm} & 84.9{\tiny\textcolor{red}{$\uparrow$0.7}} & 60.9{\tiny\textcolor{red}{$\uparrow$9.9}} & 46.6{\tiny\textcolor{red}{$\uparrow$0.7}} & 69.7{\tiny\textcolor{green}{$\downarrow$2.5}} & 55.2{\tiny\textcolor{green}{$\downarrow$1.4}} & 59.1{\tiny\textcolor{green}{$\downarrow$0.7}} & 39.0{\tiny\textcolor{red}{$\uparrow$0.4}} & 48.8{\tiny\textcolor{green}{$\downarrow$1.0}} \\
			& \texttt{MacNet} & 80.3{\tiny\textcolor{green}{$\downarrow$3.9}} & 48.1{\tiny\textcolor{green}{$\downarrow$2.9}} & 41.2{\tiny\textcolor{green}{$\downarrow$4.7}} & 68.1{\tiny\textcolor{green}{$\downarrow$4.1}} & 57.2{\tiny\textcolor{red}{$\uparrow$0.6}} & 64.1{\tiny\textcolor{red}{$\uparrow$4.3}} & 42.0{\tiny\textcolor{red}{$\uparrow$3.4}} & 49.8{\tiny\textcolor{red}{$\uparrow$0.0}} \\
			& {\alg{} (ours)} & \cellcolor{blue!10}93.5{\tiny\textcolor{red}{$\uparrow$9.3}} & \cellcolor{blue!10}69.3{\tiny\textcolor{red}{$\uparrow$18.3}} & \cellcolor{blue!10}73.1{\tiny\textcolor{red}{$\uparrow$27.2}} & \cellcolor{blue!10}88.3{\tiny\textcolor{red}{$\uparrow$16.1}} & \cellcolor{blue!10}73.9{\tiny\textcolor{red}{$\uparrow$17.3}} & \cellcolor{blue!10}81.9{\tiny\textcolor{red}{$\uparrow$22.1}} & \cellcolor{blue!10}60.3{\tiny\textcolor{red}{$\uparrow$21.7}} & \cellcolor{blue!10}65.3{\tiny\textcolor{red}{$\uparrow$15.5}} \\
			\bottomrule
		\end{tabular}
	}
	\label{tab:number_4}
\end{table*}

\begin{table*}[h]
	\centering
	\caption{Accuracy with varying number of LLMs on different datasets. The number of LLMs is $8$ for all datasets. We exhibit the performance advantage with \texttt{BoN} and highlight the \colorbox{blue!10}{best} result.}
	\vskip 0.1in
	\resizebox{\textwidth}{!}{
		\begin{tabular}{YZYYYYYYYY}
			\toprule
			\textbf{Model} & \textbf{Method} & \textbf{MMLU} & \textbf{MATH} & \textbf{GPQA} & \textbf{Code} & \textbf{ALFWorld} & \textbf{SciWorld} & \textbf{GAIA} & \textbf{PDDL} \\
			\midrule
			\multirow{7}{*}{Llama-7B}
			& \texttt{Single} & 45.3{\tiny\textcolor{green}{$\downarrow$5.2}} & 2.9{\tiny\textcolor{green}{$\downarrow$4.0}} & 16.8{\tiny\textcolor{green}{$\downarrow$10.7}} & 15.8{\tiny\textcolor{green}{$\downarrow$6.0}} & 22.6{\tiny\textcolor{green}{$\downarrow$3.6}} & 19.7{\tiny\textcolor{green}{$\downarrow$5.3}} & 14.9{\tiny\textcolor{green}{$\downarrow$4.0}} & 23.6{\tiny\textcolor{green}{$\downarrow$3.1}} \\
			& \texttt{BoN} & 50.5{\tiny\textcolor{red}{$\uparrow$0.0}} & 6.9{\tiny\textcolor{red}{$\uparrow$0.0}} & 27.5{\tiny\textcolor{red}{$\uparrow$0.0}} & 21.8{\tiny\textcolor{red}{$\uparrow$0.0}} & 26.2{\tiny\textcolor{red}{$\uparrow$0.0}} & 25.0{\tiny\textcolor{red}{$\uparrow$0.0}} & 18.9{\tiny\textcolor{red}{$\uparrow$0.0}} & 26.7{\tiny\textcolor{red}{$\uparrow$0.0}} \\
			& \texttt{Debate} & 49.5{\tiny\textcolor{green}{$\downarrow$1.0}} & 6.2{\tiny\textcolor{green}{$\downarrow$0.7}} & 24.8{\tiny\textcolor{green}{$\downarrow$2.7}} & 19.2{\tiny\textcolor{green}{$\downarrow$2.6}} & 25.9{\tiny\textcolor{green}{$\downarrow$0.3}} & 24.7{\tiny\textcolor{green}{$\downarrow$0.3}} & 15.9{\tiny\textcolor{green}{$\downarrow$3.0}} & 26.9{\tiny\textcolor{red}{$\uparrow$0.2}} \\
			& \texttt{DyLAN} & 54.2{\tiny\textcolor{red}{$\uparrow$3.7}} & 9.7{\tiny\textcolor{red}{$\uparrow$2.8}} & 35{\tiny\textcolor{red}{$\uparrow$7.5}} & 26{\tiny\textcolor{red}{$\uparrow$4.2}} & 25.0{\tiny\textcolor{green}{$\downarrow$1.2}} & 22.0{\tiny\textcolor{green}{$\downarrow$3.0}} & 17.2{\tiny\textcolor{green}{$\downarrow$1.7}} & 25.4{\tiny\textcolor{green}{$\downarrow$1.3}} \\
			& \texttt{GPTSwarm} & 49.2{\tiny\textcolor{green}{$\downarrow$1.3}} & 6.8{\tiny\textcolor{green}{$\downarrow$0.1}} & 28.7{\tiny\textcolor{red}{$\uparrow$1.2}} & 20.7{\tiny\textcolor{green}{$\downarrow$1.1}} & 26.2{\tiny\textcolor{red}{$\uparrow$0.0}} & 26.4{\tiny\textcolor{red}{$\uparrow$1.4}} & 20.0{\tiny\textcolor{red}{$\uparrow$1.1}} & 29.3{\tiny\textcolor{red}{$\uparrow$2.6}} \\
			& \texttt{MacNet} & 56.5{\tiny\textcolor{red}{$\uparrow$6.0}} & 11.0{\tiny\textcolor{red}{$\uparrow$4.1}} & 39.6{\tiny\textcolor{red}{$\uparrow$12.1}} & 30.0{\tiny\textcolor{red}{$\uparrow$8.2}} & 30.4{\tiny\textcolor{red}{$\uparrow$4.2}} & 30.2{\tiny\textcolor{red}{$\uparrow$5.2}} & 20.6{\tiny\textcolor{red}{$\uparrow$1.7}} & 31.7{\tiny\textcolor{red}{$\uparrow$5.0}} \\
			& {\alg{} (ours)} & \cellcolor{blue!10}63.8{\tiny\textcolor{red}{$\uparrow$13.3}} & \cellcolor{blue!10}24.9{\tiny\textcolor{red}{$\uparrow$18.0}} & \cellcolor{blue!10}59.3{\tiny\textcolor{red}{$\uparrow$31.8}} & \cellcolor{blue!10}38.8{\tiny\textcolor{red}{$\uparrow$17.0}} & \cellcolor{blue!10}37.2{\tiny\textcolor{red}{$\uparrow$11.0}} & \cellcolor{blue!10}36.1{\tiny\textcolor{red}{$\uparrow$11.1}} & \cellcolor{blue!10}27.1{\tiny\textcolor{red}{$\uparrow$8.2}} & \cellcolor{blue!10}34.7{\tiny\textcolor{red}{$\uparrow$8.0}} \\
			\midrule
			\multirow{7}{*}{Llama-14B}
			& \texttt{Single} & 54.8{\tiny\textcolor{green}{$\downarrow$16.2}} & 3.9{\tiny\textcolor{green}{$\downarrow$4.9}} & 25.0{\tiny\textcolor{green}{$\downarrow$12.7}} & 23.7{\tiny\textcolor{green}{$\downarrow$8.3}} & 27.8{\tiny\textcolor{green}{$\downarrow$4.1}} & 24.1{\tiny\textcolor{green}{$\downarrow$8.2}} & 18.6{\tiny\textcolor{green}{$\downarrow$3.9}} & 29.4{\tiny\textcolor{green}{$\downarrow$5.4}} \\
			& \texttt{BoN} & 71.0{\tiny\textcolor{red}{$\uparrow$0.0}} & 8.8{\tiny\textcolor{red}{$\uparrow$0.0}} & 37.7{\tiny\textcolor{red}{$\uparrow$0.0}} & 32.0{\tiny\textcolor{red}{$\uparrow$0.0}} & 31.9{\tiny\textcolor{red}{$\uparrow$0.0}} & 32.3{\tiny\textcolor{red}{$\uparrow$0.0}} & 22.5{\tiny\textcolor{red}{$\uparrow$0.0}} & 34.8{\tiny\textcolor{red}{$\uparrow$0.0}} \\
			& \texttt{Debate} & 66.9{\tiny\textcolor{green}{$\downarrow$4.1}} & 8.0{\tiny\textcolor{green}{$\downarrow$0.8}} & 34.5{\tiny\textcolor{green}{$\downarrow$3.2}} & 30.8{\tiny\textcolor{green}{$\downarrow$1.2}} & 31.7{\tiny\textcolor{green}{$\downarrow$0.2}} & 30.3{\tiny\textcolor{green}{$\downarrow$2.0}} & 20.0{\tiny\textcolor{green}{$\downarrow$2.5}} & 32.0{\tiny\textcolor{green}{$\downarrow$2.8}} \\
			& \texttt{DyLAN} & 82.1{\tiny\textcolor{red}{$\uparrow$11.1}} & 12.4{\tiny\textcolor{red}{$\uparrow$3.6}} & 46.7{\tiny\textcolor{red}{$\uparrow$9.0}} & 37.5{\tiny\textcolor{red}{$\uparrow$5.5}} & 31.7{\tiny\textcolor{green}{$\downarrow$0.2}} & 29.6{\tiny\textcolor{green}{$\downarrow$2.7}} & 19.4{\tiny\textcolor{green}{$\downarrow$3.1}} & 31.3{\tiny\textcolor{green}{$\downarrow$3.5}} \\
			& \texttt{MacNet} & 81.7{\tiny\textcolor{red}{$\uparrow$10.7}} & 14.2{\tiny\textcolor{red}{$\uparrow$5.4}} & 53.9{\tiny\textcolor{red}{$\uparrow$16.2}} & 43.3{\tiny\textcolor{red}{$\uparrow$11.3}} & 33.4{\tiny\textcolor{red}{$\uparrow$1.5}} & 35.6{\tiny\textcolor{red}{$\uparrow$3.3}} & 24.9{\tiny\textcolor{red}{$\uparrow$2.4}} & 37.0{\tiny\textcolor{red}{$\uparrow$2.2}} \\
			& \texttt{MacNet} & 81.7{\tiny\textcolor{red}{$\uparrow$10.7}} & 14.2{\tiny\textcolor{red}{$\uparrow$5.4}} & 53.9{\tiny\textcolor{red}{$\uparrow$16.2}} & 43.3{\tiny\textcolor{red}{$\uparrow$11.3}} & 33.4{\tiny\textcolor{red}{$\uparrow$1.5}} & 35.6{\tiny\textcolor{red}{$\uparrow$3.3}} & 24.9{\tiny\textcolor{red}{$\uparrow$2.4}} & 37.0{\tiny\textcolor{red}{$\uparrow$2.2}} \\
			& {\alg{} (ours)} & \cellcolor{blue!10}92.7{\tiny\textcolor{red}{$\uparrow$21.7}} & \cellcolor{blue!10}23.0{\tiny\textcolor{red}{$\uparrow$14.2}} & \cellcolor{blue!10}70.7{\tiny\textcolor{red}{$\uparrow$33.0}} & \cellcolor{blue!10}55.4{\tiny\textcolor{red}{$\uparrow$23.4}} & \cellcolor{blue!10}46.1{\tiny\textcolor{red}{$\uparrow$14.2}} & \cellcolor{blue!10}44.7{\tiny\textcolor{red}{$\uparrow$12.4}} & \cellcolor{blue!10}33.6{\tiny\textcolor{red}{$\uparrow$11.1}} & \cellcolor{blue!10}43.8{\tiny\textcolor{red}{$\uparrow$9.0}} \\
			\midrule
			\multirow{7}{*}{Llama-70B}
			& \texttt{Single} & 68.9{\tiny\textcolor{green}{$\downarrow$14.1}} & 31.6{\tiny\textcolor{green}{$\downarrow$3.2}} & 27.6{\tiny\textcolor{green}{$\downarrow$32.0}} & 35.5{\tiny\textcolor{green}{$\downarrow$12.5}} & 44.1{\tiny\textcolor{green}{$\downarrow$8.3}} & 39.1{\tiny\textcolor{green}{$\downarrow$9.7}} & 29.1{\tiny\textcolor{green}{$\downarrow$3.8}} & 47.1{\tiny\textcolor{green}{$\downarrow$5.7}} \\
			& \texttt{BoN} & 83.0{\tiny\textcolor{red}{$\uparrow$0.0}} & 34.8{\tiny\textcolor{red}{$\uparrow$0.0}} & 59.6{\tiny\textcolor{red}{$\uparrow$0.0}} & 48.0{\tiny\textcolor{red}{$\uparrow$0.0}} & 52.4{\tiny\textcolor{red}{$\uparrow$0.0}} & 48.8{\tiny\textcolor{red}{$\uparrow$0.0}} & 32.9{\tiny\textcolor{red}{$\uparrow$0.0}} & 52.8{\tiny\textcolor{red}{$\uparrow$0.0}} \\
			& \texttt{Debate} & 82.2{\tiny\textcolor{green}{$\downarrow$0.8}} & 34.3{\tiny\textcolor{green}{$\downarrow$0.5}} & 58.1{\tiny\textcolor{green}{$\downarrow$1.5}} & 47.1{\tiny\textcolor{green}{$\downarrow$0.9}} & 51.7{\tiny\textcolor{green}{$\downarrow$0.7}} & 50.1{\tiny\textcolor{red}{$\uparrow$1.3}} & 32.2{\tiny\textcolor{green}{$\downarrow$0.7}} & 53.3{\tiny\textcolor{red}{$\uparrow$0.5}} \\
			& \texttt{DyLAN} & 93.5{\tiny\textcolor{red}{$\uparrow$10.5}} & 37.1{\tiny\textcolor{red}{$\uparrow$2.3}} & 72.8{\tiny\textcolor{red}{$\uparrow$13.2}} & 56.8{\tiny\textcolor{red}{$\uparrow$8.8}} & 50.2{\tiny\textcolor{green}{$\downarrow$2.2}} & 46.4{\tiny\textcolor{green}{$\downarrow$2.4}} & 31.7{\tiny\textcolor{green}{$\downarrow$1.2}} & 49.7{\tiny\textcolor{green}{$\downarrow$3.1}} \\
			& \texttt{GPTSwarm} & 93.6{\tiny\textcolor{red}{$\uparrow$10.6}} & 36.6{\tiny\textcolor{red}{$\uparrow$1.8}} & 71.2{\tiny\textcolor{red}{$\uparrow$11.6}} & 57.1{\tiny\textcolor{red}{$\uparrow$9.1}} & 51.4{\tiny\textcolor{green}{$\downarrow$1.0}} & 46.5{\tiny\textcolor{green}{$\downarrow$2.3}} & 32.2{\tiny\textcolor{green}{$\downarrow$0.7}} & 52.2{\tiny\textcolor{green}{$\downarrow$0.6}} \\
			& \texttt{MacNet} & 82.0{\tiny\textcolor{green}{$\downarrow$1.0}} & 36.9{\tiny\textcolor{red}{$\uparrow$2.1}} & 62.2{\tiny\textcolor{red}{$\uparrow$2.6}} & 48.8{\tiny\textcolor{red}{$\uparrow$0.8}} & 57.3{\tiny\textcolor{red}{$\uparrow$4.9}} & 64.0{\tiny\textcolor{red}{$\uparrow$15.2}} & 40.7{\tiny\textcolor{red}{$\uparrow$7.8}} & 62.6{\tiny\textcolor{red}{$\uparrow$9.8}} \\
			& {\alg{} (ours)} & \cellcolor{blue!10}95.4{\tiny\textcolor{red}{$\uparrow$12.4}} & \cellcolor{blue!10}43.5{\tiny\textcolor{red}{$\uparrow$8.7}} & \cellcolor{blue!10}87.6{\tiny\textcolor{red}{$\uparrow$28.0}} & \cellcolor{blue!10}81.0{\tiny\textcolor{red}{$\uparrow$33.0}} & \cellcolor{blue!10}78.9{\tiny\textcolor{red}{$\uparrow$26.5}} & \cellcolor{blue!10}76.2{\tiny\textcolor{red}{$\uparrow$27.4}} & \cellcolor{blue!10}58.7{\tiny\textcolor{red}{$\uparrow$25.8}} & \cellcolor{blue!10}73.6{\tiny\textcolor{red}{$\uparrow$20.8}} \\
			\midrule
			\multirow{7}{*}{Qwen-7B}
			& \texttt{Single} & 61.2{\tiny\textcolor{green}{$\downarrow$13.0}} & 12.9{\tiny\textcolor{green}{$\downarrow$12.0}} & 20.2{\tiny\textcolor{green}{$\downarrow$16.2}} & 51.2{\tiny\textcolor{green}{$\downarrow$11.3}} & 23.8{\tiny\textcolor{green}{$\downarrow$7.7}} & 25.2{\tiny\textcolor{green}{$\downarrow$10.1}} & 15.6{\tiny\textcolor{green}{$\downarrow$5.5}} & 21.0{\tiny\textcolor{green}{$\downarrow$5.3}} \\
			& \texttt{BoN} & 74.2{\tiny\textcolor{red}{$\uparrow$0.0}} & 24.9{\tiny\textcolor{red}{$\uparrow$0.0}} & 36.4{\tiny\textcolor{red}{$\uparrow$0.0}} & 62.5{\tiny\textcolor{red}{$\uparrow$0.0}} & 31.5{\tiny\textcolor{red}{$\uparrow$0.0}} & 35.3{\tiny\textcolor{red}{$\uparrow$0.0}} & 21.1{\tiny\textcolor{red}{$\uparrow$0.0}} & 26.3{\tiny\textcolor{red}{$\uparrow$0.0}} \\
			& \texttt{Debate} & 71.1{\tiny\textcolor{green}{$\downarrow$3.1}} & 19.9{\tiny\textcolor{green}{$\downarrow$5.0}} & 28.7{\tiny\textcolor{green}{$\downarrow$7.7}} & 60.0{\tiny\textcolor{green}{$\downarrow$2.5}} & 30.6{\tiny\textcolor{green}{$\downarrow$0.9}} & 32.2{\tiny\textcolor{green}{$\downarrow$3.1}} & 21.0{\tiny\textcolor{green}{$\downarrow$0.1}} & 24.8{\tiny\textcolor{green}{$\downarrow$1.5}} \\
			& \texttt{DyLAN} & 74.3{\tiny\textcolor{red}{$\uparrow$0.1}} & 26.7{\tiny\textcolor{red}{$\uparrow$1.8}} & 35.4{\tiny\textcolor{green}{$\downarrow$1.0}} & 63.4{\tiny\textcolor{red}{$\uparrow$0.9}} & 29.8{\tiny\textcolor{green}{$\downarrow$1.7}} & 29.4{\tiny\textcolor{green}{$\downarrow$5.9}} & 18.4{\tiny\textcolor{green}{$\downarrow$2.7}} & 23.4{\tiny\textcolor{green}{$\downarrow$2.9}} \\
			& \texttt{GPTSwarm} & 76.3{\tiny\textcolor{red}{$\uparrow$2.1}} & 26.2{\tiny\textcolor{red}{$\uparrow$1.3}} & 35.6{\tiny\textcolor{green}{$\downarrow$0.8}} & 62.7{\tiny\textcolor{red}{$\uparrow$0.2}} & 29.9{\tiny\textcolor{green}{$\downarrow$1.6}} & 30.8{\tiny\textcolor{green}{$\downarrow$4.5}} & 20.6{\tiny\textcolor{green}{$\downarrow$0.5}} & 24.5{\tiny\textcolor{green}{$\downarrow$1.8}} \\
			& \texttt{MacNet} & 71.5{\tiny\textcolor{green}{$\downarrow$2.7}} & 21.4{\tiny\textcolor{green}{$\downarrow$3.5}} & 30.8{\tiny\textcolor{green}{$\downarrow$5.6}} & 59.3{\tiny\textcolor{green}{$\downarrow$3.2}} & 33.6{\tiny\textcolor{red}{$\uparrow$2.1}} & 37.0{\tiny\textcolor{red}{$\uparrow$1.7}} & 21.5{\tiny\textcolor{red}{$\uparrow$0.4}} & 29.5{\tiny\textcolor{red}{$\uparrow$3.2}} \\
			& {\alg{} (ours)} & \cellcolor{blue!10}92.5{\tiny\textcolor{red}{$\uparrow$18.3}} & \cellcolor{blue!10}47.8{\tiny\textcolor{red}{$\uparrow$22.9}} & \cellcolor{blue!10}66.1{\tiny\textcolor{red}{$\uparrow$29.7}} & \cellcolor{blue!10}80.3{\tiny\textcolor{red}{$\uparrow$17.8}} & \cellcolor{blue!10}39.9{\tiny\textcolor{red}{$\uparrow$8.4}} & \cellcolor{blue!10}45.3{\tiny\textcolor{red}{$\uparrow$10.0}} & \cellcolor{blue!10}33.6{\tiny\textcolor{red}{$\uparrow$12.5}} & \cellcolor{blue!10}34.7{\tiny\textcolor{red}{$\uparrow$8.4}} \\
			\midrule
			\multirow{7}{*}{Qwen-14B}
			& \texttt{Single} & 67.2{\tiny\textcolor{green}{$\downarrow$12.5}} & 21.6{\tiny\textcolor{green}{$\downarrow$6.2}} & 21.2{\tiny\textcolor{green}{$\downarrow$11.6}} & 56.7{\tiny\textcolor{green}{$\downarrow$12.7}} & 30.1{\tiny\textcolor{green}{$\downarrow$7.4}} & 31.3{\tiny\textcolor{green}{$\downarrow$10.3}} & 18.7{\tiny\textcolor{green}{$\downarrow$6.9}} & 25.7{\tiny\textcolor{green}{$\downarrow$2.8}} \\
			& \texttt{BoN} & 79.7{\tiny\textcolor{red}{$\uparrow$0.0}} & 27.8{\tiny\textcolor{red}{$\uparrow$0.0}} & 32.8{\tiny\textcolor{red}{$\uparrow$0.0}} & 69.4{\tiny\textcolor{red}{$\uparrow$0.0}} & 37.5{\tiny\textcolor{red}{$\uparrow$0.0}} & 41.6{\tiny\textcolor{red}{$\uparrow$0.0}} & 25.6{\tiny\textcolor{red}{$\uparrow$0.0}} & 28.5{\tiny\textcolor{red}{$\uparrow$0.0}} \\
			& \texttt{Debate} & 77.2{\tiny\textcolor{green}{$\downarrow$2.5}} & 27.4{\tiny\textcolor{green}{$\downarrow$0.4}} & 30.3{\tiny\textcolor{green}{$\downarrow$2.5}} & 66.9{\tiny\textcolor{green}{$\downarrow$2.5}} & 36.8{\tiny\textcolor{green}{$\downarrow$0.7}} & 39.4{\tiny\textcolor{green}{$\downarrow$2.2}} & 26.4{\tiny\textcolor{red}{$\uparrow$0.8}} & 30.0{\tiny\textcolor{red}{$\uparrow$1.5}} \\
			& \texttt{DyLAN} & 86.8{\tiny\textcolor{red}{$\uparrow$7.1}} & 31.2{\tiny\textcolor{red}{$\uparrow$3.4}} & 39.4{\tiny\textcolor{red}{$\uparrow$6.6}} & 76.6{\tiny\textcolor{red}{$\uparrow$7.2}} & 34.8{\tiny\textcolor{green}{$\downarrow$2.7}} & 36.0{\tiny\textcolor{green}{$\downarrow$5.6}} & 22.8{\tiny\textcolor{green}{$\downarrow$2.8}} & 28.0{\tiny\textcolor{green}{$\downarrow$0.5}} \\
			& \texttt{GPTSwarm} & 87.0{\tiny\textcolor{red}{$\uparrow$7.3}} & 31.3{\tiny\textcolor{red}{$\uparrow$3.5}} & 38.7{\tiny\textcolor{red}{$\uparrow$5.9}} & 75.9{\tiny\textcolor{red}{$\uparrow$6.5}} & 34.9{\tiny\textcolor{green}{$\downarrow$2.6}} & 35.9{\tiny\textcolor{green}{$\downarrow$5.7}} & 25.0{\tiny\textcolor{green}{$\downarrow$0.6}} & 28.4{\tiny\textcolor{green}{$\downarrow$0.1}} \\
			& \texttt{MacNet} & 78.9{\tiny\textcolor{green}{$\downarrow$0.8}} & 28.7{\tiny\textcolor{red}{$\uparrow$0.9}} & 31.9{\tiny\textcolor{green}{$\downarrow$0.9}} & 70.4{\tiny\textcolor{red}{$\uparrow$1.0}} & 39.2{\tiny\textcolor{red}{$\uparrow$1.7}} & 46.0{\tiny\textcolor{red}{$\uparrow$4.4}} & 32.2{\tiny\textcolor{red}{$\uparrow$6.6}} & 32.7{\tiny\textcolor{red}{$\uparrow$4.2}} \\
			& {\alg{} (ours)} & \cellcolor{blue!10}93.7{\tiny\textcolor{red}{$\uparrow$14.0}} & \cellcolor{blue!10}51.7{\tiny\textcolor{red}{$\uparrow$23.9}} & \cellcolor{blue!10}66.2{\tiny\textcolor{red}{$\uparrow$33.4}} & \cellcolor{blue!10}91.1{\tiny\textcolor{red}{$\uparrow$21.7}} & \cellcolor{blue!10}48.2{\tiny\textcolor{red}{$\uparrow$10.7}} & \cellcolor{blue!10}56.1{\tiny\textcolor{red}{$\uparrow$14.5}} & \cellcolor{blue!10}42.0{\tiny\textcolor{red}{$\uparrow$16.4}} & \cellcolor{blue!10}43.0{\tiny\textcolor{red}{$\uparrow$14.5}} \\
			\midrule
			\multirow{7}{*}{Qwen-72B}
			& \texttt{Single} & 72.5{\tiny\textcolor{green}{$\downarrow$11.7}} & 31.6{\tiny\textcolor{green}{$\downarrow$19.4}} & 34.9{\tiny\textcolor{green}{$\downarrow$11.0}} & 59.1{\tiny\textcolor{green}{$\downarrow$13.1}} & 48.2{\tiny\textcolor{green}{$\downarrow$9.3}} & 50.4{\tiny\textcolor{green}{$\downarrow$11.7}} & 31.1{\tiny\textcolor{green}{$\downarrow$10.1}} & 41.0{\tiny\textcolor{green}{$\downarrow$10.0}} \\
			& \texttt{BoN} & 84.2{\tiny\textcolor{red}{$\uparrow$0.0}} & 51.0{\tiny\textcolor{red}{$\uparrow$0.0}} & 45.9{\tiny\textcolor{red}{$\uparrow$0.0}} & 72.2{\tiny\textcolor{red}{$\uparrow$0.0}} & 57.5{\tiny\textcolor{red}{$\uparrow$0.0}} & 62.1{\tiny\textcolor{red}{$\uparrow$0.0}} & 41.2{\tiny\textcolor{red}{$\uparrow$0.0}} & 51.0{\tiny\textcolor{red}{$\uparrow$0.0}} \\
			& \texttt{Debate} & 82.7{\tiny\textcolor{green}{$\downarrow$1.5}} & 48.4{\tiny\textcolor{green}{$\downarrow$2.6}} & 43.4{\tiny\textcolor{green}{$\downarrow$2.5}} & 69.1{\tiny\textcolor{green}{$\downarrow$3.1}} & 60.4{\tiny\textcolor{red}{$\uparrow$2.9}} & 65.2{\tiny\textcolor{red}{$\uparrow$3.1}} & 42.9{\tiny\textcolor{red}{$\uparrow$1.7}} & 49.1{\tiny\textcolor{green}{$\downarrow$1.9}} \\
			& \texttt{DyLAN} & 91.5{\tiny\textcolor{red}{$\uparrow$7.3}} & 63.1{\tiny\textcolor{red}{$\uparrow$12.1}} & 51.6{\tiny\textcolor{red}{$\uparrow$5.7}} & 80.4{\tiny\textcolor{red}{$\uparrow$8.2}} & 55.1{\tiny\textcolor{green}{$\downarrow$2.4}} & 58.0{\tiny\textcolor{green}{$\downarrow$4.1}} & 40.4{\tiny\textcolor{green}{$\downarrow$0.8}} & 45.5{\tiny\textcolor{green}{$\downarrow$5.5}} \\
			& \texttt{GPTSwarm} & 91.5{\tiny\textcolor{red}{$\uparrow$7.3}} & 64.7{\tiny\textcolor{red}{$\uparrow$13.7}} & 52.4{\tiny\textcolor{red}{$\uparrow$6.5}} & 79.6{\tiny\textcolor{red}{$\uparrow$7.4}} & 56.8{\tiny\textcolor{green}{$\downarrow$0.7}} & 60.2{\tiny\textcolor{green}{$\downarrow$1.9}} & 40.3{\tiny\textcolor{green}{$\downarrow$0.9}} & 49.7{\tiny\textcolor{green}{$\downarrow$1.3}} \\
			& \texttt{MacNet} & 83.8{\tiny\textcolor{green}{$\downarrow$0.4}} & 52.9{\tiny\textcolor{red}{$\uparrow$1.9}} & 46.2{\tiny\textcolor{red}{$\uparrow$0.3}} & 70.5{\tiny\textcolor{green}{$\downarrow$1.7}} & 61.8{\tiny\textcolor{red}{$\uparrow$4.3}} & 68.4{\tiny\textcolor{red}{$\uparrow$6.3}} & 46.4{\tiny\textcolor{red}{$\uparrow$5.2}} & 53.7{\tiny\textcolor{red}{$\uparrow$2.7}} \\
			& {\alg{} (ours)} & \cellcolor{blue!10}95.1{\tiny\textcolor{red}{$\uparrow$10.9}} & \cellcolor{blue!10}72.5{\tiny\textcolor{red}{$\uparrow$21.5}} & \cellcolor{blue!10}78.9{\tiny\textcolor{red}{$\uparrow$33.0}} & \cellcolor{blue!10}90.7{\tiny\textcolor{red}{$\uparrow$18.5}} & \cellcolor{blue!10}79.0{\tiny\textcolor{red}{$\uparrow$21.5}} & \cellcolor{blue!10}88.9{\tiny\textcolor{red}{$\uparrow$26.8}} & \cellcolor{blue!10}67.2{\tiny\textcolor{red}{$\uparrow$26.0}} & \cellcolor{blue!10}70.5{\tiny\textcolor{red}{$\uparrow$19.5}} \\
			\bottomrule
		\end{tabular}
	}
	\label{tab:number_8}
\end{table*}

\begin{table*}[h]
	\centering
	\caption{Accuracy with varying number of LLMs on different datasets. The number of LLMs is $16$ for all datasets. We exhibit the performance advantage with \texttt{BoN} and highlight the \colorbox{blue!10}{best} result.}
	\vskip 0.1in
	\resizebox{\textwidth}{!}{
		\begin{tabular}{YZYYYYYYYY}
			\toprule
			\textbf{Model} & \textbf{Method} & \textbf{MMLU} & \textbf{MATH} & \textbf{GPQA} & \textbf{Code} & \textbf{ALFWorld} & \textbf{SciWorld} & \textbf{GAIA} & \textbf{PDDL} \\
			\midrule
			\multirow{7}{*}{Llama-7B}
			& \texttt{Single} & 45.3{\tiny\textcolor{green}{$\downarrow$5.2}} & 2.9{\tiny\textcolor{green}{$\downarrow$4.0}} & 16.8{\tiny\textcolor{green}{$\downarrow$10.7}} & 15.8{\tiny\textcolor{green}{$\downarrow$6.0}} & 22.4{\tiny\textcolor{green}{$\downarrow$4.3}} & 19.6{\tiny\textcolor{green}{$\downarrow$6.1}} & 14.7{\tiny\textcolor{green}{$\downarrow$4.5}} & 24.2{\tiny\textcolor{green}{$\downarrow$2.1}} \\
			& \texttt{BoN} & 50.5{\tiny\textcolor{red}{$\uparrow$0.0}} & 6.9{\tiny\textcolor{red}{$\uparrow$0.0}} & 27.5{\tiny\textcolor{red}{$\uparrow$0.0}} & 21.8{\tiny\textcolor{red}{$\uparrow$0.0}} & 26.7{\tiny\textcolor{red}{$\uparrow$0.0}} & 25.7{\tiny\textcolor{red}{$\uparrow$0.0}} & 19.2{\tiny\textcolor{red}{$\uparrow$0.0}} & 26.3{\tiny\textcolor{red}{$\uparrow$0.0}} \\
			& \texttt{Debate} & 51{\tiny\textcolor{red}{$\uparrow$0.5}} & 6.4{\tiny\textcolor{green}{$\downarrow$0.5}} & 26.1{\tiny\textcolor{green}{$\downarrow$1.4}} & 21.3{\tiny\textcolor{green}{$\downarrow$0.5}} & 25.7{\tiny\textcolor{green}{$\downarrow$1.0}} & 25.3{\tiny\textcolor{green}{$\downarrow$0.4}} & 16.6{\tiny\textcolor{green}{$\downarrow$2.6}} & 26.4{\tiny\textcolor{red}{$\uparrow$0.1}} \\
			& \texttt{DyLAN} & 54.5{\tiny\textcolor{red}{$\uparrow$4.0}} & 9.9{\tiny\textcolor{red}{$\uparrow$3.0}} & 35.5{\tiny\textcolor{red}{$\uparrow$8.0}} & 26.2{\tiny\textcolor{red}{$\uparrow$4.4}} & 24.6{\tiny\textcolor{green}{$\downarrow$2.1}} & 22.7{\tiny\textcolor{green}{$\downarrow$3.0}} & 17.2{\tiny\textcolor{green}{$\downarrow$2.0}} & 25.8{\tiny\textcolor{green}{$\downarrow$0.5}} \\
			& \texttt{GPTSwarm} & 51.6{\tiny\textcolor{red}{$\uparrow$1.1}} & 6.9{\tiny\textcolor{red}{$\uparrow$0.0}} & 29.1{\tiny\textcolor{red}{$\uparrow$1.6}} & 23.0{\tiny\textcolor{red}{$\uparrow$1.2}} & 27.2{\tiny\textcolor{red}{$\uparrow$0.5}} & 26.8{\tiny\textcolor{red}{$\uparrow$1.1}} & 20.8{\tiny\textcolor{red}{$\uparrow$1.6}} & 29.7{\tiny\textcolor{red}{$\uparrow$3.4}} \\
			& \texttt{MacNet} & 59.3{\tiny\textcolor{red}{$\uparrow$8.8}} & 12.5{\tiny\textcolor{red}{$\uparrow$5.6}} & 45.0{\tiny\textcolor{red}{$\uparrow$17.5}} & 33.7{\tiny\textcolor{red}{$\uparrow$11.9}} & 32.8{\tiny\textcolor{red}{$\uparrow$6.1}} & 33.4{\tiny\textcolor{red}{$\uparrow$7.7}} & 22.6{\tiny\textcolor{red}{$\uparrow$3.4}} & 34.2{\tiny\textcolor{red}{$\uparrow$7.9}} \\
			& {\alg{} (ours)} & \cellcolor{blue!10}71.5{\tiny\textcolor{red}{$\uparrow$21.0}} & \cellcolor{blue!10}23.9{\tiny\textcolor{red}{$\uparrow$17.0}} & \cellcolor{blue!10}69.9{\tiny\textcolor{red}{$\uparrow$42.4}} & \cellcolor{blue!10}48.3{\tiny\textcolor{red}{$\uparrow$26.5}} & \cellcolor{blue!10}39.8{\tiny\textcolor{red}{$\uparrow$13.1}} & \cellcolor{blue!10}38.4{\tiny\textcolor{red}{$\uparrow$12.7}} & \cellcolor{blue!10}28.9{\tiny\textcolor{red}{$\uparrow$9.7}} & \cellcolor{blue!10}37.0{\tiny\textcolor{red}{$\uparrow$10.7}} \\
			\midrule
			\multirow{7}{*}{Llama-14B}
			& \texttt{Single} & 54.8{\tiny\textcolor{green}{$\downarrow$16.2}} & 3.9{\tiny\textcolor{green}{$\downarrow$4.9}} & 25.0{\tiny\textcolor{green}{$\downarrow$12.7}} & 23.7{\tiny\textcolor{green}{$\downarrow$8.3}} & 27.4{\tiny\textcolor{green}{$\downarrow$5.6}} & 24.0{\tiny\textcolor{green}{$\downarrow$9.5}} & 18.3{\tiny\textcolor{green}{$\downarrow$5.0}} & 29.2{\tiny\textcolor{green}{$\downarrow$4.6}} \\
			& \texttt{BoN} & 71.0{\tiny\textcolor{red}{$\uparrow$0.0}} & 8.8{\tiny\textcolor{red}{$\uparrow$0.0}} & 37.7{\tiny\textcolor{red}{$\uparrow$0.0}} & 32.0{\tiny\textcolor{red}{$\uparrow$0.0}} & 33.0{\tiny\textcolor{red}{$\uparrow$0.0}} & 33.5{\tiny\textcolor{red}{$\uparrow$0.0}} & 23.3{\tiny\textcolor{red}{$\uparrow$0.0}} & 33.8{\tiny\textcolor{red}{$\uparrow$0.0}} \\
			& \texttt{Debate} & 72.2{\tiny\textcolor{red}{$\uparrow$1.2}} & 8.2{\tiny\textcolor{green}{$\downarrow$0.6}} & 37.0{\tiny\textcolor{green}{$\downarrow$0.7}} & 32.0{\tiny\textcolor{red}{$\uparrow$0.0}} & 32.0{\tiny\textcolor{green}{$\downarrow$1.0}} & 30.9{\tiny\textcolor{green}{$\downarrow$2.6}} & 19.6{\tiny\textcolor{green}{$\downarrow$3.7}} & 32.7{\tiny\textcolor{green}{$\downarrow$1.1}} \\
			& \texttt{DyLAN} & 83.2{\tiny\textcolor{red}{$\uparrow$12.2}} & 12.4{\tiny\textcolor{red}{$\uparrow$3.6}} & 47.0{\tiny\textcolor{red}{$\uparrow$9.3}} & 38.1{\tiny\textcolor{red}{$\uparrow$6.1}} & 31.9{\tiny\textcolor{green}{$\downarrow$1.1}} & 30.2{\tiny\textcolor{green}{$\downarrow$3.3}} & 20.0{\tiny\textcolor{green}{$\downarrow$3.3}} & 31.2{\tiny\textcolor{green}{$\downarrow$2.6}} \\
			& \texttt{MacNet} & 86.8{\tiny\textcolor{red}{$\uparrow$15.8}} & 16.0{\tiny\textcolor{red}{$\uparrow$7.2}} & 59.7{\tiny\textcolor{red}{$\uparrow$22.0}} & 49.1{\tiny\textcolor{red}{$\uparrow$17.1}} & 34.8{\tiny\textcolor{red}{$\uparrow$1.8}} & 37.8{\tiny\textcolor{red}{$\uparrow$4.3}} & 26.7{\tiny\textcolor{red}{$\uparrow$3.4}} & 39.9{\tiny\textcolor{red}{$\uparrow$6.1}} \\
			& \texttt{MacNet} & 86.8{\tiny\textcolor{red}{$\uparrow$15.8}} & 16.0{\tiny\textcolor{red}{$\uparrow$7.2}} & 59.7{\tiny\textcolor{red}{$\uparrow$22.0}} & 49.1{\tiny\textcolor{red}{$\uparrow$17.1}} & 34.8{\tiny\textcolor{red}{$\uparrow$1.8}} & 37.8{\tiny\textcolor{red}{$\uparrow$4.3}} & 26.7{\tiny\textcolor{red}{$\uparrow$3.4}} & 39.9{\tiny\textcolor{red}{$\uparrow$6.1}} \\
			& {\alg{} (ours)} & \cellcolor{blue!10}94.5{\tiny\textcolor{red}{$\uparrow$23.5}} & \cellcolor{blue!10}32.7{\tiny\textcolor{red}{$\uparrow$23.9}} & \cellcolor{blue!10}88.5{\tiny\textcolor{red}{$\uparrow$50.8}} & \cellcolor{blue!10}66.1{\tiny\textcolor{red}{$\uparrow$34.1}} & \cellcolor{blue!10}48.7{\tiny\textcolor{red}{$\uparrow$15.7}} & \cellcolor{blue!10}47.1{\tiny\textcolor{red}{$\uparrow$13.6}} & \cellcolor{blue!10}36.2{\tiny\textcolor{red}{$\uparrow$12.9}} & \cellcolor{blue!10}45.5{\tiny\textcolor{red}{$\uparrow$11.7}} \\
			\midrule
			\multirow{7}{*}{Llama-70B}
			& \texttt{Single} & 68.9{\tiny\textcolor{green}{$\downarrow$14.1}} & 31.6{\tiny\textcolor{green}{$\downarrow$3.2}} & 27.6{\tiny\textcolor{green}{$\downarrow$32.0}} & 35.5{\tiny\textcolor{green}{$\downarrow$12.5}} & 43.7{\tiny\textcolor{green}{$\downarrow$10.2}} & 38.3{\tiny\textcolor{green}{$\downarrow$11.5}} & 28.6{\tiny\textcolor{green}{$\downarrow$4.6}} & 46.8{\tiny\textcolor{green}{$\downarrow$6.2}} \\
			& \texttt{BoN} & 83.0{\tiny\textcolor{red}{$\uparrow$0.0}} & 34.8{\tiny\textcolor{red}{$\uparrow$0.0}} & 59.6{\tiny\textcolor{red}{$\uparrow$0.0}} & 48.0{\tiny\textcolor{red}{$\uparrow$0.0}} & 53.9{\tiny\textcolor{red}{$\uparrow$0.0}} & 49.8{\tiny\textcolor{red}{$\uparrow$0.0}} & 33.2{\tiny\textcolor{red}{$\uparrow$0.0}} & 53.0{\tiny\textcolor{red}{$\uparrow$0.0}} \\
			& \texttt{Debate} & 82.9{\tiny\textcolor{green}{$\downarrow$0.1}} & 34.6{\tiny\textcolor{green}{$\downarrow$0.2}} & 58.5{\tiny\textcolor{green}{$\downarrow$1.1}} & 46.8{\tiny\textcolor{green}{$\downarrow$1.2}} & 52.2{\tiny\textcolor{green}{$\downarrow$1.7}} & 51.2{\tiny\textcolor{red}{$\uparrow$1.4}} & 33.0{\tiny\textcolor{green}{$\downarrow$0.2}} & 52.8{\tiny\textcolor{green}{$\downarrow$0.2}} \\
			& \texttt{DyLAN} & 93.2{\tiny\textcolor{red}{$\uparrow$10.2}} & 37.2{\tiny\textcolor{red}{$\uparrow$2.4}} & 73.2{\tiny\textcolor{red}{$\uparrow$13.6}} & 57.2{\tiny\textcolor{red}{$\uparrow$9.2}} & 50.5{\tiny\textcolor{green}{$\downarrow$3.4}} & 46.8{\tiny\textcolor{green}{$\downarrow$3.0}} & 31.1{\tiny\textcolor{green}{$\downarrow$2.1}} & 49.7{\tiny\textcolor{green}{$\downarrow$3.3}} \\
			& \texttt{GPTSwarm} & 93.0{\tiny\textcolor{red}{$\uparrow$10.0}} & 37.4{\tiny\textcolor{red}{$\uparrow$2.6}} & 69.7{\tiny\textcolor{red}{$\uparrow$10.1}} & 55.5{\tiny\textcolor{red}{$\uparrow$7.5}} & 51.7{\tiny\textcolor{green}{$\downarrow$2.2}} & 47.5{\tiny\textcolor{green}{$\downarrow$2.3}} & 32.4{\tiny\textcolor{green}{$\downarrow$0.8}} & 51.6{\tiny\textcolor{green}{$\downarrow$1.4}} \\
			& \texttt{MacNet} & 86.7{\tiny\textcolor{red}{$\uparrow$3.7}} & 38.9{\tiny\textcolor{red}{$\uparrow$4.1}} & 68.3{\tiny\textcolor{red}{$\uparrow$8.7}} & 49.8{\tiny\textcolor{red}{$\uparrow$1.8}} & 61.6{\tiny\textcolor{red}{$\uparrow$7.7}} & 70.8{\tiny\textcolor{red}{$\uparrow$21.0}} & 43.8{\tiny\textcolor{red}{$\uparrow$10.6}} & 67.1{\tiny\textcolor{red}{$\uparrow$14.1}} \\
			& {\alg{} (ours)} & \cellcolor{blue!10}93.9{\tiny\textcolor{red}{$\uparrow$10.9}} & \cellcolor{blue!10}51.5{\tiny\textcolor{red}{$\uparrow$16.7}} & \cellcolor{blue!10}91.3{\tiny\textcolor{red}{$\uparrow$31.7}} & \cellcolor{blue!10}97.2{\tiny\textcolor{red}{$\uparrow$49.2}} & \cellcolor{blue!10}84.9{\tiny\textcolor{red}{$\uparrow$31.0}} & \cellcolor{blue!10}81.7{\tiny\textcolor{red}{$\uparrow$31.9}} & \cellcolor{blue!10}61.5{\tiny\textcolor{red}{$\uparrow$28.3}} & \cellcolor{blue!10}78.1{\tiny\textcolor{red}{$\uparrow$25.1}} \\
			\midrule
			\multirow{7}{*}{Qwen-7B}
			& \texttt{Single} & 61.2{\tiny\textcolor{green}{$\downarrow$13.0}} & 12.9{\tiny\textcolor{green}{$\downarrow$12.0}} & 20.2{\tiny\textcolor{green}{$\downarrow$16.2}} & 51.2{\tiny\textcolor{green}{$\downarrow$11.3}} & 23.4{\tiny\textcolor{green}{$\downarrow$9.5}} & 24.9{\tiny\textcolor{green}{$\downarrow$10.5}} & 15.6{\tiny\textcolor{green}{$\downarrow$6.4}} & 20.6{\tiny\textcolor{green}{$\downarrow$6.5}} \\
			& \texttt{BoN} & 74.2{\tiny\textcolor{red}{$\uparrow$0.0}} & 24.9{\tiny\textcolor{red}{$\uparrow$0.0}} & 36.4{\tiny\textcolor{red}{$\uparrow$0.0}} & 62.5{\tiny\textcolor{red}{$\uparrow$0.0}} & 32.9{\tiny\textcolor{red}{$\uparrow$0.0}} & 35.4{\tiny\textcolor{red}{$\uparrow$0.0}} & 22.0{\tiny\textcolor{red}{$\uparrow$0.0}} & 27.1{\tiny\textcolor{red}{$\uparrow$0.0}} \\
			& \texttt{Debate} & 73.4{\tiny\textcolor{green}{$\downarrow$0.8}} & 24.1{\tiny\textcolor{green}{$\downarrow$0.8}} & 35.6{\tiny\textcolor{green}{$\downarrow$0.8}} & 62.1{\tiny\textcolor{green}{$\downarrow$0.4}} & 31.7{\tiny\textcolor{green}{$\downarrow$1.2}} & 32.6{\tiny\textcolor{green}{$\downarrow$2.8}} & 22.1{\tiny\textcolor{red}{$\uparrow$0.1}} & 24.2{\tiny\textcolor{green}{$\downarrow$2.9}} \\
			& \texttt{DyLAN} & 77.4{\tiny\textcolor{red}{$\uparrow$3.2}} & 26.0{\tiny\textcolor{red}{$\uparrow$1.1}} & 38.9{\tiny\textcolor{red}{$\uparrow$2.5}} & 64.8{\tiny\textcolor{red}{$\uparrow$2.3}} & 30.4{\tiny\textcolor{green}{$\downarrow$2.5}} & 28.7{\tiny\textcolor{green}{$\downarrow$6.7}} & 19.2{\tiny\textcolor{green}{$\downarrow$2.8}} & 22.7{\tiny\textcolor{green}{$\downarrow$4.4}} \\
			& \texttt{GPTSwarm} & 78.0{\tiny\textcolor{red}{$\uparrow$3.8}} & 27.6{\tiny\textcolor{red}{$\uparrow$2.7}} & 41.1{\tiny\textcolor{red}{$\uparrow$4.7}} & 66.2{\tiny\textcolor{red}{$\uparrow$3.7}} & 30.1{\tiny\textcolor{green}{$\downarrow$2.8}} & 30.4{\tiny\textcolor{green}{$\downarrow$5.0}} & 22.1{\tiny\textcolor{red}{$\uparrow$0.1}} & 24.4{\tiny\textcolor{green}{$\downarrow$2.7}} \\
			& \texttt{MacNet} & 76.8{\tiny\textcolor{red}{$\uparrow$2.6}} & 27.8{\tiny\textcolor{red}{$\uparrow$2.9}} & 40.3{\tiny\textcolor{red}{$\uparrow$3.9}} & 69.4{\tiny\textcolor{red}{$\uparrow$6.9}} & 35.4{\tiny\textcolor{red}{$\uparrow$2.5}} & 39.7{\tiny\textcolor{red}{$\uparrow$4.3}} & 24.6{\tiny\textcolor{red}{$\uparrow$2.6}} & 30.7{\tiny\textcolor{red}{$\uparrow$3.6}} \\
			& {\alg{} (ours)} & \cellcolor{blue!10}94.6{\tiny\textcolor{red}{$\uparrow$20.4}} & \cellcolor{blue!10}47.4{\tiny\textcolor{red}{$\uparrow$22.5}} & \cellcolor{blue!10}72.8{\tiny\textcolor{red}{$\uparrow$36.4}} & \cellcolor{blue!10}84.9{\tiny\textcolor{red}{$\uparrow$22.4}} & \cellcolor{blue!10}41.8{\tiny\textcolor{red}{$\uparrow$8.9}} & \cellcolor{blue!10}47.6{\tiny\textcolor{red}{$\uparrow$12.2}} & \cellcolor{blue!10}36.5{\tiny\textcolor{red}{$\uparrow$14.5}} & \cellcolor{blue!10}37.4{\tiny\textcolor{red}{$\uparrow$10.3}} \\
			\midrule
			\multirow{7}{*}{Qwen-14B}
			& \texttt{Single} & 67.2{\tiny\textcolor{green}{$\downarrow$12.5}} & 21.6{\tiny\textcolor{green}{$\downarrow$6.2}} & 21.2{\tiny\textcolor{green}{$\downarrow$11.6}} & 56.7{\tiny\textcolor{green}{$\downarrow$12.7}} & 28.3{\tiny\textcolor{green}{$\downarrow$9.8}} & 30.9{\tiny\textcolor{green}{$\downarrow$11.4}} & 18.5{\tiny\textcolor{green}{$\downarrow$7.6}} & 25.5{\tiny\textcolor{green}{$\downarrow$3.8}} \\
			& \texttt{BoN} & 79.7{\tiny\textcolor{red}{$\uparrow$0.0}} & 27.8{\tiny\textcolor{red}{$\uparrow$0.0}} & 32.8{\tiny\textcolor{red}{$\uparrow$0.0}} & 69.4{\tiny\textcolor{red}{$\uparrow$0.0}} & 38.1{\tiny\textcolor{red}{$\uparrow$0.0}} & 42.3{\tiny\textcolor{red}{$\uparrow$0.0}} & 26.1{\tiny\textcolor{red}{$\uparrow$0.0}} & 29.3{\tiny\textcolor{red}{$\uparrow$0.0}} \\
			& \texttt{Debate} & 78.7{\tiny\textcolor{green}{$\downarrow$1.0}} & 26.7{\tiny\textcolor{green}{$\downarrow$1.1}} & 33.4{\tiny\textcolor{red}{$\uparrow$0.6}} & 70.1{\tiny\textcolor{red}{$\uparrow$0.7}} & 37.7{\tiny\textcolor{green}{$\downarrow$0.4}} & 40.5{\tiny\textcolor{green}{$\downarrow$1.8}} & 26.7{\tiny\textcolor{red}{$\uparrow$0.6}} & 30.4{\tiny\textcolor{red}{$\uparrow$1.1}} \\
			& \texttt{DyLAN} & 88.0{\tiny\textcolor{red}{$\uparrow$8.3}} & 31.9{\tiny\textcolor{red}{$\uparrow$4.1}} & 41.0{\tiny\textcolor{red}{$\uparrow$8.2}} & 78.2{\tiny\textcolor{red}{$\uparrow$8.8}} & 34.9{\tiny\textcolor{green}{$\downarrow$3.2}} & 35.7{\tiny\textcolor{green}{$\downarrow$6.6}} & 23.9{\tiny\textcolor{green}{$\downarrow$2.2}} & 28.3{\tiny\textcolor{green}{$\downarrow$1.0}} \\
			& \texttt{GPTSwarm} & 88.0{\tiny\textcolor{red}{$\uparrow$8.3}} & 31.8{\tiny\textcolor{red}{$\uparrow$4.0}} & 40.4{\tiny\textcolor{red}{$\uparrow$7.6}} & 78.4{\tiny\textcolor{red}{$\uparrow$9.0}} & 35.5{\tiny\textcolor{green}{$\downarrow$2.6}} & 36.8{\tiny\textcolor{green}{$\downarrow$5.5}} & 25.1{\tiny\textcolor{green}{$\downarrow$1.0}} & 28.7{\tiny\textcolor{green}{$\downarrow$0.6}} \\
			& \texttt{MacNet} & 83.1{\tiny\textcolor{red}{$\uparrow$3.4}} & 30.8{\tiny\textcolor{red}{$\uparrow$3.0}} & 35.1{\tiny\textcolor{red}{$\uparrow$2.3}} & 74.8{\tiny\textcolor{red}{$\uparrow$5.4}} & 42.1{\tiny\textcolor{red}{$\uparrow$4.0}} & 49.3{\tiny\textcolor{red}{$\uparrow$7.0}} & 35.5{\tiny\textcolor{red}{$\uparrow$9.4}} & 33.8{\tiny\textcolor{red}{$\uparrow$4.5}} \\
			& {\alg{} (ours)} & \cellcolor{blue!10}95.9{\tiny\textcolor{red}{$\uparrow$16.2}} & \cellcolor{blue!10}52.7{\tiny\textcolor{red}{$\uparrow$24.9}} & \cellcolor{blue!10}69.0{\tiny\textcolor{red}{$\uparrow$36.2}} & \cellcolor{blue!10}93.3{\tiny\textcolor{red}{$\uparrow$23.9}} & \cellcolor{blue!10}51.2{\tiny\textcolor{red}{$\uparrow$13.1}} & \cellcolor{blue!10}59.4{\tiny\textcolor{red}{$\uparrow$17.1}} & \cellcolor{blue!10}43.9{\tiny\textcolor{red}{$\uparrow$17.8}} & \cellcolor{blue!10}46.0{\tiny\textcolor{red}{$\uparrow$16.7}} \\
			\midrule
			\multirow{7}{*}{Qwen-72B}
			& \texttt{Single} & 72.5{\tiny\textcolor{green}{$\downarrow$11.7}} & 31.6{\tiny\textcolor{green}{$\downarrow$19.4}} & 34.9{\tiny\textcolor{green}{$\downarrow$11.0}} & 59.1{\tiny\textcolor{green}{$\downarrow$13.1}} & 47.7{\tiny\textcolor{green}{$\downarrow$10.8}} & 50.1{\tiny\textcolor{green}{$\downarrow$12.8}} & 31.2{\tiny\textcolor{green}{$\downarrow$12.1}} & 40.8{\tiny\textcolor{green}{$\downarrow$11.5}} \\
			& \texttt{BoN} & 84.2{\tiny\textcolor{red}{$\uparrow$0.0}} & 51.0{\tiny\textcolor{red}{$\uparrow$0.0}} & 45.9{\tiny\textcolor{red}{$\uparrow$0.0}} & 72.2{\tiny\textcolor{red}{$\uparrow$0.0}} & 58.5{\tiny\textcolor{red}{$\uparrow$0.0}} & 62.9{\tiny\textcolor{red}{$\uparrow$0.0}} & 43.3{\tiny\textcolor{red}{$\uparrow$0.0}} & 52.3{\tiny\textcolor{red}{$\uparrow$0.0}} \\
			& \texttt{Debate} & 83.9{\tiny\textcolor{green}{$\downarrow$0.3}} & 56.1{\tiny\textcolor{red}{$\uparrow$5.1}} & 44.1{\tiny\textcolor{green}{$\downarrow$1.8}} & 71.6{\tiny\textcolor{green}{$\downarrow$0.6}} & 60.8{\tiny\textcolor{red}{$\uparrow$2.3}} & 66.7{\tiny\textcolor{red}{$\uparrow$3.8}} & 44.1{\tiny\textcolor{red}{$\uparrow$0.8}} & 49.3{\tiny\textcolor{green}{$\downarrow$3.0}} \\
			& \texttt{DyLAN} & 92.7{\tiny\textcolor{red}{$\uparrow$8.5}} & 65.8{\tiny\textcolor{red}{$\uparrow$14.8}} & 53.6{\tiny\textcolor{red}{$\uparrow$7.7}} & 81.4{\tiny\textcolor{red}{$\uparrow$9.2}} & 56.1{\tiny\textcolor{green}{$\downarrow$2.4}} & 59.1{\tiny\textcolor{green}{$\downarrow$3.8}} & 41.1{\tiny\textcolor{green}{$\downarrow$2.2}} & 46.3{\tiny\textcolor{green}{$\downarrow$6.0}} \\
			& \texttt{GPTSwarm} & 92.4{\tiny\textcolor{red}{$\uparrow$8.2}} & 66.0{\tiny\textcolor{red}{$\uparrow$15.0}} & 53.3{\tiny\textcolor{red}{$\uparrow$7.4}} & 81.1{\tiny\textcolor{red}{$\uparrow$8.9}} & 57.3{\tiny\textcolor{green}{$\downarrow$1.2}} & 59.8{\tiny\textcolor{green}{$\downarrow$3.1}} & 41.3{\tiny\textcolor{green}{$\downarrow$2.0}} & 50.6{\tiny\textcolor{green}{$\downarrow$1.7}} \\
			& \texttt{MacNet} & 88.1{\tiny\textcolor{red}{$\uparrow$3.9}} & 60.6{\tiny\textcolor{red}{$\uparrow$9.6}} & 49.2{\tiny\textcolor{red}{$\uparrow$3.3}} & 72.6{\tiny\textcolor{red}{$\uparrow$0.4}} & 65.8{\tiny\textcolor{red}{$\uparrow$7.3}} & 74.7{\tiny\textcolor{red}{$\uparrow$11.8}} & 50.3{\tiny\textcolor{red}{$\uparrow$7.0}} & 58.1{\tiny\textcolor{red}{$\uparrow$5.8}} \\
			& {\alg{} (ours)} & \cellcolor{blue!10}96.6{\tiny\textcolor{red}{$\uparrow$12.4}} & \cellcolor{blue!10}74.3{\tiny\textcolor{red}{$\uparrow$23.3}} & \cellcolor{blue!10}80.0{\tiny\textcolor{red}{$\uparrow$34.1}} & \cellcolor{blue!10}94.4{\tiny\textcolor{red}{$\uparrow$22.2}} & \cellcolor{blue!10}82.9{\tiny\textcolor{red}{$\uparrow$24.4}} & \cellcolor{blue!10}94.6{\tiny\textcolor{red}{$\uparrow$31.7}} & \cellcolor{blue!10}71.9{\tiny\textcolor{red}{$\uparrow$28.6}} & \cellcolor{blue!10}74.3{\tiny\textcolor{red}{$\uparrow$22.0}} \\
			\bottomrule
		\end{tabular}
	}
	\label{tab:number_16}
\end{table*}

\begin{table*}[h]
	\centering
	\caption{Accuracy with varying number of LLMs on different datasets. The number of LLMs is $32$ for all datasets. We exhibit the performance advantage with \texttt{BoN} and highlight the \colorbox{blue!10}{best} result.}
	\vskip 0.1in
	\resizebox{\textwidth}{!}{
		\begin{tabular}{YZYYYYYYYY}
			\toprule
			\textbf{Model} & \textbf{Method} & \textbf{MMLU} & \textbf{MATH} & \textbf{GPQA} & \textbf{Code} & \textbf{ALFWorld} & \textbf{SciWorld} & \textbf{GAIA} & \textbf{PDDL} \\
			\midrule
			\multirow{7}{*}{Llama-7B}
			& \texttt{Single} & 45.3{\tiny\textcolor{green}{$\downarrow$5.2}} & 2.9{\tiny\textcolor{green}{$\downarrow$4.0}} & 16.8{\tiny\textcolor{green}{$\downarrow$10.7}} & 15.8{\tiny\textcolor{green}{$\downarrow$6.0}} & 22.4{\tiny\textcolor{green}{$\downarrow$4.4}} & 18.8{\tiny\textcolor{green}{$\downarrow$8.1}} & 14.7{\tiny\textcolor{green}{$\downarrow$5.1}} & 23.2{\tiny\textcolor{green}{$\downarrow$4.0}} \\
			& \texttt{BoN} & 50.5{\tiny\textcolor{red}{$\uparrow$0.0}} & 6.9{\tiny\textcolor{red}{$\uparrow$0.0}} & 27.5{\tiny\textcolor{red}{$\uparrow$0.0}} & 21.8{\tiny\textcolor{red}{$\uparrow$0.0}} & 26.8{\tiny\textcolor{red}{$\uparrow$0.0}} & 26.9{\tiny\textcolor{red}{$\uparrow$0.0}} & 19.8{\tiny\textcolor{red}{$\uparrow$0.0}} & 27.2{\tiny\textcolor{red}{$\uparrow$0.0}} \\
			& \texttt{Debate} & 51.8{\tiny\textcolor{red}{$\uparrow$1.3}} & 7.7{\tiny\textcolor{red}{$\uparrow$0.8}} & 29.5{\tiny\textcolor{red}{$\uparrow$2.0}} & 23.5{\tiny\textcolor{red}{$\uparrow$1.7}} & 26.4{\tiny\textcolor{green}{$\downarrow$0.4}} & 25.6{\tiny\textcolor{green}{$\downarrow$1.3}} & 15.8{\tiny\textcolor{green}{$\downarrow$4.0}} & 26.2{\tiny\textcolor{green}{$\downarrow$1.0}} \\
			& \texttt{DyLAN} & 54.7{\tiny\textcolor{red}{$\uparrow$4.2}} & 10.1{\tiny\textcolor{red}{$\uparrow$3.2}} & 36.2{\tiny\textcolor{red}{$\uparrow$8.7}} & 26.7{\tiny\textcolor{red}{$\uparrow$4.9}} & 24.9{\tiny\textcolor{green}{$\downarrow$1.9}} & 23.1{\tiny\textcolor{green}{$\downarrow$3.8}} & 17.6{\tiny\textcolor{green}{$\downarrow$2.2}} & 26.0{\tiny\textcolor{green}{$\downarrow$1.2}} \\
			& \texttt{GPTSwarm} & 52.4{\tiny\textcolor{red}{$\uparrow$1.9}} & 7.9{\tiny\textcolor{red}{$\uparrow$1.0}} & 30.8{\tiny\textcolor{red}{$\uparrow$3.3}} & 23.7{\tiny\textcolor{red}{$\uparrow$1.9}} & 27.3{\tiny\textcolor{red}{$\uparrow$0.5}} & 27.9{\tiny\textcolor{red}{$\uparrow$1.0}} & 21.8{\tiny\textcolor{red}{$\uparrow$2.0}} & 30.4{\tiny\textcolor{red}{$\uparrow$3.2}} \\
			& \texttt{MacNet} & 62.7{\tiny\textcolor{red}{$\uparrow$12.2}} & 14.0{\tiny\textcolor{red}{$\uparrow$7.1}} & 50.3{\tiny\textcolor{red}{$\uparrow$22.8}} & 37.2{\tiny\textcolor{red}{$\uparrow$15.4}} & 35.7{\tiny\textcolor{red}{$\uparrow$8.9}} & 37.5{\tiny\textcolor{red}{$\uparrow$10.6}} & 23.8{\tiny\textcolor{red}{$\uparrow$4.0}} & 37.5{\tiny\textcolor{red}{$\uparrow$10.3}} \\
			& {\alg{} (ours)} & \cellcolor{blue!10}79.1{\tiny\textcolor{red}{$\uparrow$28.6}} & \cellcolor{blue!10}28.8{\tiny\textcolor{red}{$\uparrow$21.9}} & \cellcolor{blue!10}84.5{\tiny\textcolor{red}{$\uparrow$57.0}} & \cellcolor{blue!10}56.5{\tiny\textcolor{red}{$\uparrow$34.7}} & \cellcolor{blue!10}42.5{\tiny\textcolor{red}{$\uparrow$15.7}} & \cellcolor{blue!10}40.7{\tiny\textcolor{red}{$\uparrow$13.8}} & \cellcolor{blue!10}31.4{\tiny\textcolor{red}{$\uparrow$11.6}} & \cellcolor{blue!10}38.9{\tiny\textcolor{red}{$\uparrow$11.7}} \\
			\midrule
			\multirow{7}{*}{Llama-14B}
			& \texttt{Single} & 54.8{\tiny\textcolor{green}{$\downarrow$16.2}} & 3.9{\tiny\textcolor{green}{$\downarrow$4.9}} & 25.0{\tiny\textcolor{green}{$\downarrow$12.7}} & 23.7{\tiny\textcolor{green}{$\downarrow$8.3}} & 27.0{\tiny\textcolor{green}{$\downarrow$5.8}} & 23.5{\tiny\textcolor{green}{$\downarrow$11.3}} & 18.1{\tiny\textcolor{green}{$\downarrow$6.1}} & 28.9{\tiny\textcolor{green}{$\downarrow$5.4}} \\
			& \texttt{BoN} & 71.0{\tiny\textcolor{red}{$\uparrow$0.0}} & 8.8{\tiny\textcolor{red}{$\uparrow$0.0}} & 37.7{\tiny\textcolor{red}{$\uparrow$0.0}} & 32.0{\tiny\textcolor{red}{$\uparrow$0.0}} & 32.8{\tiny\textcolor{red}{$\uparrow$0.0}} & 34.8{\tiny\textcolor{red}{$\uparrow$0.0}} & 24.2{\tiny\textcolor{red}{$\uparrow$0.0}} & 34.3{\tiny\textcolor{red}{$\uparrow$0.0}} \\
			& \texttt{Debate} & 73.1{\tiny\textcolor{red}{$\uparrow$2.1}} & 9.5{\tiny\textcolor{red}{$\uparrow$0.7}} & 40.7{\tiny\textcolor{red}{$\uparrow$3.0}} & 34.3{\tiny\textcolor{red}{$\uparrow$2.3}} & 31.9{\tiny\textcolor{green}{$\downarrow$0.9}} & 31.7{\tiny\textcolor{green}{$\downarrow$3.1}} & 19.9{\tiny\textcolor{green}{$\downarrow$4.3}} & 33.0{\tiny\textcolor{green}{$\downarrow$1.3}} \\
			& \texttt{DyLAN} & 83.6{\tiny\textcolor{red}{$\uparrow$12.6}} & 12.7{\tiny\textcolor{red}{$\uparrow$3.9}} & 47.7{\tiny\textcolor{red}{$\uparrow$10.0}} & 38.7{\tiny\textcolor{red}{$\uparrow$6.7}} & 32.1{\tiny\textcolor{green}{$\downarrow$0.7}} & 30.8{\tiny\textcolor{green}{$\downarrow$4.0}} & 19.8{\tiny\textcolor{green}{$\downarrow$4.4}} & 31.0{\tiny\textcolor{green}{$\downarrow$3.3}} \\
			& \texttt{MacNet} & 90.7{\tiny\textcolor{red}{$\uparrow$19.7}} & 17.6{\tiny\textcolor{red}{$\uparrow$8.8}} & 67.2{\tiny\textcolor{red}{$\uparrow$29.5}} & 53.9{\tiny\textcolor{red}{$\uparrow$21.9}} & 37.1{\tiny\textcolor{red}{$\uparrow$4.3}} & 41.7{\tiny\textcolor{red}{$\uparrow$6.9}} & 29.1{\tiny\textcolor{red}{$\uparrow$4.9}} & 42.3{\tiny\textcolor{red}{$\uparrow$8.0}} \\
			& \texttt{MacNet} & 90.7{\tiny\textcolor{red}{$\uparrow$19.7}} & 17.6{\tiny\textcolor{red}{$\uparrow$8.8}} & 67.2{\tiny\textcolor{red}{$\uparrow$29.5}} & 53.9{\tiny\textcolor{red}{$\uparrow$21.9}} & 37.1{\tiny\textcolor{red}{$\uparrow$4.3}} & 41.7{\tiny\textcolor{red}{$\uparrow$6.9}} & 29.1{\tiny\textcolor{red}{$\uparrow$4.9}} & 42.3{\tiny\textcolor{red}{$\uparrow$8.0}} \\
			& {\alg{} (ours)} & \cellcolor{blue!10}95.8{\tiny\textcolor{red}{$\uparrow$24.8}} & \cellcolor{blue!10}41.4{\tiny\textcolor{red}{$\uparrow$32.6}} & \cellcolor{blue!10}94.7{\tiny\textcolor{red}{$\uparrow$57.0}} & \cellcolor{blue!10}75.6{\tiny\textcolor{red}{$\uparrow$43.6}} & \cellcolor{blue!10}52.2{\tiny\textcolor{red}{$\uparrow$19.4}} & \cellcolor{blue!10}50.7{\tiny\textcolor{red}{$\uparrow$15.9}} & \cellcolor{blue!10}38.5{\tiny\textcolor{red}{$\uparrow$14.3}} & \cellcolor{blue!10}48.0{\tiny\textcolor{red}{$\uparrow$13.7}} \\
			\midrule
			\multirow{7}{*}{Llama-70B}
			& \texttt{Single} & 68.9{\tiny\textcolor{green}{$\downarrow$14.1}} & 31.6{\tiny\textcolor{green}{$\downarrow$3.2}} & 27.6{\tiny\textcolor{green}{$\downarrow$32.0}} & 35.5{\tiny\textcolor{green}{$\downarrow$12.5}} & 43.1{\tiny\textcolor{green}{$\downarrow$10.0}} & 38.6{\tiny\textcolor{green}{$\downarrow$11.5}} & 28.1{\tiny\textcolor{green}{$\downarrow$4.7}} & 46.0{\tiny\textcolor{green}{$\downarrow$7.2}} \\
			& \texttt{BoN} & 83.0{\tiny\textcolor{red}{$\uparrow$0.0}} & 34.8{\tiny\textcolor{red}{$\uparrow$0.0}} & 59.6{\tiny\textcolor{red}{$\uparrow$0.0}} & 48.0{\tiny\textcolor{red}{$\uparrow$0.0}} & 53.1{\tiny\textcolor{red}{$\uparrow$0.0}} & 50.1{\tiny\textcolor{red}{$\uparrow$0.0}} & 32.8{\tiny\textcolor{red}{$\uparrow$0.0}} & 53.2{\tiny\textcolor{red}{$\uparrow$0.0}} \\
			& \texttt{Debate} & 85.7{\tiny\textcolor{red}{$\uparrow$2.7}} & 35.3{\tiny\textcolor{red}{$\uparrow$0.5}} & 63.4{\tiny\textcolor{red}{$\uparrow$3.8}} & 51.2{\tiny\textcolor{red}{$\uparrow$3.2}} & 52.6{\tiny\textcolor{green}{$\downarrow$0.5}} & 52.1{\tiny\textcolor{red}{$\uparrow$2.0}} & 32.9{\tiny\textcolor{red}{$\uparrow$0.1}} & 53.7{\tiny\textcolor{red}{$\uparrow$0.5}} \\
			& \texttt{DyLAN} & 94.5{\tiny\textcolor{red}{$\uparrow$11.5}} & 37.3{\tiny\textcolor{red}{$\uparrow$2.5}} & 74.6{\tiny\textcolor{red}{$\uparrow$15.0}} & 57.8{\tiny\textcolor{red}{$\uparrow$9.8}} & 50.1{\tiny\textcolor{green}{$\downarrow$3.0}} & 47.2{\tiny\textcolor{green}{$\downarrow$2.9}} & 31.4{\tiny\textcolor{green}{$\downarrow$1.4}} & 49.7{\tiny\textcolor{green}{$\downarrow$3.5}} \\
			& \texttt{GPTSwarm} & 96.4{\tiny\textcolor{red}{$\uparrow$13.4}} & 37.8{\tiny\textcolor{red}{$\uparrow$3.0}} & 73.4{\tiny\textcolor{red}{$\uparrow$13.8}} & 56.4{\tiny\textcolor{red}{$\uparrow$8.4}} & 52.2{\tiny\textcolor{green}{$\downarrow$0.9}} & 48.4{\tiny\textcolor{green}{$\downarrow$1.7}} & 32.5{\tiny\textcolor{green}{$\downarrow$0.3}} & 51.6{\tiny\textcolor{green}{$\downarrow$1.6}} \\
			& \texttt{MacNet} & 90.0{\tiny\textcolor{red}{$\uparrow$7.0}} & 41.2{\tiny\textcolor{red}{$\uparrow$6.4}} & 73.9{\tiny\textcolor{red}{$\uparrow$14.3}} & 54.1{\tiny\textcolor{red}{$\uparrow$6.1}} & 64.5{\tiny\textcolor{red}{$\uparrow$11.4}} & 76.6{\tiny\textcolor{red}{$\uparrow$26.5}} & 46.5{\tiny\textcolor{red}{$\uparrow$13.7}} & 71.6{\tiny\textcolor{red}{$\uparrow$18.4}} \\
			& {\alg{} (ours)} & \cellcolor{blue!10}97.0{\tiny\textcolor{red}{$\uparrow$14.0}} & \cellcolor{blue!10}54.5{\tiny\textcolor{red}{$\uparrow$19.7}} & \cellcolor{blue!10}95.1{\tiny\textcolor{red}{$\uparrow$35.5}} & \cellcolor{blue!10}93.7{\tiny\textcolor{red}{$\uparrow$45.7}} & \cellcolor{blue!10}88.5{\tiny\textcolor{red}{$\uparrow$35.4}} & \cellcolor{blue!10}86.0{\tiny\textcolor{red}{$\uparrow$35.9}} & \cellcolor{blue!10}65.7{\tiny\textcolor{red}{$\uparrow$32.9}} & \cellcolor{blue!10}79.8{\tiny\textcolor{red}{$\uparrow$26.6}} \\
			\midrule
			\multirow{7}{*}{Qwen-7B}
			& \texttt{Single} & 61.2{\tiny\textcolor{green}{$\downarrow$13.0}} & 12.9{\tiny\textcolor{green}{$\downarrow$12.0}} & 20.2{\tiny\textcolor{green}{$\downarrow$16.2}} & 51.2{\tiny\textcolor{green}{$\downarrow$11.3}} & 24.2{\tiny\textcolor{green}{$\downarrow$8.6}} & 24.5{\tiny\textcolor{green}{$\downarrow$13.3}} & 14.6{\tiny\textcolor{green}{$\downarrow$7.7}} & 20.4{\tiny\textcolor{green}{$\downarrow$7.3}} \\
			& \texttt{BoN} & 74.2{\tiny\textcolor{red}{$\uparrow$0.0}} & 24.9{\tiny\textcolor{red}{$\uparrow$0.0}} & 36.4{\tiny\textcolor{red}{$\uparrow$0.0}} & 62.5{\tiny\textcolor{red}{$\uparrow$0.0}} & 32.8{\tiny\textcolor{red}{$\uparrow$0.0}} & 37.8{\tiny\textcolor{red}{$\uparrow$0.0}} & 22.3{\tiny\textcolor{red}{$\uparrow$0.0}} & 27.7{\tiny\textcolor{red}{$\uparrow$0.0}} \\
			& \texttt{Debate} & 76.8{\tiny\textcolor{red}{$\uparrow$2.6}} & 26.8{\tiny\textcolor{red}{$\uparrow$1.9}} & 40.0{\tiny\textcolor{red}{$\uparrow$3.6}} & 64.3{\tiny\textcolor{red}{$\uparrow$1.8}} & 31.7{\tiny\textcolor{green}{$\downarrow$1.1}} & 33.4{\tiny\textcolor{green}{$\downarrow$4.4}} & 23.0{\tiny\textcolor{red}{$\uparrow$0.7}} & 25.3{\tiny\textcolor{green}{$\downarrow$2.4}} \\
			& \texttt{DyLAN} & 78.2{\tiny\textcolor{red}{$\uparrow$4.0}} & 29.5{\tiny\textcolor{red}{$\uparrow$4.6}} & 40.3{\tiny\textcolor{red}{$\uparrow$3.9}} & 66.5{\tiny\textcolor{red}{$\uparrow$4.0}} & 30.7{\tiny\textcolor{green}{$\downarrow$2.1}} & 29.2{\tiny\textcolor{green}{$\downarrow$8.6}} & 19.0{\tiny\textcolor{green}{$\downarrow$3.3}} & 22.6{\tiny\textcolor{green}{$\downarrow$5.1}} \\
			& \texttt{GPTSwarm} & 79.6{\tiny\textcolor{red}{$\uparrow$5.4}} & 29.9{\tiny\textcolor{red}{$\uparrow$5.0}} & 41.8{\tiny\textcolor{red}{$\uparrow$5.4}} & 67.8{\tiny\textcolor{red}{$\uparrow$5.3}} & 31.2{\tiny\textcolor{green}{$\downarrow$1.6}} & 31.3{\tiny\textcolor{green}{$\downarrow$6.5}} & 22.6{\tiny\textcolor{red}{$\uparrow$0.3}} & 25.0{\tiny\textcolor{green}{$\downarrow$2.7}} \\
			& \texttt{MacNet} & 80.6{\tiny\textcolor{red}{$\uparrow$6.4}} & 31.5{\tiny\textcolor{red}{$\uparrow$6.6}} & 45.8{\tiny\textcolor{red}{$\uparrow$9.4}} & 73.7{\tiny\textcolor{red}{$\uparrow$11.2}} & 38.6{\tiny\textcolor{red}{$\uparrow$5.8}} & 43.1{\tiny\textcolor{red}{$\uparrow$5.3}} & 26.6{\tiny\textcolor{red}{$\uparrow$4.3}} & 34.3{\tiny\textcolor{red}{$\uparrow$6.6}} \\
			& {\alg{} (ours)} & \cellcolor{blue!10}96.0{\tiny\textcolor{red}{$\uparrow$21.8}} & \cellcolor{blue!10}47.1{\tiny\textcolor{red}{$\uparrow$22.2}} & \cellcolor{blue!10}74.9{\tiny\textcolor{red}{$\uparrow$38.5}} & \cellcolor{blue!10}87.3{\tiny\textcolor{red}{$\uparrow$24.8}} & \cellcolor{blue!10}44.4{\tiny\textcolor{red}{$\uparrow$11.6}} & \cellcolor{blue!10}50.4{\tiny\textcolor{red}{$\uparrow$12.6}} & \cellcolor{blue!10}39.7{\tiny\textcolor{red}{$\uparrow$17.4}} & \cellcolor{blue!10}40.2{\tiny\textcolor{red}{$\uparrow$12.5}} \\
			\midrule
			\multirow{7}{*}{Qwen-14B}
			& \texttt{Single} & 67.2{\tiny\textcolor{green}{$\downarrow$12.5}} & 21.6{\tiny\textcolor{green}{$\downarrow$6.2}} & 21.2{\tiny\textcolor{green}{$\downarrow$11.6}} & 56.7{\tiny\textcolor{green}{$\downarrow$12.7}} & 29.4{\tiny\textcolor{green}{$\downarrow$8.0}} & 29.5{\tiny\textcolor{green}{$\downarrow$15.6}} & 18.9{\tiny\textcolor{green}{$\downarrow$8.2}} & 24.5{\tiny\textcolor{green}{$\downarrow$4.5}} \\
			& \texttt{BoN} & 79.7{\tiny\textcolor{red}{$\uparrow$0.0}} & 27.8{\tiny\textcolor{red}{$\uparrow$0.0}} & 32.8{\tiny\textcolor{red}{$\uparrow$0.0}} & 69.4{\tiny\textcolor{red}{$\uparrow$0.0}} & 37.4{\tiny\textcolor{red}{$\uparrow$0.0}} & 45.1{\tiny\textcolor{red}{$\uparrow$0.0}} & 27.1{\tiny\textcolor{red}{$\uparrow$0.0}} & 29.0{\tiny\textcolor{red}{$\uparrow$0.0}} \\
			& \texttt{Debate} & 82.7{\tiny\textcolor{red}{$\uparrow$3.0}} & 29.1{\tiny\textcolor{red}{$\uparrow$1.3}} & 35.5{\tiny\textcolor{red}{$\uparrow$2.7}} & 71.8{\tiny\textcolor{red}{$\uparrow$2.4}} & 39.0{\tiny\textcolor{red}{$\uparrow$1.6}} & 41.7{\tiny\textcolor{green}{$\downarrow$3.4}} & 28.7{\tiny\textcolor{red}{$\uparrow$1.6}} & 30.8{\tiny\textcolor{red}{$\uparrow$1.8}} \\
			& \texttt{DyLAN} & 89.3{\tiny\textcolor{red}{$\uparrow$9.6}} & 32.6{\tiny\textcolor{red}{$\uparrow$4.8}} & 42.2{\tiny\textcolor{red}{$\uparrow$9.4}} & 78.9{\tiny\textcolor{red}{$\uparrow$9.5}} & 35.4{\tiny\textcolor{green}{$\downarrow$2.0}} & 36.2{\tiny\textcolor{green}{$\downarrow$8.9}} & 23.2{\tiny\textcolor{green}{$\downarrow$3.9}} & 29.3{\tiny\textcolor{red}{$\uparrow$0.3}} \\
			& \texttt{GPTSwarm} & 89.0{\tiny\textcolor{red}{$\uparrow$9.3}} & 32.6{\tiny\textcolor{red}{$\uparrow$4.8}} & 42.2{\tiny\textcolor{red}{$\uparrow$9.4}} & 78.9{\tiny\textcolor{red}{$\uparrow$9.5}} & 36.4{\tiny\textcolor{green}{$\downarrow$1.0}} & 37.5{\tiny\textcolor{green}{$\downarrow$7.6}} & 26.6{\tiny\textcolor{green}{$\downarrow$0.5}} & 28.9{\tiny\textcolor{green}{$\downarrow$0.1}} \\
			& \texttt{MacNet} & 87.8{\tiny\textcolor{red}{$\uparrow$8.1}} & 33.7{\tiny\textcolor{red}{$\uparrow$5.9}} & 40.1{\tiny\textcolor{red}{$\uparrow$7.3}} & 83.0{\tiny\textcolor{red}{$\uparrow$13.6}} & 44.8{\tiny\textcolor{red}{$\uparrow$7.4}} & 54.2{\tiny\textcolor{red}{$\uparrow$9.1}} & 39.8{\tiny\textcolor{red}{$\uparrow$12.7}} & 36.2{\tiny\textcolor{red}{$\uparrow$7.2}} \\
			& {\alg{} (ours)} & \cellcolor{blue!10}96.1{\tiny\textcolor{red}{$\uparrow$16.4}} & \cellcolor{blue!10}58.0{\tiny\textcolor{red}{$\uparrow$30.2}} & \cellcolor{blue!10}70.5{\tiny\textcolor{red}{$\uparrow$37.7}} & \cellcolor{blue!10}94.9{\tiny\textcolor{red}{$\uparrow$25.5}} & \cellcolor{blue!10}54.9{\tiny\textcolor{red}{$\uparrow$17.5}} & \cellcolor{blue!10}62.1{\tiny\textcolor{red}{$\uparrow$17.0}} & \cellcolor{blue!10}49.0{\tiny\textcolor{red}{$\uparrow$21.9}} & \cellcolor{blue!10}49.6{\tiny\textcolor{red}{$\uparrow$20.6}} \\
			\midrule
			\multirow{7}{*}{Qwen-72B}
			& \texttt{Single} & 72.5{\tiny\textcolor{green}{$\downarrow$11.7}} & 31.6{\tiny\textcolor{green}{$\downarrow$19.4}} & 34.9{\tiny\textcolor{green}{$\downarrow$11.0}} & 59.1{\tiny\textcolor{green}{$\downarrow$13.1}} & 46.7{\tiny\textcolor{green}{$\downarrow$12.6}} & 49.0{\tiny\textcolor{green}{$\downarrow$14.2}} & 30.4{\tiny\textcolor{green}{$\downarrow$13.5}} & 40.1{\tiny\textcolor{green}{$\downarrow$13.2}} \\
			& \texttt{BoN} & 84.2{\tiny\textcolor{red}{$\uparrow$0.0}} & 51.0{\tiny\textcolor{red}{$\uparrow$0.0}} & 45.9{\tiny\textcolor{red}{$\uparrow$0.0}} & 72.2{\tiny\textcolor{red}{$\uparrow$0.0}} & 59.3{\tiny\textcolor{red}{$\uparrow$0.0}} & 63.2{\tiny\textcolor{red}{$\uparrow$0.0}} & 43.9{\tiny\textcolor{red}{$\uparrow$0.0}} & 53.3{\tiny\textcolor{red}{$\uparrow$0.0}} \\
			& \texttt{Debate} & 86.5{\tiny\textcolor{red}{$\uparrow$2.3}} & 62.9{\tiny\textcolor{red}{$\uparrow$11.9}} & 48.3{\tiny\textcolor{red}{$\uparrow$2.4}} & 74.5{\tiny\textcolor{red}{$\uparrow$2.3}} & 62.8{\tiny\textcolor{red}{$\uparrow$3.5}} & 67.8{\tiny\textcolor{red}{$\uparrow$4.6}} & 46.6{\tiny\textcolor{red}{$\uparrow$2.7}} & 50.2{\tiny\textcolor{green}{$\downarrow$3.1}} \\
			& \texttt{DyLAN} & 92.8{\tiny\textcolor{red}{$\uparrow$8.6}} & 67.7{\tiny\textcolor{red}{$\uparrow$16.7}} & 54.8{\tiny\textcolor{red}{$\uparrow$8.9}} & 82.7{\tiny\textcolor{red}{$\uparrow$10.5}} & 56.0{\tiny\textcolor{green}{$\downarrow$3.3}} & 59.1{\tiny\textcolor{green}{$\downarrow$4.1}} & 42.9{\tiny\textcolor{green}{$\downarrow$1.0}} & 46.7{\tiny\textcolor{green}{$\downarrow$6.6}} \\
			& \texttt{GPTSwarm} & 93.1{\tiny\textcolor{red}{$\uparrow$8.9}} & 68.1{\tiny\textcolor{red}{$\uparrow$17.1}} & 54.7{\tiny\textcolor{red}{$\uparrow$8.8}} & 82.2{\tiny\textcolor{red}{$\uparrow$10.0}} & 58.3{\tiny\textcolor{green}{$\downarrow$1.0}} & 62.6{\tiny\textcolor{green}{$\downarrow$0.6}} & 42.7{\tiny\textcolor{green}{$\downarrow$1.2}} & 51.4{\tiny\textcolor{green}{$\downarrow$1.9}} \\
			& \texttt{MacNet} & 92.5{\tiny\textcolor{red}{$\uparrow$8.3}} & 72.1{\tiny\textcolor{red}{$\uparrow$21.1}} & 55.2{\tiny\textcolor{red}{$\uparrow$9.3}} & 79.9{\tiny\textcolor{red}{$\uparrow$7.7}} & 70.5{\tiny\textcolor{red}{$\uparrow$11.2}} & 79.5{\tiny\textcolor{red}{$\uparrow$16.3}} & 54.7{\tiny\textcolor{red}{$\uparrow$10.8}} & 62.3{\tiny\textcolor{red}{$\uparrow$9.0}} \\
			& {\alg{} (ours)} & \cellcolor{blue!10}97.5{\tiny\textcolor{red}{$\uparrow$13.3}} & \cellcolor{blue!10}74.7{\tiny\textcolor{red}{$\uparrow$23.7}} & \cellcolor{blue!10}82.0{\tiny\textcolor{red}{$\uparrow$36.1}} & \cellcolor{blue!10}95.9{\tiny\textcolor{red}{$\uparrow$23.7}} & \cellcolor{blue!10}89.3{\tiny\textcolor{red}{$\uparrow$30.0}} & \cellcolor{blue!10}95{\tiny\textcolor{red}{$\uparrow$31.8}} & \cellcolor{blue!10}79.7{\tiny\textcolor{red}{$\uparrow$35.8}} & \cellcolor{blue!10}80.0{\tiny\textcolor{red}{$\uparrow$26.7}} \\
			\bottomrule
		\end{tabular}
	}
	\label{tab:number_32}
\end{table*}

\begin{table*}[h]
	\centering
	\caption{Accuracy with varying number of LLMs on different datasets. The number of LLMs is $64$ for all datasets. We exhibit the performance advantage with \texttt{BoN} and highlight the \colorbox{blue!10}{best} result.}
	\vskip 0.1in
	\resizebox{\textwidth}{!}{
		\begin{tabular}{YZYYYYYYYY}
			\toprule
			\textbf{Model} & \textbf{Method} & \textbf{MMLU} & \textbf{MATH} & \textbf{GPQA} & \textbf{Code} & \textbf{ALFWorld} & \textbf{SciWorld} & \textbf{GAIA} & \textbf{PDDL} \\
			\midrule
			\multirow{7}{*}{Llama-7B}
			& \texttt{Single} & 45.3{\tiny\textcolor{green}{$\downarrow$5.2}} & 2.9{\tiny\textcolor{green}{$\downarrow$4.0}} & 16.8{\tiny\textcolor{green}{$\downarrow$10.7}} & 15.8{\tiny\textcolor{green}{$\downarrow$6.0}} & 22.0{\tiny\textcolor{green}{$\downarrow$5.2}} & 18.8{\tiny\textcolor{green}{$\downarrow$7.9}} & 15.0{\tiny\textcolor{green}{$\downarrow$4.8}} & 22.5{\tiny\textcolor{green}{$\downarrow$4.4}} \\
			& \texttt{BoN} & 50.5{\tiny\textcolor{red}{$\uparrow$0.0}} & 6.9{\tiny\textcolor{red}{$\uparrow$0.0}} & 27.5{\tiny\textcolor{red}{$\uparrow$0.0}} & 21.8{\tiny\textcolor{red}{$\uparrow$0.0}} & 27.2{\tiny\textcolor{red}{$\uparrow$0.0}} & 26.7{\tiny\textcolor{red}{$\uparrow$0.0}} & 19.8{\tiny\textcolor{red}{$\uparrow$0.0}} & 26.9{\tiny\textcolor{red}{$\uparrow$0.0}} \\
			& \texttt{Debate} & 53.4{\tiny\textcolor{red}{$\uparrow$2.9}} & 9.2{\tiny\textcolor{red}{$\uparrow$2.3}} & 33.5{\tiny\textcolor{red}{$\uparrow$6.0}} & 25{\tiny\textcolor{red}{$\uparrow$3.2}} & 26.9{\tiny\textcolor{green}{$\downarrow$0.3}} & 26.3{\tiny\textcolor{green}{$\downarrow$0.4}} & 16.5{\tiny\textcolor{green}{$\downarrow$3.3}} & 26.9{\tiny\textcolor{red}{$\uparrow$0.0}} \\
			& \texttt{DyLAN} & 55.1{\tiny\textcolor{red}{$\uparrow$4.6}} & 10.4{\tiny\textcolor{red}{$\uparrow$3.5}} & 36.9{\tiny\textcolor{red}{$\uparrow$9.4}} & 27{\tiny\textcolor{red}{$\uparrow$5.2}} & 24.6{\tiny\textcolor{green}{$\downarrow$2.6}} & 23.3{\tiny\textcolor{green}{$\downarrow$3.4}} & 17.9{\tiny\textcolor{green}{$\downarrow$1.9}} & 25.0{\tiny\textcolor{green}{$\downarrow$1.9}} \\
			& \texttt{GPTSwarm} & 53.5{\tiny\textcolor{red}{$\uparrow$3.0}} & 9{\tiny\textcolor{red}{$\uparrow$2.1}} & 34{\tiny\textcolor{red}{$\uparrow$6.5}} & 25.3{\tiny\textcolor{red}{$\uparrow$3.5}} & 27.0{\tiny\textcolor{green}{$\downarrow$0.2}} & 28.8{\tiny\textcolor{red}{$\uparrow$2.1}} & 22.1{\tiny\textcolor{red}{$\uparrow$2.3}} & 31.0{\tiny\textcolor{red}{$\uparrow$4.1}} \\
			& \texttt{MacNet} & 66.4{\tiny\textcolor{red}{$\uparrow$15.9}} & 15.3{\tiny\textcolor{red}{$\uparrow$8.4}} & 54.4{\tiny\textcolor{red}{$\uparrow$26.9}} & 40.0{\tiny\textcolor{red}{$\uparrow$18.2}} & 38.8{\tiny\textcolor{red}{$\uparrow$11.6}} & 39.6{\tiny\textcolor{red}{$\uparrow$12.9}} & 27.0{\tiny\textcolor{red}{$\uparrow$7.2}} & 38.4{\tiny\textcolor{red}{$\uparrow$11.5}} \\
			& {\alg{} (ours)} & \cellcolor{blue!10}81.5{\tiny\textcolor{red}{$\uparrow$31.0}} & \cellcolor{blue!10}35.4{\tiny\textcolor{red}{$\uparrow$28.5}} & \cellcolor{blue!10}82.5{\tiny\textcolor{red}{$\uparrow$55.0}} & \cellcolor{blue!10}60.6{\tiny\textcolor{red}{$\uparrow$38.8}} & \cellcolor{blue!10}44.2{\tiny\textcolor{red}{$\uparrow$17.0}} & \cellcolor{blue!10}42.9{\tiny\textcolor{red}{$\uparrow$16.2}} & \cellcolor{blue!10}32.9{\tiny\textcolor{red}{$\uparrow$13.1}} & \cellcolor{blue!10}40.9{\tiny\textcolor{red}{$\uparrow$14.0}} \\
			\midrule
			\multirow{7}{*}{Llama-14B}
			& \texttt{Single} & 54.8{\tiny\textcolor{green}{$\downarrow$16.2}} & 3.9{\tiny\textcolor{green}{$\downarrow$4.9}} & 25.0{\tiny\textcolor{green}{$\downarrow$12.7}} & 23.7{\tiny\textcolor{green}{$\downarrow$8.3}} & 26.9{\tiny\textcolor{green}{$\downarrow$5.7}} & 23.6{\tiny\textcolor{green}{$\downarrow$12.1}} & 17.8{\tiny\textcolor{green}{$\downarrow$6.7}} & 28.6{\tiny\textcolor{green}{$\downarrow$5.6}} \\
			& \texttt{BoN} & 71.0{\tiny\textcolor{red}{$\uparrow$0.0}} & 8.8{\tiny\textcolor{red}{$\uparrow$0.0}} & 37.7{\tiny\textcolor{red}{$\uparrow$0.0}} & 32.0{\tiny\textcolor{red}{$\uparrow$0.0}} & 32.6{\tiny\textcolor{red}{$\uparrow$0.0}} & 35.7{\tiny\textcolor{red}{$\uparrow$0.0}} & 24.5{\tiny\textcolor{red}{$\uparrow$0.0}} & 34.2{\tiny\textcolor{red}{$\uparrow$0.0}} \\
			& \texttt{Debate} & 79.6{\tiny\textcolor{red}{$\uparrow$8.6}} & 11.1{\tiny\textcolor{red}{$\uparrow$2.3}} & 45.0{\tiny\textcolor{red}{$\uparrow$7.3}} & 36.5{\tiny\textcolor{red}{$\uparrow$4.5}} & 32.6{\tiny\textcolor{red}{$\uparrow$0.0}} & 32.0{\tiny\textcolor{green}{$\downarrow$3.7}} & 19.9{\tiny\textcolor{green}{$\downarrow$4.6}} & 32.6{\tiny\textcolor{green}{$\downarrow$1.6}} \\
			& \texttt{DyLAN} & 85.0{\tiny\textcolor{red}{$\uparrow$14.0}} & 13.0{\tiny\textcolor{red}{$\uparrow$4.2}} & 48.8{\tiny\textcolor{red}{$\uparrow$11.1}} & 39.3{\tiny\textcolor{red}{$\uparrow$7.3}} & 31.9{\tiny\textcolor{green}{$\downarrow$0.7}} & 31.0{\tiny\textcolor{green}{$\downarrow$4.7}} & 19.9{\tiny\textcolor{green}{$\downarrow$4.6}} & 30.5{\tiny\textcolor{green}{$\downarrow$3.7}} \\
			& \texttt{MacNet} & 90.8{\tiny\textcolor{red}{$\uparrow$19.8}} & 18.1{\tiny\textcolor{red}{$\uparrow$9.3}} & 67.5{\tiny\textcolor{red}{$\uparrow$29.8}} & 55.2{\tiny\textcolor{red}{$\uparrow$23.2}} & 39.1{\tiny\textcolor{red}{$\uparrow$6.5}} & 44.4{\tiny\textcolor{red}{$\uparrow$8.7}} & 31.1{\tiny\textcolor{red}{$\uparrow$6.6}} & 44.0{\tiny\textcolor{red}{$\uparrow$9.8}} \\
			& \texttt{MacNet} & 90.8{\tiny\textcolor{red}{$\uparrow$19.8}} & 18.1{\tiny\textcolor{red}{$\uparrow$9.3}} & 67.5{\tiny\textcolor{red}{$\uparrow$29.8}} & 55.2{\tiny\textcolor{red}{$\uparrow$23.2}} & 39.1{\tiny\textcolor{red}{$\uparrow$6.5}} & 44.4{\tiny\textcolor{red}{$\uparrow$8.7}} & 31.1{\tiny\textcolor{red}{$\uparrow$6.6}} & 44.0{\tiny\textcolor{red}{$\uparrow$9.8}} \\
			& {\alg{} (ours)} & \cellcolor{blue!10}96.8{\tiny\textcolor{red}{$\uparrow$25.8}} & \cellcolor{blue!10}40.0{\tiny\textcolor{red}{$\uparrow$31.2}} & \cellcolor{blue!10}93.1{\tiny\textcolor{red}{$\uparrow$55.4}} & \cellcolor{blue!10}82.5{\tiny\textcolor{red}{$\uparrow$50.5}} & \cellcolor{blue!10}54.8{\tiny\textcolor{red}{$\uparrow$22.2}} & \cellcolor{blue!10}52.8{\tiny\textcolor{red}{$\uparrow$17.1}} & \cellcolor{blue!10}40.3{\tiny\textcolor{red}{$\uparrow$15.8}} & \cellcolor{blue!10}49.8{\tiny\textcolor{red}{$\uparrow$15.6}} \\
			\midrule
			\multirow{7}{*}{Llama-70B}
			& \texttt{Single} & 68.9{\tiny\textcolor{green}{$\downarrow$14.1}} & 31.6{\tiny\textcolor{green}{$\downarrow$3.2}} & 27.6{\tiny\textcolor{green}{$\downarrow$32.0}} & 35.5{\tiny\textcolor{green}{$\downarrow$12.5}} & 43.0{\tiny\textcolor{green}{$\downarrow$11.8}} & 37.3{\tiny\textcolor{green}{$\downarrow$13.7}} & 28.2{\tiny\textcolor{green}{$\downarrow$5.7}} & 45.9{\tiny\textcolor{green}{$\downarrow$6.8}} \\
			& \texttt{BoN} & 83.0{\tiny\textcolor{red}{$\uparrow$0.0}} & 34.8{\tiny\textcolor{red}{$\uparrow$0.0}} & 59.6{\tiny\textcolor{red}{$\uparrow$0.0}} & 48.0{\tiny\textcolor{red}{$\uparrow$0.0}} & 54.8{\tiny\textcolor{red}{$\uparrow$0.0}} & 51.0{\tiny\textcolor{red}{$\uparrow$0.0}} & 33.9{\tiny\textcolor{red}{$\uparrow$0.0}} & 52.7{\tiny\textcolor{red}{$\uparrow$0.0}} \\
			& \texttt{Debate} & 90.0{\tiny\textcolor{red}{$\uparrow$7.0}} & 36.6{\tiny\textcolor{red}{$\uparrow$1.8}} & 69.0{\tiny\textcolor{red}{$\uparrow$9.4}} & 54.0{\tiny\textcolor{red}{$\uparrow$6.0}} & 53.6{\tiny\textcolor{green}{$\downarrow$1.2}} & 53.5{\tiny\textcolor{red}{$\uparrow$2.5}} & 32.7{\tiny\textcolor{green}{$\downarrow$1.2}} & 53.2{\tiny\textcolor{red}{$\uparrow$0.5}} \\
			& \texttt{DyLAN} & 95.4{\tiny\textcolor{red}{$\uparrow$12.4}} & 37.6{\tiny\textcolor{red}{$\uparrow$2.8}} & 75.4{\tiny\textcolor{red}{$\uparrow$15.8}} & 58.9{\tiny\textcolor{red}{$\uparrow$10.9}} & 50.4{\tiny\textcolor{green}{$\downarrow$4.4}} & 47.6{\tiny\textcolor{green}{$\downarrow$3.4}} & 30.8{\tiny\textcolor{green}{$\downarrow$3.1}} & 50.4{\tiny\textcolor{green}{$\downarrow$2.3}} \\
			& \texttt{GPTSwarm} & 95.7{\tiny\textcolor{red}{$\uparrow$12.7}} & 37.9{\tiny\textcolor{red}{$\uparrow$3.1}} & 78.9{\tiny\textcolor{red}{$\uparrow$19.3}} & 58.1{\tiny\textcolor{red}{$\uparrow$10.1}} & 52.4{\tiny\textcolor{green}{$\downarrow$2.4}} & 49.3{\tiny\textcolor{green}{$\downarrow$1.7}} & 32.3{\tiny\textcolor{green}{$\downarrow$1.6}} & 51.3{\tiny\textcolor{green}{$\downarrow$1.4}} \\
			& \texttt{MacNet} & 95.3{\tiny\textcolor{red}{$\uparrow$12.3}} & 44.2{\tiny\textcolor{red}{$\uparrow$9.4}} & 84.2{\tiny\textcolor{red}{$\uparrow$24.6}} & 57.9{\tiny\textcolor{red}{$\uparrow$9.9}} & 67.9{\tiny\textcolor{red}{$\uparrow$13.1}} & 81.6{\tiny\textcolor{red}{$\uparrow$30.6}} & 49.0{\tiny\textcolor{red}{$\uparrow$15.1}} & 75.4{\tiny\textcolor{red}{$\uparrow$22.7}} \\
			& {\alg{} (ours)} & \cellcolor{blue!10}96.3{\tiny\textcolor{red}{$\uparrow$13.3}} & \cellcolor{blue!10}58.0{\tiny\textcolor{red}{$\uparrow$23.2}} & \cellcolor{blue!10}96.9{\tiny\textcolor{red}{$\uparrow$37.3}} & \cellcolor{blue!10}92.0{\tiny\textcolor{red}{$\uparrow$44.0}} & \cellcolor{blue!10}90.8{\tiny\textcolor{red}{$\uparrow$36.0}} & \cellcolor{blue!10}88.9{\tiny\textcolor{red}{$\uparrow$37.9}} & \cellcolor{blue!10}68.2{\tiny\textcolor{red}{$\uparrow$34.3}} & \cellcolor{blue!10}82.2{\tiny\textcolor{red}{$\uparrow$29.5}} \\
			\midrule
			\multirow{7}{*}{Qwen-7B}
			& \texttt{Single} & 61.2{\tiny\textcolor{green}{$\downarrow$13.0}} & 12.9{\tiny\textcolor{green}{$\downarrow$12.0}} & 20.2{\tiny\textcolor{green}{$\downarrow$16.2}} & 51.2{\tiny\textcolor{green}{$\downarrow$11.3}} & 23.0{\tiny\textcolor{green}{$\downarrow$11.5}} & 24.4{\tiny\textcolor{green}{$\downarrow$14.2}} & 14.8{\tiny\textcolor{green}{$\downarrow$8.2}} & 20.0{\tiny\textcolor{green}{$\downarrow$7.8}} \\
			& \texttt{BoN} & 74.2{\tiny\textcolor{red}{$\uparrow$0.0}} & 24.9{\tiny\textcolor{red}{$\uparrow$0.0}} & 36.4{\tiny\textcolor{red}{$\uparrow$0.0}} & 62.5{\tiny\textcolor{red}{$\uparrow$0.0}} & 34.5{\tiny\textcolor{red}{$\uparrow$0.0}} & 38.6{\tiny\textcolor{red}{$\uparrow$0.0}} & 23.0{\tiny\textcolor{red}{$\uparrow$0.0}} & 27.8{\tiny\textcolor{red}{$\uparrow$0.0}} \\
			& \texttt{Debate} & 78.1{\tiny\textcolor{red}{$\uparrow$3.9}} & 28.9{\tiny\textcolor{red}{$\uparrow$4.0}} & 41.7{\tiny\textcolor{red}{$\uparrow$5.3}} & 66.6{\tiny\textcolor{red}{$\uparrow$4.1}} & 32.3{\tiny\textcolor{green}{$\downarrow$2.2}} & 34.1{\tiny\textcolor{green}{$\downarrow$4.5}} & 23.5{\tiny\textcolor{red}{$\uparrow$0.5}} & 26.0{\tiny\textcolor{green}{$\downarrow$1.8}} \\
			& \texttt{DyLAN} & 79.8{\tiny\textcolor{red}{$\uparrow$5.6}} & 30.7{\tiny\textcolor{red}{$\uparrow$5.8}} & 43.5{\tiny\textcolor{red}{$\uparrow$7.1}} & 67.5{\tiny\textcolor{red}{$\uparrow$5.0}} & 30.7{\tiny\textcolor{green}{$\downarrow$3.8}} & 30.0{\tiny\textcolor{green}{$\downarrow$8.6}} & 19.8{\tiny\textcolor{green}{$\downarrow$3.2}} & 23.7{\tiny\textcolor{green}{$\downarrow$4.1}} \\
			& \texttt{GPTSwarm} & 80.9{\tiny\textcolor{red}{$\uparrow$6.7}} & 29.9{\tiny\textcolor{red}{$\uparrow$5.0}} & 43.1{\tiny\textcolor{red}{$\uparrow$6.7}} & 68.0{\tiny\textcolor{red}{$\uparrow$5.5}} & 31.8{\tiny\textcolor{green}{$\downarrow$2.7}} & 32.7{\tiny\textcolor{green}{$\downarrow$5.9}} & 23.2{\tiny\textcolor{red}{$\uparrow$0.2}} & 25.3{\tiny\textcolor{green}{$\downarrow$2.5}} \\
			& \texttt{MacNet} & 85.1{\tiny\textcolor{red}{$\uparrow$10.9}} & 34.7{\tiny\textcolor{red}{$\uparrow$9.8}} & 50.4{\tiny\textcolor{red}{$\uparrow$14.0}} & 79.8{\tiny\textcolor{red}{$\uparrow$17.3}} & 41.6{\tiny\textcolor{red}{$\uparrow$7.1}} & 46.7{\tiny\textcolor{red}{$\uparrow$8.1}} & 28.4{\tiny\textcolor{red}{$\uparrow$5.4}} & 37.1{\tiny\textcolor{red}{$\uparrow$9.3}} \\
			& {\alg{} (ours)} & \cellcolor{blue!10}97.5{\tiny\textcolor{red}{$\uparrow$23.3}} & \cellcolor{blue!10}50.7{\tiny\textcolor{red}{$\uparrow$25.8}} & \cellcolor{blue!10}80.6{\tiny\textcolor{red}{$\uparrow$44.2}} & \cellcolor{blue!10}90.2{\tiny\textcolor{red}{$\uparrow$27.7}} & \cellcolor{blue!10}46.8{\tiny\textcolor{red}{$\uparrow$12.3}} & \cellcolor{blue!10}53.3{\tiny\textcolor{red}{$\uparrow$14.7}} & \cellcolor{blue!10}42.0{\tiny\textcolor{red}{$\uparrow$19.0}} & \cellcolor{blue!10}42.2{\tiny\textcolor{red}{$\uparrow$14.4}} \\
			\midrule
			\multirow{7}{*}{Qwen-14B}
			& \texttt{Single} & 67.2{\tiny\textcolor{green}{$\downarrow$12.5}} & 21.6{\tiny\textcolor{green}{$\downarrow$6.2}} & 21.2{\tiny\textcolor{green}{$\downarrow$11.6}} & 56.7{\tiny\textcolor{green}{$\downarrow$12.7}} & 28.9{\tiny\textcolor{green}{$\downarrow$9.5}} & 29.9{\tiny\textcolor{green}{$\downarrow$15.7}} & 17.8{\tiny\textcolor{green}{$\downarrow$9.8}} & 24.7{\tiny\textcolor{green}{$\downarrow$4.8}} \\
			& \texttt{BoN} & 79.7{\tiny\textcolor{red}{$\uparrow$0.0}} & 27.8{\tiny\textcolor{red}{$\uparrow$0.0}} & 32.8{\tiny\textcolor{red}{$\uparrow$0.0}} & 69.4{\tiny\textcolor{red}{$\uparrow$0.0}} & 38.4{\tiny\textcolor{red}{$\uparrow$0.0}} & 45.6{\tiny\textcolor{red}{$\uparrow$0.0}} & 27.6{\tiny\textcolor{red}{$\uparrow$0.0}} & 29.5{\tiny\textcolor{red}{$\uparrow$0.0}} \\
			& \texttt{Debate} & 86.6{\tiny\textcolor{red}{$\uparrow$6.9}} & 31.1{\tiny\textcolor{red}{$\uparrow$3.3}} & 39.9{\tiny\textcolor{red}{$\uparrow$7.1}} & 75.8{\tiny\textcolor{red}{$\uparrow$6.4}} & 38.8{\tiny\textcolor{red}{$\uparrow$0.4}} & 42.7{\tiny\textcolor{green}{$\downarrow$2.9}} & 29.0{\tiny\textcolor{red}{$\uparrow$1.4}} & 31.7{\tiny\textcolor{red}{$\uparrow$2.2}} \\
			& \texttt{DyLAN} & 90.2{\tiny\textcolor{red}{$\uparrow$10.5}} & 33.6{\tiny\textcolor{red}{$\uparrow$5.8}} & 42.8{\tiny\textcolor{red}{$\uparrow$10.0}} & 80.1{\tiny\textcolor{red}{$\uparrow$10.7}} & 36.2{\tiny\textcolor{green}{$\downarrow$2.2}} & 37.0{\tiny\textcolor{green}{$\downarrow$8.6}} & 24.4{\tiny\textcolor{green}{$\downarrow$3.2}} & 28.7{\tiny\textcolor{green}{$\downarrow$0.8}} \\
			& \texttt{GPTSwarm} & 90.2{\tiny\textcolor{red}{$\uparrow$10.5}} & 33.0{\tiny\textcolor{red}{$\uparrow$5.2}} & 43.2{\tiny\textcolor{red}{$\uparrow$10.4}} & 80.4{\tiny\textcolor{red}{$\uparrow$11.0}} & 35.9{\tiny\textcolor{green}{$\downarrow$2.5}} & 36.0{\tiny\textcolor{green}{$\downarrow$9.6}} & 26.5{\tiny\textcolor{green}{$\downarrow$1.1}} & 29.0{\tiny\textcolor{green}{$\downarrow$0.5}} \\
			& \texttt{MacNet} & 93.7{\tiny\textcolor{red}{$\uparrow$14.0}} & 37.7{\tiny\textcolor{red}{$\uparrow$9.9}} & 46.7{\tiny\textcolor{red}{$\uparrow$13.9}} & 91.2{\tiny\textcolor{red}{$\uparrow$21.8}} & 47.5{\tiny\textcolor{red}{$\uparrow$9.1}} & 56.8{\tiny\textcolor{red}{$\uparrow$11.2}} & 42.7{\tiny\textcolor{red}{$\uparrow$15.1}} & 38.4{\tiny\textcolor{red}{$\uparrow$8.9}} \\
			& {\alg{} (ours)} & \cellcolor{blue!10}99.1{\tiny\textcolor{red}{$\uparrow$19.4}} & \cellcolor{blue!10}61.9{\tiny\textcolor{red}{$\uparrow$34.1}} & \cellcolor{blue!10}80.4{\tiny\textcolor{red}{$\uparrow$47.6}} & \cellcolor{blue!10}98.7{\tiny\textcolor{red}{$\uparrow$29.3}} & \cellcolor{blue!10}56.4{\tiny\textcolor{red}{$\uparrow$18.0}} & \cellcolor{blue!10}65.8{\tiny\textcolor{red}{$\uparrow$20.2}} & \cellcolor{blue!10}51.8{\tiny\textcolor{red}{$\uparrow$24.2}} & \cellcolor{blue!10}50.9{\tiny\textcolor{red}{$\uparrow$21.4}} \\
			\midrule
			\multirow{7}{*}{Qwen-72B}
			& \texttt{Single} & 72.5{\tiny\textcolor{green}{$\downarrow$11.7}} & 31.6{\tiny\textcolor{green}{$\downarrow$19.4}} & 34.9{\tiny\textcolor{green}{$\downarrow$11.0}} & 59.1{\tiny\textcolor{green}{$\downarrow$13.1}} & 46.7{\tiny\textcolor{green}{$\downarrow$13.4}} & 48.3{\tiny\textcolor{green}{$\downarrow$17.7}} & 29.8{\tiny\textcolor{green}{$\downarrow$15.1}} & 40.4{\tiny\textcolor{green}{$\downarrow$14.0}} \\
			& \texttt{BoN} & 84.2{\tiny\textcolor{red}{$\uparrow$0.0}} & 51.0{\tiny\textcolor{red}{$\uparrow$0.0}} & 45.9{\tiny\textcolor{red}{$\uparrow$0.0}} & 72.2{\tiny\textcolor{red}{$\uparrow$0.0}} & 60.1{\tiny\textcolor{red}{$\uparrow$0.0}} & 66.0{\tiny\textcolor{red}{$\uparrow$0.0}} & 44.9{\tiny\textcolor{red}{$\uparrow$0.0}} & 54.4{\tiny\textcolor{red}{$\uparrow$0.0}} \\
			& \texttt{Debate} & 91.2{\tiny\textcolor{red}{$\uparrow$7.0}} & 62.3{\tiny\textcolor{red}{$\uparrow$11.3}} & 52.4{\tiny\textcolor{red}{$\uparrow$6.5}} & 78.9{\tiny\textcolor{red}{$\uparrow$6.7}} & 62.5{\tiny\textcolor{red}{$\uparrow$2.4}} & 69.1{\tiny\textcolor{red}{$\uparrow$3.1}} & 47.6{\tiny\textcolor{red}{$\uparrow$2.7}} & 50.5{\tiny\textcolor{green}{$\downarrow$3.9}} \\
			& \texttt{DyLAN} & 95.1{\tiny\textcolor{red}{$\uparrow$10.9}} & 66.8{\tiny\textcolor{red}{$\uparrow$15.8}} & 55.7{\tiny\textcolor{red}{$\uparrow$9.8}} & 83.9{\tiny\textcolor{red}{$\uparrow$11.7}} & 56.9{\tiny\textcolor{green}{$\downarrow$3.2}} & 59.9{\tiny\textcolor{green}{$\downarrow$6.1}} & 43.7{\tiny\textcolor{green}{$\downarrow$1.2}} & 47.7{\tiny\textcolor{green}{$\downarrow$6.7}} \\
			& \texttt{GPTSwarm} & 94.3{\tiny\textcolor{red}{$\uparrow$10.1}} & 68.8{\tiny\textcolor{red}{$\uparrow$17.8}} & 55.2{\tiny\textcolor{red}{$\uparrow$9.3}} & 84.0{\tiny\textcolor{red}{$\uparrow$11.8}} & 59.6{\tiny\textcolor{green}{$\downarrow$0.5}} & 63.2{\tiny\textcolor{green}{$\downarrow$2.8}} & 43.0{\tiny\textcolor{green}{$\downarrow$1.9}} & 52.5{\tiny\textcolor{green}{$\downarrow$1.9}} \\
			& \texttt{MacNet} & 98.0{\tiny\textcolor{red}{$\uparrow$13.8}} & 73.9{\tiny\textcolor{red}{$\uparrow$22.9}} & 63.2{\tiny\textcolor{red}{$\uparrow$17.3}} & 86.7{\tiny\textcolor{red}{$\uparrow$14.5}} & 74.4{\tiny\textcolor{red}{$\uparrow$14.3}} & 85.1{\tiny\textcolor{red}{$\uparrow$19.1}} & 59.5{\tiny\textcolor{red}{$\uparrow$14.6}} & 66.6{\tiny\textcolor{red}{$\uparrow$12.2}} \\
			& {\alg{} (ours)} & \cellcolor{blue!10}99.7{\tiny\textcolor{red}{$\uparrow$15.5}} & \cellcolor{blue!10}79.7{\tiny\textcolor{red}{$\uparrow$28.7}} & \cellcolor{blue!10}86.0{\tiny\textcolor{red}{$\uparrow$40.1}} & \cellcolor{blue!10}97.5{\tiny\textcolor{red}{$\uparrow$25.3}} & \cellcolor{blue!10}92.9{\tiny\textcolor{red}{$\uparrow$32.8}} & \cellcolor{blue!10}95.9{\tiny\textcolor{red}{$\uparrow$29.9}} & \cellcolor{blue!10}84.2{\tiny\textcolor{red}{$\uparrow$39.3}} & \cellcolor{blue!10}83.9{\tiny\textcolor{red}{$\uparrow$29.5}} \\
			\bottomrule
		\end{tabular}
	}
	\label{tab:number_64}
\end{table*}

\begin{figure*}[h]
	\centering
	\includegraphics[width=\linewidth]{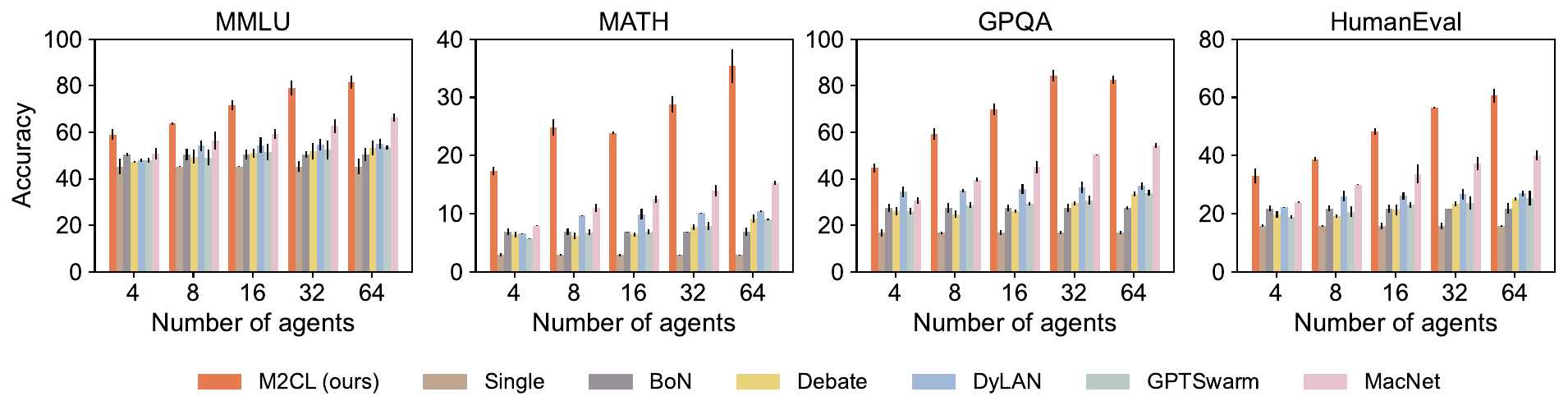}
	\includegraphics[width=\linewidth]{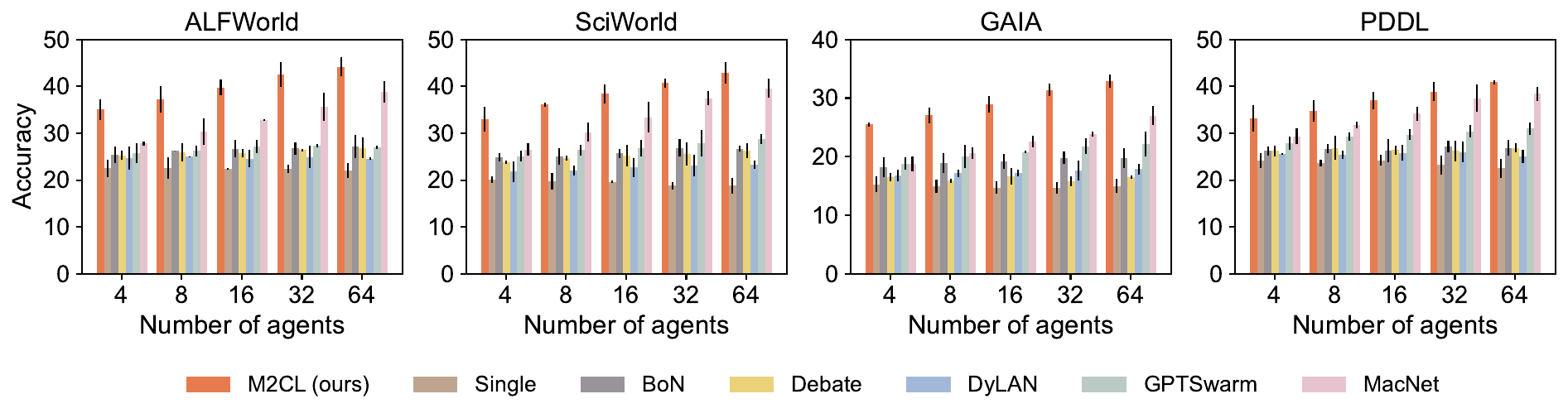}
	\caption{Performance of llama-7b as the base model with varying number of LLMs. Uncertainty intervals depict standard deviation over three seeds. \alg{} exhibits higher performance and increasing tendency with more LLMs, demonstrating its great collaboration efficiency compared to existing methods. Of note, academic and agentic tasks reasoning are challenging because they require more diverse thinking perspectives and more rigorous analysis. The outperformance of \alg{} reveals its capability of enabling LLMs to collaborate in changing discussion state.}
	\label{fig:number_agents_1}
\end{figure*}
\begin{figure*}[h]
	\centering
	\includegraphics[width=\linewidth]{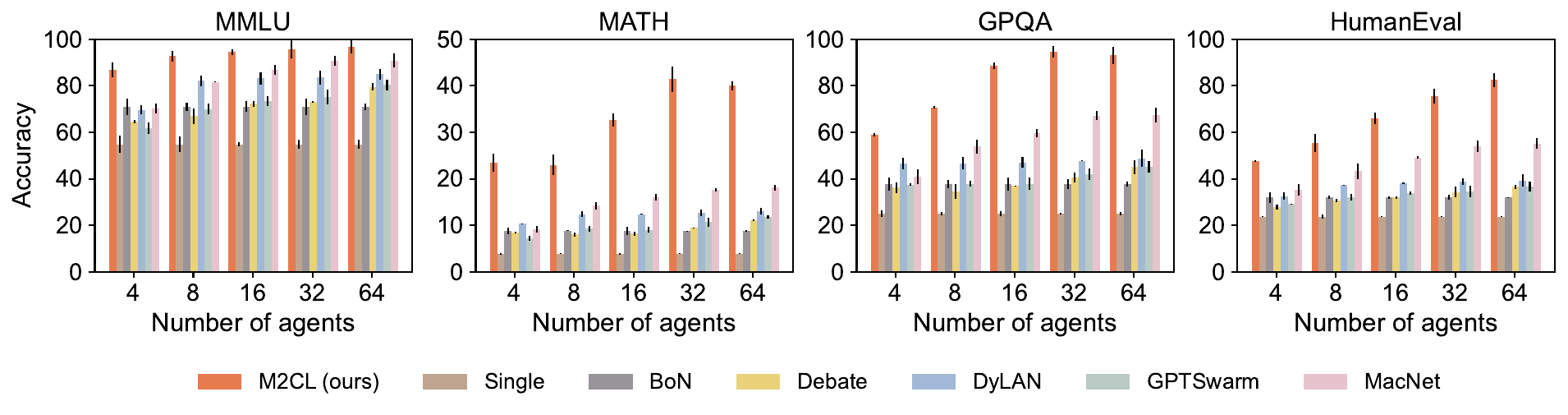}
	\includegraphics[width=\linewidth]{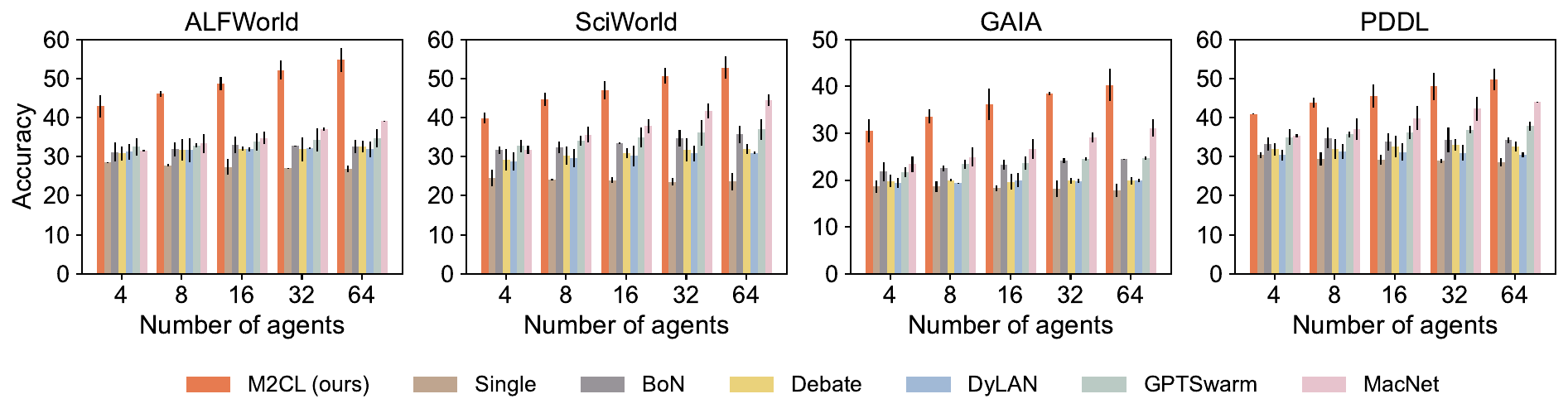}
	\caption{Performance of llama-13b as the base model with varying number of LLMs. Uncertainty intervals depict standard deviation over three seeds. \alg{} exhibits higher performance and increasing tendency with more LLMs, demonstrating its great collaboration efficiency compared to existing methods.}
	\label{fig:number_agents_2}
\end{figure*}

\begin{figure*}[h]
	\centering
	\includegraphics[width=\linewidth]{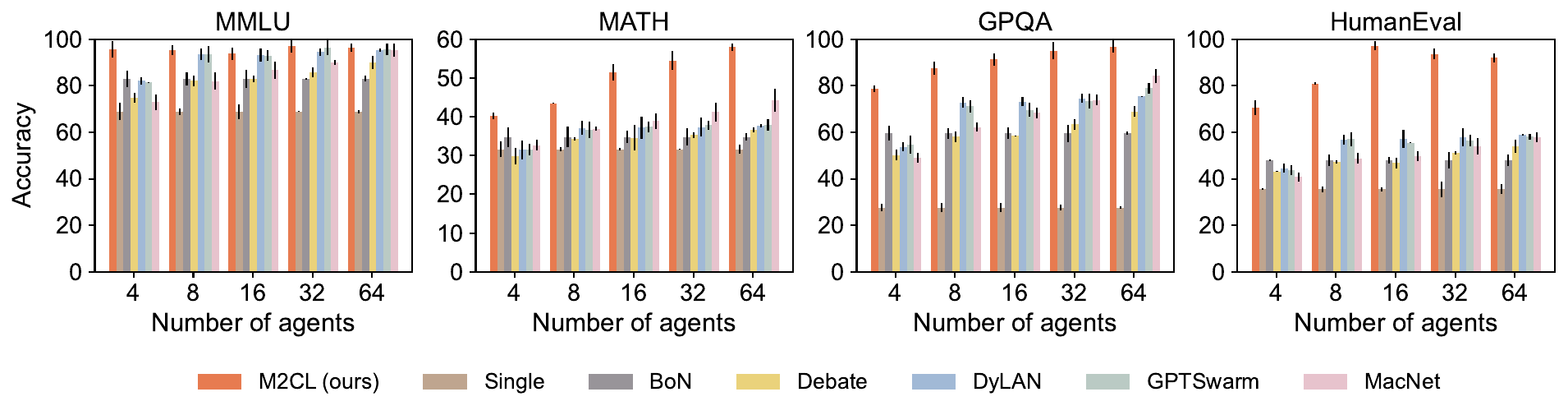}
	\includegraphics[width=\linewidth]{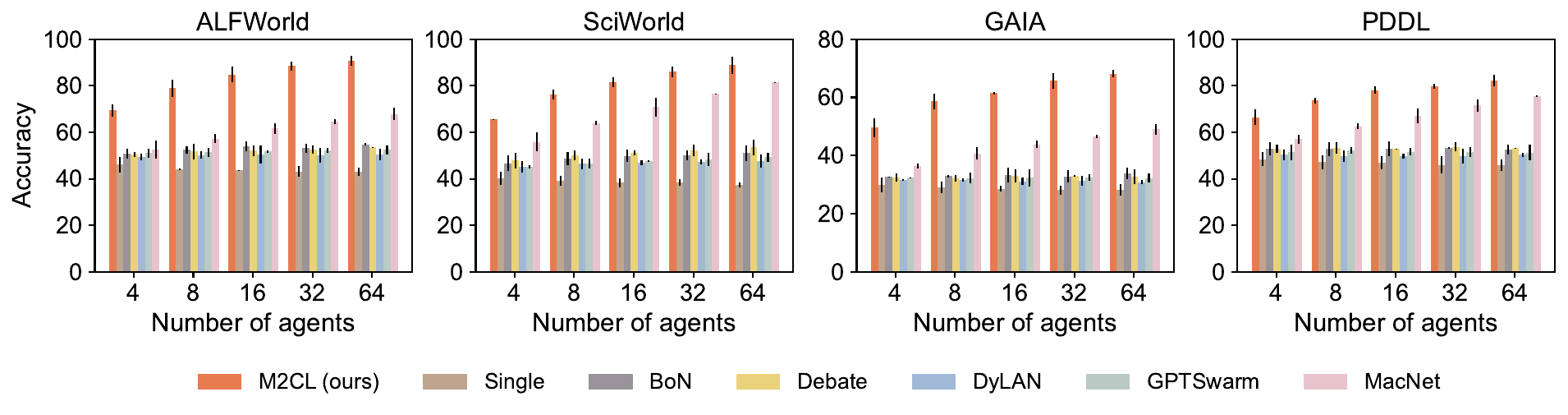}
	\caption{Performance as llama-70b as the base model with varying number of LLMs. Uncertainty intervals depict standard deviation over three seeds. \alg{} exhibits higher performance and increasing tendency with more LLMs, demonstrating its great collaboration efficiency compared to existing methods.}
	\label{fig:number_agents_3}
\end{figure*}

\begin{figure*}[h]
	\centering
	\includegraphics[width=\linewidth]{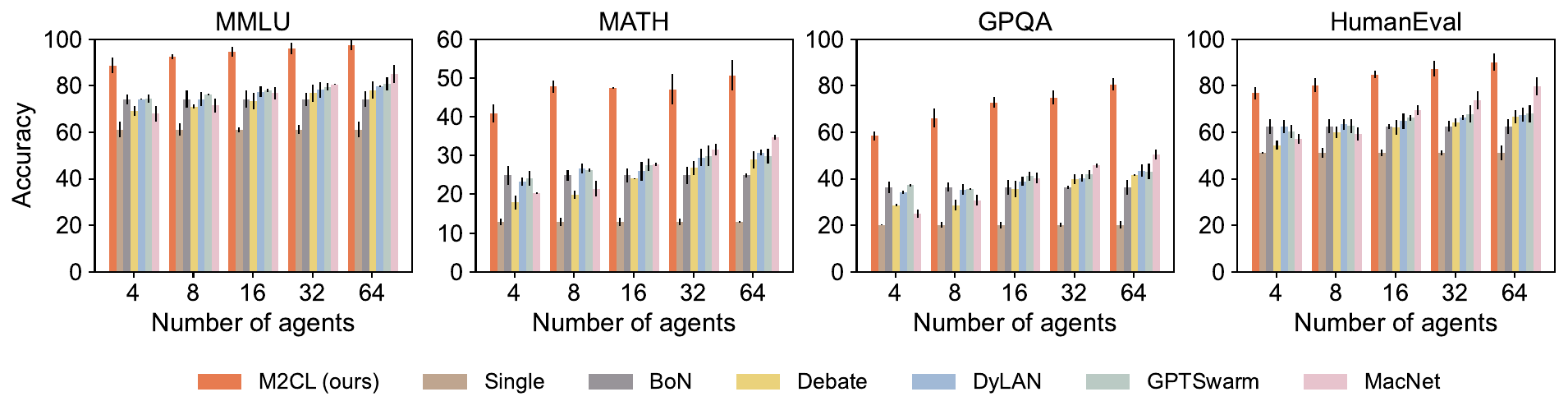}
	\includegraphics[width=\linewidth]{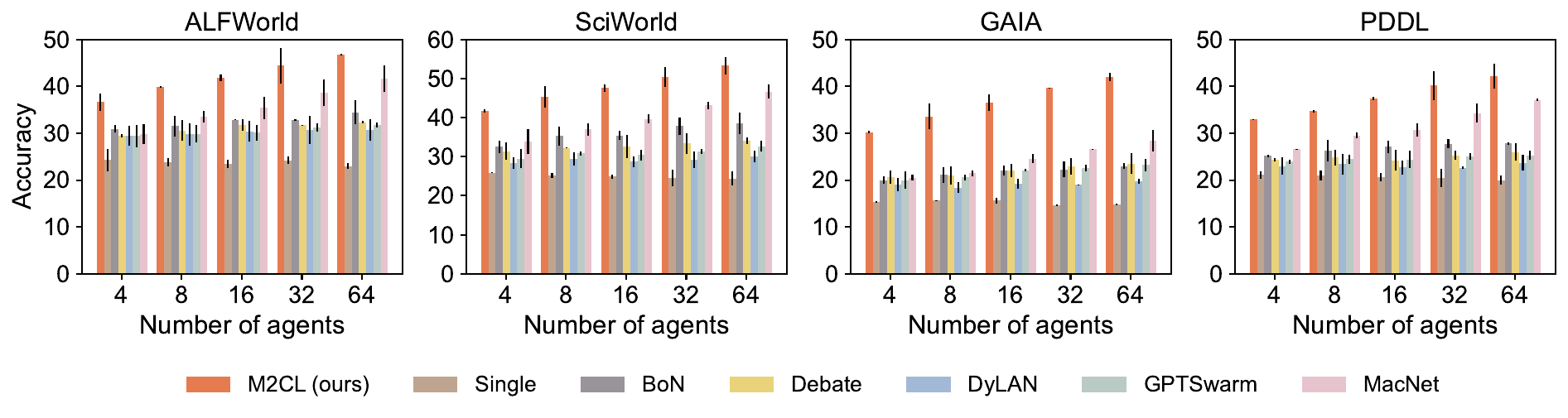}
	\caption{Performance of Qwen-7b as the base model with varying number of LLMs. Uncertainty intervals depict standard deviation over three seeds. \alg{} exhibits higher performance and increasing tendency with more LLMs, demonstrating its great collaboration efficiency compared to existing methods.}
	\label{fig:number_agents_4}
\end{figure*}
\begin{figure*}[h]
	\centering
	\includegraphics[width=\linewidth]{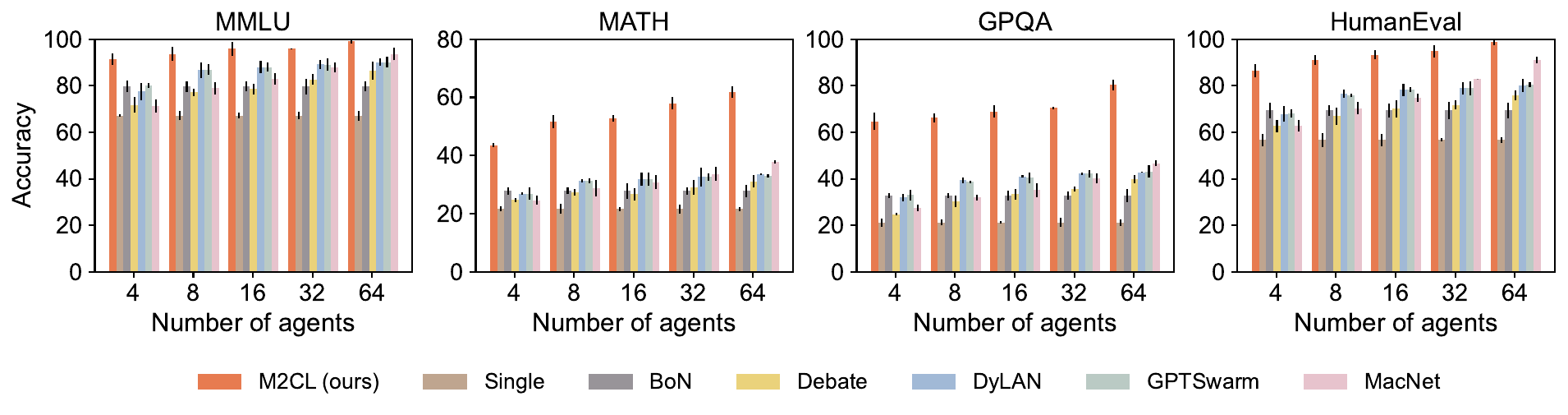}
	\includegraphics[width=\linewidth]{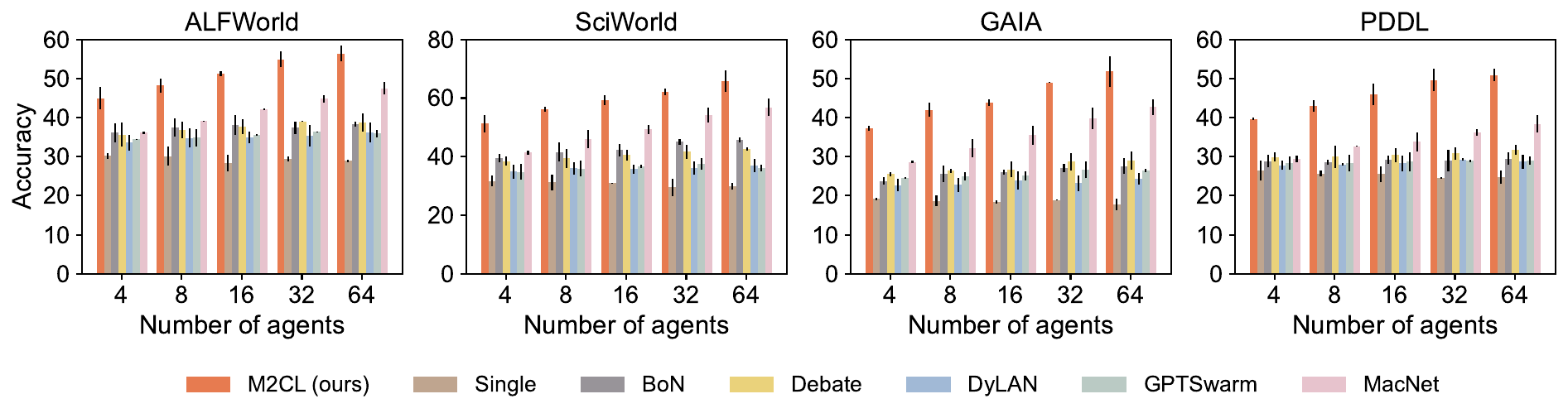}
	\caption{Performance of Qwen-14b as the base model with varying number of LLMs. Uncertainty intervals depict standard deviation over three seeds. \alg{} exhibits higher performance and increasing tendency with more LLMs, demonstrating its great collaboration efficiency compared to existing methods.}
	\label{fig:number_agents_5}
\end{figure*}

\begin{figure*}[h]
	\centering
	\includegraphics[width=\linewidth]{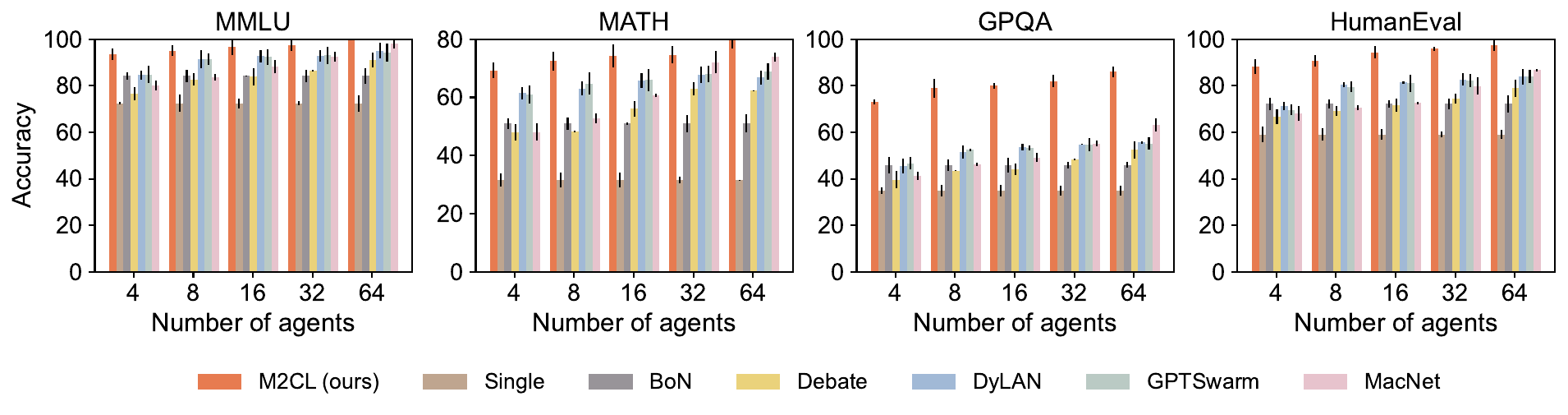}
	\includegraphics[width=\linewidth]{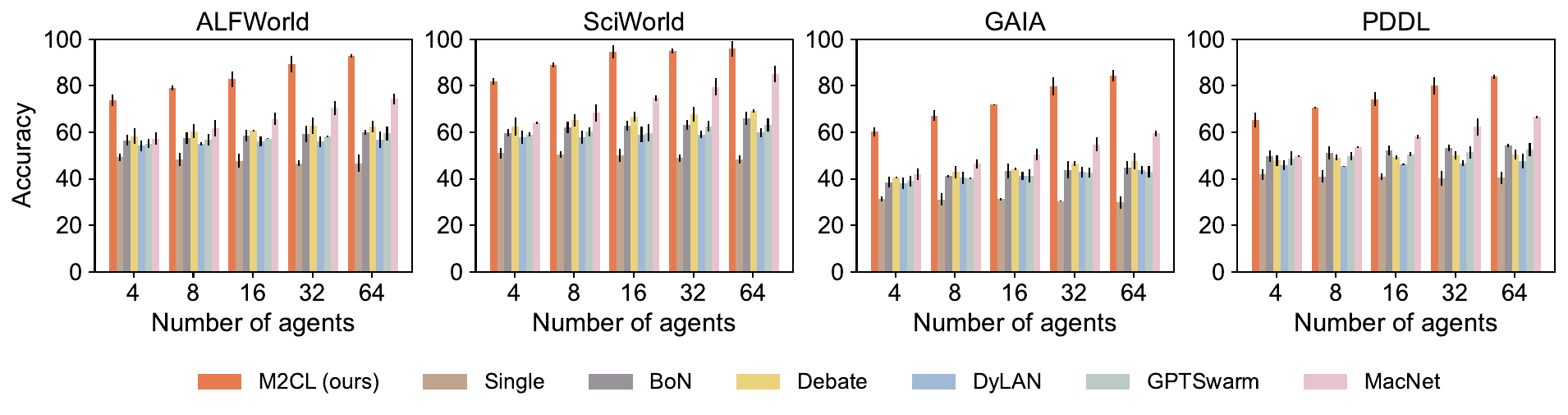}
	\caption{Performance of Qwen-72b as the base model with varying number of LLMs. Uncertainty intervals depict standard deviation over three seeds. \alg{} exhibits higher performance and increasing tendency with more LLMs, demonstrating its great collaboration efficiency compared to existing methods.}
	\label{fig:number_agents_6}
\end{figure*}

\clearpage
\subsection{GUI agent}
\label{sec:agent}
We evaluate \alg{}'s performance on more challenging AndroidWorld, which requires GUI identification, long-horizon planning, and accurate action execution capability. Comparative results are shown in \cref{tab:androidworld,fig:androidworld}.

We observe that \alg{} consistently outperforms existing baselines across different model scales up to $50\%$, with performance gains becoming more pronounced as the number of participating LLMs increases. The superior scalability of \alg{} in this setting highlights its ability to exploit diverse responses and maintain consistent under complex, real-world style interactions.
\begin{table*}[h]
	\centering
	\vskip 0.1in
	\resizebox{\textwidth}{!}{
		\begin{tabular}{YYYYYYYYY}
			\toprule
			\textbf{N} & \textbf{Model} & \texttt{Single} & \texttt{BoN} & \texttt{Debate} & \texttt{DyLAN} & \texttt{GPTSwarm} & \texttt{MacNet} & \alg{} \\
			\midrule
			\multirow{2}{*}{N=4}
			& 3B & 10.9{\tiny\textcolor{green}{$\downarrow$6.1}} & 17.0{\tiny\textcolor{red}{$\uparrow$0.0}} & 12.5{\tiny\textcolor{green}{$\downarrow$4.5}} & 13.0{\tiny\textcolor{green}{$\downarrow$4.0}} & 15.0{\tiny\textcolor{green}{$\downarrow$2.0}} & 16.8{\tiny\textcolor{green}{$\downarrow$0.2}} & \cellcolor{blue!10}25.5{\tiny\textcolor{red}{$\uparrow$8.5}} \\
			& 7B & 27.0{\tiny\textcolor{green}{$\downarrow$11.0}} & 38.0{\tiny\textcolor{red}{$\uparrow$0.0}} & 30.5{\tiny\textcolor{green}{$\downarrow$7.5}} & 31.0{\tiny\textcolor{green}{$\downarrow$7.0}} & 37.0{\tiny\textcolor{green}{$\downarrow$1.0}} & 39.6{\tiny\textcolor{red}{$\uparrow$1.6}} & \cellcolor{blue!10}44.9{\tiny\textcolor{red}{$\uparrow$6.9}} \\
			\midrule
			\multirow{2}{*}{N=8}
			& 3B & 10.9{\tiny\textcolor{green}{$\downarrow$6.1}} & 17.0{\tiny\textcolor{red}{$\uparrow$0.0}} & 15.7{\tiny\textcolor{green}{$\downarrow$1.3}} & 15.2{\tiny\textcolor{green}{$\downarrow$1.8}} & 15.6{\tiny\textcolor{green}{$\downarrow$1.4}} & 19.7{\tiny\textcolor{red}{$\uparrow$2.7}} & \cellcolor{blue!10}32.0{\tiny\textcolor{red}{$\uparrow$15.0}} \\
			& 7B & 27.0{\tiny\textcolor{green}{$\downarrow$11.0}} & 38.0{\tiny\textcolor{red}{$\uparrow$0.0}} & 33.0{\tiny\textcolor{green}{$\downarrow$5.0}} & 34.5{\tiny\textcolor{green}{$\downarrow$3.5}} & 38.0{\tiny\textcolor{red}{$\uparrow$0.0}} & 44.0{\tiny\textcolor{red}{$\uparrow$6.0}} & \cellcolor{blue!10}50.0{\tiny\textcolor{red}{$\uparrow$12.0}} \\
			\midrule
			\multirow{2}{*}{N=16}
			& 3B & 10.9{\tiny\textcolor{green}{$\downarrow$6.1}} & 17.0{\tiny\textcolor{red}{$\uparrow$0.0}} & 16.9{\tiny\textcolor{green}{$\downarrow$0.1}} & 18.4{\tiny\textcolor{red}{$\uparrow$1.4}} & 18.1{\tiny\textcolor{red}{$\uparrow$1.1}} & 22.6{\tiny\textcolor{red}{$\uparrow$5.6}} & \cellcolor{blue!10}38.0{\tiny\textcolor{red}{$\uparrow$21.0}} \\
			& 7B & 27.0{\tiny\textcolor{green}{$\downarrow$11.0}} & 38.0{\tiny\textcolor{red}{$\uparrow$0.0}} & 37.3{\tiny\textcolor{green}{$\downarrow$0.7}} & 37.8{\tiny\textcolor{green}{$\downarrow$0.2}} & 42.5{\tiny\textcolor{red}{$\uparrow$4.5}} & 47.9{\tiny\textcolor{red}{$\uparrow$9.9}} & \cellcolor{blue!10}55.0{\tiny\textcolor{red}{$\uparrow$17.0}} \\
			\midrule
			\multirow{2}{*}{N=32}
			& 3B & 10.9{\tiny\textcolor{green}{$\downarrow$6.1}} & 17.0{\tiny\textcolor{red}{$\uparrow$0.0}} & 18.0{\tiny\textcolor{red}{$\uparrow$1.0}} & 21.5{\tiny\textcolor{red}{$\uparrow$4.5}} & 20.6{\tiny\textcolor{red}{$\uparrow$3.6}} & 25.2{\tiny\textcolor{red}{$\uparrow$8.2}} & \cellcolor{blue!10}42.0{\tiny\textcolor{red}{$\uparrow$25.0}} \\
			& 7B & 27.0{\tiny\textcolor{green}{$\downarrow$11.0}} & 38.0{\tiny\textcolor{red}{$\uparrow$0.0}} & 41.5{\tiny\textcolor{red}{$\uparrow$3.5}} & 40.0{\tiny\textcolor{red}{$\uparrow$2.0}} & 45.0{\tiny\textcolor{red}{$\uparrow$7.0}} & 52.8{\tiny\textcolor{red}{$\uparrow$14.8}} & \cellcolor{blue!10}59.0{\tiny\textcolor{red}{$\uparrow$21.0}} \\
			\midrule
			\multirow{2}{*}{N=64}
			& 3B & 10.9{\tiny\textcolor{green}{$\downarrow$6.1}} & 17.0{\tiny\textcolor{red}{$\uparrow$0.0}} & 20.1{\tiny\textcolor{red}{$\uparrow$3.1}} & 23.6{\tiny\textcolor{red}{$\uparrow$6.6}} & 24.2{\tiny\textcolor{red}{$\uparrow$7.2}} & 28.6{\tiny\textcolor{red}{$\uparrow$11.6}} & \cellcolor{blue!10}45.0{\tiny\textcolor{red}{$\uparrow$28.0}} \\
			& 7B & 27.0{\tiny\textcolor{green}{$\downarrow$11.0}} & 38.0{\tiny\textcolor{red}{$\uparrow$0.0}} & 41.7{\tiny\textcolor{red}{$\uparrow$3.7}} & 42.2{\tiny\textcolor{red}{$\uparrow$4.2}} & 49.3{\tiny\textcolor{red}{$\uparrow$11.3}} & 55.2{\tiny\textcolor{red}{$\uparrow$17.2}} & \cellcolor{blue!10}62.0{\tiny\textcolor{red}{$\uparrow$24.0}} \\
			\bottomrule
		\end{tabular}
	}
	\caption{Accuracy with varying number of LLMs from $4$ to $64$ on AndroidWorld. We exhibit the performance advantage with \texttt{BoN} and highlight the \colorbox{blue!10}{best} result.}
	\label{tab:androidworld}
\end{table*}

\begin{figure}[ht]
	\centering
	\includegraphics[width=0.7\linewidth]{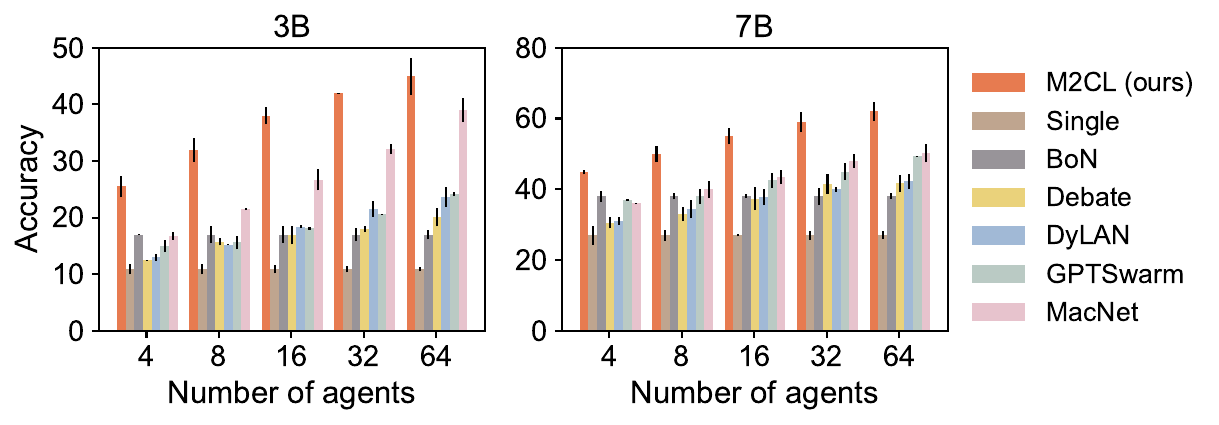}
	\caption{Performance of Qwen2.5-VL (3B and 7B) as the base model with varying number of LLMs. \alg{} exhibits higher performance and increasing tendency with more LLMs, demonstrating its great collaboration efficiency compared to existing methods. Uncertainty intervals depict standard deviation over three seeds.}
	\label{fig:androidworld}
\end{figure}

\clearpage
\subsection{Context Constraint}
To assess the effect of context constraint, we carry out experiments by varying the context constraint $\beta$ from 0 to 10. As illustrated in \cref{fig:context_constraint_1,fig:context_constraint_2,fig:context_constraint_3,fig:context_constraint_4,fig:context_constraint_5}, a larger value of $\beta$ results in a high degree of consistency among LLMs, leading them to produce similar answers. Conversely, a smaller value of $\beta$ is associated with reduced collaboration among LLMs. Therefore, it is important to adjust $\beta$ to control the degree of consistency among LLMs for better collaboration.
\label{sec:context_constraint}
\begin{figure*}[h]
	\centering
	\includegraphics[width=0.99\linewidth]{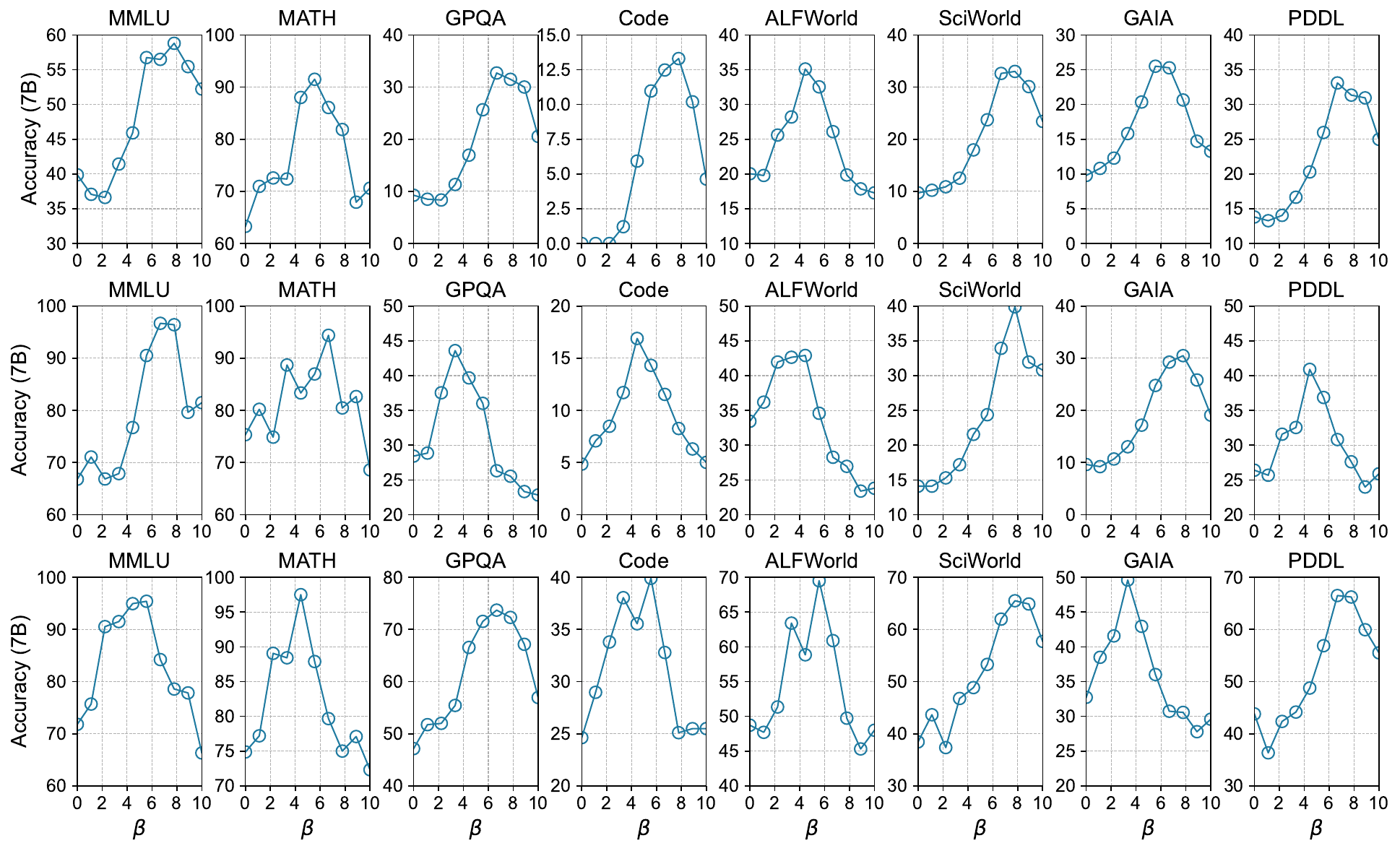}
	\caption{Performance with varying context constraint when 4 LLMs participate. All the curves display the same trend of rising first and then falling, which is consistent with our theory.}
	\label{fig:context_constraint_1}
\end{figure*}
\begin{figure*}[h]
	\centering
	\includegraphics[width=0.99\linewidth]{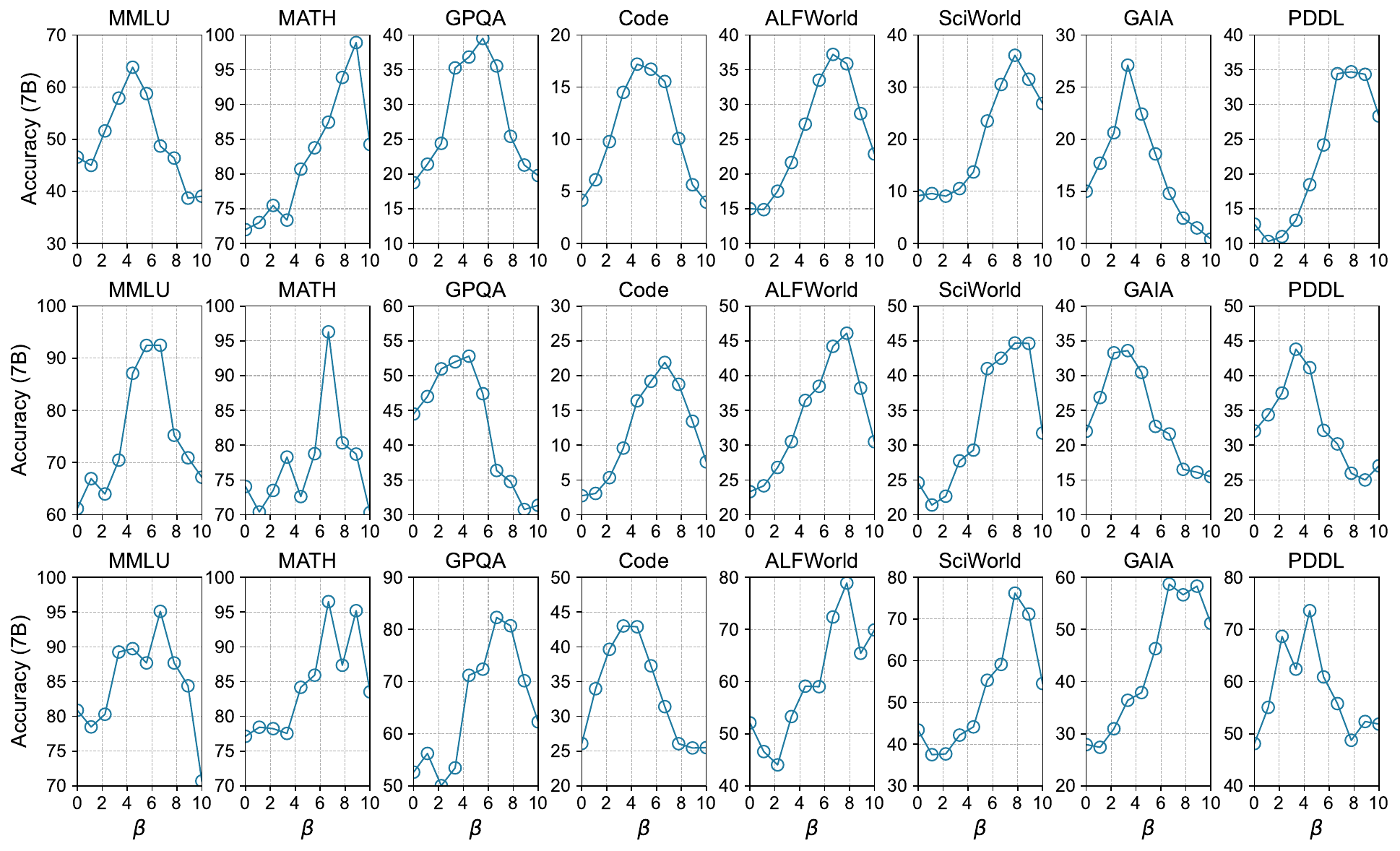}
	\caption{Performance with varying context constraint when 8 LLMs participate.}
	\label{fig:context_constraint_2}
\end{figure*}
\begin{figure*}[h]
	\centering
	\includegraphics[width=0.99\linewidth]{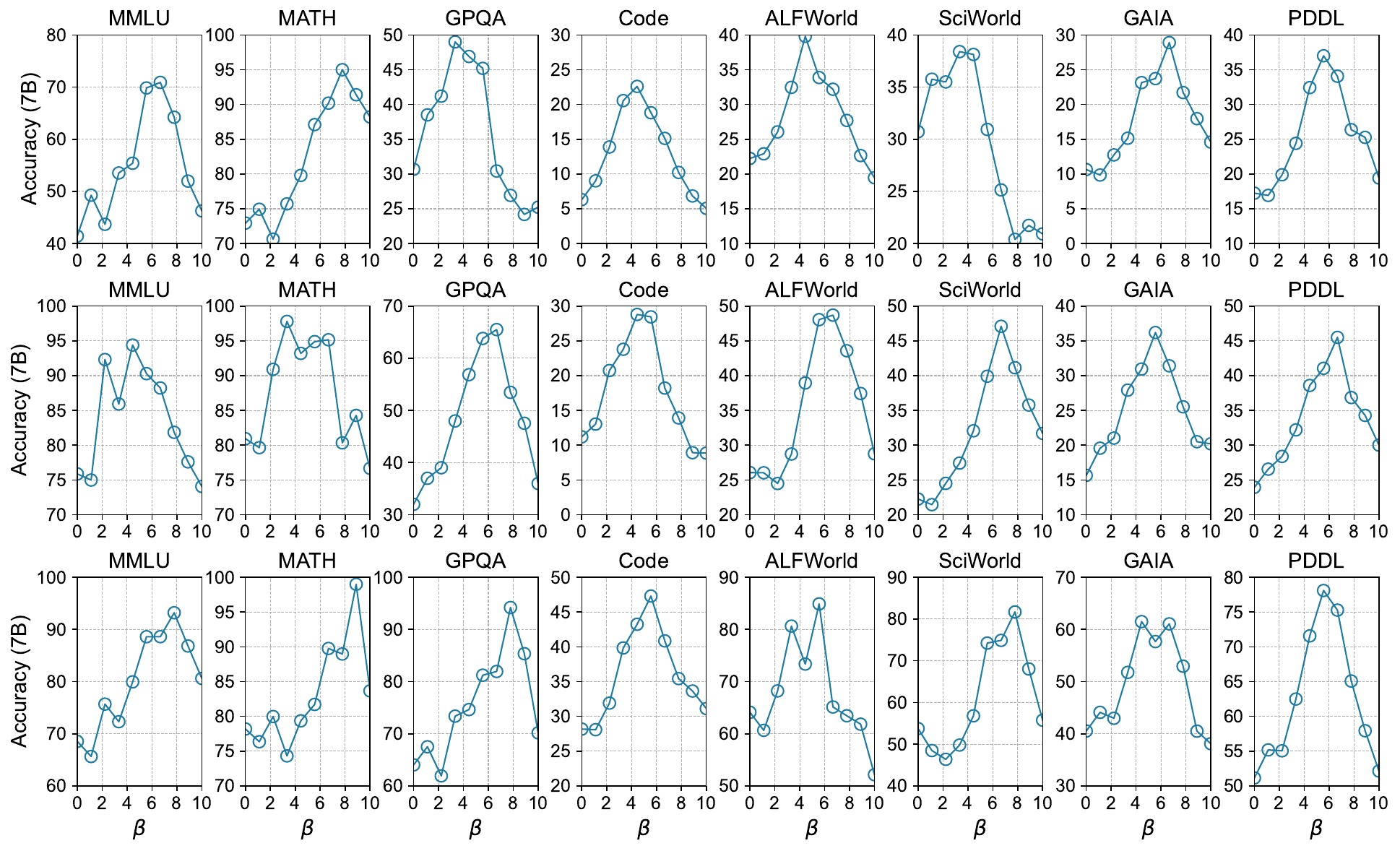}
	\caption{Performance with varying context constraint when 16 LLMs participate.}
	\label{fig:context_constraint_3}
\end{figure*}
\begin{figure*}[h]
	\centering
	\includegraphics[width=0.99\linewidth]{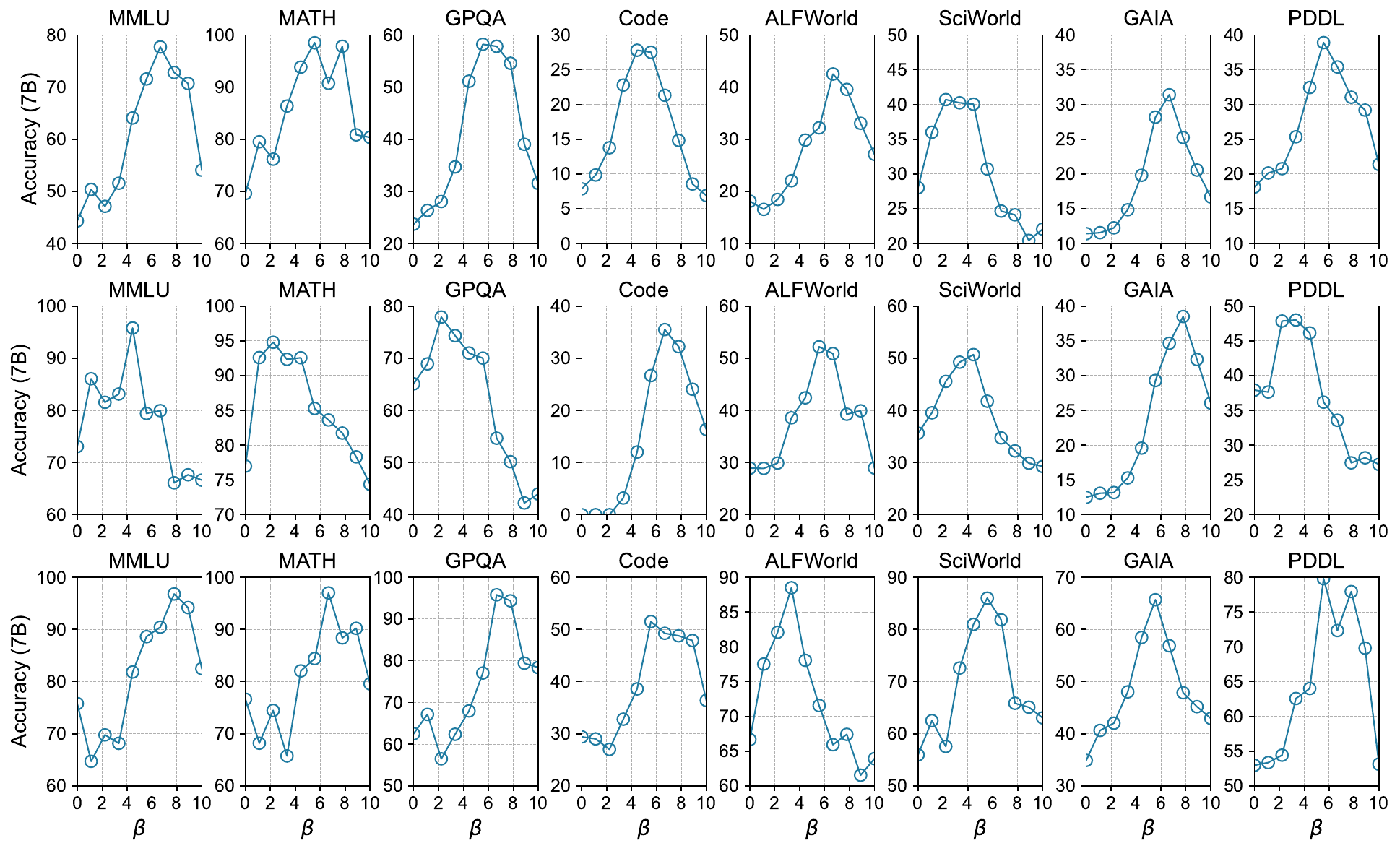}
	\caption{Performance with varying context constraint when 32 LLMs participate.}
	\label{fig:context_constraint_4}
\end{figure*}
\begin{figure*}[h]
	\centering
	\includegraphics[width=0.99\linewidth]{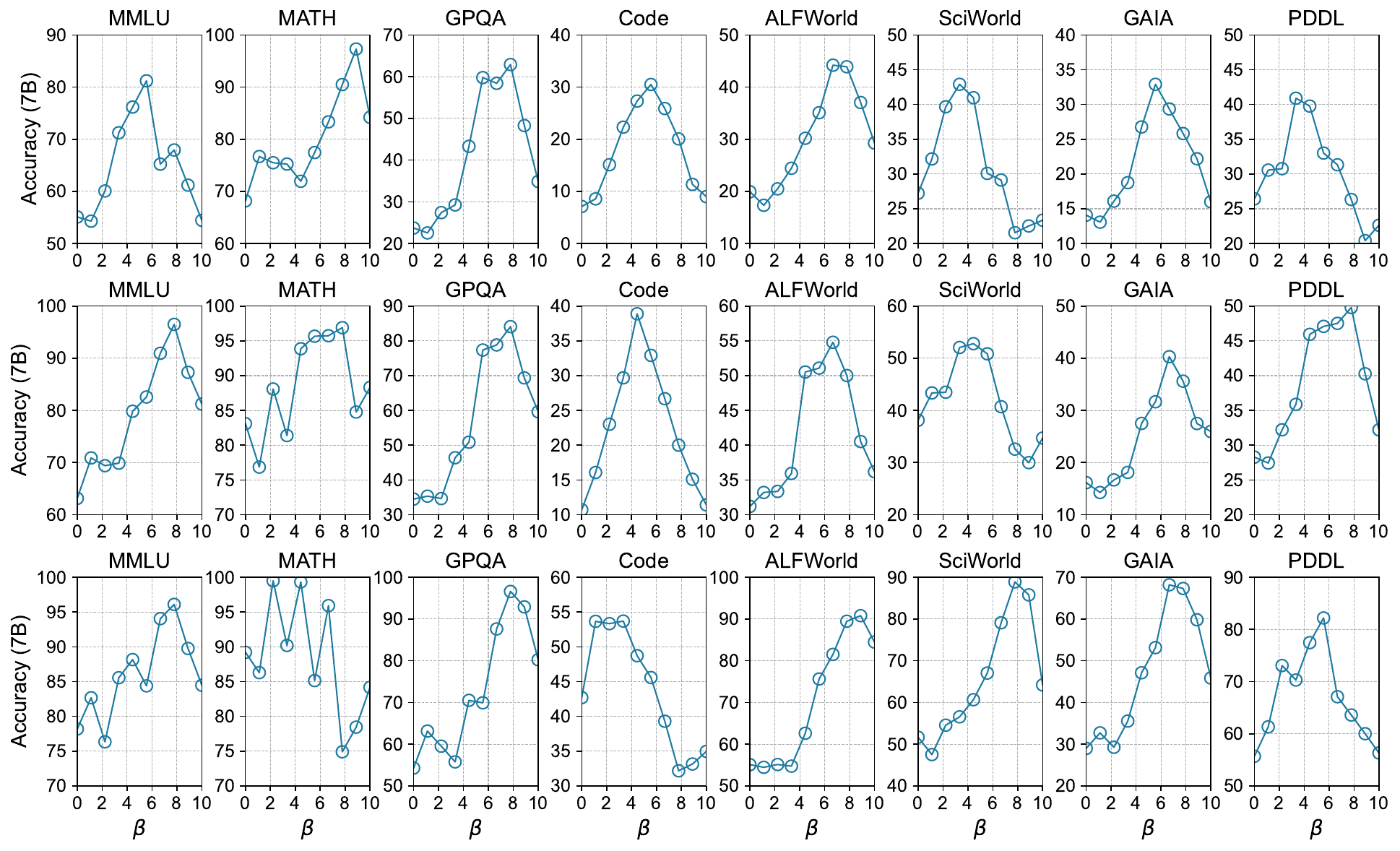}
	\caption{Performance with varying context constraint when 64 LLMs participate.}
	\label{fig:context_constraint_5}
\end{figure*}
\clearpage
\subsection{Discrepancy intensity by Rounds}
\label{sec:agreement_intensity}
To validate that \alg{} can collaborate LLMs to reach an agreement by rounds, we visualize the discrepancy intensity. As illustrated in \cref{fig:agreement_intensity_0,fig:agreement_intensity_1,fig:agreement_intensity_2,fig:agreement_intensity_3,fig:agreement_intensity_4}, the initial discrepancy intensity of \alg{} is higher as its context initialization can make the discussion more creative. The discrepancy intensity of \alg{} increases faster because the dynamic adjustment of the context provides LLMs with better ability to effectively receive information from other LLMs, resulting in reduced disagreement and faster collaboration among LLMs to reach a consensus.
\begin{figure*}[h]
	\centering
	\includegraphics[width=0.99\linewidth]{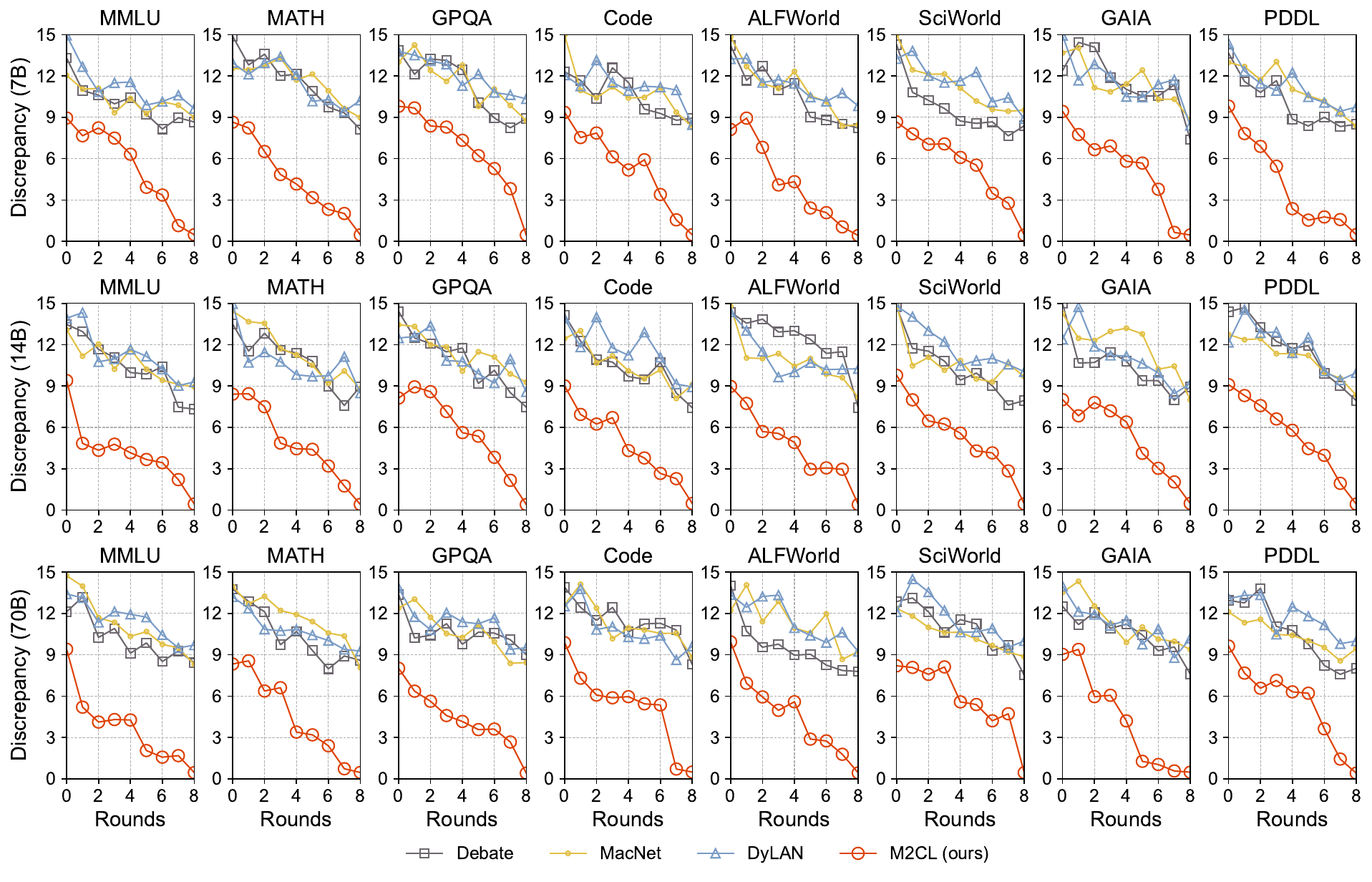}
	\caption{Comparative results on discrepancy intensity with varying model size (from top to bottom correspond to 7B, 14B, and 70B). The number of agents is set as $4$. Lower values represent a lower degree of disagreement. \alg{} can improve consistency with fewer rounds. Of note, \alg{} displays both a lower initial value and a faster decreasing speed, indicating its capability of assigning appropriate contexts based on the given question and current discussion situation.}
	\label{fig:agreement_intensity_0}
\end{figure*}

\begin{figure*}[h]
	\centering
	\includegraphics[width=0.99\linewidth]{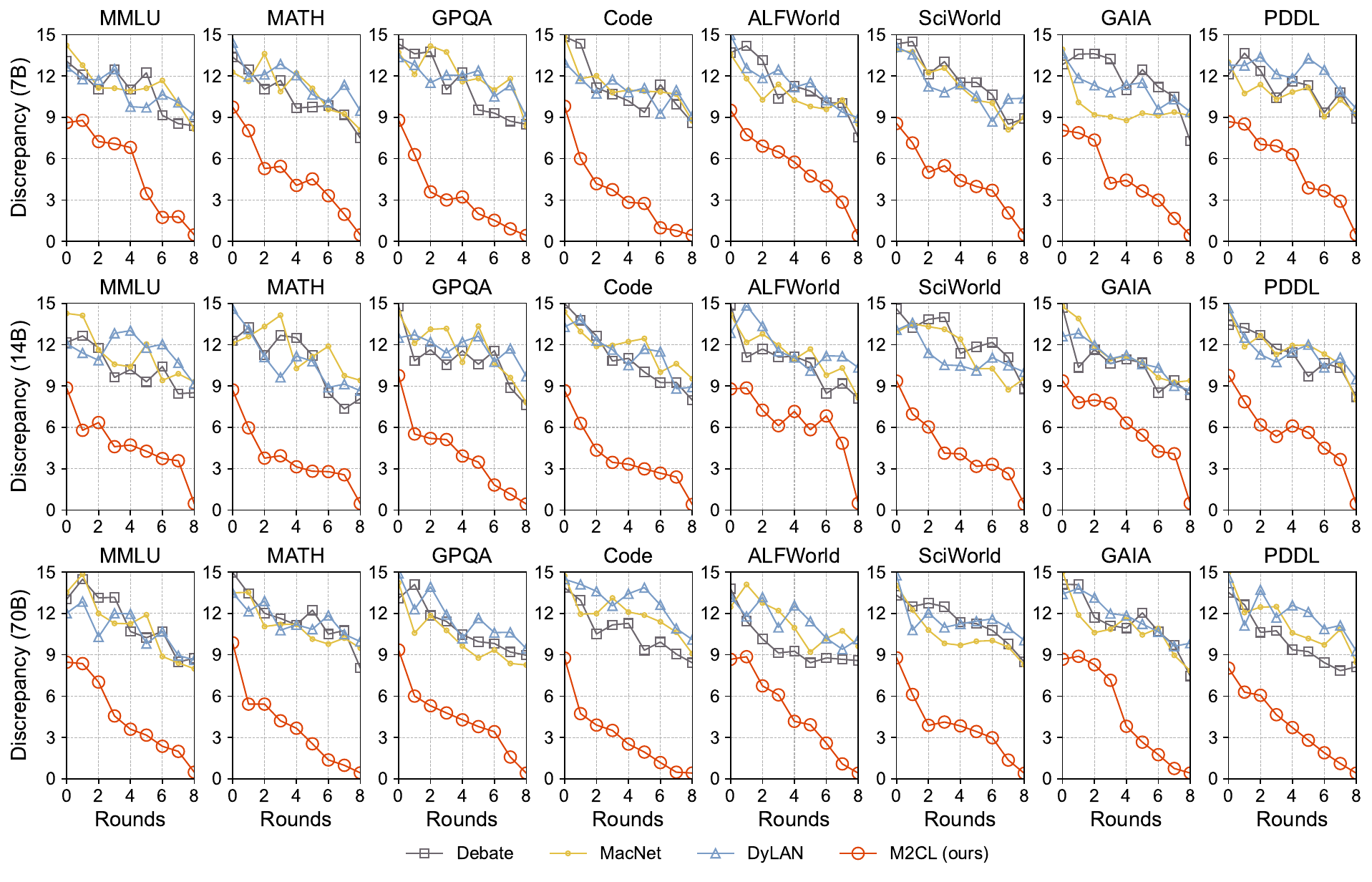}
	\caption{Comparative results on discrepancy intensity with varying model size (from top to bottom correspond to 7B, 14B, and 70B). The number of agents is set as $8$. Lower values represent a lower degree of disagreement. \alg{} can improve consistency with fewer rounds.}
	\label{fig:agreement_intensity_1}
\end{figure*}

\begin{figure*}[h]
	\centering
	\includegraphics[width=0.99\linewidth]{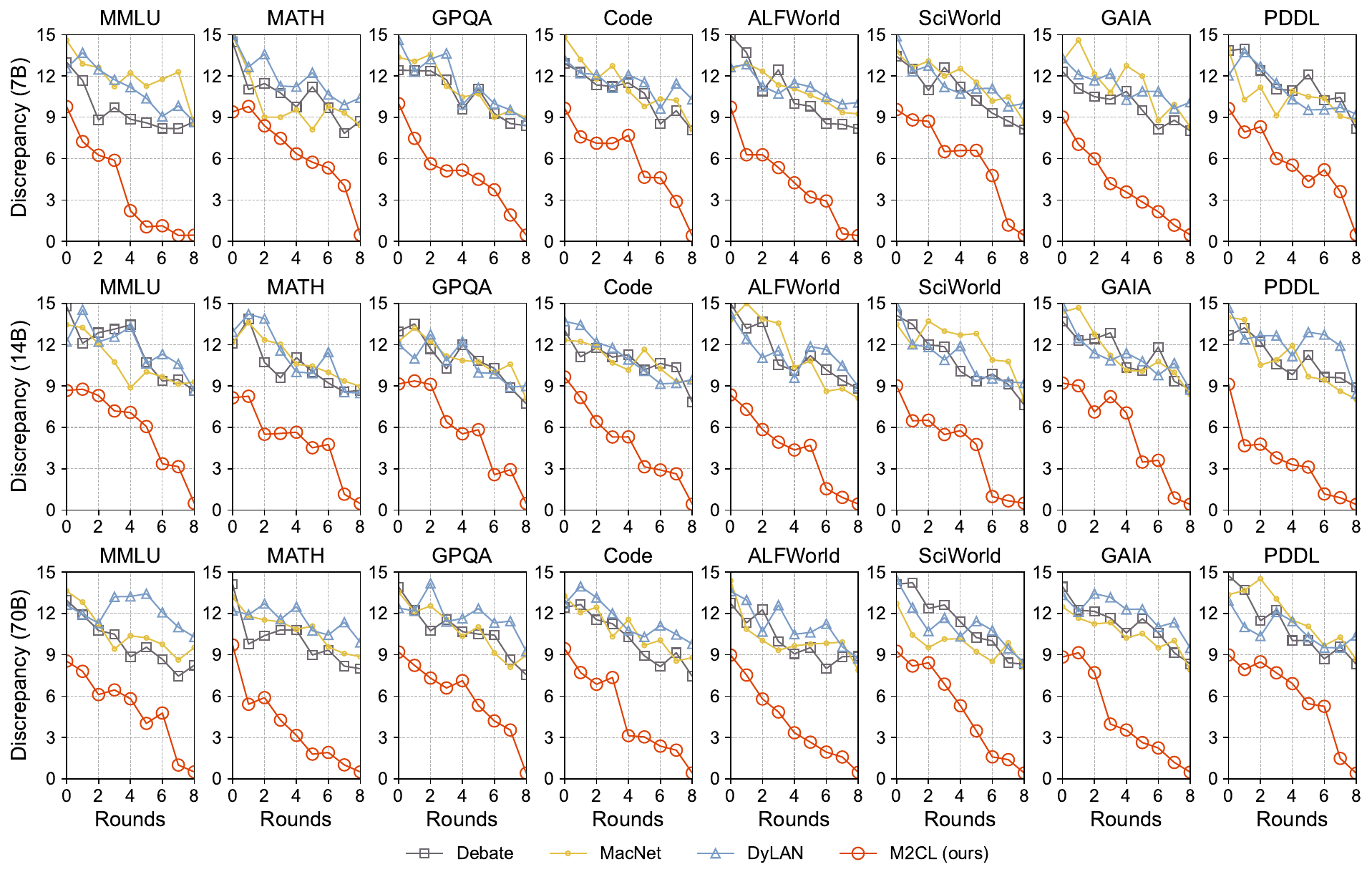}
	\caption{Comparative results on discrepancy intensity with varying model size (from top to bottom correspond to 7B, 14B, and 70B). The number of agents is set as $16$. Lower values represent a lower degree of disagreement. \alg{} can improve consistency with fewer rounds.}
	\label{fig:agreement_intensity_2}
\end{figure*}

\begin{figure*}[h]
	\centering
	\includegraphics[width=0.99\linewidth]{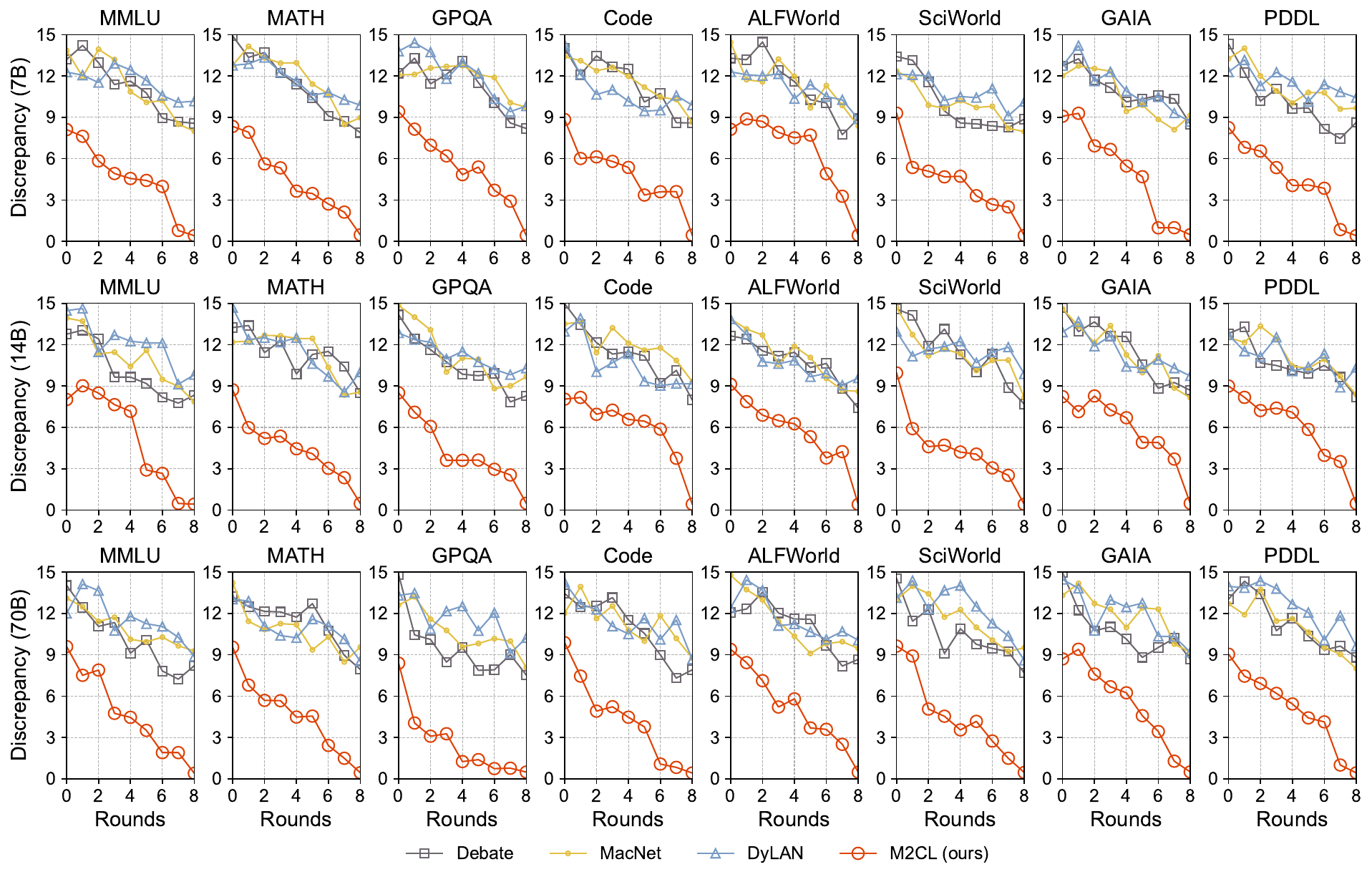}
	\caption{Comparative results on discrepancy intensity with varying model size (from top to bottom correspond to 7B, 14B, and 70B) The number of agents is set as $32$. Lower values represent a lower degree of disagreement. \alg{} can improve consistency with fewer rounds.}
	\label{fig:agreement_intensity_3}
\end{figure*}

\begin{figure*}[h]
	\centering
	\includegraphics[width=0.99\linewidth]{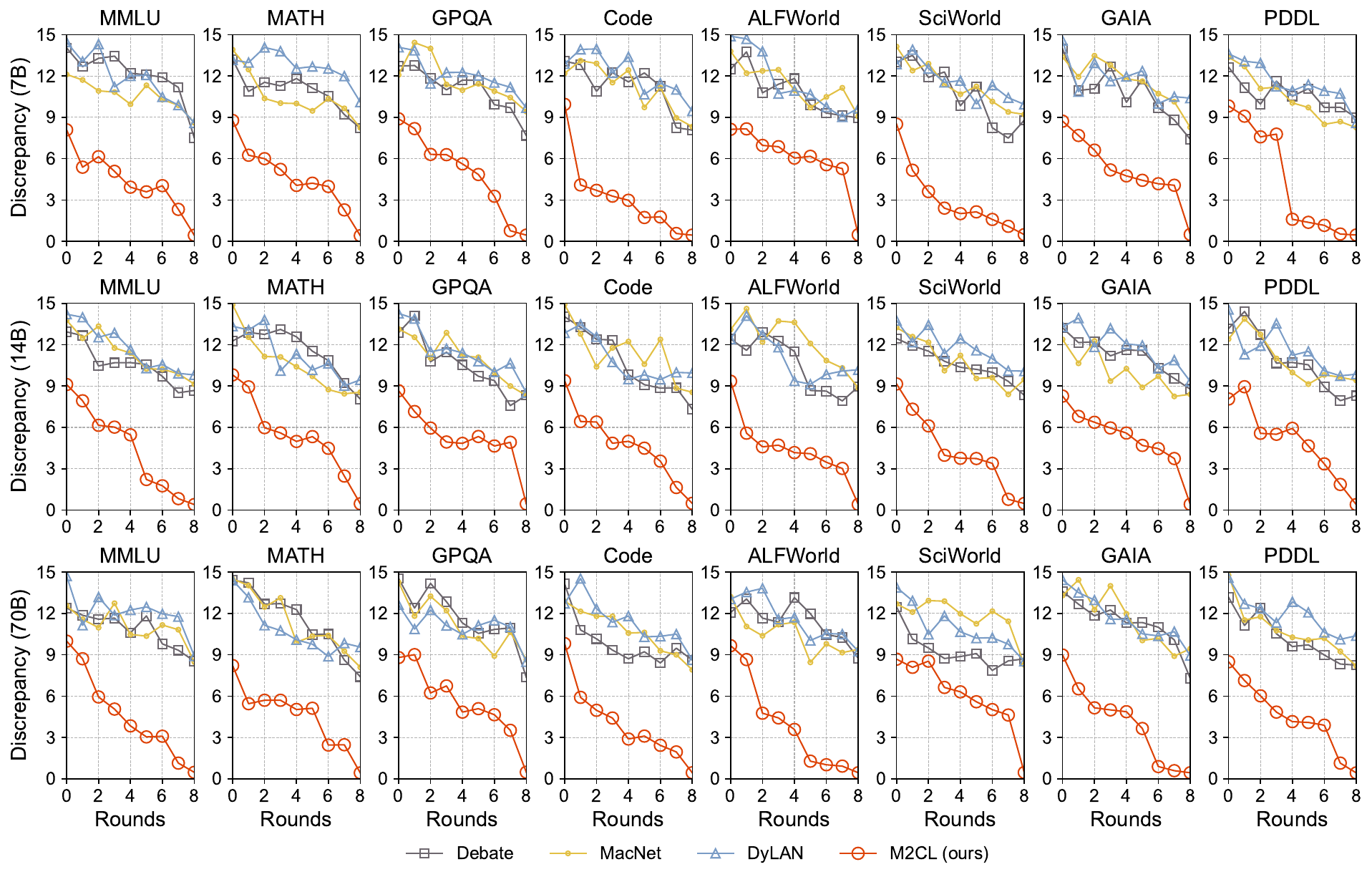}
	\caption{Comparative results on discrepancy intensity with varying model size (from top to bottom correspond to 7B, 14B, and 70B). The number of agents is set as $64$. Lower values represent a lower degree of disagreement. \alg{} can improve consistency with fewer rounds.}
	\label{fig:agreement_intensity_4}
\end{figure*}

\clearpage
\subsection{Transferability of Contexts}
\label{sec:transferability}
To further study the generalization of the generated contexts, we implement the multi-agent system using GPT-4 as the base model with the context generator trained on llama-7B and compare its performance with using initial contexts. As illustrated in \cref{tab:context_transfer_all}, the generated contexts outperform initial contexts, indicating that the trained context generator can be expanded to more models for improving overall performance through LLMs' collaboration.
\begin{table}[h] 
	\centering
	\resizebox{0.95\textwidth}{!}{
		\begin{tabular}{YZYYYYYYYY}
			\toprule
			\textbf{Dataset} & \textbf{Method} & \textbf{MMLU} & \textbf{MATH} & \textbf{GPQA} & \textbf{Code} & \textbf{ALFWorld} & \textbf{SciWorld} & \textbf{GAIA} & \textbf{PDDL} \\
			\midrule
			\multirow{2}{*}{$n=4$}
			& Initial & 84.3{\tiny\textcolor{green}{$\downarrow$0.0}} & 41.1{\tiny\textcolor{green}{$\downarrow$0.0}} & 33.3{\tiny\textcolor{green}{$\downarrow$0.0}} & 62.7{\tiny\textcolor{green}{$\downarrow$0.0}} & 70.2{\tiny\textcolor{green}{$\downarrow$0.0}} & 65.4{\tiny\textcolor{green}{$\downarrow$0.0}} & 58.7{\tiny\textcolor{green}{$\downarrow$0.0}} & 73.5{\tiny\textcolor{green}{$\downarrow$0.0}} \\
			& Generated & \cellcolor{blue!10}95.0{\tiny\textcolor{red}{$\uparrow$10.7}} & \cellcolor{blue!10}58.0{\tiny\textcolor{red}{$\uparrow$16.9}} & \cellcolor{blue!10}52.1{\tiny\textcolor{red}{$\uparrow$18.8}} & \cellcolor{blue!10}81.0{\tiny\textcolor{red}{$\uparrow$18.3}} & \cellcolor{blue!10}82.0{\tiny\textcolor{red}{$\uparrow$11.8}} & \cellcolor{blue!10}76.5{\tiny\textcolor{red}{$\uparrow$11.1}} & \cellcolor{blue!10}65.4{\tiny\textcolor{red}{$\uparrow$6.7}} & \cellcolor{blue!10}85.2{\tiny\textcolor{red}{$\uparrow$11.7}} \\
			\midrule
			\multirow{2}{*}{$n=8$}
			& Initial & 86.4{\tiny\textcolor{green}{$\downarrow$0.0}} & 43.9{\tiny\textcolor{green}{$\downarrow$0.0}} & 38.8{\tiny\textcolor{green}{$\downarrow$0.0}} & 67.0{\tiny\textcolor{green}{$\downarrow$0.0}} & 72.8{\tiny\textcolor{green}{$\downarrow$0.0}} & 68.0{\tiny\textcolor{green}{$\downarrow$0.0}} & 61.0{\tiny\textcolor{green}{$\downarrow$0.0}} & 76.0{\tiny\textcolor{green}{$\downarrow$0.0}} \\
			& Generated & \cellcolor{blue!10}98.2{\tiny\textcolor{red}{$\uparrow$11.8}} & \cellcolor{blue!10}68.5{\tiny\textcolor{red}{$\uparrow$24.6}} & \cellcolor{blue!10}62.5{\tiny\textcolor{red}{$\uparrow$23.7}} & \cellcolor{blue!10}89.1{\tiny\textcolor{red}{$\uparrow$22.1}} & \cellcolor{blue!10}82.5{\tiny\textcolor{red}{$\uparrow$9.7}} & \cellcolor{blue!10}77.8{\tiny\textcolor{red}{$\uparrow$9.8}} & \cellcolor{blue!10}64.3{\tiny\textcolor{red}{$\uparrow$3.3}} & \cellcolor{blue!10}88.2{\tiny\textcolor{red}{$\uparrow$12.2}} \\
			\midrule
			\multirow{2}{*}{$n=16$}
			& Initial & 87.9{\tiny\textcolor{green}{$\downarrow$0.0}} & 44.8{\tiny\textcolor{green}{$\downarrow$0.0}} & 40.1{\tiny\textcolor{green}{$\downarrow$0.0}} & 68.5{\tiny\textcolor{green}{$\downarrow$0.0}} & 74.3{\tiny\textcolor{green}{$\downarrow$0.0}} & 70.5{\tiny\textcolor{green}{$\downarrow$0.0}} & 63.0{\tiny\textcolor{green}{$\downarrow$0.0}} & 78.2{\tiny\textcolor{green}{$\downarrow$0.0}} \\
			& Generated & \cellcolor{blue!10}98.5{\tiny\textcolor{red}{$\uparrow$10.6}} & \cellcolor{blue!10}70.3{\tiny\textcolor{red}{$\uparrow$25.5}} & \cellcolor{blue!10}64.3{\tiny\textcolor{red}{$\uparrow$24.2}} & \cellcolor{blue!10}90.3{\tiny\textcolor{red}{$\uparrow$21.8}} & \cellcolor{blue!10}84.5{\tiny\textcolor{red}{$\uparrow$10.2}} & \cellcolor{blue!10}78.7{\tiny\textcolor{red}{$\uparrow$8.2}} & \cellcolor{blue!10}66.1{\tiny\textcolor{red}{$\uparrow$3.1}} & \cellcolor{blue!10}90.0{\tiny\textcolor{red}{$\uparrow$11.8}} \\
			\midrule
			\multirow{2}{*}{$n=32$}
			& Initial & 89.1{\tiny\textcolor{green}{$\downarrow$0.0}} & 45.4{\tiny\textcolor{green}{$\downarrow$0.0}} & 41.5{\tiny\textcolor{green}{$\downarrow$0.0}} & 69.7{\tiny\textcolor{green}{$\downarrow$0.0}} & 75.0{\tiny\textcolor{green}{$\downarrow$0.0}} & 71.0{\tiny\textcolor{green}{$\downarrow$0.0}} & 63.5{\tiny\textcolor{green}{$\downarrow$0.0}} & 78.5{\tiny\textcolor{green}{$\downarrow$0.0}} \\
			& Generated & \cellcolor{blue!10}98.8{\tiny\textcolor{red}{$\uparrow$9.7}} & \cellcolor{blue!10}72.0{\tiny\textcolor{red}{$\uparrow$26.6}} & \cellcolor{blue!10}66.0{\tiny\textcolor{red}{$\uparrow$24.5}} & \cellcolor{blue!10}91.3{\tiny\textcolor{red}{$\uparrow$21.6}} & \cellcolor{blue!10}86.0{\tiny\textcolor{red}{$\uparrow$11.0}} & \cellcolor{blue!10}79.5{\tiny\textcolor{red}{$\uparrow$8.5}} & \cellcolor{blue!10}66.8{\tiny\textcolor{red}{$\uparrow$3.3}} & \cellcolor{blue!10}91.2{\tiny\textcolor{red}{$\uparrow$12.7}} \\
			\midrule
			\multirow{2}{*}{$n=64$}
			& Initial & 88.5{\tiny\textcolor{green}{$\downarrow$0.0}} & 45.4{\tiny\textcolor{green}{$\downarrow$0.0}} & 41.8{\tiny\textcolor{green}{$\downarrow$0.0}} & 70.5{\tiny\textcolor{green}{$\downarrow$0.0}} & 75.2{\tiny\textcolor{green}{$\downarrow$0.0}} & 71.5{\tiny\textcolor{green}{$\downarrow$0.0}} & 64.0{\tiny\textcolor{green}{$\downarrow$0.0}} & 79.2{\tiny\textcolor{green}{$\downarrow$0.0}} \\
			& Generated & \cellcolor{blue!10}98.9{\tiny\textcolor{red}{$\uparrow$10.4}} & \cellcolor{blue!10}72.8{\tiny\textcolor{red}{$\uparrow$27.4}} & \cellcolor{blue!10}66.9{\tiny\textcolor{red}{$\uparrow$25.1}} & \cellcolor{blue!10}91.7{\tiny\textcolor{red}{$\uparrow$21.2}} & \cellcolor{blue!10}86.5{\tiny\textcolor{red}{$\uparrow$11.3}} & \cellcolor{blue!10}80.2{\tiny\textcolor{red}{$\uparrow$8.7}} & \cellcolor{blue!10}67.2{\tiny\textcolor{red}{$\uparrow$3.2}} & \cellcolor{blue!10}91.5{\tiny\textcolor{red}{$\uparrow$12.3}} \\
			\bottomrule
		\end{tabular}
	}
	\caption{Performance of transferring generated contexts trained on llama-7B to GPT-4 with varying number of LLMs.}
	\label{tab:context_transfer_all}
\end{table}

\subsection{Ablation Studies and Complementary Experiments}
\subsubsection{Training Convergence of Context Generator}
We verify the convergence of solving Problem~(\ref{eq:objective_multi_round_final}) by displaying the cumulative utility. As shown in \cref{fig:loss_1}, it works well in all dataset and often converges in 60 training steps.
\begin{figure*}[h]
	\centering
	\includegraphics[width=\linewidth]{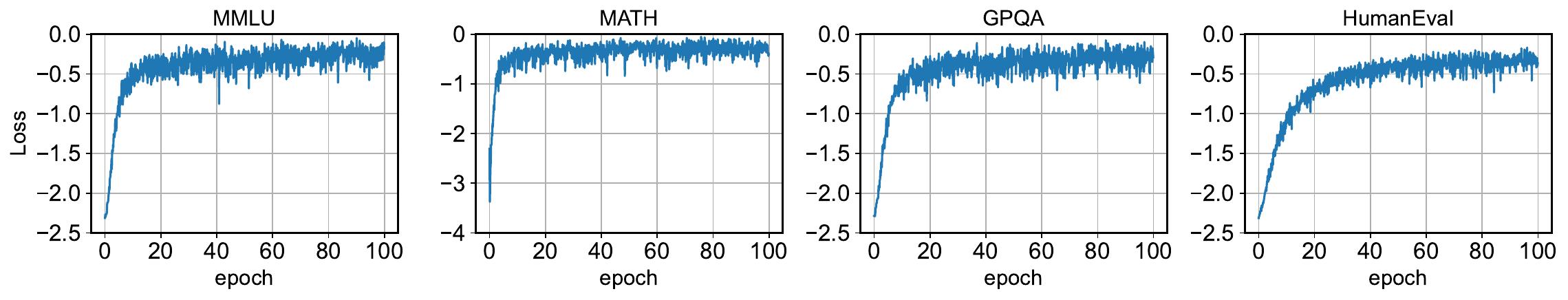}
	\includegraphics[width=\linewidth]{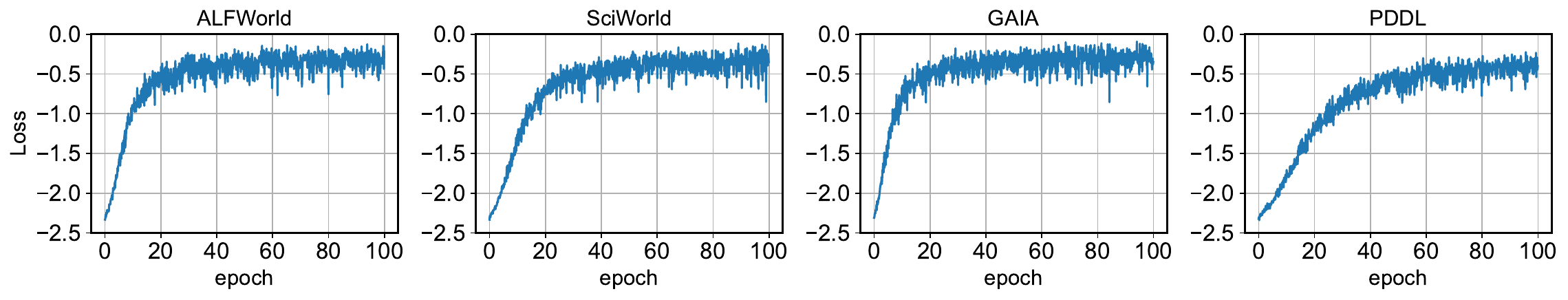}
	\caption{The value of $L(\theta)$ when training llama-7B with 8 LLMs participating.}
	\label{fig:loss_1}
\end{figure*}

\subsubsection{Ablation}
In this section, we assess the effect of key components by ablating them under the same setting.

Without context initialization, LLMs fail to develop specialized expertise, resulting in homogeneous policies that inefficiently duplicate effort, demonstrating poor adaptability when faced with novel questions. \cref{tab:ablation_llama_7b,tab:ablation_llama_13b,tab:ablation_llama_70b} underscore the importance of context initialization before discussion to enhance foundational multi-agent capabilities.

Without tuning $\alpha$ during discussion rounds, LLMs tend to reach an agreement in the first round, leading to responses that lack creativity and diversity, which ultimately reduces problem-solving ability. \cref{tab:ablation_llama_7b,tab:ablation_llama_13b,tab:ablation_llama_70b} underscore the importance of tuning $\alpha$ during discussion rounds for training.

Without context evolution, LLMs receive no guidance consider inter-LLM dependencies, making it difficult for them to fully utilize the information contributed by other LLMs. This context misalignment leads to discussion inconsistency and poor collaboration efficiency. \cref{tab:ablation_llama_7b,tab:ablation_llama_13b,tab:ablation_llama_70b} highlight the necessity of context evolution for enabling effective inter-LLM collaboration and achieving consistent improvements across discussion rounds.

\begin{table}[h] 
	\centering
	\resizebox{0.95\textwidth}{!}{
		\begin{tabular}{XZYYYYYYYY}
			\toprule
			\textbf{Dataset} & \textbf{Method} & \textbf{MMLU} & \textbf{MATH} & \textbf{GPQA} & \textbf{Code} & \textbf{ALFWorld} & \textbf{SciWorld} & \textbf{GAIA} & \textbf{PDDL} \\
			\midrule
			\multirow{4}{*}{$n=4$}
			& \alg{}      & \cellcolor{blue!10}59.1{\tiny\textcolor{red}{$\uparrow$10.6}} & \cellcolor{blue!10}17.4{\tiny\textcolor{red}{$\uparrow$9.9}}  & \cellcolor{blue!10}44.7{\tiny\textcolor{red}{$\uparrow$16.7}} & \cellcolor{blue!10}32.9{\tiny\textcolor{red}{$\uparrow$11.8}} & \cellcolor{blue!10}35.1{\tiny\textcolor{red}{$\uparrow$8.8}}  & \cellcolor{blue!10}33.0{\tiny\textcolor{red}{$\uparrow$8.3}}  & \cellcolor{blue!10}25.5{\tiny\textcolor{red}{$\uparrow$8.1}}  & \cellcolor{blue!10}33.1{\tiny\textcolor{red}{$\uparrow$6.2}} \\
			& w/o init.   & 56.4{\tiny\textcolor{red}{$\uparrow$7.9}}  & 13.3{\tiny\textcolor{red}{$\uparrow$5.8}}  & 44.5{\tiny\textcolor{red}{$\uparrow$16.5}} & 31.3{\tiny\textcolor{red}{$\uparrow$10.2}} & 31.8{\tiny\textcolor{red}{$\uparrow$5.5}}  & 30.1{\tiny\textcolor{red}{$\uparrow$5.4}}  & 23.2{\tiny\textcolor{red}{$\uparrow$5.8}}  & 30.5{\tiny\textcolor{red}{$\uparrow$3.6}} \\
			& w/o $\alpha$ & 53.9{\tiny\textcolor{red}{$\uparrow$5.4}}  & 9.4{\tiny\textcolor{red}{$\uparrow$1.9}}   & 37.8{\tiny\textcolor{red}{$\uparrow$9.8}}  & 26.0{\tiny\textcolor{red}{$\uparrow$4.9}}  & 22.5{\tiny\textcolor{green}{$\downarrow$3.8}} & 21.1{\tiny\textcolor{green}{$\downarrow$3.6}} & 20.3{\tiny\textcolor{red}{$\uparrow$2.9}}  & 23.4{\tiny\textcolor{green}{$\downarrow$3.5}} \\
			& w/o evolve   & 48.5{\tiny\textcolor{green}{$\downarrow$0.0}}  & 7.5{\tiny\textcolor{green}{$\downarrow$0.0}}   & 28.0{\tiny\textcolor{green}{$\downarrow$0.0}}  & 21.1{\tiny\textcolor{green}{$\downarrow$0.0}}  & 26.3{\tiny\textcolor{green}{$\downarrow$0.0}}  & 24.7{\tiny\textcolor{green}{$\downarrow$0.0}}  & 17.4{\tiny\textcolor{green}{$\downarrow$0.0}}  & 26.9{\tiny\textcolor{green}{$\downarrow$0.0}} \\
			\midrule
			\multirow{4}{*}{$n=8$}
			& \alg{}      & \cellcolor{blue!10}63.8{\tiny\textcolor{red}{$\uparrow$12.9}} & \cellcolor{blue!10}24.9{\tiny\textcolor{red}{$\uparrow$16.8}} & \cellcolor{blue!10}59.3{\tiny\textcolor{red}{$\uparrow$31.1}} & \cellcolor{blue!10}38.8{\tiny\textcolor{red}{$\uparrow$17.6}} & \cellcolor{blue!10}37.2{\tiny\textcolor{red}{$\uparrow$10.2}} & \cellcolor{blue!10}36.1{\tiny\textcolor{red}{$\uparrow$10.3}} & \cellcolor{blue!10}27.1{\tiny\textcolor{red}{$\uparrow$10.1}} & \cellcolor{blue!10}34.7{\tiny\textcolor{red}{$\uparrow$7.0}} \\
			& w/o init.   & 63.8{\tiny\textcolor{red}{$\uparrow$12.9}} & 17.2{\tiny\textcolor{red}{$\uparrow$9.1}}  & 54.9{\tiny\textcolor{red}{$\uparrow$26.7}} & 37.2{\tiny\textcolor{red}{$\uparrow$16.0}} & 34.7{\tiny\textcolor{red}{$\uparrow$7.7}}  & 33.1{\tiny\textcolor{red}{$\uparrow$7.3}}  & 24.5{\tiny\textcolor{red}{$\uparrow$7.5}}  & 31.5{\tiny\textcolor{red}{$\uparrow$3.8}} \\
			& w/o $\alpha$ & 57.4{\tiny\textcolor{red}{$\uparrow$6.5}}  & 12.9{\tiny\textcolor{red}{$\uparrow$4.8}}  & 43.0{\tiny\textcolor{red}{$\uparrow$14.8}} & 29.8{\tiny\textcolor{red}{$\uparrow$8.6}}  & 22.7{\tiny\textcolor{green}{$\downarrow$4.3}} & 26.8{\tiny\textcolor{red}{$\uparrow$1.0}}  & 20.9{\tiny\textcolor{red}{$\uparrow$3.9}}  & 24.8{\tiny\textcolor{green}{$\downarrow$2.9}} \\
			& w/o evolve   & 50.9{\tiny\textcolor{green}{$\downarrow$0.0}}  & 8.1{\tiny\textcolor{green}{$\downarrow$0.0}}   & 28.2{\tiny\textcolor{green}{$\downarrow$0.0}}  & 21.2{\tiny\textcolor{green}{$\downarrow$0.0}}  & 27.0{\tiny\textcolor{green}{$\downarrow$0.0}}  & 25.8{\tiny\textcolor{green}{$\downarrow$0.0}}  & 17.0{\tiny\textcolor{green}{$\downarrow$0.0}}  & 27.7{\tiny\textcolor{green}{$\downarrow$0.0}} \\
			\midrule
			\multirow{4}{*}{$n=16$}
			& \alg{}      & \cellcolor{blue!10}71.5{\tiny\textcolor{red}{$\uparrow$18.5}} & \cellcolor{blue!10}23.9{\tiny\textcolor{red}{$\uparrow$15.7}} & \cellcolor{blue!10}69.9{\tiny\textcolor{red}{$\uparrow$39.4}} & \cellcolor{blue!10}48.3{\tiny\textcolor{red}{$\uparrow$24.3}} & \cellcolor{blue!10}39.8{\tiny\textcolor{red}{$\uparrow$12.7}} & \cellcolor{blue!10}38.4{\tiny\textcolor{red}{$\uparrow$11.8}} & \cellcolor{blue!10}28.9{\tiny\textcolor{red}{$\uparrow$11.1}} & \cellcolor{blue!10}37.0{\tiny\textcolor{red}{$\uparrow$9.5}} \\
			& w/o init.   & 70.9{\tiny\textcolor{red}{$\uparrow$17.9}} & 22.6{\tiny\textcolor{red}{$\uparrow$14.4}} & 69.4{\tiny\textcolor{red}{$\uparrow$38.9}} & 45.3{\tiny\textcolor{red}{$\uparrow$21.3}} & 36.3{\tiny\textcolor{red}{$\uparrow$9.2}}  & 35.3{\tiny\textcolor{red}{$\uparrow$8.7}}  & 26.2{\tiny\textcolor{red}{$\uparrow$8.4}}  & 33.9{\tiny\textcolor{red}{$\uparrow$6.4}} \\
			& w/o $\alpha$ & 60.5{\tiny\textcolor{red}{$\uparrow$7.5}}  & 14.2{\tiny\textcolor{red}{$\uparrow$6.0}}  & 46.4{\tiny\textcolor{red}{$\uparrow$15.9}} & 32.8{\tiny\textcolor{red}{$\uparrow$8.8}}  & 26.6{\tiny\textcolor{green}{$\downarrow$0.5}} & 23.7{\tiny\textcolor{green}{$\downarrow$2.9}} & 22.2{\tiny\textcolor{red}{$\uparrow$4.4}}  & 24.9{\tiny\textcolor{green}{$\downarrow$2.6}} \\
			& w/o evolve   & 53.0{\tiny\textcolor{green}{$\downarrow$0.0}}  & 8.2{\tiny\textcolor{green}{$\downarrow$0.0}}   & 30.5{\tiny\textcolor{green}{$\downarrow$0.0}}  & 24.0{\tiny\textcolor{green}{$\downarrow$0.0}}  & 27.1{\tiny\textcolor{green}{$\downarrow$0.0}}  & 26.6{\tiny\textcolor{green}{$\downarrow$0.0}}  & 17.8{\tiny\textcolor{green}{$\downarrow$0.0}}  & 27.5{\tiny\textcolor{green}{$\downarrow$0.0}} \\
			\midrule
			\multirow{4}{*}{$n=32$}
			& \alg{}      & \cellcolor{blue!10}95.8{\tiny\textcolor{red}{$\uparrow$41.3}} & \cellcolor{blue!10}41.4{\tiny\textcolor{red}{$\uparrow$31.6}} & \cellcolor{blue!10}94.7{\tiny\textcolor{red}{$\uparrow$59.7}} & \cellcolor{blue!10}75.6{\tiny\textcolor{red}{$\uparrow$48.8}} & \cellcolor{blue!10}42.5{\tiny\textcolor{red}{$\uparrow$14.5}} & \cellcolor{blue!10}40.7{\tiny\textcolor{red}{$\uparrow$13.6}} & \cellcolor{blue!10}31.4{\tiny\textcolor{red}{$\uparrow$14.0}} & \cellcolor{blue!10}38.9{\tiny\textcolor{red}{$\uparrow$11.4}} \\
			& w/o init.   & 77.7{\tiny\textcolor{red}{$\uparrow$23.2}} & 27.8{\tiny\textcolor{red}{$\uparrow$18.0}} & 83.5{\tiny\textcolor{red}{$\uparrow$48.5}} & 53.2{\tiny\textcolor{red}{$\uparrow$26.4}} & 40.4{\tiny\textcolor{red}{$\uparrow$12.4}} & 36.7{\tiny\textcolor{red}{$\uparrow$9.6}}  & 28.8{\tiny\textcolor{red}{$\uparrow$11.4}} & 35.3{\tiny\textcolor{red}{$\uparrow$7.8}} \\
			& w/o $\alpha$ & 64.4{\tiny\textcolor{red}{$\uparrow$9.9}}  & 18.1{\tiny\textcolor{red}{$\uparrow$8.3}}  & 50.0{\tiny\textcolor{red}{$\uparrow$15.0}} & 34.8{\tiny\textcolor{red}{$\uparrow$8.0}}  & 32.5{\tiny\textcolor{red}{$\uparrow$4.5}}  & 28.6{\tiny\textcolor{red}{$\uparrow$1.5}}  & 24.6{\tiny\textcolor{red}{$\uparrow$7.2}}  & 28.4{\tiny\textcolor{red}{$\uparrow$0.9}} \\
			& w/o evolve   & 54.5{\tiny\textcolor{green}{$\downarrow$0.0}}  & 9.8{\tiny\textcolor{green}{$\downarrow$0.0}}   & 35.0{\tiny\textcolor{green}{$\downarrow$0.0}}  & 26.8{\tiny\textcolor{green}{$\downarrow$0.0}}  & 28.0{\tiny\textcolor{green}{$\downarrow$0.0}}  & 27.1{\tiny\textcolor{green}{$\downarrow$0.0}}  & 17.4{\tiny\textcolor{green}{$\downarrow$0.0}}  & 27.5{\tiny\textcolor{green}{$\downarrow$0.0}} \\
			\midrule
			\multirow{4}{*}{$n=64$}
			& \alg{}      & \cellcolor{blue!10}81.5{\tiny\textcolor{red}{$\uparrow$25.3}} & \cellcolor{blue!10}35.4{\tiny\textcolor{red}{$\uparrow$23.6}} & \cellcolor{blue!10}82.5{\tiny\textcolor{red}{$\uparrow$44.1}} & \cellcolor{blue!10}60.6{\tiny\textcolor{red}{$\uparrow$32.0}} & \cellcolor{blue!10}44.2{\tiny\textcolor{red}{$\uparrow$15.6}} & \cellcolor{blue!10}42.9{\tiny\textcolor{red}{$\uparrow$14.9}} & \cellcolor{blue!10}32.9{\tiny\textcolor{red}{$\uparrow$14.8}} & \cellcolor{blue!10}40.9{\tiny\textcolor{red}{$\uparrow$12.6}} \\
			& w/o init.   & 81.2{\tiny\textcolor{red}{$\uparrow$25.0}} & 30.5{\tiny\textcolor{red}{$\uparrow$18.7}} & 80.6{\tiny\textcolor{red}{$\uparrow$42.2}} & 57.2{\tiny\textcolor{red}{$\uparrow$28.6}} & 40.7{\tiny\textcolor{red}{$\uparrow$12.1}} & 39.5{\tiny\textcolor{red}{$\uparrow$11.5}} & 30.4{\tiny\textcolor{red}{$\uparrow$12.3}} & 37.4{\tiny\textcolor{red}{$\uparrow$9.1}} \\
			& w/o $\alpha$ & 63.3{\tiny\textcolor{red}{$\uparrow$7.1}}  & 19.2{\tiny\textcolor{red}{$\uparrow$7.4}}  & 49.7{\tiny\textcolor{red}{$\uparrow$11.3}} & 37.0{\tiny\textcolor{red}{$\uparrow$8.4}}  & 31.7{\tiny\textcolor{red}{$\uparrow$3.1}}  & 32.9{\tiny\textcolor{red}{$\uparrow$4.9}}  & 24.6{\tiny\textcolor{red}{$\uparrow$6.5}}  & 27.8{\tiny\textcolor{green}{$\downarrow$0.5}} \\
			& w/o evolve   & 56.2{\tiny\textcolor{green}{$\downarrow$0.0}}  & 11.8{\tiny\textcolor{green}{$\downarrow$0.0}}  & 38.4{\tiny\textcolor{green}{$\downarrow$0.0}}  & 28.6{\tiny\textcolor{green}{$\downarrow$0.0}}  & 28.6{\tiny\textcolor{green}{$\downarrow$0.0}}  & 28.0{\tiny\textcolor{green}{$\downarrow$0.0}}  & 18.1{\tiny\textcolor{green}{$\downarrow$0.0}}  & 28.3{\tiny\textcolor{green}{$\downarrow$0.0}} \\
			\bottomrule
		\end{tabular}
	}
	\caption{Ablation study on context initialization, tuning $\alpha$, and context evolution when using llama-7B with varying number of LLMs.}
	\label{tab:ablation_llama_7b}
\end{table}

\begin{table}[h]  
	\centering
	\resizebox{0.95\textwidth}{!}{
		\begin{tabular}{XZYYYYYYYY}
			\toprule
			\textbf{Dataset} & \textbf{Method} & \textbf{MMLU} & \textbf{MATH} & \textbf{GPQA} & \textbf{Code} & \textbf{ALFWorld} & \textbf{SciWorld} & \textbf{GAIA} & \textbf{PDDL} \\
			\midrule
			\multirow{4}{*}{$n=4$}
			& {\alg{} (ours)} & \cellcolor{blue!10}86.9{\tiny\textcolor{red}{$\uparrow$20.2}} & \cellcolor{blue!10}23.4{\tiny\textcolor{red}{$\uparrow$13.5}} & \cellcolor{blue!10}58.9{\tiny\textcolor{red}{$\uparrow$20.3}} & \cellcolor{blue!10}47.5{\tiny\textcolor{red}{$\uparrow$17.6}} & \cellcolor{blue!10}42.9{\tiny\textcolor{red}{$\uparrow$10.9}} & \cellcolor{blue!10}39.9{\tiny\textcolor{red}{$\uparrow$9.5}} & \cellcolor{blue!10}30.5{\tiny\textcolor{red}{$\uparrow$9.6}} & \cellcolor{blue!10}40.9{\tiny\textcolor{red}{$\uparrow$8.1}}\\
			& w/o init. & 86.7{\tiny\textcolor{red}{$\uparrow$20.0}} & 16.9{\tiny\textcolor{red}{$\uparrow$7.0}} & 57.9{\tiny\textcolor{red}{$\uparrow$19.3}} & 45.2{\tiny\textcolor{red}{$\uparrow$15.3}} & 40.7{\tiny\textcolor{red}{$\uparrow$8.7}} & 47.8{\tiny\textcolor{red}{$\uparrow$17.4}} & 33.7{\tiny\textcolor{red}{$\uparrow$12.8}} & 37.0{\tiny\textcolor{red}{$\uparrow$4.2}}\\
			& w/o $\alpha$ & 81.5{\tiny\textcolor{red}{$\uparrow$14.8}} & 13.2{\tiny\textcolor{red}{$\uparrow$3.3}} & 50.9{\tiny\textcolor{red}{$\uparrow$12.3}} & 37.3{\tiny\textcolor{red}{$\uparrow$7.4}} & 32.1{\tiny\textcolor{red}{$\uparrow$0.1}} & 31.9{\tiny\textcolor{red}{$\uparrow$1.5}} & 24.2{\tiny\textcolor{red}{$\uparrow$3.3}} & 24.3{\tiny\textcolor{green}{$\downarrow$8.5}}\\
			& w/o evolve & 66.7{\tiny\textcolor{green}{$\downarrow$0.0}} & 9.9{\tiny\textcolor{green}{$\downarrow$0.0}} & 38.6{\tiny\textcolor{green}{$\downarrow$0.0}} & 29.9{\tiny\textcolor{green}{$\downarrow$0.0}} & 32.0{\tiny\textcolor{green}{$\downarrow$0.0}} & 30.4{\tiny\textcolor{green}{$\downarrow$0.0}} & 20.9{\tiny\textcolor{green}{$\downarrow$0.0}} & 32.8{\tiny\textcolor{green}{$\downarrow$0.0}} \\
			\midrule
			\multirow{4}{*}{$n=8$}
			& {\alg{} (ours)} & \cellcolor{blue!10}92.7{\tiny\textcolor{red}{$\uparrow$23.2}} & \cellcolor{blue!10}23.0{\tiny\textcolor{red}{$\uparrow$13.5}} & \cellcolor{blue!10}70.7{\tiny\textcolor{red}{$\uparrow$32.6}} & \cellcolor{blue!10}55.4{\tiny\textcolor{red}{$\uparrow$22.1}} & \cellcolor{blue!10}46.1{\tiny\textcolor{red}{$\uparrow$13.0}} & \cellcolor{blue!10}44.7{\tiny\textcolor{red}{$\uparrow$13.0}} & \cellcolor{blue!10}33.6{\tiny\textcolor{red}{$\uparrow$12.2}} & \cellcolor{blue!10}43.8{\tiny\textcolor{red}{$\uparrow$10.6}}\\
			& w/o init. & 92.5{\tiny\textcolor{red}{$\uparrow$23.0}} & 21.9{\tiny\textcolor{red}{$\uparrow$12.4}} & 70.3{\tiny\textcolor{red}{$\uparrow$32.2}} & 53.3{\tiny\textcolor{red}{$\uparrow$20.0}} & 45.3{\tiny\textcolor{red}{$\uparrow$12.2}} & 51.8{\tiny\textcolor{red}{$\uparrow$20.1}} & 38.4{\tiny\textcolor{red}{$\uparrow$17.0}} & 39.9{\tiny\textcolor{red}{$\uparrow$6.7}}\\
			& w/o $\alpha$ & 85.4{\tiny\textcolor{red}{$\uparrow$15.9}} & 16.9{\tiny\textcolor{red}{$\uparrow$7.4}} & 53.9{\tiny\textcolor{red}{$\uparrow$15.8}} & 43.3{\tiny\textcolor{red}{$\uparrow$10.0}} & 34.7{\tiny\textcolor{red}{$\uparrow$1.6}} & 37.6{\tiny\textcolor{red}{$\uparrow$5.9}} & 26.2{\tiny\textcolor{red}{$\uparrow$4.8}} & 29.0{\tiny\textcolor{green}{$\downarrow$4.2}}\\
			& w/o evolve & 69.5{\tiny\textcolor{green}{$\downarrow$0.0}} & 9.5{\tiny\textcolor{green}{$\downarrow$0.0}} & 38.1{\tiny\textcolor{green}{$\downarrow$0.0}} & 33.3{\tiny\textcolor{green}{$\downarrow$0.0}} & 33.1{\tiny\textcolor{green}{$\downarrow$0.0}} & 31.7{\tiny\textcolor{green}{$\downarrow$0.0}} & 21.4{\tiny\textcolor{green}{$\downarrow$0.0}} & 33.2{\tiny\textcolor{green}{$\downarrow$0.0}} \\
			\midrule
			\multirow{4}{*}{$n=16$}
			& {\alg{} (ours)} & \cellcolor{blue!10}94.9{\tiny\textcolor{red}{$\uparrow$22.9}} & \cellcolor{blue!10}27.3{\tiny\textcolor{red}{$\uparrow$16.2}} & \cellcolor{blue!10}73.2{\tiny\textcolor{red}{$\uparrow$33.8}} & \cellcolor{blue!10}61.5{\tiny\textcolor{red}{$\uparrow$25.3}} & \cellcolor{blue!10}49.7{\tiny\textcolor{red}{$\uparrow$15.3}} & \cellcolor{blue!10}48.9{\tiny\textcolor{red}{$\uparrow$15.6}} & \cellcolor{blue!10}36.0{\tiny\textcolor{red}{$\uparrow$12.9}} & \cellcolor{blue!10}45.2{\tiny\textcolor{red}{$\uparrow$10.2}}\\
			& w/o init. & 94.7{\tiny\textcolor{red}{$\uparrow$22.7}} & 25.0{\tiny\textcolor{red}{$\uparrow$13.9}} & 72.3{\tiny\textcolor{red}{$\uparrow$32.9}} & 60.4{\tiny\textcolor{red}{$\uparrow$24.2}} & 48.5{\tiny\textcolor{red}{$\uparrow$14.1}} & 54.5{\tiny\textcolor{red}{$\uparrow$21.2}} & 39.1{\tiny\textcolor{red}{$\uparrow$16.0}} & 41.8{\tiny\textcolor{red}{$\uparrow$6.8}}\\
			& w/o $\alpha$ & 87.2{\tiny\textcolor{red}{$\uparrow$15.2}} & 18.7{\tiny\textcolor{red}{$\uparrow$7.6}} & 57.7{\tiny\textcolor{red}{$\uparrow$18.3}} & 47.6{\tiny\textcolor{red}{$\uparrow$11.4}} & 36.3{\tiny\textcolor{red}{$\uparrow$1.9}} & 39.9{\tiny\textcolor{red}{$\uparrow$6.6}} & 27.0{\tiny\textcolor{red}{$\uparrow$3.9}} & 31.0{\tiny\textcolor{green}{$\downarrow$4.0}}\\
			& w/o evolve & 72.0{\tiny\textcolor{green}{$\downarrow$0.0}} & 11.1{\tiny\textcolor{green}{$\downarrow$0.0}} & 39.4{\tiny\textcolor{green}{$\downarrow$0.0}} & 36.2{\tiny\textcolor{green}{$\downarrow$0.0}} & 34.4{\tiny\textcolor{green}{$\downarrow$0.0}} & 33.3{\tiny\textcolor{green}{$\downarrow$0.0}} & 23.1{\tiny\textcolor{green}{$\downarrow$0.0}} & 35.0{\tiny\textcolor{green}{$\downarrow$0.0}} \\
			\midrule
			\multirow{4}{*}{$n=32$}
			& {\alg{} (ours)} & \cellcolor{blue!10}95.5{\tiny\textcolor{red}{$\uparrow$21.8}} & \cellcolor{blue!10}27.9{\tiny\textcolor{red}{$\uparrow$15.3}} & \cellcolor{blue!10}74.0{\tiny\textcolor{red}{$\uparrow$32.3}} & \cellcolor{blue!10}63.1{\tiny\textcolor{red}{$\uparrow$24.6}} & \cellcolor{blue!10}50.8{\tiny\textcolor{red}{$\uparrow$15.5}} & \cellcolor{blue!10}49.9{\tiny\textcolor{red}{$\uparrow$14.7}} & \cellcolor{blue!10}36.6{\tiny\textcolor{red}{$\uparrow$12.3}} & \cellcolor{blue!10}46.7{\tiny\textcolor{red}{$\uparrow$10.5}}\\
			& w/o init. & 95.3{\tiny\textcolor{red}{$\uparrow$21.6}} & 25.7{\tiny\textcolor{red}{$\uparrow$13.1}} & 73.0{\tiny\textcolor{red}{$\uparrow$31.3}} & 62.0{\tiny\textcolor{red}{$\uparrow$23.5}} & 49.4{\tiny\textcolor{red}{$\uparrow$14.1}} & 55.2{\tiny\textcolor{red}{$\uparrow$20.0}} & 39.8{\tiny\textcolor{red}{$\uparrow$15.5}} & 43.0{\tiny\textcolor{red}{$\uparrow$6.8}}\\
			& w/o $\alpha$ & 87.6{\tiny\textcolor{red}{$\uparrow$13.9}} & 19.1{\tiny\textcolor{red}{$\uparrow$6.5}} & 58.5{\tiny\textcolor{red}{$\uparrow$16.8}} & 48.1{\tiny\textcolor{red}{$\uparrow$9.6}} & 36.7{\tiny\textcolor{red}{$\uparrow$1.4}} & 40.5{\tiny\textcolor{red}{$\uparrow$5.3}} & 27.4{\tiny\textcolor{red}{$\uparrow$3.1}} & 31.7{\tiny\textcolor{green}{$\downarrow$4.5}}\\
			& w/o evolve & 73.7{\tiny\textcolor{green}{$\downarrow$0.0}} & 12.6{\tiny\textcolor{green}{$\downarrow$0.0}} & 41.7{\tiny\textcolor{green}{$\downarrow$0.0}} & 38.5{\tiny\textcolor{green}{$\downarrow$0.0}} & 35.3{\tiny\textcolor{green}{$\downarrow$0.0}} & 35.2{\tiny\textcolor{green}{$\downarrow$0.0}} & 24.3{\tiny\textcolor{green}{$\downarrow$0.0}} & 36.2{\tiny\textcolor{green}{$\downarrow$0.0}} \\
			\midrule
			\multirow{4}{*}{$n=64$}
			& {\alg{} (ours)} & \cellcolor{blue!10}95.8{\tiny\textcolor{red}{$\uparrow$21.6}} & \cellcolor{blue!10}28.1{\tiny\textcolor{red}{$\uparrow$14.9}} & \cellcolor{blue!10}74.8{\tiny\textcolor{red}{$\uparrow$31.3}} & \cellcolor{blue!10}63.5{\tiny\textcolor{red}{$\uparrow$23.5}} & \cellcolor{blue!10}51.2{\tiny\textcolor{red}{$\uparrow$15.2}} & \cellcolor{blue!10}50.3{\tiny\textcolor{red}{$\uparrow$13.9}} & \cellcolor{blue!10}36.9{\tiny\textcolor{red}{$\uparrow$11.9}} & \cellcolor{blue!10}47.1{\tiny\textcolor{red}{$\uparrow$10.3}}\\
			& w/o init. & 95.6{\tiny\textcolor{red}{$\uparrow$21.4}} & 25.9{\tiny\textcolor{red}{$\uparrow$12.7}} & 73.5{\tiny\textcolor{red}{$\uparrow$30.0}} & 62.2{\tiny\textcolor{red}{$\uparrow$22.2}} & 49.6{\tiny\textcolor{red}{$\uparrow$13.6}} & 55.6{\tiny\textcolor{red}{$\uparrow$19.2}} & 40.1{\tiny\textcolor{red}{$\uparrow$15.1}} & 43.3{\tiny\textcolor{red}{$\uparrow$6.5}}\\
			& w/o $\alpha$ & 87.8{\tiny\textcolor{red}{$\uparrow$13.6}} & 19.4{\tiny\textcolor{red}{$\uparrow$6.2}} & 59.1{\tiny\textcolor{red}{$\uparrow$15.6}} & 48.6{\tiny\textcolor{red}{$\uparrow$8.6}} & 37.0{\tiny\textcolor{red}{$\uparrow$1.0}} & 40.8{\tiny\textcolor{red}{$\uparrow$4.4}} & 27.7{\tiny\textcolor{red}{$\uparrow$2.7}} & 32.0{\tiny\textcolor{green}{$\downarrow$4.8}}\\
			& w/o evolve & 74.2{\tiny\textcolor{green}{$\downarrow$0.0}} & 13.2{\tiny\textcolor{green}{$\downarrow$0.0}} & 43.5{\tiny\textcolor{green}{$\downarrow$0.0}} & 40.0{\tiny\textcolor{green}{$\downarrow$0.0}} & 36.0{\tiny\textcolor{green}{$\downarrow$0.0}} & 36.4{\tiny\textcolor{green}{$\downarrow$0.0}} & 25.0{\tiny\textcolor{green}{$\downarrow$0.0}} & 36.8{\tiny\textcolor{green}{$\downarrow$0.0}} \\
			\bottomrule
		\end{tabular}
	}
	\caption{Ablation study on context initialization, tuning $\alpha$, and context evolution when using llama-13B with varying number of LLMs.}
	\label{tab:ablation_llama_13b}
\end{table}

\begin{table}[h]  
	\centering
	\resizebox{0.95\textwidth}{!}{
		\begin{tabular}{XZYYYYYYYY}
			\toprule
			\textbf{Dataset} & \textbf{Method} & \textbf{MMLU} & \textbf{MATH} & \textbf{GPQA} & \textbf{Code} & \textbf{ALFWorld} & \textbf{SciWorld} & \textbf{GAIA} & \textbf{PDDL} \\
			\midrule
			\multirow{4}{*}{$n=4$}
			& {\alg{} (ours)} & \cellcolor{blue!10}95.6{\tiny\textcolor{red}{$\uparrow$18.7}} & \cellcolor{blue!10}40.3{\tiny\textcolor{red}{$\uparrow$9.4}} & \cellcolor{blue!10}78.7{\tiny\textcolor{red}{$\uparrow$25.6}} & \cellcolor{blue!10}70.6{\tiny\textcolor{red}{$\uparrow$24.7}} & \cellcolor{blue!10}69.5{\tiny\textcolor{red}{$\uparrow$17.2}} & \cellcolor{blue!10}65.5{\tiny\textcolor{red}{$\uparrow$15.9}} & \cellcolor{blue!10}49.6{\tiny\textcolor{red}{$\uparrow$15.5}} & \cellcolor{blue!10}66.5{\tiny\textcolor{red}{$\uparrow$12.2}}\\
			& w/o init. & 95.4{\tiny\textcolor{red}{$\uparrow$18.5}} & 39.9{\tiny\textcolor{red}{$\uparrow$9.0}} & 88.5{\tiny\textcolor{red}{$\uparrow$35.4}} & 67.8{\tiny\textcolor{red}{$\uparrow$21.9}} & 64.8{\tiny\textcolor{red}{$\uparrow$12.5}} & 62.9{\tiny\textcolor{red}{$\uparrow$13.3}} & 47.0{\tiny\textcolor{red}{$\uparrow$12.9}} & 62.6{\tiny\textcolor{red}{$\uparrow$8.3}}\\
			& w/o $\alpha$ & 86.7{\tiny\textcolor{red}{$\uparrow$9.8}} & 36.9{\tiny\textcolor{red}{$\uparrow$6.0}} & 72.3{\tiny\textcolor{red}{$\uparrow$19.2}} & 53.9{\tiny\textcolor{red}{$\uparrow$8.0}} & 43.3{\tiny\textcolor{green}{$\downarrow$9.0}} & 47.1{\tiny\textcolor{green}{$\downarrow$2.5}} & 38.0{\tiny\textcolor{green}{$\downarrow$-0.1}} & 40.3{\tiny\textcolor{green}{$\downarrow$14.0}}\\
			& w/o evolve & 76.9{\tiny\textcolor{green}{$\downarrow$0.0}} & 30.9{\tiny\textcolor{green}{$\downarrow$0.0}} & 53.1{\tiny\textcolor{green}{$\downarrow$0.0}} & 45.9{\tiny\textcolor{green}{$\downarrow$0.0}} & 52.3{\tiny\textcolor{green}{$\downarrow$0.0}} & 49.6{\tiny\textcolor{green}{$\downarrow$0.0}} & 34.1{\tiny\textcolor{green}{$\downarrow$0.0}} & 54.3{\tiny\textcolor{green}{$\downarrow$0.0}} \\
			\midrule
			\multirow{4}{*}{$n=8$}
			& {\alg{} (ours)} & \cellcolor{blue!10}95.4{\tiny\textcolor{red}{$\uparrow$11.9}} & \cellcolor{blue!10}43.5{\tiny\textcolor{red}{$\uparrow$8.3}} & \cellcolor{blue!10}87.6{\tiny\textcolor{red}{$\uparrow$26.6}} & \cellcolor{blue!10}81.0{\tiny\textcolor{red}{$\uparrow$30.5}} & \cellcolor{blue!10}78.9{\tiny\textcolor{red}{$\uparrow$24.5}} & \cellcolor{blue!10}76.2{\tiny\textcolor{red}{$\uparrow$23.5}} & \cellcolor{blue!10}58.7{\tiny\textcolor{red}{$\uparrow$23.8}} & \cellcolor{blue!10}73.6{\tiny\textcolor{red}{$\uparrow$18.3}}\\
			& w/o init. & 95.1{\tiny\textcolor{red}{$\uparrow$11.6}} & 43.0{\tiny\textcolor{red}{$\uparrow$7.8}} & 86.3{\tiny\textcolor{red}{$\uparrow$25.3}} & 80.0{\tiny\textcolor{red}{$\uparrow$29.5}} & 75.8{\tiny\textcolor{red}{$\uparrow$21.4}} & 71.7{\tiny\textcolor{red}{$\uparrow$18.5}} & 56.1{\tiny\textcolor{red}{$\uparrow$21.2}} & 71.3{\tiny\textcolor{red}{$\uparrow$16.0}}\\
			& w/o $\alpha$ & 94.0{\tiny\textcolor{red}{$\uparrow$10.5}} & 39.3{\tiny\textcolor{red}{$\uparrow$4.1}} & 78.0{\tiny\textcolor{red}{$\uparrow$17.0}} & 63.8{\tiny\textcolor{red}{$\uparrow$13.3}} & 49.3{\tiny\textcolor{green}{$\downarrow$5.1}} & 49.0{\tiny\textcolor{green}{$\downarrow$3.7}} & 42.6{\tiny\textcolor{red}{$\uparrow$7.7}} & 47.8{\tiny\textcolor{green}{$\downarrow$7.5}}\\
			& w/o evolve & 83.5{\tiny\textcolor{green}{$\downarrow$0.0}} & 35.2{\tiny\textcolor{green}{$\downarrow$0.0}} & 61.0{\tiny\textcolor{green}{$\downarrow$0.0}} & 50.5{\tiny\textcolor{green}{$\downarrow$0.0}} & 54.4{\tiny\textcolor{green}{$\downarrow$0.0}} & 52.7{\tiny\textcolor{green}{$\downarrow$0.0}} & 34.9{\tiny\textcolor{green}{$\downarrow$0.0}} & 55.3{\tiny\textcolor{green}{$\downarrow$0.0}} \\
			\midrule
			\multirow{4}{*}{$n=16$}
			& {\alg{} (ours)} & \cellcolor{blue!10}93.9{\tiny\textcolor{red}{$\uparrow$9.9}} & \cellcolor{blue!10}51.5{\tiny\textcolor{red}{$\uparrow$15.2}} & \cellcolor{blue!10}91.3{\tiny\textcolor{red}{$\uparrow$29.5}} & \cellcolor{blue!10}97.2{\tiny\textcolor{red}{$\uparrow$45.4}} & \cellcolor{blue!10}84.9{\tiny\textcolor{red}{$\uparrow$29.4}} & \cellcolor{blue!10}81.7{\tiny\textcolor{red}{$\uparrow$27.4}} & \cellcolor{blue!10}61.5{\tiny\textcolor{red}{$\uparrow$25.6}} & \cellcolor{blue!10}78.1{\tiny\textcolor{red}{$\uparrow$22.8}}\\
			& w/o init. & 93.2{\tiny\textcolor{red}{$\uparrow$9.2}} & 47.3{\tiny\textcolor{red}{$\uparrow$11.0}} & 90.9{\tiny\textcolor{red}{$\uparrow$29.1}} & 96.9{\tiny\textcolor{red}{$\uparrow$45.1}} & 81.3{\tiny\textcolor{red}{$\uparrow$25.8}} & 79.3{\tiny\textcolor{red}{$\uparrow$23.0}} & 56.6{\tiny\textcolor{red}{$\uparrow$20.7}} & 73.5{\tiny\textcolor{red}{$\uparrow$18.2}}\\
			& w/o $\alpha$ & 93.2{\tiny\textcolor{red}{$\uparrow$9.2}} & 40.5{\tiny\textcolor{red}{$\uparrow$4.2}} & 78.6{\tiny\textcolor{red}{$\uparrow$16.8}} & 72.2{\tiny\textcolor{red}{$\uparrow$20.4}} & 64.2{\tiny\textcolor{green}{$\downarrow$-8.7}} & 64.8{\tiny\textcolor{red}{$\uparrow$10.5}} & 47.0{\tiny\textcolor{red}{$\uparrow$11.1}} & 59.3{\tiny\textcolor{red}{$\uparrow$4.0}}\\
			& w/o evolve & 84.0{\tiny\textcolor{green}{$\downarrow$0.0}} & 36.3{\tiny\textcolor{green}{$\downarrow$0.0}} & 61.8{\tiny\textcolor{green}{$\downarrow$0.0}} & 51.8{\tiny\textcolor{green}{$\downarrow$0.0}} & 55.5{\tiny\textcolor{green}{$\downarrow$0.0}} & 54.3{\tiny\textcolor{green}{$\downarrow$0.0}} & 35.9{\tiny\textcolor{green}{$\downarrow$0.0}} & 55.3{\tiny\textcolor{green}{$\downarrow$0.0}} \\
			\midrule
						\multirow{4}{*}{$n=32$}
			& {\alg{} (ours)} & \cellcolor{blue!10}97.0{\tiny\textcolor{red}{$\uparrow$10.2}} & \cellcolor{blue!10}54.5{\tiny\textcolor{red}{$\uparrow$17.3}} & \cellcolor{blue!10}95.1{\tiny\textcolor{red}{$\uparrow$28.5}} & \cellcolor{blue!10}93.7{\tiny\textcolor{red}{$\uparrow$38.2}} & \cellcolor{blue!10}88.5{\tiny\textcolor{red}{$\uparrow$32.3}} & \cellcolor{blue!10}86.0{\tiny\textcolor{red}{$\uparrow$30.5}} & \cellcolor{blue!10}65.7{\tiny\textcolor{red}{$\uparrow$29.5}} & \cellcolor{blue!10}79.8{\tiny\textcolor{red}{$\uparrow$23.5}}\\
			& w/o init. & 96.8{\tiny\textcolor{red}{$\uparrow$10.0}} & 51.5{\tiny\textcolor{red}{$\uparrow$14.3}} & 94.8{\tiny\textcolor{red}{$\uparrow$27.9}} & 93.4{\tiny\textcolor{red}{$\uparrow$37.9}} & 85.8{\tiny\textcolor{red}{$\uparrow$29.6}} & 81.8{\tiny\textcolor{red}{$\uparrow$26.3}} & 61.2{\tiny\textcolor{red}{$\uparrow$24.3}} & 77.1{\tiny\textcolor{red}{$\uparrow$20.8}}\\
			& w/o $\alpha$ & 95.2{\tiny\textcolor{red}{$\uparrow$8.4}} & 44.0{\tiny\textcolor{red}{$\uparrow$6.8}} & 83.9{\tiny\textcolor{red}{$\uparrow$17.3}} & 68.6{\tiny\textcolor{red}{$\uparrow$13.1}} & 57.3{\tiny\textcolor{green}{$\downarrow$-1.1}} & 57.9{\tiny\textcolor{red}{$\uparrow$2.4}} & 50.9{\tiny\textcolor{red}{$\uparrow$14.7}} & 61.3{\tiny\textcolor{red}{$\uparrow$5.0}}\\
			& w/o evolve & 86.8{\tiny\textcolor{green}{$\downarrow$0.0}} & 37.2{\tiny\textcolor{green}{$\downarrow$0.0}} & 66.6{\tiny\textcolor{green}{$\downarrow$0.0}} & 55.5{\tiny\textcolor{green}{$\downarrow$0.0}} & 56.2{\tiny\textcolor{green}{$\downarrow$0.0}} & 55.5{\tiny\textcolor{green}{$\downarrow$0.0}} & 36.2{\tiny\textcolor{green}{$\downarrow$0.0}} & 56.3{\tiny\textcolor{green}{$\downarrow$0.0}} \\
			\midrule
			\multirow{4}{*}{$n=64$}
			& {\alg{} (ours)} & \cellcolor{blue!10}96.3{\tiny\textcolor{red}{$\uparrow$5.7}} & \cellcolor{blue!10}58.0{\tiny\textcolor{red}{$\uparrow$19.3}} & \cellcolor{blue!10}96.9{\tiny\textcolor{red}{$\uparrow$25.1}} & \cellcolor{blue!10}92.0{\tiny\textcolor{red}{$\uparrow$34.2}} & \cellcolor{blue!10}90.8{\tiny\textcolor{red}{$\uparrow$33.5}} & \cellcolor{blue!10}88.9{\tiny\textcolor{red}{$\uparrow$31.9}} & \cellcolor{blue!10}68.2{\tiny\textcolor{red}{$\uparrow$32.0}} & \cellcolor{blue!10}82.2{\tiny\textcolor{red}{$\uparrow$26.1}}\\
			& w/o init. & 96.1{\tiny\textcolor{red}{$\uparrow$5.5}} & 53.7{\tiny\textcolor{red}{$\uparrow$15.0}} & 96.9{\tiny\textcolor{red}{$\uparrow$25.1}} & 91.7{\tiny\textcolor{red}{$\uparrow$33.9}} & 88.7{\tiny\textcolor{red}{$\uparrow$31.4}} & 85.5{\tiny\textcolor{red}{$\uparrow$28.5}} & 65.0{\tiny\textcolor{red}{$\uparrow$28.8}} & 77.5{\tiny\textcolor{red}{$\uparrow$21.4}}\\
			& w/o $\alpha$ & 95.7{\tiny\textcolor{red}{$\uparrow$5.1}} & 43.3{\tiny\textcolor{red}{$\uparrow$4.6}} & 82.8{\tiny\textcolor{red}{$\uparrow$11.0}} & 69.6{\tiny\textcolor{red}{$\uparrow$11.8}} & 66.9{\tiny\textcolor{green}{$\downarrow$-9.6}} & 70.3{\tiny\textcolor{red}{$\uparrow$13.3}} & 41.8{\tiny\textcolor{red}{$\uparrow$5.6}} & 55.6{\tiny\textcolor{red}{$\uparrow$-0.5}}\\
			& w/o evolve & 90.6{\tiny\textcolor{green}{$\downarrow$0.0}} & 38.7{\tiny\textcolor{green}{$\downarrow$0.0}} & 71.8{\tiny\textcolor{green}{$\downarrow$0.0}} & 57.8{\tiny\textcolor{green}{$\downarrow$0.0}} & 57.3{\tiny\textcolor{green}{$\downarrow$0.0}} & 57.0{\tiny\textcolor{green}{$\downarrow$0.0}} & 36.2{\tiny\textcolor{green}{$\downarrow$0.0}} & 56.1{\tiny\textcolor{green}{$\downarrow$0.0}} \\
			\bottomrule
			\end{tabular}
		}
	\caption{Ablation study on context initialization, tuning $\alpha$, and context evolution when using llama-70B with varying number of LLMs.}
\label{tab:ablation_llama_70b}
\end{table}
\clearpage
\subsubsection{Run-time}
\label{sec:runtime}
To demonstrate the efficiency of our context initialization, we verify its runtime with varying the number of LLMs.
\begin{figure*}[h]
	\centering
	\includegraphics[width=0.95\linewidth]{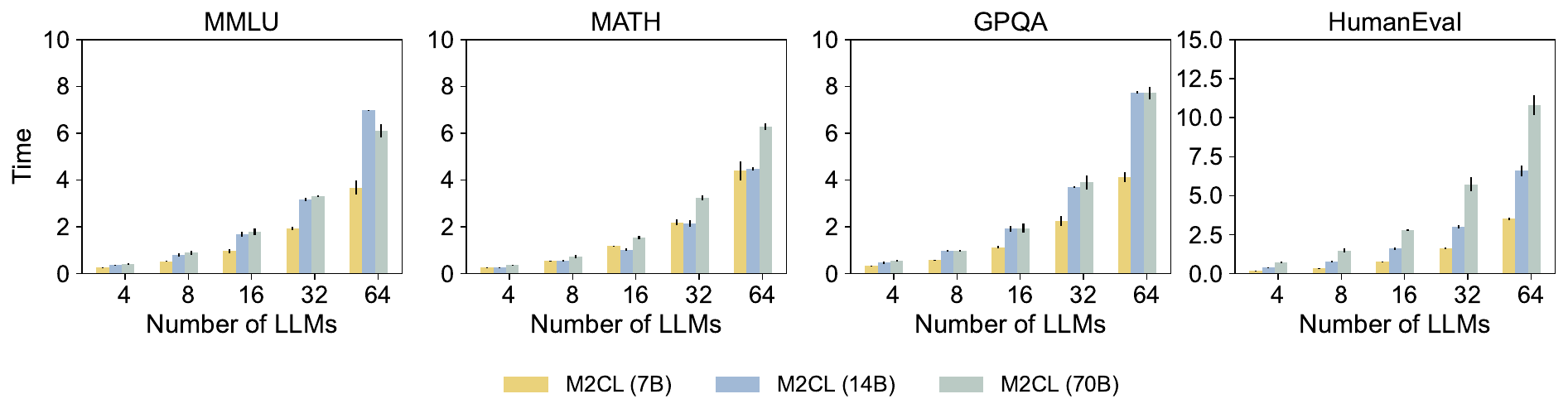}
	\includegraphics[width=0.95\linewidth]{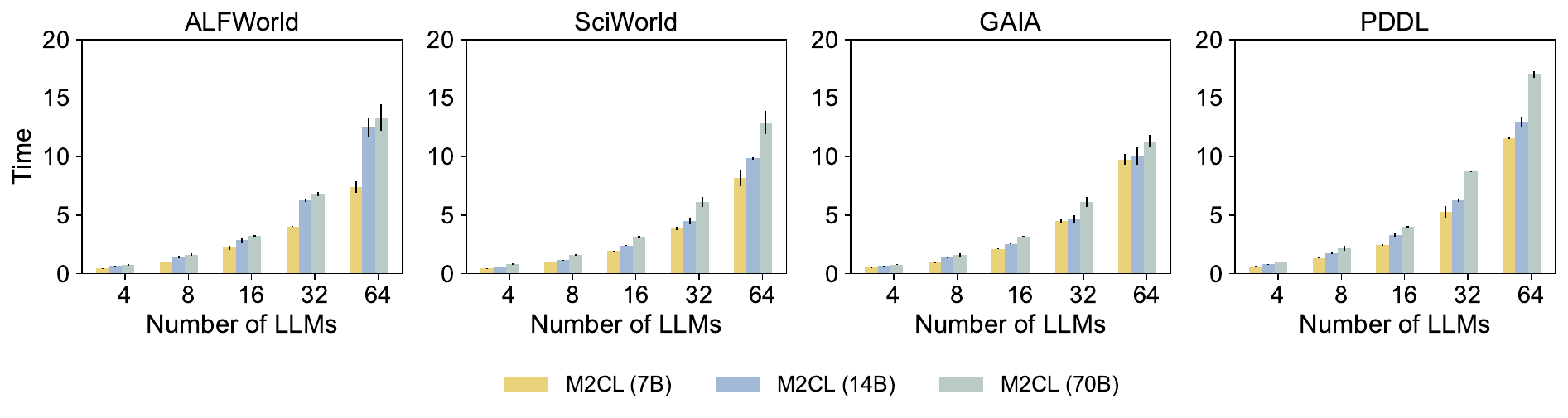}
	\caption{Runtime of initialization. Uncertainty intervals depict standard deviation over three seeds.}
	\label{fig:runtime_init}
\end{figure*}
Then, we evaluate the runtime of \alg{} compared with baseline algorithms for average testing time, utilizing the same model size on 8 NVIDIA H800 GPUs. As illustrated by \cref{fig:runtime_1}, the runtime of \alg{} is slightly longer than other multi-LLM discussion methods as the runtime of context generators is negligible compared with the inference time of LLM, which substantiates the low computational cost of \alg{}.
\begin{figure*}[h]
	\centering
	\includegraphics[width=0.95\linewidth]{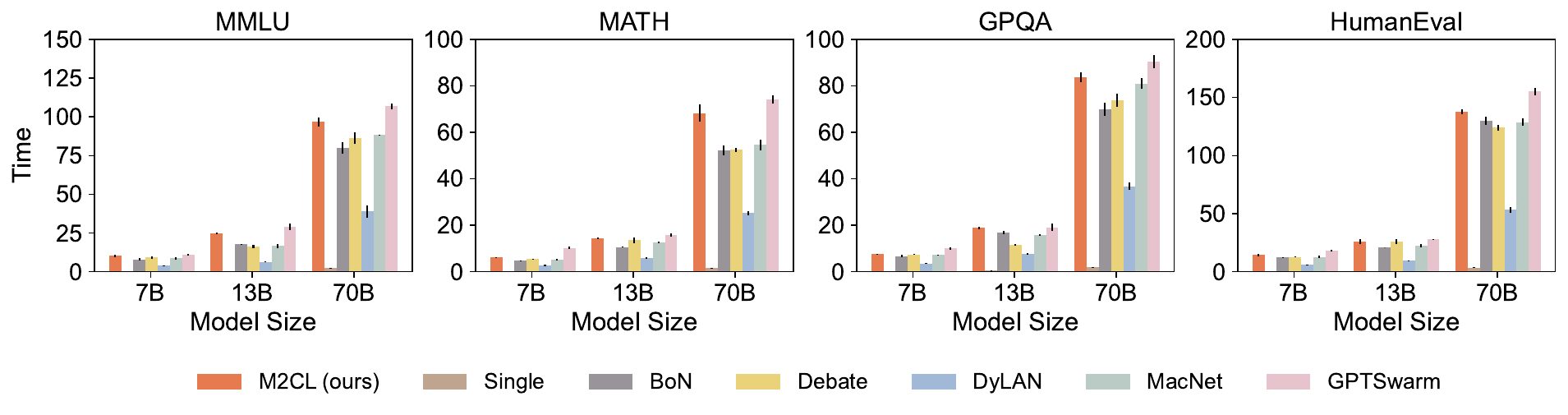}
	\includegraphics[width=0.95\linewidth]{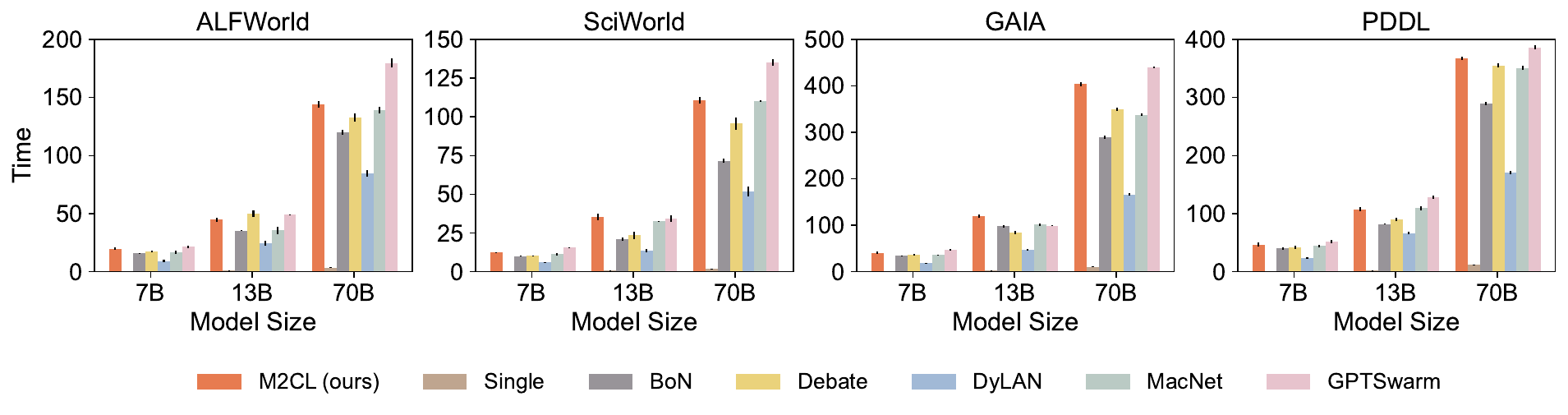}
	\caption{Runtime when varying the size of the LLama series models. The number of LLMs is $8$. Uncertainty intervals depict standard deviation over three seeds.}
	\label{fig:runtime_1}
\end{figure*}
\clearpage
\section{Case study}
\label{sec:case_study}
We used a problem from the MATH~\cite{hendrycks2021measuring}. The number of LLMs is set as $8$. For each LLM, we present their instructions, responses, and final answers for $4$ discussion rounds.
\subsection{Case study of \alg{} (ours)}
We provide the case of \alg{} solving the problem. As illustrated in \cref{fig:first_round,fig:second_round,fig:third_round,fig:last_round}, we observe that the generated contexts evolve progressively to enforce stricter collaboration among LLMs. Initially, the instructions merely encourage LLMs to pay attention to others’ responses. In subsequent rounds, they guide LLMs to cross-check each other’s answers and eventually require reaching full agreement. Correspondingly, the LLMs’ outputs transition from diverse, potentially conflicting answers to a single, consistent solution. This behavior demonstrates that our method effectively modulates context evolution to guide toward consensus without sacrificing initial creativity.
\begin{figure*}[h]
	\centering
	\includegraphics[width=0.99\linewidth]{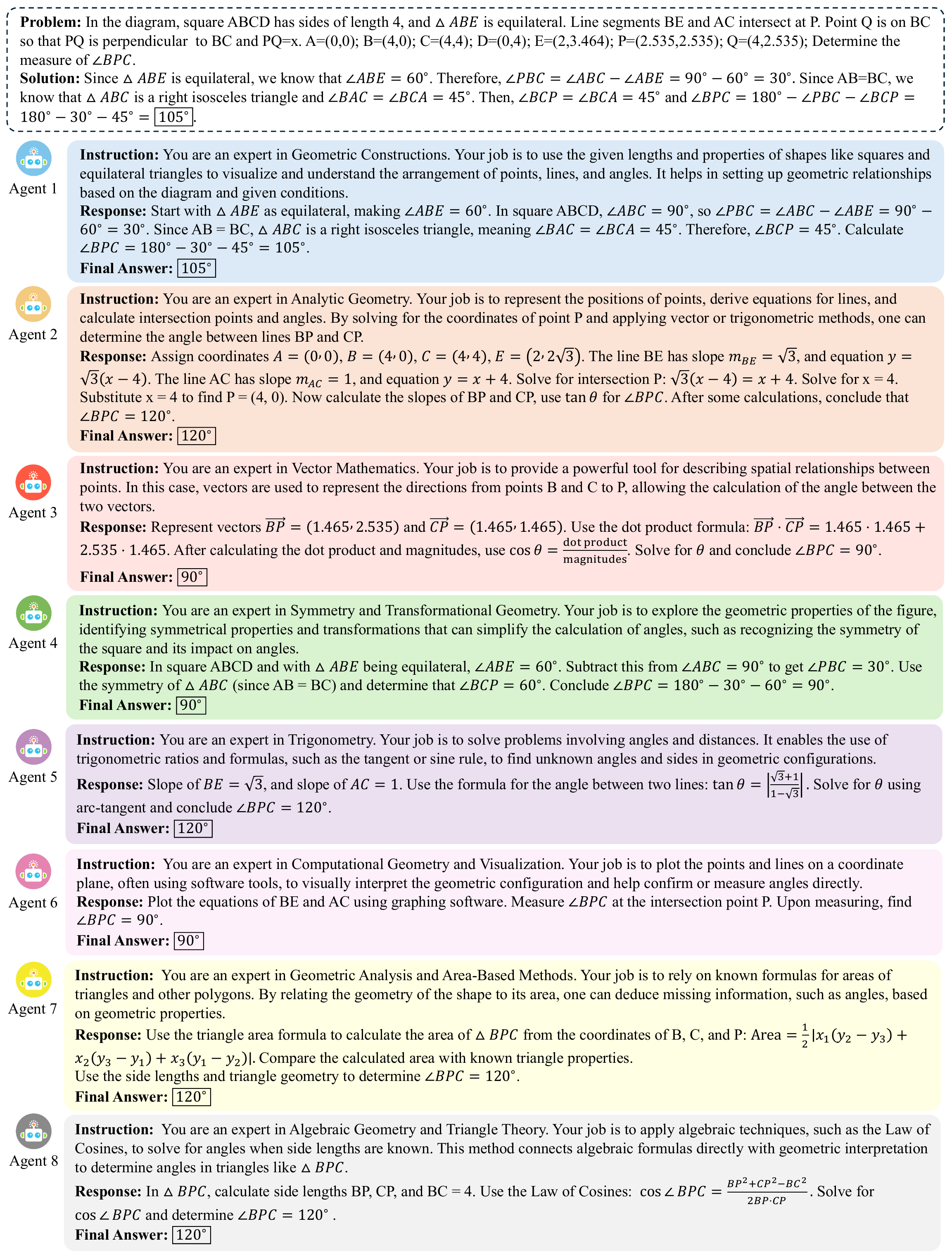}
	\caption{Visualization of \alg{} at the \textbf{first round}.}
	\label{fig:first_round}
\end{figure*}

\begin{figure*}[h]
	\centering
	\includegraphics[width=0.99\linewidth]{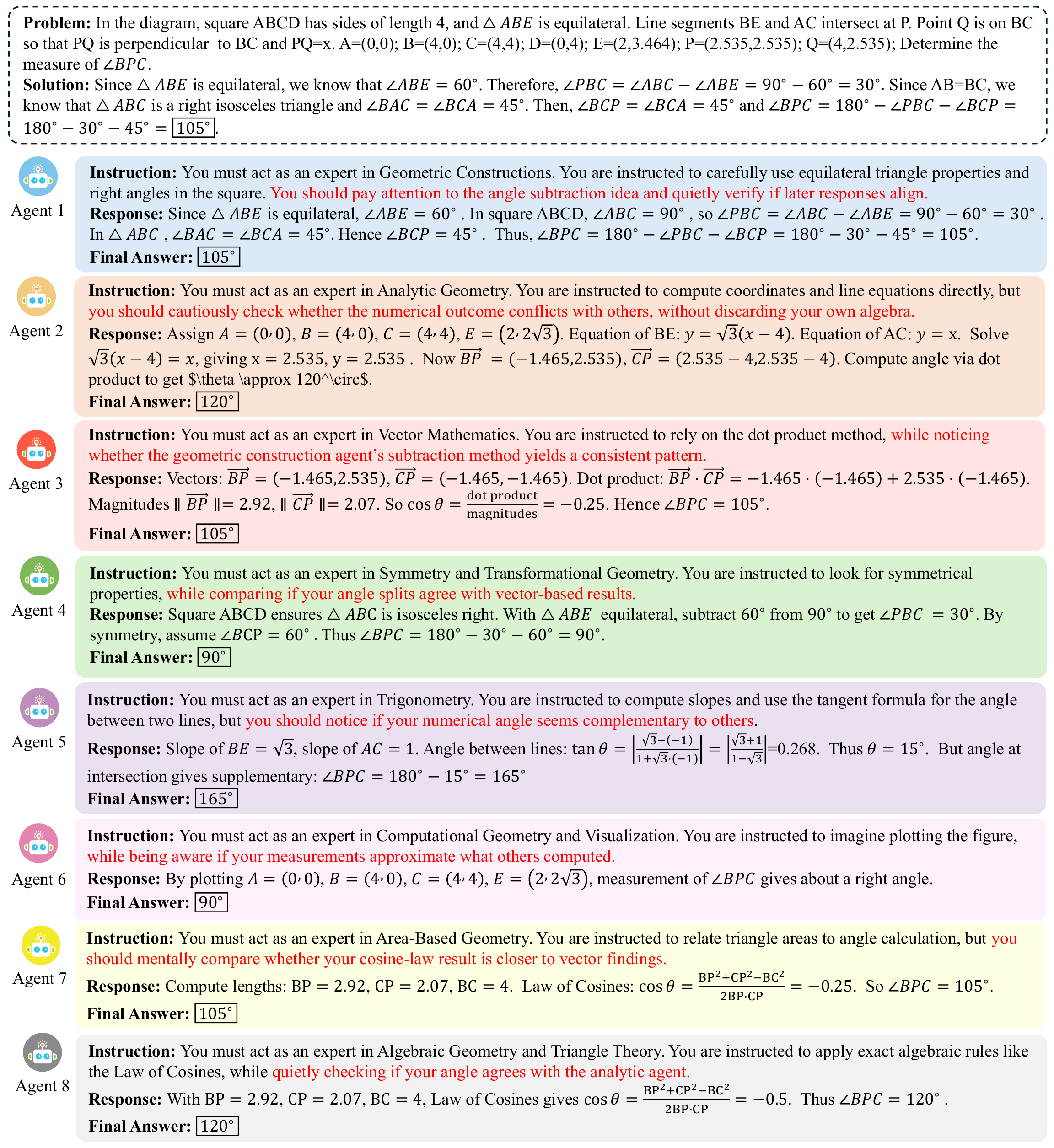}
	\caption{Visualization of \alg{} at the \textbf{second round}. We highlight the guidance on how to cooperate with other LLMs. At the beginning, instructions encourage diverse perspectives and consideration of others’ responses, but the requirements for discussion consistency are not yet strict.}
	\label{fig:second_round}
\end{figure*}

\begin{figure*}[h]
	\centering
	\includegraphics[width=0.99\linewidth]{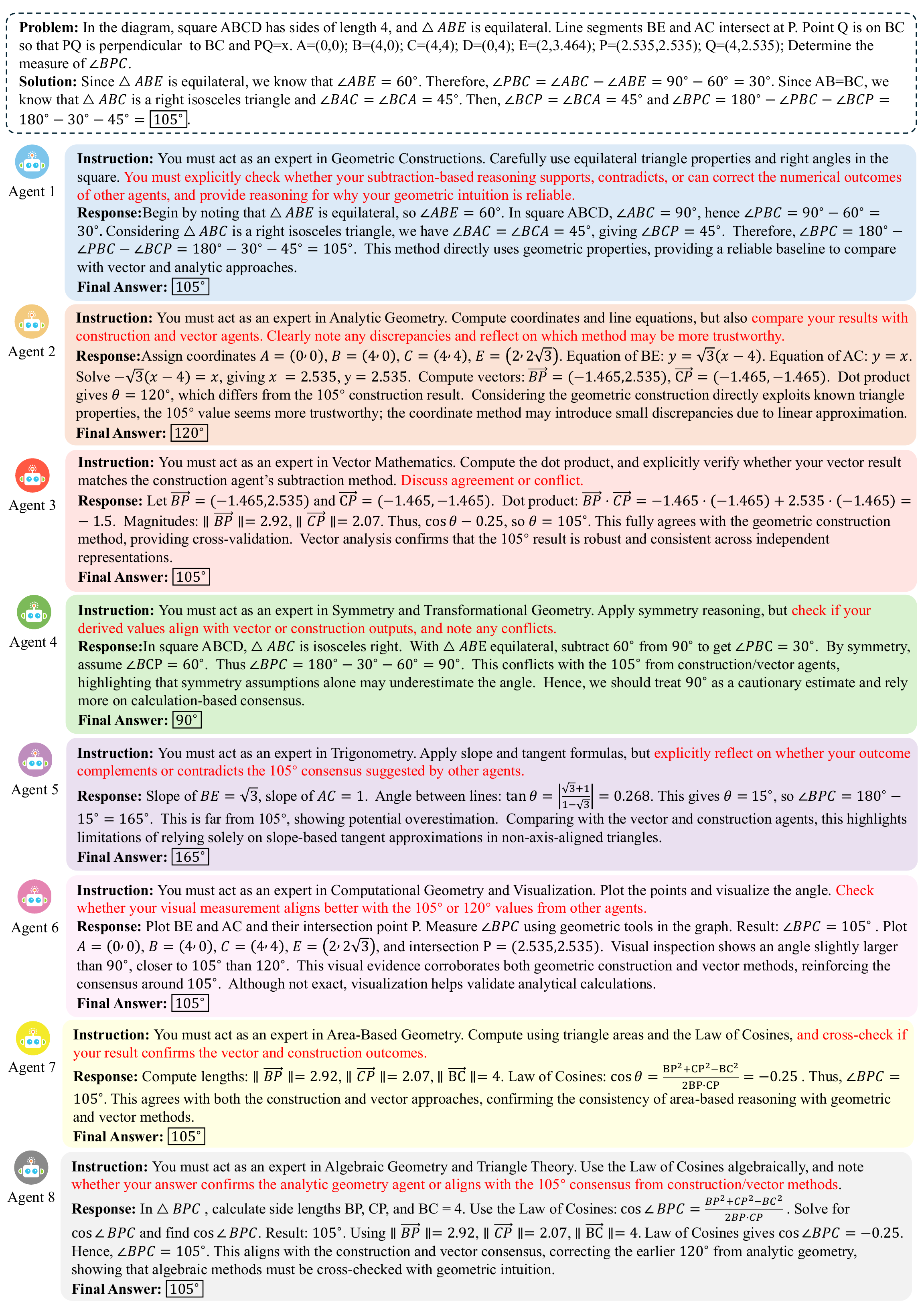}
	\caption{Visualization of \alg{} at the \textbf{third round}. We highlight the guidance on how to cooperate with other LLMs. As the discussion progresses, the instructions increasingly enforce stricter requirements for cross-checking and aligning answers, helping the models converge toward a consistent solution.}
	\label{fig:third_round}
\end{figure*}

\begin{figure*}[h]
	\centering
	\includegraphics[width=0.99\linewidth]{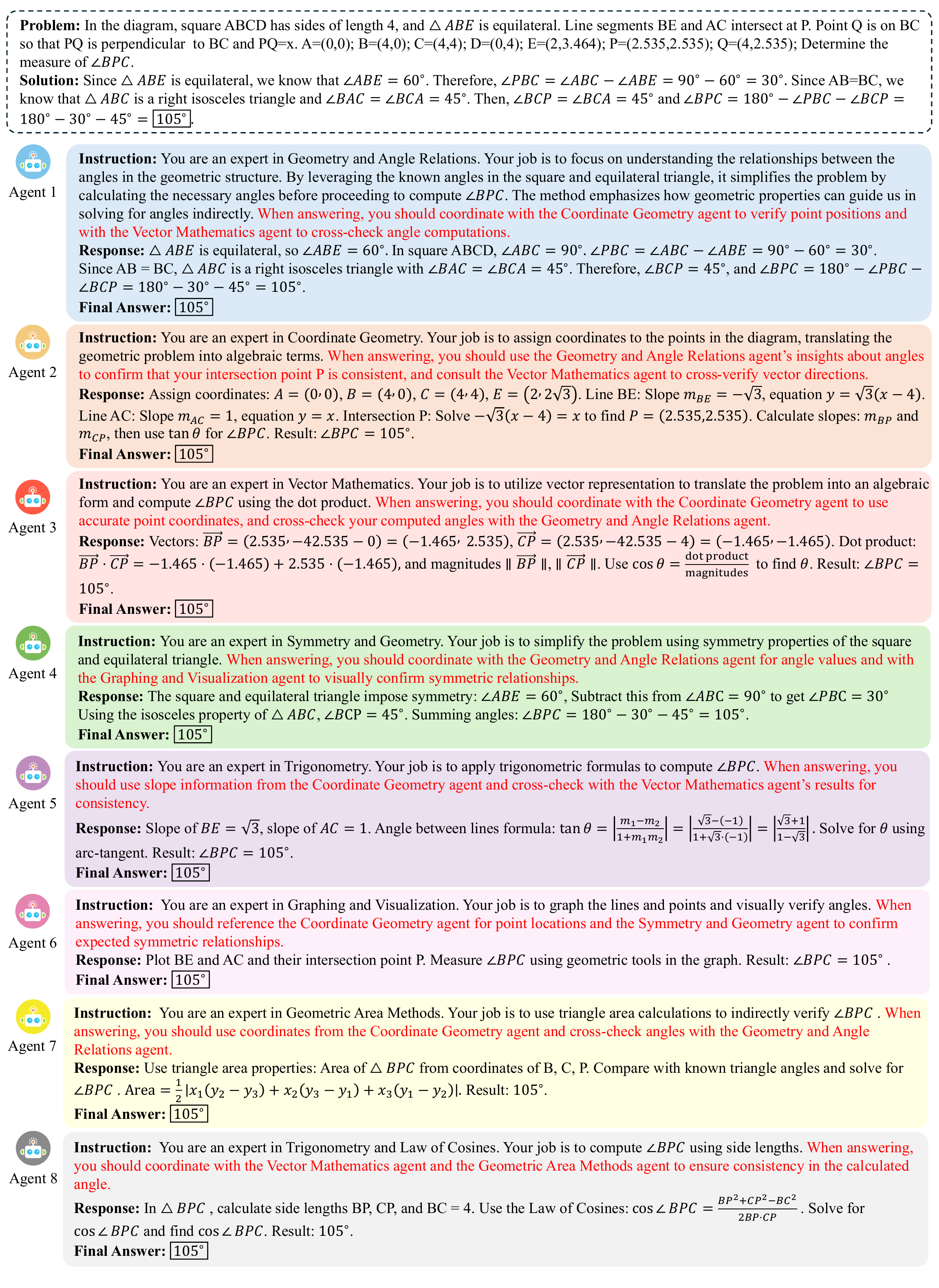}
	\caption{Visualization of \alg{} at the \textbf{last round}. We highlight the guidance on how to cooperate with other LLMs. Although the initial round produced divergent answers, the collaborative instructions enable LLMs to exchange and integrate information, ultimately reaching a correct consensus.}
	\label{fig:last_round}
\end{figure*}
\clearpage
\subsection{Case study of \texttt{Debate}}
We also provide the case of \texttt{Debate} solving the problem. As illustrated in \cref{fig:debate_first_round,fig:debate_second_round,fig:debate_third_round,fig:debate_last_round}, we observe an inconsistency during discussion when using fixed instructions.
\begin{figure*}[h]
	\centering
	\includegraphics[width=0.99\linewidth]{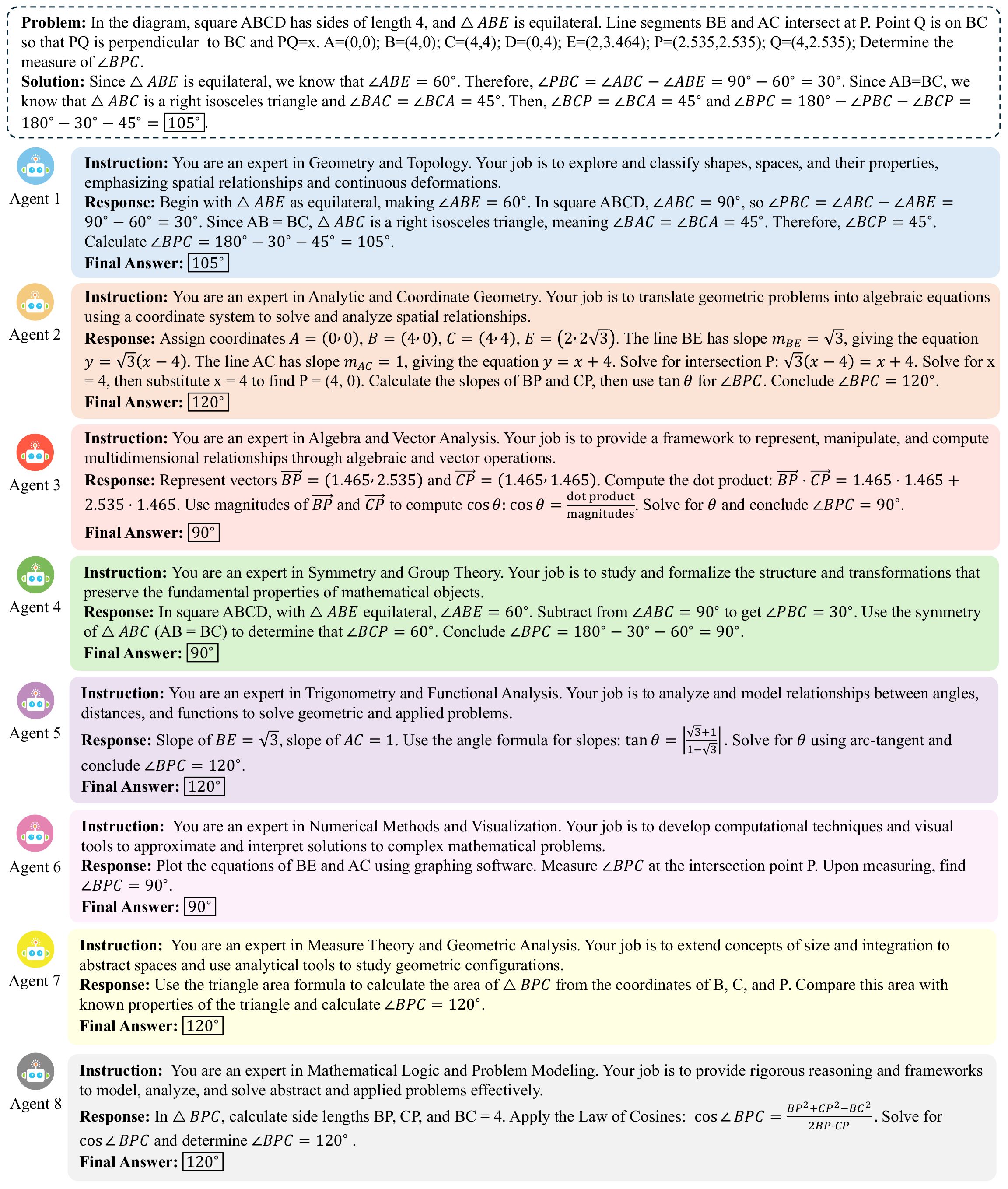}
	\caption{Visualization of \texttt{Debate} at the \textbf{first round}.}
	\label{fig:debate_first_round}
\end{figure*}
\begin{figure*}[h]
	\centering
	\includegraphics[width=0.99\linewidth]{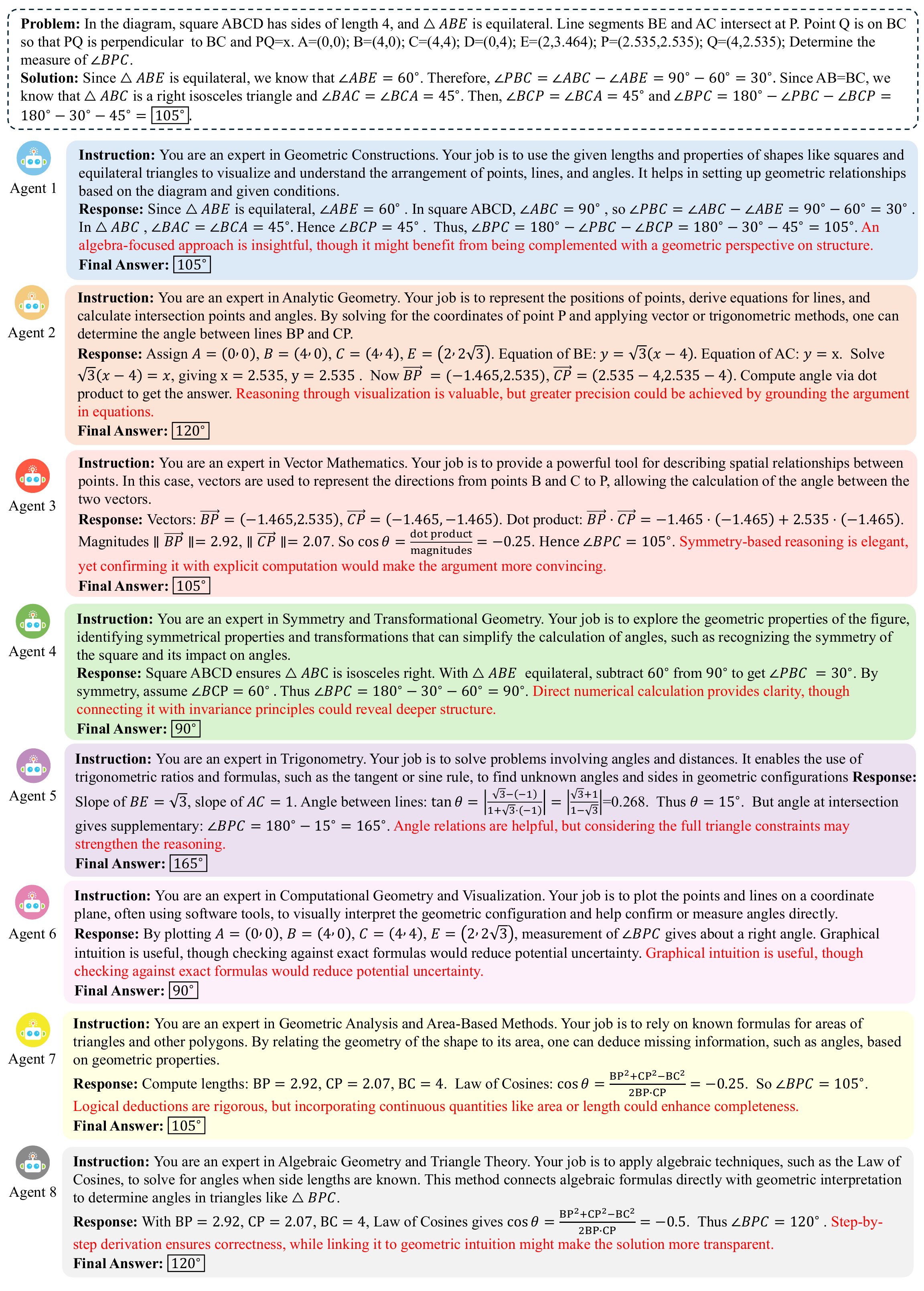}
	\caption{Visualization of \texttt{Debate} at the \textbf{second round}.}
	\label{fig:debate_second_round}
\end{figure*}
\begin{figure*}[h]
	\centering
	\includegraphics[width=0.99\linewidth]{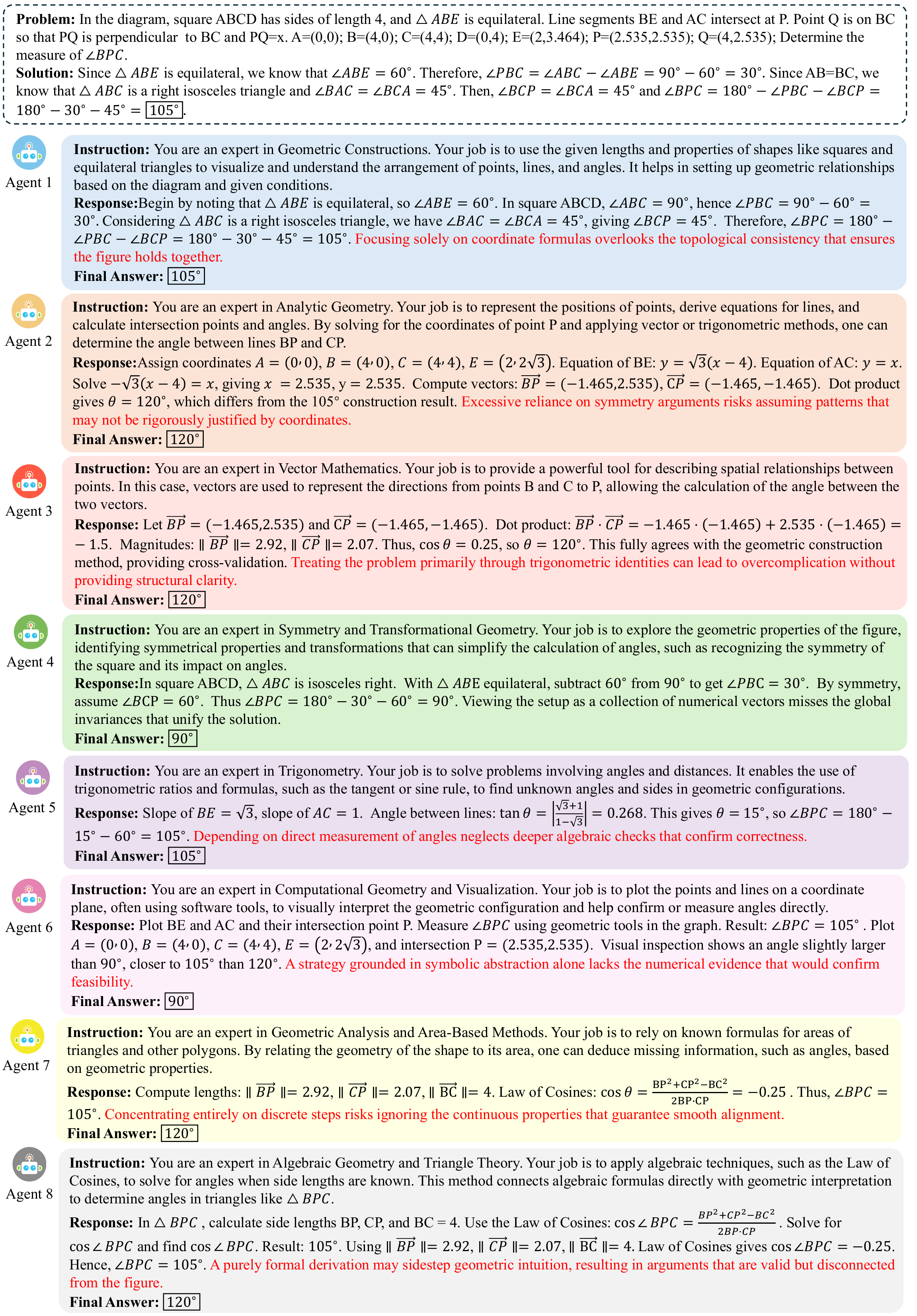}
	\caption{Visualization of \texttt{Debate} at the \textbf{third round}.}
	\label{fig:debate_third_round}
\end{figure*}
\begin{figure*}[h]
	\centering
	\includegraphics[width=0.99\linewidth]{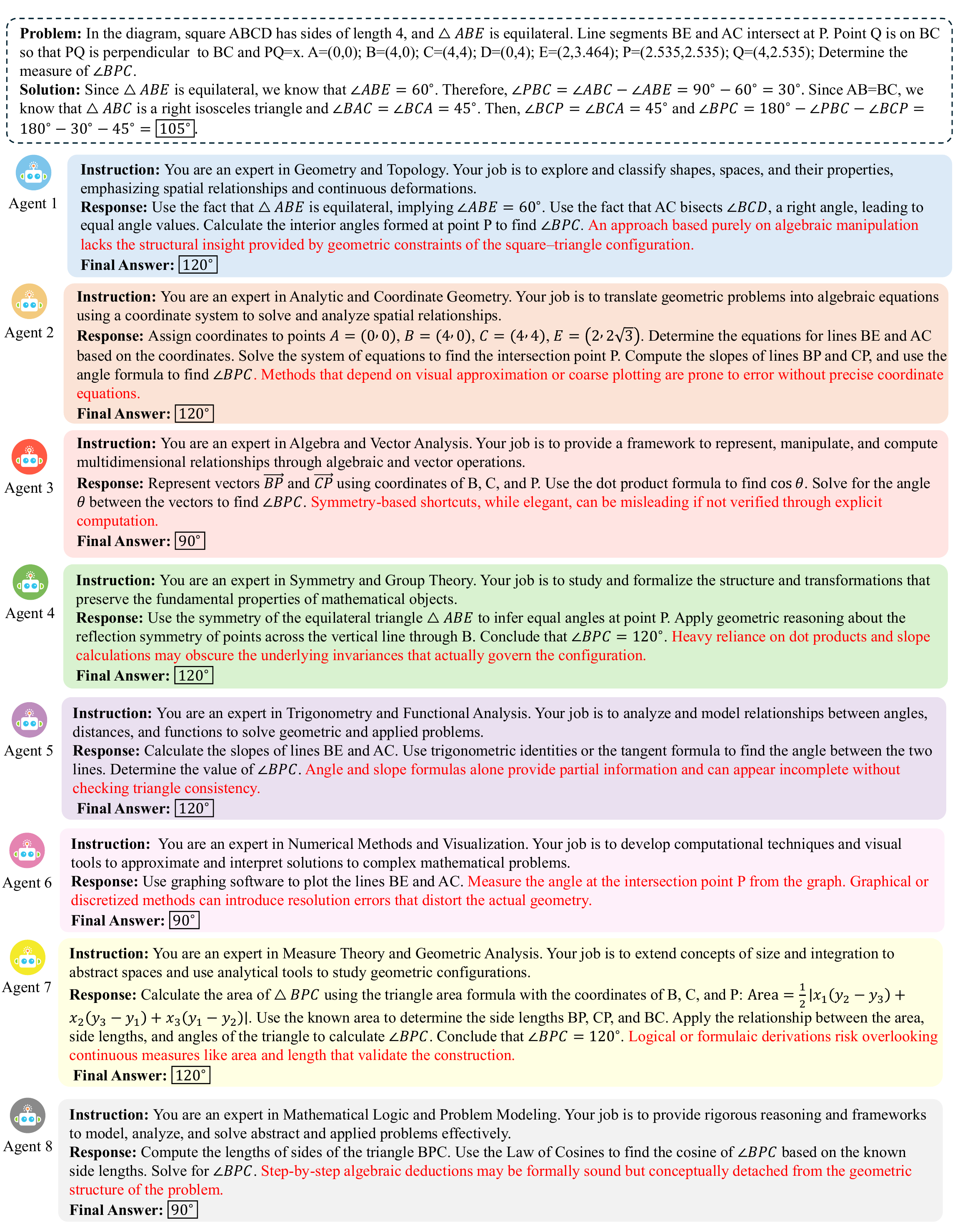}
	\caption{Visualization of \texttt{Debate} at the \textbf{last round}.}
	\label{fig:debate_last_round}
\end{figure*}
\end{document}